\documentclass[twoside,11pt]{article}

\usepackage[T1]{fontenc}

\usepackage{blindtext}

\usepackage{subcaption}
\usepackage[preprint]{jmlr2e}

\usepackage{lastpage}
\jmlrheading{25}{2024}{1-\pageref{LastPage}}{7/23; Revised
5/24}{8/24}{23-0888}{Shu Hu and George H. Chen}
\ShortHeadings{Fairness in Survival Analysis with Distributionally Robust Optimization}{Hu and Chen}
\firstpageno{1}

\usepackage[utf8]{inputenc}
\usepackage{color}
\usepackage{amsfonts}
\usepackage{amsmath}
\usepackage{amssymb}
\usepackage{bm}
\usepackage{enumitem}
\usepackage{multirow}
\usepackage{multicol}

\usepackage{dsfont}
\usepackage{cancel}
\usepackage{booktabs}

\let\emptyset\varnothing

\usepackage[ruled, vlined, linesnumbered]{algorithm2e}
\usepackage{threeparttable}
\usepackage{wrapfig}
\usepackage{multirow}
\usepackage{tcolorbox}
\usepackage{makecell}
\SetKwInOut{Parameter}{parameter}
\usepackage{color}
\usepackage{xcolor}
\usepackage{float}

\def\pAp@k{\texttt{pAp@k}}

\usepackage{colortbl}
\usepackage{collcell,xcolor,xfp}
\def\1{\cellcolor{green!30}}
\usepackage{tikz}
\newcommand\mybox[2][]{\tikz[overlay]\node[fill=blue!20,inner sep=2pt, anchor=text, rectangle, rounded corners=1mm,#1] {#2};\phantom{#2}}

\newcommand{\alphamin}{\alpha}

\newcommand{\fmtnum}[1]{%
  \ifnum\fpeval{#1 < 0} = 1
    \textcolor{red}{$#1$}%
  \else
    \ifnum\fpeval{#1 < 0.5} = 1
      \textcolor{green}{$#1$}%
    \else
      \textcolor{blue}{$#1$}%
    \fi
  \fi
}

\newcommand{\tableCox}{Cox}
\newcommand{\tableCoxKeyaInd}{Cox$_I$(Keya~et~al.)}
\newcommand{\tableCoxRPInd}{Cox$_I$(R\&P)}
\newcommand{\tableCoxKeyaGroup}{Cox$_G$(Keya~et~al.)}
\newcommand{\tableCoxRPGroup}{Cox$_G$(R\&P)}
\newcommand{\tableCoxKeyaInt}{Cox$_{\cap}$(Keya~et~al.)}

\newcommand{\tableDeepSurvKeyaInd}{DeepSurv$_I$(Keya~et~al.)}
\newcommand{\tableDeepSurvRPInd}{DeepSurv$_I$(R\&P)}
\newcommand{\tableDeepSurvKeyaGroup}{DeepSurv$_G$(Keya~et~al.)}
\newcommand{\tableDeepSurvRPGroup}{DeepSurv$_G$(R\&P)}
\newcommand{\tableDeepPseudoRPInd}{FIDP}
\newcommand{\tablePseudoNAMRPInd}{FIPNAM}
\newcommand{\tableDeepSurvKeyaInt}{DeepSurv$_{\cap}$(Keya~et~al.)}

\newcommand{\tableDROCox}{\textsc{DRO-COX}}
\newcommand{\tableDROCoxSplit}{\textsc{DRO-COX (SPLIT)}}
\newcommand{\tableExactDROCox}{\textsc{EXACT DRO-COX}}
\newcommand{\tableDeepDROCox}{Deep \textsc{DRO-COX}}
\newcommand{\tableDeepDROCoxSplit}{Deep \textsc{DRO-COX (SPLIT)}}
\newcommand{\tableDeepExactDROCox}{Deep \textsc{EXACT DRO-COX}}

\newcommand{\tableDeepHitRPGroup}{DeepHit$_G$(R\&P)}
\newcommand{\tableSODENRPGroup}{SODEN$_G$(R\&P)}

\usepackage{placeins}

\begin{document}

\title{Fairness in Survival Analysis with \\ Distributionally Robust Optimization}

\author{\name Shu Hu\thanks{equal contribution} \email hu968@purdue.edu \\
       \addr Department of Computer and Information Technology \\
       Purdue University \\
       Indianapolis, IN, 46202, USA
       \AND
       \name George H.~Chen$^*$\thanks{corresponding author} \email georgechen@cmu.edu \\
       \addr Heinz College of Information Systems and Public Policy \\
       Carnegie Mellon University \\
       Pittsburgh, PA, 15213, USA
       }

\editor{David Sontag}

\maketitle

\begin{abstract}
We propose a general approach for encouraging fairness in survival analysis models that is based on minimizing a worst-case error across \emph{all} subpopulations that are ``large enough'' (occurring with at least a user-specified probability threshold). This approach can be used to convert a wide variety of existing survival analysis models into ones that simultaneously encourage fairness, \emph{without} requiring the user to specify which attributes or features to treat as sensitive in the training loss function.
From a technical standpoint, our approach applies recent methodological developments of \emph{distributionally robust optimization} (DRO) to survival analysis. The complication is that existing DRO theory uses a training loss function that decomposes across contributions of individual data points, i.e., any term that shows up in the loss function depends only on a single training point. This decomposition does not hold for commonly used survival loss functions, including for the standard Cox proportional hazards model, its deep neural network variants, and many other recently developed survival analysis models that use loss functions involving ranking or similarity score calculations. We address this technical hurdle using a sample splitting strategy.
We demonstrate our sample splitting DRO approach by using it to create fair versions of a diverse set of existing survival analysis models including the classical Cox model (and its deep neural network variant DeepSurv), the discrete-time model DeepHit, and the neural ODE model SODEN. We also establish a finite-sample theoretical guarantee to show what our sample splitting DRO loss converges to.
Specifically for the Cox model, we further derive an exact DRO approach that does not use sample splitting.
For all the survival models that we convert into DRO variants, we show that the DRO variants often score better on recently established fairness metrics (without incurring a significant drop in accuracy) compared to existing survival analysis fairness regularization techniques, including ones which directly use sensitive demographic information in their training loss functions.

Our code is available at: \url{https://github.com/discovershu/DRO_survival}.

\end{abstract}


\section{Introduction}

\sloppy

Survival analysis aims to model time durations before a critical event happens. Examples of such critical events include a patient dying, a convicted criminal reoffending, or a customer cancelling a subscription service. Predicting such time durations accurately could help plan patient treatments, make bail decisions, or target subscription pricing promotions. If a survival analysis model is to be used in high-stakes decision making, fairness could be an important design criterion. For example, in the case of making bail decisions with the help of predicted time durations until a criminal reoffends, we may want a survival analysis model that produces these predictions to be similarly accurate across different races.

One of the major recent advances in encouraging fairness for machine learning models is to minimize a worst-case error over \emph{all} subpopulations that are ``large enough'' (e.g., \citealt{hashimoto2018fairness,duchi2021learning,li2021evaluating,duchi2022distributionally,hu2022rank}). In particular, a modeler specifies a minimum probability threshold~$\alpha$. The goal then is to ensure that all subpopulations that occur with probability at least~$\alpha$ have low average error, whereas we make no promises for subpopulations that occur with probability less than~$\alpha$. The modeler need not provide a list of subpopulations to account for. This problem can be tractably solved in practice and is called \emph{distributionally robust optimization}~(DRO).

We emphasize that curating a list of all subpopulations to account for can be challenging for various reasons. For example, one major challenge is \emph{intersectionality}: subpopulations that a machine learning model yields the worst accuracy scores for can be defined by complex intersections of sensitive attributes (such as age, race, and gender simultaneously taking on specific values) \citep{buolamwini2018gender}. Some of these attributes might require discretization (e.g., dividing age into bins), for which choosing the ``best'' discretization strategy might not be straightforward. Moreover, if there is a large number of features and we suspect that the sensitive attributes (encoded by specific features) could possibly be correlated with other features (not flagged as sensitive), then there is a question of whether these other features should also be accounted for in a listing of what the sensitive attributes are. DRO provides a theoretically sound alternative to having to specify which attributes to treat as sensitive in a training loss function. %

Our main contribution in this paper is to show how to apply DRO to survival analysis. Specifically, we propose a general strategy for converting a wide variety of survival analysis models into ones that simultaneously encourage fairness. Our strategy supports all survival analysis models we are aware of that minimize a loss function (details on the general form of survival analysis models that our approach supports are in Section~\ref{sec:conversion}). %

The key technical challenge is that existing DRO theory assumes that the overall training loss is the sum of individual loss terms, where each such term only depends on a single data point. This assumption fails to hold for commonly used survival analysis loss functions---including that of the standard Cox proportional hazards model \citep{cox1972regression}---that involve pairwise comparisons from ranking or similarity score evaluations (e.g., \citealt{steck2007ranking,lee2018deephit,chen2020deep,wu2021uncertainty}). In particular, there are loss terms that arise that incorporate information from multiple data points at once. We propose a sample splitting approach to address this technical challenge, and we establish a finite-sample theoretical guarantee on what our sample splitting DRO loss converges to. We point out that there are also parametric survival analysis models with loss functions that directly adhere to existing DRO theory (e.g., parametric accelerated failure time models \citep[Chapter 12]{klein2003survival} or, as a more exotic example, the recently proposed neural ordinary differential equation (ODE) model called SODEN \citep{tang2022soden}); such models can trivially be modified to use DRO without the sample splitting approach that we propose.

We specifically show how to derive DRO variants of the standard Cox model \citep{cox1972regression} (and its deep neural network variant DeepSurv \citep{faraggi1995neural, katzman2018deepsurv}), the discrete-time DeepHit model \citep{lee2018deephit}, and the neural ODE model called SODEN \citep{tang2022soden}. Again, we emphasize that our strategy for converting an existing survival analysis model to its DRO variant is fairly general and is not limited to only the few models that we showcase as illustrative examples.

We further derive an exact DRO approach specific to the Cox model that does not require sample splitting. In particular, by introducing additional parameters to optimize over for the Cox model's standard negative partial likelihood loss, it is possible to convert this loss function into one that decouples across training points. This derivation is specific to the Cox model though and does not easily generalize to other survival models.

On three standard datasets that have been previously used for research on fair survival analysis, we show that our DRO modification often outperforms various baseline fairness regularization techniques in terms of existing fairness metrics that focus on user-specified sensitive attributes. Most of these baselines require the user to specify which attributes to treat as sensitive attributes within the added regularization term. As with other fairness methods recently developed for survival analysis (e.g., \citet{keya2021equitable,rahman2022fair}), our approach also results in a drop in accuracy (compared to using a loss that does not encourage fairness). %
Note that our paper does not aim to find which survival model is the most accurate or the most fair. In fact, per survival model, there is in general a tradeoff between accuracy and fairness that can be tuned by the modeler. We show how to visualize this tradeoff using a plot inspired by an ROC curve.

\paragraph{Related work on fair survival analysis}
Despite many recent advances in survival analysis methodology (see, for instance, the survey by \citet{wang2019machine}), very few of these advances study fairness \citep{keya2021equitable, zhang2022longitudinal, sonabend2022flexible, rahman2022fair}. We provide an overview of these existing papers, and we discuss how they differ from our work. %

\citet{keya2021equitable} adapted existing fairness definitions to the survival analysis setting and showed how to encourage different notions of fairness by adding fairness regularization terms. %
Specifically, \citet{keya2021equitable} came up with individual \citep{dwork2012fairness}, group \citep{dwork2012fairness}, and intersectional \citep{foulds2020intersectional} fairness definitions specialized to Cox models. Keya et al.~define individual fairness in terms of model predictions being similar for similar individuals, and group fairness in terms of different user-specified groups having similar average predicted outcomes. Intersectional fairness further considers subgroups defined by intersections of protected groups (e.g., individuals of a specific race and simultaneously a specific gender). %
However, a major limitation of the notions of fairness defined by Keya et al.~is that they focus on predicted model outputs and do not actually use any of the ground truth label information. %
For example, if one uses age as a sensitive attribute and suppose we discretize age into two groups, then the notion of group fairness by Keya et al.~would ask for the predicted outcomes of the two age groups to be similar, which for healthcare problems often does not make sense (since age is often highly predictive of different health outcomes). Instead, in such a scenario, a more desirable notion of fairness is that the model's \emph{accuracy} for the different age groups be similar.

To account for model accuracy, \citet{zhang2022longitudinal} introduced a fairness metric called \emph{concordance imparity} that computes a quantity similar to the standard survival analysis accuracy metric of \emph{concordance index} \citep{harrell1982evaluating} for different groups and then looks at the worst-case difference between any two groups' accuracy scores. Meanwhile, \citet{rahman2022fair} directly modified the fairness definitions of \citet{keya2021equitable} to account for ground truth label information, and also generalized these definitions to survival models beyond Cox models.

Separately, \citet{sonabend2022flexible} empirically explored how well existing survival analysis accuracy and calibration metrics measure bias by synthetically modifying datasets (e.g., undersampling disadvantaged groups). However, they do not propose any new fairness metric or survival model that encourages fairness.

The papers mentioned above that propose new methods for learning fair survival models all either require user-specified demographic information to treat as sensitive (possibly as a list of subpopulations or groups to account for) or are simply adding a regularization term that encourages smoothness in the model outputs (the individual fairness regularization by \citet{keya2021equitable} and \citet{rahman2022fair} are directly related to encouraging Lipschitz continuity; for details, see Appendix~\ref{sec:fairness-measures}). In contrast, our proposed DRO approach does not require the user to indicate which attributes to treat as sensitive in the training loss function, and is not simply encouraging the model output to be Lipschitz continuous.

\paragraph{Bibliographical note}
This paper significantly extends our previous conference paper \citep{hu2022distributionally} in methodological development and in experiments. For methodological development, whereas our conference paper only considered Cox models, we show in this journal paper version how to convert a much wider class of survival analysis models into their DRO variants that encourage fairness. In fact, this wider class of models consists of all survival models we are aware of that are learned by minimizing an overall loss function. Furthermore, this journal paper extension includes theoretical analysis of our sample splitting DRO approach and also an exact DRO approach for the Cox model without sample splitting; neither of these contributions were in our conference paper. For experiments, we demonstrate our conversion strategy on not only Cox models but also on DeepHit and SODEN models. Our experiments are overall more extensive, and the SEER dataset we now use is much larger ($\sim$28k data points in this version vs $\sim$4k in the conference paper). Lastly, we also add a new visualization for seeing the tradeoff between accuracy and fairness across multiple models within a single plot. %

\paragraph{Outline}
The rest of the paper is organized as follows. We provide background on survival analysis, existing research on fairness in survival analysis, and DRO in Section~\ref{sec:background}.
We then present our strategy for converting a wide family of existing survival analysis models into their corresponding DRO variants that encourage fairness in Section~\ref{sec:conversion}; notably, this section introduces a sample splitting DRO approach and formally establishes its rate of convergence.
Specifically for the Cox model, we present an exact DRO Cox model without sample splitting in Section~\ref{sec:dro-cox-exact}. We conduct experiments to compare DRO variants of Cox, DeepHit, and SODEN models to their original non-DRO variants as well as to variants of these models that encourage fairness using non-DRO baseline regularization strategies. We conclude the paper in Section~\ref{sec:conclusion}.

\section{Background}\label{sec:background}

We begin by reviewing the basic survival analysis problem setup in Section~\ref{sec:survival-analysis} and then provide three examples of survival analysis models (Cox, DeepHit, and SODEN) in Section~\ref{sec:survival-models}. %
We then review DRO in Section~\ref{sec:dro}. Throughout the paper, we frequently use the notation $[\ell] \triangleq \{1, 2, \dots, \ell\}$ for any positive integer~$\ell$.

\subsection{Survival Analysis Setup}
\label{sec:survival-analysis}

Survival analysis aims to model the amount of time that will elapse before a critical event of interest happens. We assume that we have training data $\{(X_i,Y_i,\Delta_i)\}_{i=1}^n$, where training data point $i\in[n]$ has raw input $X_i\in\mathcal{X}$, observed duration $Y_i\ge0$, and event indicator $\Delta_i\in\{0,1\}$. If $\Delta_i=1$ (i.e., the critical event of interest happened for the $i$-th data point), then $Y_i$ is the time until the event happens. Otherwise, if $\Delta_i=0$, then $Y_i$ is the time until censoring for the $i$-th point, i.e., the true time until event is unknown but we know that it is at least $Y_i$. The raw input space $\mathcal{X}$ could be any input space supported by standard neural network software (e.g., tabular data, images, time series). %

Each training data point $(X_i,Y_i,\Delta_i)$ is assumed to be generated as follows:
\begin{enumerate}[itemsep=0pt,parsep=0pt,topsep=2pt] %
\item Sample raw input $X_i$ from a raw input distribution $\mathbb{P}_{X}$.
\item Sample nonnegative time duration $T_i$ (this is the true time until the critical event happens) from a conditional distribution $\mathbb{P}_{T|X=X_i}$.
\item Sample nonnegative time duration $C_i$ (this is the true time until the data point is censored) from a conditional distribution $\mathbb{P}_{C|X=X_i}$.
\item If $T_i \le C_i$ (the critical event happens before censoring), then set $Y_i = T_i$ and $\Delta_i=1$. Otherwise, set $Y_i = C_i$ and $\Delta_i=0$.
\end{enumerate}
Distributions $\mathbb{P}_X$, $\mathbb{P}_{T|X}$, and $\mathbb{P}_{C|X}$ are shared across data points and are unknown. %
We assume that the random variables $T_i$ and $C_i$ are independent given raw input $X_i$ (since the training data are i.i.d., this means that conditioned on the raw input, the censoring times are random and independent of each other, and they are also independent of the true survival times). %
We denote the CDF of distribution $\mathbb{P}_{T|X=x}$ as $F(\cdot|x)$. %

\paragraph{Prediction}
A standard prediction task is to estimate the probability that a data point with raw input~$x\in\mathcal{X}$ survives beyond time~$t$. Formally, this is defined as the \emph{conditional survival function}
\begin{equation}
S(t|x) \triangleq \mathbb{P}(T > t | X=x) = 1 - F(t|x)
\qquad\text{for }t\ge0.
\label{eq:conditional-survival}
\end{equation}
Importantly, for raw input~$x$, we are predicting an entire probability distribution (since $S(\cdot|x)$ encodes the same information as the CDF $F(\cdot|x)$).

Some survival analysis models, such as the Cox proportional hazards model \citep{cox1972regression}, estimate a transformed version of $S(\cdot|x)$ called the \emph{hazard function}, given by
\begin{equation}
h(t|x) \triangleq - \frac{\partial}{\partial t}\log S(t|x)
\qquad\text{for }t\ge0.
\label{eq:hazard}
\end{equation}
From negating both sides of this equation, integrating over time, and exponentiating, we get $S(t|x)=\exp(-\int_0^t h(u|x)\textrm{d} u)$. Thus, if we have an estimate of $h(\cdot|x)$, then we can readily estimate the conditional survival function $S(\cdot|x)$.

\subsection{Examples of Survival Analysis Models}
\label{sec:survival-models}

We now review three examples of survival analysis models (Cox, DeepHit, and SODEN) that can be modified to encourage fairness using DRO. In reviewing these models, we focus on aspects most relevant to our exposition later for how to convert these models into their DRO variants. %
For all three examples, we denote the neural network to be learned as $f(\cdot;\theta)$, where~$\theta$ denotes the parameters of the neural network. The domain and range of $f$ depends on the specific survival model we look at. Meanwhile, the architecture of $f$ is up to the modeler to specify, where standard strategies could be used (e.g., if the raw inputs are tabular data, then a multilayer perceptron could be used; if the raw inputs are images, a convolutional neural network could be used; etc).

\subsubsection{Classical and Deep Cox Models}
\label{sec:cox-models}

The classical Cox model assumes that the hazard function has the factorization
\begin{equation}
h(t|x) = h_0(t)\exp(f(x;\theta)),
\label{eq:hazard-factorization}
\end{equation}
where $h_0$ is called the baseline hazard function ($h_0$ maps a nonnegative time $t\ge0$ to a nonnegative number), and neural network $f(\cdot;\theta)$ maps a raw input from $\mathcal{X}$ to a single real number (i.e., $f(\cdot;\theta)$ has domain $\mathcal{X}$ and range $\mathbb{R}$). %
In particular, $f(x;\theta)$ models the so-called \emph{log partial hazard function} and could be thought of as assigning a real-valued ``risk score'' to raw input~$x\in\mathcal{X}$: when $f(x;\theta)$ is higher, then~$x$ has a higher risk of the critical event happening, so that the survival time of~$x$ will tend to be lower.

The original Cox model \citep{cox1972regression} defines $f$ to be a dot product: $f(x;\theta)=\theta^T x$, where~$\theta$ and~$x$ are in the same Euclidean vector space. More recently, researchers replaced~$f$ with a neural network \citep{faraggi1995neural,katzman2018deepsurv}, resulting in a method called DeepSurv (which could be viewed as a generalization of the original Cox model in that the classical definition $f(x;\theta)=\theta^T x$ is a simple neural network consisting of a linear layer with no bias and no nonlinear activation). In either case, the standard approach for learning a Cox model is to first learn the neural network parameters~$\theta$ by minimizing the negative log partial likelihood:
\begin{equation}
L^{\text{Cox}}(\theta) = \frac{1}{n} \sum_{i=1}^n L_i^{\text{Cox}}(\theta),
\label{eq:cox}
\end{equation}
where the $i$-th data point's loss is
\begin{align}
L_i^{\text{Cox}}(\theta)
\triangleq
  -\Delta_i
   \bigg[
     f(X_i;\theta)
     - \log
         \sum_{j\in[n]\text{ s.t.~}Y_j\geq Y_i}
           \exp(f(X_j;\theta))
   \bigg]. \label{eq:cox-individual-loss}
\end{align}
If the $i$-th data point is censored (i.e., $\Delta_i=0$), then $L_i^{\text{Cox}}(\theta)=0$. Thus, the overall loss $L^{\text{Cox}}(\theta)$ could be viewed as weighting the \emph{uncensored} training points equally. After learning $\theta$, we then estimate~$h_0$; as this step is not essential to our exposition, we explain it in Appendix~\ref{sec:breslow}, along with details on constructing the final estimate of $S(\cdot|x)$.

We remark that the factorization in equation~\eqref{eq:hazard-factorization} is referred to as the \emph{proportional hazards assumption}: regardless of what the input $x$ is, the hazard function $h(\cdot|x)$ is always proportional to the baseline hazard function $h_0$. A consequence of this assumption is that the resulting conditional survival function $S(\cdot|x)$ is heavily constrained in terms of its shape. In particular, regardless of what $x$ is, $S(t|x)$ must be a power of the function $S_0(t) \triangleq \exp(-\int_0^t h_0(u)\textrm{d} u)$ (for details, see Appendix~\ref{sec:proportional-hazards-survival-curve-shape}). The next two survival analysis models that we describe do not have this assumption and can more flexibly estimate $S(\cdot|x)$.

\subsubsection{DeepHit}
\label{sec:deephit}

A wide class of survival analysis models directly estimate (some transformed version of) the conditional survival function $S(\cdot|x)$ along a discretized time grid, without requiring the proportional hazards assumption. The time grid itself is up to the modeler to choose and can depend on the observed time $Y_i$ and event indicator $\Delta_i$ variables in the training data. For example, we could use a uniformly-spaced time grid between the minimum and maximum observed times (for some user-specified number of discretized time steps), or we could have the time grid consist of all unique times in the training data in which the critical event happened (in fact, this how the classical Kaplan-Meier estimator \citep{kaplan1958nonparametric} discretizes time). Some other time grids are discussed by \citet{kvamme2021continuous}.

An example of a model that uses a discretized time grid is DeepHit \citep{lee2018deephit}. Note that DeepHit supports the so-called \emph{competing risks} setting where there are multiple critical events of interest. %
For simplicity, we review DeepHit where we only present the case where there is a single critical event of interest, which reduces the problem setup to the same one we specified in Section~\ref{sec:survival-analysis}. %

Let $t_1<t_2<\cdots<t_m$ denote the $m$ discretized time points based on some user-specified grid. We assume that these are the only time points in which the critical event or censoring could happen (if a critical event or censoring happens at some other time point, we quantize it to one of these $m$ time points). Then DeepHit parameterizes the following conditional probability mass function using a neural network:
\begin{equation}
\mathbb{P}(T = t_j | X = x) = f_j(x; \theta) \qquad\text{for }j\in[m],
\label{eq:deephit-model}
\end{equation}
where neural network $f(x;\theta) = \big(f_1(x;\theta),f_2(x;\theta),\dots,f_m(x;\theta)
\big)\in[0,1]^m$ has parameters~$\theta$ and maps a raw input $x\in\mathcal{X}$ to a probability distribution over the $m$ time steps. In other words, the domain of $f(\cdot;\theta)$ is $\mathcal{X}$ and the range of $f(\cdot;\theta)$ is the probability simplex $\big\{ z\in\mathbb{R}^m : z_j \ge 0 \text{ for all }j\in[m]\text{ and }\sum_{j=1}^m z_j = 1 \big\}$.
For example, when working with tabular data, $f$ could be a multilayer perceptron, where the last linear layer outputs~$m$ numbers and has softmax activation.

Because of the parameterization in equation~\eqref{eq:deephit-model}, we can write the conditional survival function $S(t|x)$ at any discrete time point $t_j$ in terms of the neural network $f(\cdot;\theta)$:
\[
S_j(x;\theta) \triangleq S(t_j|x) = \mathbb{P}(T > t_j | X = x) %
= \sum_{\ell =j+1}^m f_{\ell}(x;\theta)
\quad\;\text{for }j\in[m].
\]
To learn $\theta$, DeepHit uses the sum of two loss terms, corresponding to a negative log likelihood term and, separately, a ranking loss term. In what follows, we use the notation $\kappa(Y_i)\in[m]$ to denote the time step index corresponding to the $i$-th training point's observed time $Y_i$ (i.e., $Y_i$ gets quantized to integer time step $\kappa(Y_i)$). For example, one way to define the function $\kappa:[t_1,\infty)\rightarrow[m]$ is as follows:\footnote{Note that it is possible to instead define the domain of $\kappa$ to be $[0,\infty)$, where we either require $t_1\triangleq0$ or alternatively we define a special time point $t_0\triangleq0$ (and have $\kappa(t)=0$ when $t<t_1$, where $t_1$ is assumed to be positive). In the latter case where we introduce time point $t_0$, the range of $\kappa$ would of course be $\{0,1,\dots,m\}$ instead of $[m]=\{1,2,\dots,m\}$.}
\begin{equation}
\kappa(t)
\triangleq
\begin{cases}
\ell & \text{if there exists time index }\ell\in[m]\text{ s.t.~}t_\ell = t, \\
\max\{\ell\in[m] : t_\ell < t\} & \text{otherwise}.
\end{cases}
\label{eq:deephit-kappa}
\end{equation}
Then the overall DeepHit loss is
\begin{align}
L^{\text{DeepHit}}(\theta)
\triangleq & ~\beta\cdot \overbrace{\frac{1}{n}\sum_{i=1}^n
     \big[
       - \Delta_i \log(f_{\kappa(Y_i)}(X_i;\theta))
       - ( 1 - \Delta_i ) \log( S_{\kappa(Y_i)}(X_i;\theta) )
     \big]}^{\text{negative log likelihood loss term}}
\nonumber \\
& + (1-\beta)\cdot
    \underbrace{\frac{1}{n^2}\sum_{i=1}^n \Delta_i \sum_{j\in[n]\text{ s.t.~}\kappa(Y_j) > \kappa(Y_i)}
     \exp\bigg(\frac{S_{\kappa(Y_i)}(X_i;\theta) - S_{\kappa(Y_i)}(X_j;\theta)}{\sigma}\bigg)}_{\text{ranking loss term}},
\label{eq:deephit-loss-pycox}
\end{align}
where $\beta\in[0,1]$ and $\sigma>0$ are hyperparameters. Note that this formulation of the overall loss follows the implementation of DeepHit by \citet{kvamme2019time} in the now-standard \texttt{pycox} software package and is slightly different from the original formulation by \citet{lee2018deephit} (the only difference is in the weights used to combine the two main loss terms).

For how we convert DeepHit into its DRO variant later, it will be helpful to rewrite the DeepHit loss in terms of individual losses:
\begin{equation}
L^{\text{DeepHit}}(\theta)
\triangleq
  \frac{1}{n}\sum_{i=1}^n L^{\text{DeepHit}}_i(\theta),
\label{eq:deephit-loss}
\end{equation}
where the $i$-th individual loss is
\begin{align}
L^{\text{DeepHit}}_i(\theta)
=
     & ~
     \beta\cdot
     \big[
       - \Delta_i \log(f_{\kappa(Y_i)}(X_i;\theta))
       - ( 1 - \Delta_i ) \log( S_{\kappa(Y_i)}(X_i;\theta) )
     \big]
\nonumber \\
& + (1-\beta)\cdot
    \frac{1}{n} \cdot \Delta_i \sum_{j\in[n]\text{ s.t.~}\kappa(Y_j) > \kappa(Y_i)}
     \exp\bigg(\frac{S_{\kappa(Y_i)}(X_i;\theta) - S_{\kappa(Y_i)}(X_j;\theta)}{\sigma}\bigg).
\label{eq:deephit-individual-loss}
\end{align}

\subsubsection{SODEN}
\label{sec:soden}

Recently, a number of researchers have considered a differential-equation approach to setting up a survival analysis model that can avoid the proportional hazards assumption while also not requiring the modeler to explicitly specify a discrete time grid \citep{groha2020general, moon2022survlatent, tang2022survival, tang2022soden}. We review one such model called SODEN (Survival model through Ordinary Differential Equation Networks), proposed by \citet{tang2022soden}. Note that we review a special case that is easier to describe and that corresponds to our survival analysis problem setup in Section~\ref{sec:survival-analysis}, where survival times are all nonnegative.

In what follows, we denote $H(t|x) \triangleq -\log S(t|x)$. From how we defined the hazard function $h(t|x)$ in equation~\eqref{eq:hazard}, we have $h(t|x) = \frac{\partial}{\partial t} H(t|x)$, so $H(t|x)=\int_0^t h(u|x)\textrm{d}u$; this integral expression reveals why $H(t|x)$ is commonly called the \emph{cumulative hazard function}. %

SODEN uses a neural network $f(\cdot;\theta)$ to parameterize the hazard function as the solution to an ordinary differential equation (ODE):
\begin{equation}
\begin{cases}
& \frac{\partial}{\partial t}H(t|x) = h(t|x) = f\big((t, H(t|x), x); \theta\big)\qquad\text{for }t > 0, \\
& H(0|x)= 0\qquad\text{(initial condition at time 0)},
\end{cases}
\label{eq:SODEN-ODE-constraint}
\end{equation}
where the neural network $f(\cdot;\theta)$ has domain $[0,\infty)\times[0,\infty)\times\mathcal{X}$ and range $\mathbb{R}$. Specifically $f(\cdot;\theta)$ takes as input time $t\ge 0$, a cumulative hazard value $H(t|x)$ (which is nonnegative), and a raw input $x\in\mathcal{X}$, and $f(\cdot;\theta)$ outputs a single real number that is $h(t|x)$. For example, $f(\cdot;\theta)$ could concatenate all its inputs to form a single vector of numbers that is then treated as the input to a multilayer perceptron, where the final linear layer outputs a single number and has softplus activation (to ensure that the output is always positive). The initial condition follows from the fact that $H(0|x)=\int_0^0 h(u|x)\textrm{d}u = 0$.

Learning neural networks in terms of ODEs (as in equation~\eqref{eq:SODEN-ODE-constraint}) is possible thanks to the landmark paper by \citet{chen2018neural}. Importantly, using any user-specified ODE solver, given any raw input $x\in\mathcal{X}$ and neural network parameters $\theta$, we can numerically solve the ODE in equation \eqref{eq:SODEN-ODE-constraint} (going from time 0 to any user-specified time $t>0$) to obtain an estimate for $H(t|x)$; we denote this estimate as $H_{\text{ODE-solve}}(t|x;\theta)$. In particular, a major result of \citet{chen2018neural} is that the loss function we use to train the neural network can contain terms involving $h(t|x)=f((t,H(t|x),x);\theta)$ and $H(t|x)$, where we replace $H(t|x)$ with $H_{\text{ODE-solve}}(t|x;\theta)$. Backpropagation is possible with the help of any ODE solver. %

To train the SODEN model, \citet{tang2022soden} use the overall loss function
\begin{equation}
L^{\text{SODEN}}(\theta) \triangleq \frac{1}{n}\sum_{i=1}^n L_i^{\text{SODEN}}(\theta),
\label{eq:SODEN-overall}
\end{equation}
where the $i$-th individual loss is
\begin{equation}
L_i^{\text{SODEN}}(\theta) = -\Delta_i\log f\big( (Y_i, H_{\text{ODE-solve}}(Y_i|X_i;\theta), X_i) ;\theta\big) + H_{\text{ODE-solve}}(Y_i|X_i;\theta).
\label{eq:SODEN-loss}
\end{equation}
Note that the overall loss~\eqref{eq:SODEN-overall} is just a negative log likelihood expression, so that minimizing this loss corresponds to solving a maximum likelihood problem.

\subsection{Distributionally Robust Optimization (DRO)}
\label{sec:dro}

DRO uses a worst-case average error over ``large enough'' subpopulations. Note that there are now a number of different versions of DRO (e.g., \citealt{hashimoto2018fairness,sagawa2020distributionally,duchi2021learning,duchi2022distributionally}). We specifically use the one by \citet{hashimoto2018fairness}. %
Even though existing literature on DRO does not consider survival analysis to the best of our knowledge, we intentionally review DRO here using survival analysis notation that we have introduced in Section~\ref{sec:survival-analysis}. In fact, existing DRO theory actually works with many existing survival analysis loss functions already, without modification. In particular, survival analysis models for which each data point's individual loss does not depend on any other data points could trivially use existing DRO machinery. Examples of such survival analysis models include DeepHit when $\beta=1$ (see equation~\eqref{eq:deephit-individual-loss}), SODEN (see equation~\eqref{eq:SODEN-loss}), as well as exponential, Weibull, log-logistic, log-normal, and generalized Gamma accelerated failure time models \citep[Chapter 12]{klein2003survival}.

\paragraph{Problem setup}
Let $\mathbb{P}$ denote the joint distribution over each data point $(X_i,Y_i,\Delta_i)$. This joint distribution corresponds to the generative procedure described in Section~\ref{sec:survival-analysis}. We assume that there are $K$ groups that comprise $\mathbb{P}$. In particular, $\mathbb{P}$ is a mixture of $K$ distributions $\mathbb{P} \triangleq \sum_{k=1}^{K}\pi_{k}\mathbb{P}_{k}$, where the $k$-th group occurs with probability $\pi_{k}\in(0,1)$ and has associated distribution $\mathbb{P}_{k}$. Moreover, $\sum_{k=1}^{K}\pi_{k}=1$. We assume that we do not know $\{(\pi_{k},\mathbb{P}_{k})\}_{k=1}^{K}$, nor do we know $K$. This setting, for instance, handles the case where we do not exhaustively know all subpopulations to consider. The smallest minority group corresponds to whichever group has the smallest $\pi_k$ value. A simple special case would be when $K=2$, where data are drawn from either a minority group or a majority group.

We would like to minimize the risk
\[
R_{\max}(\theta) \triangleq \max_{k=1,\dots,K}\mathbb{E}_{(X,Y,\Delta)\sim\mathbb{P}_{k}}[L_{\text{indiv}}(\theta;X,Y,\Delta)],
\]
where $L_{\text{indiv}}$ is a loss function that depends only on the parameters $\theta$ (for a survival analysis model that we aim to learn) and on a single data point $(X,Y,\Delta)$. However, minimizing $R_{\max}(\theta)$ is not possible as we do not know any of the latent groups. Nevertheless, it turns out that there is an optimization problem that we can tractably solve that minimizes an empirical version of an upper bound on $R_{\max}(\theta)$. We explain what the upper bound is in Section~\ref{sec:dro-risk-upper-bound} and how to empirically minimize the upper bound in Section~\ref{sec:empirical-dro-risk-minimization}.

\subsubsection{Upper Bound on the Risk \texorpdfstring{$R_{\max}(\theta)$}{R\_\{max\}(theta)} Using DRO}
\label{sec:dro-risk-upper-bound}

For a set of distributions $\mathcal{B}_{r}(\mathbb{P})$ to be defined shortly, we consider minimizing the following alternative risk instead:
\begin{equation}
R_{\text{DRO}}(\theta;r) \triangleq \sup_{\mathbb{Q}\in\mathcal{B}_{r}(\mathbb{P})}\mathbb{E}_{(X,Y,\Delta)\sim\mathbb{Q}}[L_{\text{indiv}}(\theta;X,Y,\Delta)].
\label{eq:dro}
\end{equation}
This is the worst-case expected loss when we sample from any distribution in $\mathcal{B}_{r}(\mathbb{P})$.

The definition for $\mathcal{B}_{r}(\mathbb{P})$ is somewhat technical; we first give its precise definition and then state how to choose $r$ so that $R_{\text{DRO}}(\theta;r)$ is an upper bound on $R_{\max}(\theta)$. Importantly, we will be able to efficiently minimize an empirical version of $R_{\text{DRO}}(\theta;r)$.

\begin{definition}
The set $\mathcal{B}_{r}(\mathbb{P})$ consists of all distributions $\mathbb{Q}$ that have the same (or smaller) support as $\mathbb{P}$ and have $\chi^{2}$-divergence at most $r$ from distribution $\mathbb{P}$. Formally,
\[
    \mathcal{B}_{r}(\mathbb{P}) \triangleq \{\text{distribution }\mathbb{Q}\text{ such that }\mathbb{Q}\ll\mathbb{P}\text{ and }D_{\chi^{2}}(\mathbb{Q}\|\mathbb{P})\le r\},
\]
where, using standard measure theory notation, ``~\!$\mathbb{Q}\ll\mathbb{P}$'' means that $\mathbb{Q}$ is absolutely continuous with respect to~$\mathbb{P}$, and $D_{\chi^2}(\mathbb{Q}\|\mathbb{P}) \triangleq \int (\frac{\textrm{\emph{d}} \mathbb{Q}}{\textrm{\emph{d}} \mathbb{P}} - 1)^2 \textrm{\emph{d}}\mathbb{P}$.
\end{definition}
Working with $\mathcal{B}_{r}(\mathbb{P})$ turns out to be straightforward so long as we have a lower bound on the smallest group's probability (i.e., a lower bound on $\min_{k=1,\dots,K}\pi_k$). %

\begin{proposition}
\label{prop:hashimoto-et-al}
(Directly follows from Proposition 2 of \citet{hashimoto2018fairness})
Suppose that we have a lower bound $\alphamin>0$ on the $K$ latent groups' probabilities of occurring (i.e., $\alphamin\le\min_{k=1,\dots,K}\pi_{k}$). Then $R_{\text{DRO}}(\theta;r_{\max})\ge R_{\max}(\theta)$, where $r_{\max} \triangleq (\frac{1}{\alphamin}-1)^{2}$.
\end{proposition}
In other words, if we have a guess for $\alphamin\in(0,\min_{k=1,\dots,K}\pi_{k}]$, then it suffices to choose $r$ for $\mathcal{B}_{r}(\mathbb{P})$ to be $r_{\max}=(\frac{1}{\alphamin}-1)^{2}$. Furthermore, the risk $R_{\text{DRO}}(\theta;r_{\max})$ is an upper bound on $R_{\max}(\theta)$. In practice, $\alphamin\in(0,1)$ is a user-specified hyperparameter since we do not know $\pi_1,\dots,\pi_K$ nor $K$. Choosing $\alphamin$ to be smaller means that we want to ensure that groups with smaller probabilities of occurring also have low expected loss. For example, setting $\alpha=0.1$ means that the ``rarest'' group that we want to ensure low expected loss for occurs with probability at least~0.1.

To provide some more detail, if there is some underlying true $K$ unknown subpopulation distributions $\mathbb{P}_1,\dots,\mathbb{P}_K$ (where the value of $K$ itself is also unknown), it is important to keep in mind that often times it suffices to consider there to simply be two subpopulations: a minority subpopulation which we could without loss of generality take to be $\mathbb{P}_1$ (since we can reorder the subpopulations so that the minority subpopulation of interest is the first one) and everyone else (the combination of $\mathbb{P}_2,\dots,\mathbb{P}_K$, i.e., $\sum_{k=2}^K \pi_k \mathbb{P}_k$); note that this is a mixture of two distributions now. Then we would be tuning $\alpha$ with the hope that if $\mathbb{P}_1$ occurs with probability $\pi_1$ that is at least $\alpha$, then we ensure low expected loss for~$\mathbb{P}_1$. However, if $\alpha > \pi_1$, then we would no longer ensure that $\mathbb{P}_1$ has low expected loss. However, if the combination of $\mathbb{P}_1$ and any of the other mixture components, say, $\mathbb{P}_2$ have probability $\pi_1+\pi_2\ge\alpha$, then we would be ensuring that this larger subpopulation of $\pi_1\mathbb{P}_1 + \pi_2\mathbb{P}_2$ has low expected loss. In this manner, as we increase $\alpha$, we are ensuring low expected loss across a larger subpopulation, where this subpopulation could be the combination of multiple of the $\mathbb{P}_k$ distributions. More precisely, for any choice of $\alpha$, if $\mathcal{K}$ is any subset of $[K]$ such that $\sum_{k\in\mathcal{K}} \pi_k \ge \alpha$, then we would be ensuring low expected loss for the subpopulation $\sum_{k\in\mathcal{K}} \pi_k\mathbb{P}_k$.

\subsubsection{Empirical DRO Risk}
\label{sec:empirical-dro-risk-minimization}

The next issue is how to minimize the risk $R_{\text{DRO}}(\theta;r_{\max})$, which at a first glance might appear daunting due to the supremum over all distributions in $\mathcal{B}_{r_{\max}}(\mathbb{P})$. However, a fundamental theoretical result from DRO literature is that $R_{\text{DRO}}(\theta;r_{\max})$ can be written in a form that is amenable to computation.

\begin{proposition}\label{prop:duchi-et-al}
(Lemma 1 in \citet{duchi2021learning})
    Suppose $\widehat{\ell}(\theta;X,Y,\Delta)$ is upper semi-continuous with respect to $\theta$. Let $[\cdot]_+$ denote the ReLU function (i.e., $[a]_+ \triangleq \max\{a,0\}$ for any $a\in\mathbb{R}$), and $C_{\alpha} \triangleq \sqrt{2(\frac{1}{\alphamin} -1)^2+1}$. Then
\begin{align}
    R_{\text{DRO}}(\theta;r_{\max})= \inf_{\eta\in \mathbb{R}}\Big\{C_{\alpha}\sqrt{\mathbb{E}_{(X,Y,\Delta)\sim\mathbb{P}}\big[[L_{\text{indiv}}(\theta;X,Y,\Delta)-\eta]_+^2\big]}+\eta\Big\},
\label{eq:dro_dual}
\end{align}
where, as a reminder, $r_{\max} = (\frac{1}{\alpha} - 1)^2$.
\end{proposition}
The right-hand side of equation~\eqref{eq:dro_dual} could be interpreted as follows. Suppose that we have achieved the optimal value $\eta^*$. Then the loss from a data point will be ignored if it is less than $\eta^*$ (due to the ReLU function). Thus, only the data points with losses above $\eta^*$ are considered for learning the survival model.

\begin{algorithm}[t!]\footnotesize
    \caption{\textsc{dro}}\label{alg:dro-cox}
    \SetAlgoLined
    \KwIn{A training dataset $\{(X_i,Y_i,\Delta_i)\}_{i=1}^n$, minimum subpopulation probability hyperparameter $\alpha$, learning rate $\xi$, max\_iterations}
    \KwOut{Survival model parameters $\widehat{\theta}$} 
    Obtain initial survival model parameters $\widehat{\theta}_0$ (e.g., using default PyTorch parameter initialization).
    
    \For{$\ell=0$ to \emph{max\_iterations}}{
    \For{$i=1$ to $n$}{
        Set $u_i \leftarrow  L_{\text{indiv}}(\widehat{\theta}_\ell;X_i,Y_i,\Delta_i)$.
    }
    Set $\widehat{\eta} \leftarrow \arg\min_{\eta\in\mathbb{R}}
    \Big\{\Big(\sqrt{2\big(\frac{1}{\alpha} - 1\big)^2 + 1}~\!\Big) \sqrt{\frac{1}{n}\sum_{i=1}^n [u_i - \eta]_+^2} + \eta\Big\}$,
    where this minimization is solved using binary search. (This step directly corresponds to minimizing $L_{\text{DRO}}(\widehat{\theta}_\ell, \eta)$ as given in equation~\eqref{eq:empirical-dro-minimization}.) %
    
    Set $\widehat{\theta}_{\ell+1} \leftarrow \widehat{\theta}_\ell -\xi\cdot\nabla_{\theta}{L}_{\text{DRO}}(\widehat{\theta}_\ell, \widehat{\eta})$.
    
    }
    \Return{$\widehat{\theta} \leftarrow \widehat{\theta}_{\emph{max\_iterations}+1}$}
\end{algorithm}

Note that as we vary the model parameters~$\theta$, the different data points' losses change. Thus, as a function of $\theta$, the DRO risk $R_{\text{DRO}}(\theta;r_{\max})$ dynamically adjusts which data points to focus on, always prioritizing the points with the highest loss values (again, we only consider the points with a loss greater than the optimal value of $\eta$).

We can readily minimize an empirical version of $R_{\text{DRO}}(\theta;r_{\max})$. Specifically, we replace the expectation on the right-hand side of equation~\eqref{eq:dro_dual} with an empirical average to arrive at the following optimization problem:
\begin{equation}
\min_{\theta\in\Theta,\eta\in\mathbb{R}}
\Bigg(~ \underbrace{C_{\alpha}\sqrt{\frac{1}{n}\sum_{i=1}^n[L_{\text{indiv}}(\theta;X_i,Y_i,\Delta_i)-\eta]_+^2}+\eta}_{\triangleq L_{\text{DRO}}(\theta, \eta)} ~\Bigg)
,
\label{eq:empirical-dro-minimization}
\end{equation}
where $\Theta$ denotes the feasible set of the model parameters.

\paragraph{Numerical optimization}
The optimization problem in equation~\eqref{eq:empirical-dro-minimization} can be solved with an iterative gradient descent approach \citep{hu2020learning, hu2021tkml, hu2022sum}. Specifically, we first initialize the model parameters $\theta$. Then, following \citet{hashimoto2018fairness}, we alternate between two steps:
\begin{itemize}[itemsep=0pt,topsep=2pt,parsep=0pt]
\item We fix $\theta$ and update $\eta$ by finding the value of $\eta$ that minimizes $L_{\text{DRO}}(\theta, \eta)$. To do this, we use binary search to find the global optimum of $\eta$ since $L_{\text{DRO}}(\theta, \eta)$ is a convex function with respect to $\eta$.
\item We fix $\eta$ and update $\theta$ by minimizing $L_{\text{DRO}}(\theta, \eta)$ (e.g., using gradient descent).
\end{itemize}
We stop iterating after user-specified stopping criteria are reached (e.g., maximum number of iterations reached, early stopping due to no improvement in a validation metric after a pre-specified number of epochs). The pseudocode can be found in Algorithm \ref{alg:dro-cox}.

\section{Converting Existing Survival Analysis Models into DRO Variants}
\label{sec:conversion}

Throughout this section, we assume that the training points $\{(X_{i},Y_{i},\Delta_{i})\}_{i=1}^{n}$ are generated by the procedure stated in Section~\ref{sec:survival-analysis}. We describe the general class of survival models that we can convert into DRO variants in Section~\ref{sec:convertible-class}. For some models (such as SODEN), the existing DRO approach stated in Section~\ref{sec:dro} directly works without modification. For other survival models (such as Cox models), existing DRO theory does not work as advertised and we propose a sample splitting DRO approach in Section~\ref{sec:dro-split} to obtain an approximate loss to minimize that does comply with DRO theory. We establish theoretical guarantees for this sample splitting DRO approach in Section~\ref{sec:sample-splitting-dro-theory}.

\subsection{Class of Survival Models Convertible Into DRO Variants}
\label{sec:convertible-class}

Our technique for converting a survival model into its DRO variant works with any survival model that minimizes a loss of the form
\begin{equation}
L_{\text{average}}(\theta)
=\frac{1}{n}
 \sum_{i=1}^n
   L_i(\theta;\mathcal{A}_i),
\label{eq:average}
\end{equation}
where the $i$-th loss term $L_i$ depends on training point $i\in[n]$ as well as possibly other training points $\mathcal{A}_i\subseteq[n]\setminus\{i\}$. We refer to $\mathcal{A}_i$ as the \emph{adjacency set} for the $i$-th training point, where $\mathcal{A}_i$ can be empty. The basic idea is that the $i$-th training point's loss term could potentially depend on training points aside from the $i$-th training point.

Importantly, in what follows, we assume that we have access to a function $\mathcal{A}^*$ that tell us for any data point what its adjacency set is, and we also have access to a function $L^*(\cdot,\cdot;\theta)$ (with parameter variable $\theta$) that tells us for any data point what its individual loss term is. Note that $\theta$ contains all the underlying survival model's parameters and is thus shared across data points. The functions $\mathcal{A}^*$ and $L^*(\cdot,\cdot;\theta)$ (to be defined shortly) can be evaluated even for points that are not in the training data. We first formally describe these functions and then we state what they are for the survival models we presented in Section~\ref{sec:survival-models}.

\paragraph{Adjacency function}
Let $\mathcal{Z}\triangleq\mathcal{X}\times[0,\infty)\times\{0,1\}$ denote the set of possible data points (for instance, each training point $(X_i,Y_i,\Delta_i)$ belongs to $\mathcal{Z}$). We assume that given any $(x,y,\delta)\in\mathcal{Z}$ and any set $\mathcal{C}$ of data points from $\mathcal{Z}$, there is a function $\mathcal{A}^*$ that tells us which points in $\mathcal{C}$ are adjacent to $(x,y,\delta)$; namely, $\mathcal{A}^*( (x,y,\delta), \mathcal{C} )$ denotes the subset of points in $\mathcal{C}$ that are adjacent to $(x,y,\delta)$. This means that for the $i$-th training point $(X_i,Y_i,\Delta_i)$, the training points adjacent to the $i$-th point (and excluding the $i$-th point itself) is given by
\[
\mathcal{Z}_i
\triangleq
\mathcal{A}^*\big( (X_i,Y_i,\Delta_i), \{ (X_j,Y_j,\Delta_j)\text{ for }j\in[n]\setminus\{i\} \} \big).
\]
Then the adjacency set $\mathcal{A}_i$ is precisely defined as the set of training data indices corresponding to the data points in $\mathcal{Z}_i$:
\begin{equation}
\mathcal{A}_i
\triangleq
\{ j\in[n]\text{ such that }(X_j,Y_j,\Delta_j) \in \mathcal{Z}_i\}.
\label{eq:adjacency-set}
\end{equation}
\paragraph{Individual loss function}
Next, individual loss functions are determined using a function $L^*(\cdot,\cdot;\theta)$, where parameter variable $\theta$. Similar to the function $\mathcal{A}^*$, $L^*(\cdot,\cdot;\theta)$ takes two inputs: a single data point from $\mathcal{Z}$ and a (possibly empty) set of data points from $\mathcal{Z}$. Specifically, for any $(x,y,\delta)\in\mathcal{Z}$ and any set $\mathcal{C}$ of data points from $\mathcal{Z}$, we use $L^*( (x,y,\delta), \mathcal{C}; \theta )$ to denote the individual loss for data point $(x,y,\delta)$. We then define
\begin{equation}
L_i(\theta; \mathcal{I})
\triangleq L^*\big( (X_i, Y_i, \Delta_i), \{ (X_j, Y_j, \Delta_j)\text{ for }j\in\mathcal{I}\}; \theta \big)
\qquad\text{for any }\mathcal{I}\subseteq[n].
\label{eq:indiv-loss}
\end{equation}
In particular, the $i$-th data point's loss in equation~\eqref{eq:average} is taken to be $L_i(\theta; \mathcal{A}_i)$.

\vspace{1em}
\noindent
We now give explicit examples for the functions $\mathcal{A}^*$ and $L^*(\cdot,\cdot;\theta)$.

\begin{example}[Cox models]
For any $(x,y,\delta)\in\mathcal{Z}$ and for any (possibly empty) set $\mathcal{C}$ of data points from $\mathcal{Z}$,
define the adjacency function %
\[
\mathcal{A}^*\big( (x,y,\delta), \mathcal{C} \big)
\triangleq
  \begin{cases}
  \emptyset & \text{if }\delta = 0, \\
  \big\{ (x',y',\delta')\in\mathcal{C} \text{ such that }y'\ge y \} & \text{otherwise},
  \end{cases}
\]
and the individual loss function
\[
L^*( (x, y, \delta), \mathcal{C}; \theta )
\triangleq
-\delta
 \bigg[
   f(x;\theta)
   -\log\bigg(
      \exp(f(x;\theta))
      +
      \sum_{(x',y',\delta')\in\mathcal{C}}
        \exp(f(x';\theta))\bigg)\bigg].
\]
One can verify that plugging in these choices for $\mathcal{A}^*$ and $L^*(\cdot,\cdot;\theta)$ %
into equations~\eqref{eq:adjacency-set} and~\eqref{eq:indiv-loss} to obtain $\mathcal{A}_i$ and $L_i(\theta; \mathcal{A}_i)$, and subsequently plugging $\mathcal{A}_i$ and $L_i(\theta; \mathcal{A}_i)$ into equation~\eqref{eq:average}, we recover the Cox loss from equation~\eqref{eq:cox}.
\end{example}
\begin{example}[DeepHit]\label{ex:deephit}
Recall that DeepHit discretizes time so as to use the user-specified grid $t_1<t_2<\cdots<t_m$, and $\kappa:[t_1,\infty)\rightarrow[m]$ maps from a continuous time to one of the discrete time indices as given in equation~\eqref{eq:deephit-kappa}. For any $(x,y,\delta)\in\mathcal{Z}$ and for any (possibly empty) set $\mathcal{C}$ of data points from $\mathcal{Z}$, define the adjacency function
\[
\mathcal{A}^*\big( (x,y,\delta), \mathcal{C} \big)
\triangleq
  \begin{cases}
  \emptyset & \text{if }\delta = 0, \\
  \big\{ (x',y',\delta')\in\mathcal{C} \text{ such that }\kappa(y')\ge \kappa(y) \} & \text{otherwise},
  \end{cases}
\]
and the individual loss function
\begin{align}
L^*( (x,y,\delta), \mathcal{C}; \theta)
\triangleq
     & ~
     \beta\cdot
     \big[
       - \delta \log(f_{\kappa(y)}(x;\theta))
       - ( 1 - \delta ) \log( S_{\kappa(y)}(x;\theta) )
     \big]
\nonumber \\
& + (1-\beta)\cdot
    \frac{1}{n}
    \cdot
    \delta
    \sum_{(x',y',\delta')\in\mathcal{C}}
      \exp\Big(
        \frac{S_{\kappa(y)}(x;\theta) - S_{\kappa(y)}(x';\theta)}{\sigma}
      \Big).
\label{eq:deephit-individual-loss-splittable}
\end{align}
Plugging in these choices for $\mathcal{A}^*$ and $L^*(\cdot,\cdot;\theta)$ into equations~\eqref{eq:adjacency-set} and~\eqref{eq:indiv-loss} to obtain $\mathcal{A}_i$ and $L_i(\theta; \mathcal{A}_i)$, and subsequently plugging $\mathcal{A}_i$ and $L_i(\theta; \mathcal{A}_i)$ into equation~\eqref{eq:average}, we recover the DeepHit loss from equation~\eqref{eq:deephit-loss}.

Specifically when $\beta=1$, note that the second term (i.e., the ranking loss) in equation~\eqref{eq:deephit-individual-loss-splittable} becomes 0, so that in this special case, the adjacency function $\mathcal{A}^*$ can just always be set to output the empty set. %
Conceptually, the ranking loss is the only reason that the DeepHit loss has terms that couple different data points.
\end{example}
If $\mathcal{A}_i = \emptyset$ for all $i\in[n]$, then we can directly use the existing DRO optimization~\eqref{eq:empirical-dro-minimization}; the overall loss decouples across the different data points so we do not run into issues where multiple data points get ``coupled''.

\begin{example}[SODEN]
For the SODEN model, the overall loss function~\eqref{eq:SODEN-overall} actually has no coupling across training points, so it suffices to define the adjacency function $\mathcal{A}^*$ to always output the empty set. Meanwhile, for any $(x,y,\delta)\in\mathcal{Z}$ and any (possibly empty) set of data points from $\mathcal{Z}$, we define the individual loss function to be
\[
L^*( (x, y, \delta), \mathcal{C}; \theta)
\triangleq
-\delta\log f\big( (y, H_{\text{ODE-solve}}(y|x;\theta), x) ;\theta\big) + H_{\text{ODE-solve}}(y|x;\theta).
\]
With these definitions of $\mathcal{A}^*$ and $L^*(\cdot,\cdot;\theta)$, one can show that equation~\eqref{eq:average} becomes the SODEN loss from equation~\eqref{eq:SODEN-loss}.
\end{example}

As our examples above illustrate, the loss function in equation~\eqref{eq:average} can vary quite a bit across different survival models. For Cox models, the loss function corresponds to a negative partial log likelihood (it is called ``partial'' since it excludes the baseline hazard function; we discuss this in more detail in Section~\ref{sec:dro-cox-exact}). For the DeepHit model, the loss function corresponds to the sum of a negative log likelihood term and a ranking loss term (e.g., if hyperparameter $\beta$ is set equal to 0, then we would only be using the ranking loss term). For the SODEN model, the loss function is just a negative log likelihood. In particular, the loss function is not required to be a negative log likelihood, such as in the case of using only the DeepHit ranking loss (without the negative log likelihood term). %

\subsection{Applying DRO When Adjacency Sets Can be Nonempty}
\label{sec:dro-split}

We now discuss how to use DRO when $\mathcal{A}_i$ is not guaranteed to be empty. %

\subsubsection{Heuristic Approach}
To convert a survival analysis model that minimizes the loss~\eqref{eq:average} into one that uses DRO, a heuristic approach that does not comply with existing DRO theory would be to solve the DRO optimization problem \eqref{eq:empirical-dro-minimization}, ignoring the fact that the individual loss terms are not guaranteed to depend only on a single data point each. To be clear, existing DRO theory effectively requires that the sets $\mathcal{A}_i$ are all empty. As a preview of our experimental results, we mention that this heuristic approach actually works well in practice but we lack any justification as to why it should be expected to work well.

\subsubsection{Sample Splitting Approach}
We now propose a sample splitting approach that creates an approximate loss function that complies with existing DRO theory. We divide the training data into two sets $\mathcal{D}_{1}\subset[n]$ and $\mathcal{D}_{2} \triangleq [n]\setminus\mathcal{D}_{1}$ of sizes $n_{1} \triangleq |\mathcal{D}_{1}|$ and $n_{2} \triangleq |\mathcal{D}_{2}|=n-n_{1}$. The basic idea is that we treat the data points in $\mathcal{D}_{2}$ as fixed, and then define a DRO loss only over data points in $\mathcal{D}_{1}$. For each $i\in\mathcal{D}_{1}$, we replace its original individual loss $L_i(\theta;\mathcal{A}_i)$ with an approximate version $L_i(\theta;\mathcal{A}_i\cap\mathcal{D}_{2})$ that only depends on the $i$-th point along with points in $\mathcal{D}_{2}$. Specifically, we minimize the new DRO loss function
\begin{equation}
L_{\text{DRO}}^{\text{split}}(\theta,\eta,\mathcal{D}_{1}\mid\mathcal{D}_{2}) \triangleq C_{\alpha}\sqrt{\frac{1}{|\mathcal{D}_{1}|}\sum_{i\in\mathcal{D}_{1}}[L_i(\theta;\mathcal{A}_i\cap \mathcal{D}_{2})-\eta]_{+}^{2}}+\eta.
\label{eq:dro-split-one-way}
\end{equation}
The key observation is that conditioned on the points in $\mathcal{D}_{2}$, the loss terms ${L_i(\theta;\mathcal{A}_i\cap \mathcal{D}_{2})}$ appear i.i.d.~across $i\in\mathcal{D}_{1}$ and the $i$-th loss only depends on the $i$-th data point (and possibly points in $\mathcal{D}_{2}$ which are treated as fixed). Hence, the original DRO theory applies. More formally, we can state the following.

\begin{proposition}
\label{prop:sample-split}
Suppose that we condition on indices $\mathcal{D}_{2}$ and the data $\{(X_{i},Y_{i},\Delta_{i}):i\in\mathcal{D}_{2}\}$. Then the individual losses $L_{i}(\theta;\mathcal{A}_{i}\cap \mathcal{D}_{2})$ for $i\in\mathcal{D}_{1}$ appear i.i.d. Consequently, we can directly apply Propositions \ref{prop:hashimoto-et-al} and \ref{prop:duchi-et-al}, where
\begin{equation}
L_{\text{indiv}}(\theta;X,Y,\Delta)\triangleq L^{*}\big((X,Y,\Delta),\underbrace{\mathcal{A}^{*}\big((X,Y,\Delta),\{(X_j,Y_j,\Delta_j):j\in\mathcal{D}_{2}\}\big)}_{\substack{\text{points in }\mathcal{D}_2\\\text{ adjacent to }(X,Y,\Delta)}};\theta\big).
\label{eq:sample-split-indiv-loss}
\end{equation}
\end{proposition}
Clearly this sample splitting strategy is introducing an approximation since we replace $L_{i}(\theta;\mathcal{A}_{i})$ with $L_{i}(\theta;\mathcal{A}_{i}\cap \mathcal{D}_{2})$. However, it is unclear how to quantify the approximation error of the resulting individual loss in equation~\eqref{eq:sample-split-indiv-loss}. The technical challenge is that to measure this approximation error, we need to state what the target individual loss function is that we are measuring the error from. However, the problem is that DRO theory, to the best of our knowledge, does not work with an individual loss function that has coupling across points. In short, it is unclear what the ``correct'' $L_{\text{indiv}}$ function (that is compliant with DRO theory) should even be for the general class of survival models that we consider.

One could view Proposition~\ref{prop:sample-split} as positing an individual loss function that can be used with DRO theory. In particular, suppose that we treat $n_2 = |\mathcal{D}_2|$ as fixed, and we sample $\{(X_j',Y_j',\Delta_j')\}_{j=1}^{n_2}$ i.i.d.~from $\mathbb{P}$ (the same distribution that the training data are sampled from), where these freshly sampled points are also independent of the training data. Then equation~\eqref{eq:sample-split-indiv-loss} could be viewed as an empirical estimate of the individual risk
\begin{equation}
R_\text{indiv}(\theta; X, Y, \Delta)
\triangleq
\mathbb{E}_{\{(X_j',Y_j',\Delta_j')\}_{j=1}^{n_2}}\big[ L^{*}\big((X,Y,\Delta),\mathcal{A}^{*}((X,Y,\Delta),\{(X_j',Y_j',\Delta_j')\}_{j=1}^{n_2});\theta\big) \big].
\end{equation}
We refer to $R_\text{indiv}$ as a risk rather than a loss since it cannot be computed exactly in practice due to the expectation. This risk says that for any individual data point, we measure its error in comparison to $n_2$ randomly sampled reference data points.

We point out that our sample splitting strategy is somewhat inspired by the ``case control'' strategy by \citet{kvamme2019time}, where instead of using the original Cox loss, they approximate each individual data point's loss (which could depend on many other data points) to only depend on a \emph{single} other randomly sampled reference data point. Kvamme et al.~found that by optimizing this modified loss, the resulting model's prediction accuracy is often about as good as using the original Cox loss.

Returning to the earlier question of quantifying the ``approximation error'' of replacing $L_{i}(\theta;\mathcal{A}_{i})$ with $L_{i}(\theta;\mathcal{A}_{i}\cap \mathcal{D}_{2})$, we reiterate that we do not know of a clear way to do this. If we state that our goal is to minimize the individual risk $R_\text{indiv}(\theta;X,Y,\Delta)$, then clearly we would use the empirical version of this individual risk as given by $L_{i}(\theta;\mathcal{A}_{i}\cap \mathcal{D}_{2})$, and in particular, it would not make sense to use $L_{i}(\theta;\mathcal{A}_{i})$, which we suspect does not correspond to any individual loss or risk that works with DRO theory when $\mathcal{A}_i$ is nonempty.

\paragraph{Cross-fitting}
Although minimizing $L_{\text{DRO}}^{\text{split}}(\theta,\eta,\mathcal{D}_{1}\mid\mathcal{D}_{2})$ is compliant with DRO theory, it uses data ``less effectively'' since at most $n_1$ data points (rather than $n$) are used to compute the empirical average inside the square root in equation~\eqref{eq:dro-split-one-way} (as compared to the empirical average inside the square root of $L_{\text{DRO}}(\theta,\eta)$ in equation~\eqref{eq:empirical-dro-minimization}); as reminder, some individual loss terms might actually be zero (in the case of the Cox model, individual loss terms are zero for censored data).
Moreover, in the new split loss $L_{\text{DRO}}^{\text{split}}(\theta,\eta,\mathcal{D}_{1}\mid\mathcal{D}_{2})$, each individual loss within the empirical average is computed using only a subset of each individual's original adjacency set (for each $i\in\mathcal{D}_1$, we approximate individual loss $L_i(\theta;\mathcal{A}_i)$ by $L_i(\theta;\mathcal{A}_i\cap\mathcal{D}_2)$).

To ``more effectively'' use data, we use the basic idea from cross-fitting (e.g., \citealt{schick1986asymptotically,chernozhukov2018double}).
Whereas the loss $L_{\text{DRO}}^{\text{split}}(\theta,\eta,\mathcal{D}_{1}\mid\mathcal{D}_{2})$ treats $\mathcal{D}_2$ as fixed and computes an average over $\mathcal{D}_1$, we also do the opposite: we treat $\mathcal{D}_1$ as fixed and compute an average over $\mathcal{D}_2$, which would corresponds precisely to using the loss $L_{\text{DRO}}^{\text{split}}(\theta,\eta',\mathcal{D}_{2}\mid\mathcal{D}_{1})$; note that we use a different variable $\eta'$ than the variable $\eta$ used in $L_{\text{DRO}}^{\text{split}}(\theta,\eta,\mathcal{D}_{1}\mid\mathcal{D}_{2})$. Overall, we minimize the loss
\[
L_{\text{DRO}}^{\text{split}}(\theta,\eta,\eta')
\triangleq
\frac{1}{2}L_{\text{DRO}}^{\text{split}}(\theta,\eta,\mathcal{D}_{1}\mid\mathcal{D}_{2})
+
\frac{1}{2}L_{\text{DRO}}^{\text{split}}(\theta,\eta',\mathcal{D}_{2}\mid\mathcal{D}_{1})
\]
via coordinate descent, alternating between the following steps:
\begin{itemize}[itemsep=0pt,parsep=0pt,topsep=2pt]
\item Treating $\eta'$ and $\theta$ as fixed, we update $\eta$ by finding the value of $\eta$ that minimizes $L_{\text{DRO}}^{\text{split}}(\theta,\eta,\eta')$. This amounts to solving ${\min_{\eta\in\mathbb{R}} L_{\text{DRO}}^{\text{split}}(\theta,\eta,\mathcal{D}_{1}\mid\mathcal{D}_{2})}$ using binary search (since $L_{\text{DRO}}^{\text{split}}(\theta,\eta,\mathcal{D}_{1}\mid\mathcal{D}_{2})$ is convex w.r.t.~$\eta$).
\item Treating $\eta$ and $\theta$ as fixed, we update $\eta'$ by finding the value of $\eta'$ that minimizes $L_{\text{DRO}}^{\text{split}}(\theta,\eta,\eta')$. This amounts to
solving ${\min_{\eta'\in\mathbb{R}} L_{\text{DRO}}^{\text{split}}(\theta,\eta',\mathcal{D}_{2}\mid\mathcal{D}_{1})}$ using binary search.
\item Treating $\eta$ and $\eta'$ as fixed, we update $\theta$ by minimizing $L_{\text{DRO}}^{\text{split}}(\theta,\eta,\eta')$ (e.g., using gradient descent).
\end{itemize}
We provide the pseudocode in Algorithm~\ref{alg:dro-cox-split}. 

Note that it is possible to do cross-fitting where we partition the training data into more than 2 sets $\mathcal{D}_1$ and $\mathcal{D}_2$, similar to how for K-fold cross-validation, one could use more than~2 folds. We explain how to do this in Appendix~\ref{sec:k-fold-cross-fitting}. However, for simplicity, we only use 2-fold cross-fitting in our experiments later.

\begin{algorithm}[t!]\footnotesize
    \caption{\textsc{dro (split)}}\label{alg:dro-cox-split}
    \SetAlgoLined
    \KwIn{A training dataset $\{(X_i,Y_i,\Delta_i)\}_{i=1}^n$, minimum subpopulation probability hyperparameter $\alpha$, subset~size~$n_1$, learning rate $\xi$, max\_iterations}
    \KwOut{Survival model parameters $\widehat{\theta}$} 
    Obtain initial survival model parameters $\widehat{\theta}_0$ (e.g., using default PyTorch parameter initialization).
    
    Set $\mathcal{D}_1 \leftarrow \{1,2,\dots,n_1\}$ and $\mathcal{D}_2 \leftarrow \{n_1+1,\dots,n\}$.
    
    \For{$\ell=0$ to \emph{max\_iterations}}{
    \For{$i\in \mathcal{D}_1$}{
        Set $u_i \leftarrow L_i(\widehat{\theta}_{\ell};\mathcal{A}_i\cap \mathcal{D}_{2})$. %
    }
    Set $\widehat{\eta} \leftarrow \arg\min_{\eta\in\mathbb{R}}
    \Big\{\Big(\sqrt{2\big(\frac{1}{\alpha} - 1\big)^2 + 1}~\!\Big) \sqrt{\frac{1}{n_1}\sum_{i\in\mathcal{D}_1} [u_i - \eta]_+^2} + \eta\Big\}$,
    where this minimization is solved using binary search. %
    
    \For{$i\in\mathcal{D}_2$}{
        Set $v_i \leftarrow L_i(\widehat{\theta}_{\ell};\mathcal{A}_i\cap D_{1})$.
    }

    Set $\widehat{\eta}' \leftarrow \arg\min_{\eta'\in\mathbb{R}}
    \Big\{\Big(\sqrt{2\big(\frac{1}{\alpha} - 1\big)^2 + 1}~\!\Big) \sqrt{\frac{1}{n_2}\sum_{i\in\mathcal{D}_2} [v_i - \eta']_+^2} + \eta'\Big\}$,
    where this minimization is solved using binary search.
    
    Set $\widehat{\theta}_{\ell+1} \leftarrow \widehat{\theta}_\ell - \frac{\xi}{2}\big(\nabla_{\theta} L_{\text{DRO}}^{\text{split}}(\widehat{\theta}_\ell,\widehat{\eta},\mathcal{D}_{1}\mid\mathcal{D}_{2})+\nabla_{\theta} L_{\text{DRO}}^{\text{split}}(\widehat{\theta}_\ell,\widehat{\eta}',\mathcal{D}_{2}\mid\mathcal{D}_{1})\big)$.

    }
    \Return{$\widehat{\theta} \leftarrow \widehat{\theta}_{\emph{max\_iterations}+1}$}
\end{algorithm}

\subsection{Theoretical Guarantees for the Sample Splitting Approach}
\label{sec:sample-splitting-dro-theory}

We now derive a theoretical guarantee for our sample splitting approach. We begin by stating assumptions on the survival data generating process:
\begin{itemize}
\item \textbf{A1 (compact raw input space).} We assume that the raw input space $\mathcal{X}\subseteq\mathbb{R}^{d}$ is compact, and we denote its $\varepsilon$-covering number in Euclidean distance by $\mathbb{N}(\varepsilon,\mathcal{X})$.
\item \textbf{A2 (discrete time).} We assume that the survival and censoring times are discrete along a grid $t_{1}<t_{2}<\cdots<t_{m}$ (with $m\ge2$), and that all of these time points are used in the sense that there exists a positive constant $\zeta>0$ such that
\[
\mathbb{P}(Y=t_{\ell})\ge\zeta\qquad\text{for all time indices }\ell\in[m].
\]
In other words, the probability of an observed time being equal to $t_{\ell}$ is never 0.
\end{itemize}
Next, we state assumptions on the adjacency function $\mathcal{A}^{*}$ and the loss function $L^{*}$ (note that special instances of Cox and DeepHit models satisfy these conditions, as we explain later):
\begin{itemize}
\item \textbf{A3 (adjacency function).} We assume that the adjacency function is as stated for the DeepHit model (in fact, under ssumption A2, the adjacency function for the Cox model would be equivalent but for simplicity, we use the version stated for the DeepHit model that explicitly has time discretized).
\item \textbf{A4 (loss function).} We assume that the individual loss function $L^{*}$ is of the form
\begin{align*}
&L^{*}((x,y,\delta),\mathcal{C};\theta)\\
&\quad
  =\phi_{\text{indiv}}((x,y,\delta);\theta)+\delta\cdot\phi_{\text{transform}}\bigg(\sum_{(x',y',\delta')\in\mathcal{C}}\phi_{\text{couple}}((x,y,\delta),(x',y',\delta');\theta)\bigg)
\end{align*}
for some functions $\phi_{\text{indiv}}(\cdot;\theta):\mathcal{Z}\rightarrow\mathbb{R}$, $\phi_{\text{couple}}(\cdot,\cdot;\theta):\mathcal{Z}\times\mathcal{Z}\rightarrow\mathbb{R}$, and $\phi_{\text{transform}}:\mathbb{R}\rightarrow\mathbb{R}$. These functions satisfy the following conditions:
\begin{itemize}
\item [(a)]There exist constants $M_{\text{indiv}}\in[0,\infty)$ and $M_{\text{couple-min}},M_{\text{couple-max}}\in(0,\infty)$ such that
\[
\phi_{\text{indiv}}((x,y,\delta);\theta)
\in[0,M_{\text{indiv}}] \qquad\text{for all }(x,y,\delta)\in\mathcal{Z}\text{ and }\theta\in\Theta,
\]
and
\begin{gather*}
\phi_{\text{couple}}((x,y,\delta),(x',y',\delta');\theta)
\in[M_{\text{couple-min}},M_{\text{couple-max}}] \\[.25em]
\qquad\qquad\qquad\qquad\qquad\qquad\text{for all }(x,y,\delta),(x',y',\delta')\in\mathcal{Z}\text{ and }\theta\in\Theta.
\end{gather*}
Importantly, when a coupling term appears, it is nontrivial in the sense that it is not just equal to 0, whereas we allow for the possibility that $M_{\text{indiv}}=0$.
\item [(b)]The function $\phi_{\text{transform}}$ is either (a) the identity function $\phi_{\text{transform}}(s)=s$, or (b) the function $\phi_{\text{transform}}(s)=\log(1+s)$.
\item [(c)]For any fixed $y\in\{t_{1},t_{2},\dots,t_{m}\}$, $\delta\in\{0,1\}$, set of data points $\mathcal{C}$ from $\mathcal{Z}$, and parameter setting $\theta\in\Theta$, the map $x\mapsto L^{*}((x,y,\delta),\mathcal{C};\theta)$ is $\mathcal{L}$-Lipschitz with respect to Euclidean norm, i.e., for all $x,x'\in\mathcal{X}$,
\[
|L^{*}((x,y,\delta),\mathcal{C};\theta)-L^{*}((x',y,\delta),\mathcal{C};\theta)|\le\mathcal{L}\|x-x'\|_{2}.
\]
\end{itemize}
\end{itemize}
We define
\[
R_{\text{DRO}}^{\text{split}}(\theta,\eta)\triangleq C_{\alpha}\sqrt{\mathbb{E}_{(X,Y,\Delta)\sim\mathbb{P}}\big[[R_{\text{indiv}}(\theta;X,Y,\Delta)-\eta]_{+}^{2}\big]}+\eta,
\]
where, as a reminder, $C_{\alpha}=\sqrt{2(\frac{1}{\alpha}-1)^{2}+1}$, and $R_{\text{indiv}}$ is defined in a manner that depends on taking the expectation with respect to a fresh sample $\{(X_{i}',Y_{i}',\Delta_{i}')\}_{i=1}^{n_{2}}$ with $n_{2}$ treated as a constant. Put another way, $R_{\text{DRO}}^{\text{split}}(\theta,\eta)$ is simply the population-level version of $L_{\text{DRO}}^{\text{split}}(\theta,\eta,\mathcal{D}_{1}\mid\mathcal{D}_{2})$. In particular, $R_{\text{DRO}}^{\text{split}}(\theta,\eta)$ does not depend on the training data.

We are now ready to state our main theoretical result.
\begin{theorem}
\label{thm:main-result}Fix $n\in\mathbb{N}$ even and randomly split the training data into $\mathcal{D}_{1}$ and $\mathcal{D}_{2}$ of sizes $n_{1}=n_{2}=n/2$. Let $\omega>0$. Suppose that Assumptions A1--A4 hold. If $\phi_{\text{transform}}(s)=s$, then define
\begin{align*}
M & \triangleq M_{\text{indiv}}+\frac{M_{\text{couple-max}}}{2}n,\\
M' & \triangleq(M_{\text{couple-max}}-M_{\text{couple-min}})\sqrt{\frac{\omega n}{2\zeta}}.
\end{align*}
If instead $\phi_{\text{transform}}(s)=\log(1+s)$, then define
\begin{align*}
M & \triangleq M_{\text{indiv}}+\log\Big(1+\frac{M_{\text{couple-max}}}{2}n\Big),\\
M' & \triangleq\frac{4(M_{\text{couple-max}}-M_{\text{couple-min}})}{\zeta M_{\text{couple-min}}}\sqrt{\frac{\omega}{n}}.
\end{align*}
Then with probability at least
\begin{equation}
1-2\bigg[\frac{M}{(C_{\alpha}-1)\big[2\sqrt{\frac{\omega}{n}}\max\{2,\frac{C_{\alpha}}{C_{\alpha}-1}\}M+(2\mathcal{L}+1)M'\big]}+\mathbb{N}(M',\mathcal{X})\bigg]e^{-\omega}-me^{-\frac{n\zeta}{16}}\label{eq:main-result-prob}
\end{equation}
over randomness in the training data, we have
\begin{align}
 & \bigg|\inf_{\eta\in\mathbb{R}}L_{\text{DRO}}^{\text{split}}(\theta,\eta,\mathcal{D}_{1}\mid\mathcal{D}_{2})-\inf_{\eta\in\mathbb{R}}R_{\text{DRO}}^{\text{split}}(\theta,\eta)\bigg|\nonumber \\
 & \quad\le10C_{\alpha}^{2}\Bigg[\sqrt{\frac{2}{n}}\max\Big\{2,\frac{C_{\alpha}}{C_{\alpha}-1}\Big\}\big(\sqrt{2\omega}+1\big)M+(2\mathcal{L}+1)M'\Bigg].\label{eq:main-result-loss-bound}
\end{align}
\end{theorem}

We defer the proof of this theorem to Appendix \ref{sec:pf-main-result}, where we state a slightly more general result in which $n$ need not be even, and $n_{1}$ need not equal $n_{2}$ (we present the theorem in the manner above to prevent the notation from getting more unwieldy). To help provide intuition for Theorem \ref{thm:main-result}, we illustrate its use for special cases of Cox and DeepHit models, where we see that as $n\rightarrow\infty$, the probability in equation (\ref{eq:main-result-prob}) goes to 1 and the right-hand side of bound (\ref{eq:main-result-loss-bound}) goes to 0.
\begin{corollary}
\label{cor:linear-cox-split-dro-guarantee}(Special case of a Cox model) Suppose that the raw input space $\mathcal{X}\subseteq\mathbb{R}^{d}$ and the parameter vector space $\Theta\in\mathbb{R}^{d}$ are both set to be the unit ball in $\mathbb{R}^{d}$, where $d\ge5$. Consider the standard Cox model where $f(x;\theta)=\theta^{\top}x$, and time is discrete along the finite grid $t_{1}<t_{2}<\cdots<t_{m}$ that satisfies Assumption A2. This setup implies that Assumptions A1, A3, and A4 are also satisfied. Specifically for Assumption A4, we have 
\begin{align*}
\phi_{\text{indiv}}((x,y,\delta);\theta) & =\underbrace{0}_{M_{\text{indiv}}},\\
\phi_{\text{couple}}((x,y,\delta),(x',y',\delta');\theta) & =e^{\theta^{\top}(x-x')}\in[\underbrace{e^{-2}}_{M_{\text{couple-min}}},\underbrace{e^{2}}_{M_{\text{couple-max}}}],\\
\phi_{\text{transform}}(s) & =\log(1+s).
\end{align*}
Furthermore, a valid Lipschitz constant in Assumption A4(c) in this case is $\mathcal{L}=1$.

Now assume that the number of training data is sufficiently large, namely
\begin{align*}
n\ge\max\Bigg\{ & \Big(\frac{4(e^{2}-e^{-2})}{\zeta e^{-2}}\Big)^{2}\Big(\frac{d+1}{2}\Big)e^{\sqrt{2\log\Big(\big(\frac{4(e^{2}-e^{-2})}{\zeta e^{-2}}\big)^{2}\big(\frac{d+1}{2}\big)\Big)-1}},\\
 & ~2e^{-\frac{6(e^{4}-1)}{\zeta\max\{2,\frac{C_{\alpha}}{C_{\alpha}-1}\}}-2},\quad e^{\sqrt{2(\log\frac{d+1}{2}-1)}+\log\frac{d+1}{2}},\quad e^{\frac{2}{d+1}}\Bigg\}.
\end{align*}
Define $\Upsilon\triangleq2\big[\frac{1}{2\max\big\{2(C_{\alpha}-1),C_{\alpha}\big\}}+\big(\frac{3\zeta e^{-2}}{4(e^{2}-e^{-2})}\big)^{d}\big]$, which is constant with respect to $n$. Then with probability at least
\begin{equation}
1-\frac{\Upsilon}{\sqrt{n}}-me^{-\frac{n\zeta}{16}}\label{eq:linear-cox-split-dro-prob}
\end{equation}
over randomness in the training data, we have
\begin{align}
 & \bigg|\inf_{\eta\in\mathbb{R}}L_{\text{DRO}}^{\text{split}}(\theta,\eta,\mathcal{D}_{1}\mid\mathcal{D}_{2})-\inf_{\eta\in\mathbb{R}}R_{\text{DRO}}^{\text{split}}(\theta,\eta)\bigg|\nonumber \\
 & \quad\le10C_{\alpha}^{2}\Bigg[\sqrt{\frac{2}{n}}\max\Big\{2,\frac{C_{\alpha}}{C_{\alpha}-1}\Big\}\big(\sqrt{(d+1)\log n}+1\big)\log\Big(1+\frac{e^{2}}{2}n\Big)\nonumber \\
 & \quad\phantom{\le10C_{\alpha}^{2}\Bigg[}~+\frac{12(e^{2}-e^{-2})}{\zeta e^{-2}}\sqrt{\frac{(d+1)\log n}{2n}}\Bigg]\nonumber \\
 & \quad=\widetilde{\mathcal{O}}\Big(\frac{1}{\sqrt{n}}\Big),\label{eq:linear-cox-split-dro-loss-bound}
\end{align}
where $\widetilde{\mathcal{O}}$ is big O notation ignoring log factors.
\end{corollary}

We provide the proof in Appendix \ref{sec:pf-linear-cox-split-dro-guarantee}. The key idea is that we set $\omega=\frac{d+1}{2}\log n$ for Theorem \ref{thm:main-result} and we further impose constraints on $n$ and $d$ that enable us to simplify the probability bound in equation (\ref{eq:main-result-prob}).
\begin{corollary}
\label{cor:deephit-split-dro-guarantee}(Special case of a DeepHit model) Just as in Corollary \ref{cor:linear-cox-split-dro-guarantee}, we assume $\mathcal{X}\subseteq\mathbb{R}^{d}$ and $\Theta\in\mathbb{R}^{d}$ are both set to be the unit ball in $\mathbb{R}^{d},$where $d\ge5$. We consider a DeepHit model defined over a discrete time grid with $m\ge3$ time points, where the conditional probabilities of each time index satisfy the bound
\[
f_{j}(x;\theta)=\mathbb{P}(T=t_{j}\mid X=x)\ge\varrho\qquad\text{for all }x\in\mathcal{X},j\in[m],\theta\in\Theta.
\]
Furthermore, we assume that the very last time step $t_{m}$ is special in that it is used to mean ``any time strictly after all the observed training times'' (so that $t_{m}>Y_{j}$ for all $j\in[n]$); of course this last time step also obeys the bound above of $f_{m}(x;\theta)\ge\varrho$. Moreover, we assume that $f_{j}(\cdot;\theta)$ is 1-Lipschitz:
\[
|f_{j}(x;\theta)-f_{j}(x';\theta)|\le\|x-x'\|_{2}\qquad\text{for all }x,x'\in\mathcal{X},j\in[m],\theta\in\Theta.
\]
We further suppose that Assumption A2 holds, that the number of training data is sufficiently large
\begin{align*}
n\ge\max\Bigg\{ & (1-\beta)^2 \big(e^{(1-\varrho)/\sigma}-e^{(\varrho-1)/\sigma}\big)^{2}\big(\frac{d+1}{4\zeta}\big)e^{\sqrt{2\log\big((1-\beta)^2 (e^{(1-\varrho)/\sigma}-e^{(\varrho-1)/\sigma})^{2}\big(\frac{d+1}{4\zeta}\big)\big)}},\\
 & e^{\sqrt{2(\log\frac{d+1}{2}-1)}+\log\frac{d+1}{2}},\quad e^{\frac{2}{d+1}}\Bigg\},
\end{align*}
and that
\[
\log\Big((1-\beta)^{2}(e^{(1-\varrho)/\sigma}-e^{(\varrho-1)/\sigma})^{2}\Big(\frac{d+1}{4\zeta}\Big)\Big)>1.
\]
Then this setup as stated also implies that Assumptions A1, A3, and A4 are satisfied. Specifically for Assumption A4, we have
\begin{align*}
\phi_{\text{indiv}}((x,y,\delta);\theta) & =\beta\cdot\big[-\delta\log(f_{\kappa(y)}(x;\theta))-(1-\delta)\log(S_{\kappa(y)}(x;\theta))\big],\\
\phi_{\text{couple}}((x,y,\delta),(x',y',\delta'),\mathcal{C};\theta) & =(1-\beta)\cdot\frac{1}{n}\cdot\exp\Big(\frac{S_{\kappa(y)}(x;\theta)-S_{\kappa(y)}(x';\theta)}{\sigma}\Big),\\
\phi_{\text{transform}}(s) & =s,
\end{align*}
where one can verify that
\begin{align*}
M_{\text{indiv}} & =\beta\log\frac{1}{\varrho},\\
M_{\text{couple-min}} & =\Big(\frac{1-\beta}{n}\Big)e^{\frac{\varrho-1}{\sigma}},\\
M_{\text{couple-max}} & =\Big(\frac{1-\beta}{n}\Big)e^{\frac{1-\varrho}{\sigma}},
\end{align*}
and that $x\mapsto L^{*}((x,y,\delta),\mathcal{C};\theta)$ has Lipschitz constant $\mathcal{L}=\frac{2\beta(m-1)}{\varrho}$.

Define the constant
\begin{align*}
\Psi & \triangleq\frac{2}{(C_{\alpha}-1)\Big[2\max\{2,\frac{C_{\alpha}}{C_{\alpha}-1}\}+(2\mathcal{L}+1)\Big(\frac{\big((1-\beta)e^{(1-\varrho)/\sigma}-(1-\beta)e^{(\varrho-1)/\sigma}\big)}{\big(2\beta\log\frac{1}{\varrho}+(1-\beta)e^{(1-\varrho)/\sigma}\big)}\sqrt{\frac{2}{\zeta}}\Big)\Big]}\\
 & \quad+2\bigg(\frac{3\sqrt{2\zeta}}{(1-\beta)e^{(1-\varrho)/\sigma}-(1-\beta)e^{(\varrho-1)/\sigma}}\bigg)^{d}.
\end{align*}
Then with probability at least
\[
1-\frac{\Psi}{\sqrt{n}}-me^{-\frac{n\zeta}{16}},
\]
we have
\begin{align*}
 & \bigg|\inf_{\eta\in\mathbb{R}}L_{\text{DRO}}^{\text{split}}(\theta,\eta,\mathcal{D}_{1}\mid\mathcal{D}_{2})-\inf_{\eta\in\mathbb{R}}R_{\text{DRO}}^{\text{split}}(\theta,\eta)\bigg|\\
 & \quad\le10C_{\alpha}^{2}\Bigg[\sqrt{\frac{2}{n}}\max\Big\{2,\frac{C_{\alpha}}{C_{\alpha}-1}\Big\}\big(\sqrt{(d+1)\log n}+1\big)\Big(\beta\log\frac{1}{\varrho}+\frac{(1-\beta)e^{(1-\varrho)/\sigma}}{2}\Big)\\
 & \phantom{\quad\le10C_{\alpha}^{2}\Bigg[}~+\Big(\frac{4\beta(m-1)}{\varrho}+1\Big)\sqrt{\frac{(d+1)\log n}{\zeta n}}\frac{(1-\beta)(e^{(1-\varrho)/\sigma}-e^{(\varrho-1)/\sigma})}{2}\Bigg]\\
 & \quad=\widetilde{\mathcal{O}}\Big(\frac{1}{\sqrt{n}}\Big).
\end{align*}
 
\end{corollary}

The proof is in Appendix \ref{sec:pf-deephit-split-dro-guarantee} and uses similar ideas as that of Corollary \ref{cor:linear-cox-split-dro-guarantee}. The constants change since $M$ and $M'$ are different, but we again set $\omega=\frac{d+1}{2}\log n$ in the statement of Theorem \ref{thm:main-result}.
\begin{corollary}
(Cross-fitting) We assume the same setting as Theorem \ref{thm:main-result}. For a fixed $\theta\in\Theta$, the cross-fitting approach solves
\[
\inf_{\eta,\eta'\in\mathbb{R}}L_{\text{DRO}}^{\text{split}}(\theta,\eta,\eta')=\frac{1}{2}\inf_{\eta\in\mathbb{R}}L_{\text{DRO}}^{\text{split}}(\theta,\eta,\mathcal{D}_{1}\mid\mathcal{D}_{2})+\frac{1}{2}\inf_{\eta'\in\mathbb{R}}L_{\text{DRO}}^{\text{split}}(\theta,\eta',\mathcal{D}_{2}\mid\mathcal{D}_{1}).
\]
With probability at least
\[
1-4\bigg[\frac{M}{(C_{\alpha}-1)\big[2\sqrt{\frac{\omega}{n}}\max\{2,\frac{C_{\alpha}}{C_{\alpha}-1}\}M+(2\mathcal{L}+1)M'\big]}+\mathbb{N}(M',\mathcal{X})\bigg]e^{-\omega}-2me^{-\frac{n\zeta}{16}},
\]
we have
\begin{align*}
 & \bigg|\inf_{\eta,\eta'\in\mathbb{R}}L_{\text{DRO}}^{\text{split}}(\theta,\eta,\eta')-\inf_{\eta\in\mathbb{R}}R_{\text{DRO}}^{\text{split}}(\theta,\eta)\bigg|\\
 & \quad\le10C_{\alpha}^{2}\Bigg[\sqrt{\frac{2}{n}}\max\Big\{2,\frac{C_{\alpha}}{C_{\alpha}-1}\Big\}\big(\sqrt{2\omega}+1\big)M+(2\mathcal{L}+1)M'\Bigg].
\end{align*}
\end{corollary}

\begin{proof}
The proof is straightforward and amounts to applying Theorem \ref{thm:main-result} for each of the two folds separately and then union bounding over the bad events of the two folds. This union bound just multiplies the bad event's probability in bound (\ref{eq:main-result-prob}) by 2. Then since this bad event does not happen for either fold, we have
\begin{align*}
 & \max\Bigg\{\bigg|\inf_{\eta\in\mathbb{R}}L_{\text{DRO}}^{\text{split}}(\theta,\eta,\mathcal{D}_{1}\mid\mathcal{D}_{2})-\inf_{\eta\in\mathbb{R}}R_{\text{DRO}}^{\text{split}}(\theta,\eta)\bigg|,\bigg|\inf_{\eta\in\mathbb{R}}L_{\text{DRO}}^{\text{split}}(\theta,\eta,\mathcal{D}_{2}\mid\mathcal{D}_{1})-\inf_{\eta\in\mathbb{R}}R_{\text{DRO}}^{\text{split}}(\theta,\eta)\bigg|\Bigg\}\\
 & \quad\le10C_{\alpha}^{2}\Bigg[\sqrt{\frac{2}{n}}\max\Big\{2,\frac{C_{\alpha}}{C_{\alpha}-1}\Big\}\big(\sqrt{2\omega}+1\big)M+(2\mathcal{L}+1)M'\Bigg].
\end{align*}
Then
\begin{align*}
 & \bigg|\inf_{\eta,\eta'\in\mathbb{R}}L_{\text{DRO}}^{\text{split}}(\theta,\eta,\eta')-\inf_{\eta\in\mathbb{R}}R_{\text{DRO}}^{\text{split}}(\theta,\eta)\bigg|\\
 & \quad=\bigg|\frac{1}{2}\inf_{\eta\in\mathbb{R}}L_{\text{DRO}}^{\text{split}}(\theta,\eta,\mathcal{D}_{1}\mid\mathcal{D}_{2})+\frac{1}{2}\inf_{\eta'\in\mathbb{R}}L_{\text{DRO}}^{\text{split}}(\theta,\eta',\mathcal{D}_{1}\mid\mathcal{D}_{2})\\
 & \quad\quad\;-\frac{1}{2}\inf_{\eta\in\mathbb{R}}R_{\text{DRO}}^{\text{split}}(\theta,\eta)-\frac{1}{2}\inf_{\eta\in\mathbb{R}}R_{\text{DRO}}^{\text{split}}(\theta,\eta)\bigg|\\
 & \quad=\frac{1}{2}\bigg|\inf_{\eta\in\mathbb{R}}L_{\text{DRO}}^{\text{split}}(\theta,\eta,\mathcal{D}_{1}\mid\mathcal{D}_{2})-\inf_{\eta\in\mathbb{R}}R_{\text{DRO}}^{\text{split}}(\theta,\eta)\\
 & \quad\quad\;\phantom{\frac{1}{2}}+\inf_{\eta'\in\mathbb{R}}L_{\text{DRO}}^{\text{split}}(\theta,\eta',\mathcal{D}_{1}\mid\mathcal{D}_{2})-\inf_{\eta\in\mathbb{R}}R_{\text{DRO}}^{\text{split}}(\theta,\eta)\bigg|\\
 & \quad\le\frac{1}{2}\bigg|\inf_{\eta\in\mathbb{R}}L_{\text{DRO}}^{\text{split}}(\theta,\eta,\mathcal{D}_{1}\mid\mathcal{D}_{2})-\inf_{\eta\in\mathbb{R}}R_{\text{DRO}}^{\text{split}}(\theta,\eta)\bigg|\\
 & \quad\quad\;\frac{1}{2}\bigg|\inf_{\eta'\in\mathbb{R}}L_{\text{DRO}}^{\text{split}}(\theta,\eta',\mathcal{D}_{1}\mid\mathcal{D}_{2})-\inf_{\eta\in\mathbb{R}}R_{\text{DRO}}^{\text{split}}(\theta,\eta)\bigg|\\
 & \quad\le10C_{\alpha}^{2}\Bigg[\sqrt{\frac{2}{n}}\max\Big\{2,\frac{C_{\alpha}}{C_{\alpha}-1}\Big\}\big(\sqrt{2\omega}+1\big)M+(2\mathcal{L}+1)M'\Bigg].
\end{align*}
\end{proof}

\section{An Exact DRO Cox Approach}
\label{sec:dro-cox-exact}

We now derive an exact DRO approach for Cox models. The rough idea is that we reparameterize the Cox model in such a way that the resulting loss function decouples across training data points, removing the coupling issue. Our derivation here is specific to the Cox model and, as far as we are aware, does not easily generalize to other survival models with nonempty adjacency sets.

To obtain an exact approach for using the Cox model with DRO that does not require sample splitting, we turn to a standard derivation of the Cox partial likelihood loss. Specifically, \citet{breslow1972discussion} showed that the Cox log partial likelihood could be derived by assuming that the baseline hazard function $h_0$ is piecewise constant. First, denote the unique times in which the critical event happened in the training data as $t_1<t_2<\cdots<t_m$ (so that there are $m$ unique times in which the event happened), with the convention that $t_0\triangleq0$ (note that we are reusing notation used for the discrete time grid for DeepHit; however, the difference is that for the Cox model, the time grid is typically set based on the unique critical event times whereas for DeepHit, the time grid is user-specified and need not be the unique critical event times). Then we parameterize the baseline hazard function as
\begin{equation}
h_0(t;\psi)
\triangleq
  \begin{cases}
    e^{\psi_{\ell}}
      & \text{if }t_{\ell-1}<t\le t_{\ell}
        \quad\text{for }\ell\in[m], \\
    0 & \text{otherwise},
\end{cases}
\label{eq:hazard-piecewise-constant}
\end{equation}
where $\psi_1,\psi_2,\dots,\psi_m\in\mathbb{R}$ are parameters to be learned, and $\psi\triangleq(\psi_1,\dots,\psi_m)$.\footnote{Note that in equation~\eqref{eq:hazard-piecewise-constant}, the exponential function can be replaced with any differentiable, strictly increasing, positive activation function (e.g., instead of $e^{\psi_\ell}$, we could use the softplus function $\log{(1+\exp(\psi_\ell))}$). For ease of exposition, we stick to using the exponential function.} %

Next, let $\kappa(Y_i,\Delta_i)\in[m]$ denote the discrete time index that $Y_i$ corresponds to in a manner that depends also on $\Delta_i$: if $\Delta_i=1$, then $\kappa(Y_i,\Delta_i)$ is set equal to the index $\ell$ such that $t_\ell = Y_i$, and if $\Delta_i=0$ (so that $Y_i$ is a censoring time), then we set $\kappa(Y_i,\Delta_i)$ to be the largest time index corresponding to when a critical event happened strictly before $Y_i$ (i.e., we use the largest index in $\{\ell\in\{0,1,\dots,m\} : t_\ell < Y_i\}$). Then the full negative Cox log likelihood can be written as
\begin{equation}
L^{\text{Cox-full}}(\theta,\psi)
\triangleq
  \frac{1}{n}
  \sum_{i=1}^n
    L_i^{\text{Cox-full}}(\theta,\psi),
\label{eq:cox-full}
\end{equation}
where
\begin{equation}
L_i^{\text{Cox-full}}(\theta,\psi)
\triangleq
  -\Delta_i
   [f(X_i;\theta)
    + \psi_{\kappa(Y_i,\Delta_i)} ]
  +
  e^{f(X_i;\theta)}
  \sum_{\ell=1}^{\kappa(Y_i,\Delta_i)}
    (t_{\ell}-t_{\ell-1})e^{\psi_{\ell}}.
\label{eq:cox-full-indiv}
\end{equation}
Then a standard result is as follows.

\begin{proposition}[slight variant of \citealt{breslow1972discussion}]\label{prop:cox-full-vs-partial}
Suppose that the baseline hazard function is piecewise constant as stated in equation~\eqref{eq:hazard-piecewise-constant}. Suppose that we preprocess the data so that for each training point $i\in[n]$ that is censored (i.e., $\Delta_i=0$), we set $Y_i \triangleq t_{\kappa(Y_i,0)}$ (we do not modify the observed times for the uncensored training points). Then the partial Cox loss $L^{\text{Cox}}$ (from equation~\eqref{eq:cox}) is related to the full Cox loss $L^{\text{Cox-full}}$ (from equation~\eqref{eq:cox-full}) by
\[
L^{\text{Cox}}(\theta) = \min_{\psi\in\mathbb{R}^{m}}L^{\text{Cox-full}}(\theta,\psi)
+
\text{constant w.r.t.~}\theta.
\]
Hence, $\arg\min_{\theta\in\Theta} L^{\text{Cox}}(\theta) = \arg\min_{\theta\in\Theta} \{ \min_{\psi\in\mathbb{R}^m}L^{\text{Cox-full}}(\theta,\psi) \}$, where $\Theta$ is the feasible set of model parameters.
\end{proposition}
While the proof is standard \citep{breslow1972discussion}, to keep the paper relatively self-contained, we provide it in Appendix~\ref{sec:cox-full-details}, where we also provide a little bit of background on how the expression for individual loss $L_{i}^{\text{Cox-full}}(\theta,\psi)$ (from equation~\eqref{eq:cox-full-indiv}) is derived. We separately point out that, as far as we are aware, the full Cox loss in equation~\eqref{eq:cox-full} is typically not used in practice and is instead mainly used for theoretically justifying the standard Cox loss (equation~\eqref{eq:cox}) that actually is extremely commonly used in practice.

An immediate consequence of Proposition~\ref{prop:cox-full-vs-partial} is that we could apply DRO to the loss $L^{\text{Cox-full}}(\theta, \psi)$ (using the individual losses given by $L_i^{\text{Cox-full}}(\theta, \psi)$ in equation~\eqref{eq:cox-full-indiv}), which does not involve coupling across training points. The high-level idea is that whereas $L^{\text{Cox}}(\theta)$ had the coupling issue, by introducing an additional parameter variable $\psi$, we remove the dependence between the training points' contributions.

\paragraph{Numerical optimization} For completeness, we now state how to use DRO with the full Cox loss. We first define
\begin{equation}
L_{\text{DRO-Cox-exact}}(\theta,\psi,\eta)
\triangleq
  C_\alpha
  \sqrt{
    \frac{1}{n}
    \sum_{i=1}^n [L_i^{\text{Cox-full}}(\theta,\psi) - \eta]_+^2
  }
  + \eta.
\end{equation}
Then we alternate between the following two steps until convergence:
\begin{itemize}[itemsep=0pt,topsep=2pt,parsep=0pt]

\item Treating $\theta$ and $\psi$ as fixed, we update $\eta$ by finding the value of $\eta$ that minimizes $L_{\text{DRO-Cox-exact}}(\theta,\psi,\eta)$. As before, this step is done using binary search to find the global minimum since $L_{\text{DRO-Cox-exact}(\theta,\psi,\eta)}$ is convex with respect to $\eta$.
\item Treating $\eta$ as fixed, we update $(\theta,\psi)$ by minimizing $L_{\text{DRO-Cox-exact}}(\theta,\psi,\eta)$ (e.g., using gradient descent).

\end{itemize}
This procedure corresponds to using Algorithm~\ref{alg:dro-cox}, where $L_{\text{indiv}}$ is set to be $L_i^{\text{Cox-full}}(\theta,\psi)$ (equation~\eqref{eq:cox-full-indiv}), and the survival model parameter variable $\theta$ is replaced by $(\theta, \psi)$. Note that we intentionally specified the parameter variable $\psi$ so that it remains unconstrained so that it could be optimized along with $\theta$ using standard gradient descent variants.

\section{Experiments} \label{sec:Experiments}

To see how well our general proposed DRO conversion strategy works in practice (the heuristic approach without guarantees and, separately, our sample splitting DRO approach), we now conduct extensive experiments to evaluate the accuracy and fairness of DRO variants of different survival models compared to the original versions of these models, as well as to versions of these models modified to encourage fairness using existing fairness regularizers. Specifically for the Cox model, we also show how well our exact Cox DRO approach works in practice.

We describe the datasets we use in Section~\ref{sec:datasets}, the experimental setup in Section~\ref{sec:experimental-setup}, the evaluation metrics in Section~\ref{sec:eval-metrics}, and the models evaluated in Section~\ref{sec:models-evaluated}. We then present our experimental results in Section~\ref{sec:experimental-results}. Lastly, we show how to compare across multiple models using a plot inspired by ROC curves in Section~\ref{sec:fairness-accuracy-tradeoff}.

\subsection{Datasets}
\label{sec:datasets}

We use three standard, publicly available survival analysis datasets: %
\begin{itemize}[itemsep=0pt,topsep=2pt,parsep=0pt]

\item The \textbf{FLC} dataset \citep{dispenzieri2012use} is from a study on the relationship between serum free light chain (FLC) and mortality of Olmsted County residents aged 50 or higher. We treat discretized age (age$\leq$65 and age$>$65) and gender (women and men) as sensitive attributes.

\item The \textbf{SUPPORT} dataset \citep{knaus1995support} is from a study at Vanderbilt University on understanding prognoses, preferences, outcomes, and risks of treatment by analyzing survival times of severely ill hospitalized patients. We treat discretized age (age$\leq$65 and age$>$65), race (white and non-white), and gender (women and men) as sensitive attributes.

\item The \textbf{SEER} dataset is on breast cancer patients from the Sureillance, Epidemiology, and End Results (SEER) program of the National Cancer Institute. %
We collected this dataset using the data extraction software from the official SEER program of the National Cancer Institute. We used 11 covariates that also appear in an existing snapshot of the SEER dataset \citep{Teng2019SEER} that only contained 4024 data points. We also treat discretized age (age$\leq$65 and age$>$65) and race (white and non-white) as sensitive attributes. %

\end{itemize} %
These datasets have appeared in existing fair survival analysis research (e.g., \citealt{keya2021equitable,rahman2022fair,zhang2022longitudinal}) although not always with all three of these appearing within the same paper. Basic characteristics of these datasets are reported in Table \ref{tab:overview_datasets}.

\begin{table}[t]\small
\caption{Basic dataset characteristics.}\vspace{.5em}
\label{tab:overview_datasets}
\centering
\begin{tabular}{cccc}
\toprule
& FLC & SUPPORT & SEER \\ \midrule
\# samples & 7,874 & 9,105 & 28,018 \\ %
\# features & 6 (9$^*$) & 14 (19$^*$) & 11 \\ %
Censoring rate & 0.725 & 0.319 & 0.654 \\ %
\makecell{Sensitive  attributes} & age, gender & age, race, gender & age, race
\\ \bottomrule
\end{tabular}
\\[3pt]
$^*$ indicates the number before preprocessing  (preprocessing removes some features)
\end{table}

\subsection{Experimental Setup}
\label{sec:experimental-setup}
For all models, we first use a random 80\%/20\% train/test split to hold out a test set that will be the same across experimental repeats for all datasets. Then we repeat the following basic experiment 10 times: (1) We hold out 20\% of the training data to treat as a validation set, which is used to tune hyperparameters. (2) We then compute evaluation metrics across the same test set. We describe the evaluation metrics and how hyperparameter tuning works shortly. When we report our experimental results, we provide the mean and standard deviation of each metric across the 10 experimental repeats.
More hyperparameter settings can be found in Appendix \ref{sec:hyperparameters-compute-env}.

\subsection{Evaluation Metrics}
\label{sec:eval-metrics}
For accuracy metrics, we use Time-dependent concordance index (C$^{td}$, higher is better) \citep{antolini2005time} and Integrated IPCW Brier Score (IBS, lower is better) \citep{graf1999assessment}.
For fairness metrics, we use the concordance imparity (CI) fairness metric by \citep{zhang2022longitudinal}, Censoring-based individual fairness (F$_{CI}$) \citep{rahman2022fair}, and Censoring-based group fairness (F$_{CG}$) \citep{rahman2022fair}. For these fairness metrics, lower is better. Definitions of these fairness metrics are in Appendix~\ref{sec:fairness-measures}.

Note that the fairness metrics CI and F$_{CG}$ require us to specify groups. For the FLC dataset, we use (discretized) age and, separately, gender (i.e., we first run experiments using only age in evaluating CI and F$_{CG}$; we then re-run experiments using gender instead of age). For the SUPPORT dataset, we separately use gender, age, and race. For the SEER dataset, we separately use race and age.

\subsection{Models Evaluated}
\label{sec:models-evaluated}

Working off our running examples from Section~\ref{sec:survival-models}, we consider Cox models (classical and deep), DeepHit, and SODEN. For each of these, we compare the original model with its DRO variants using our conversion strategy (the heuristic approach and also the sample splitting approach stated in Section~\ref{sec:dro-split}; for Cox models, we also compare with the exact DRO Cox approach, and for SODEN, there is no need to do sample splitting and the heuristic approach is actually exact). We also try versions of the original models modified to encourage fairness using existing fairness regularizers. %

Note that when we use our sample splitting DRO approach, for simplicity, we randomly split the training data so that $\mathcal{D}_1$ and $\mathcal{D}_2$ are (approximately) the same size and, moreover, we stratify the sampling so that $\mathcal{D}_1$ and $\mathcal{D}_2$ have (approximately) the same censoring rates. We include some supplemental experiments that consider deviations of both of these in Appendix~\ref{additional_exp_results} that focus specifically on Cox models.

\paragraph{Cox models}
We separately experiment on the classical \textbf{linear} setting (the log partial hazard function is $f(x;\theta)=\theta^T x$) or the ``deep'' \textbf{nonlinear} setting in which $f$ is a multilayer perceptron (MLP). In the linear case, we denote the heuristic DRO variant as \textsc{dro-cox} and the sample splitting DRO variant as \textsc{dro-cox (split)}. For the nonlinear case, we add the prefix ``Deep'' to these names for clarity.

In terms of baselines, we use the unregularized linear Cox model \citep{cox1972regression} (denoted as ``Cox'' in tables later), whereas the unregularized nonlinear Cox model \citep{katzman2018deepsurv} is denoted as ``DeepSurv''. As baselines, we use regularized versions of either the standard Cox or DeepSurv models, using different fairness regularization terms. When we use individual, group, or intersectional regularization terms by \citet{keya2021equitable} (we discuss these in Appendix~\ref{sec:fairness-measures}), we add the suffix ``$_I$(Keya~et~al.)'', ``$_G$(Keya~et~al.)'', or ``$_{\cap}$(Keya~et~al.)'' respectively to a model name; for example, ``DeepSurv$_G$(Keya~et~al.)'' corresponds to DeepSurv with group fairness regularization by \citet{keya2021equitable}. When we use the individual or group fairness regularization terms that account for censoring information \citep{rahman2022fair}, we instead use the suffix ``$_I$(R\&P)'' or ``$_G$(R\&P)''.\footnote{\citet{rahman2022fair} did not propose an intersectional fairness regularizer and technically did not try regularized versions of Cox models using their fairness definitions. However, it is straightforward to adapt their individual and group fairness definitions as regularization terms for a Cox model, especially as their work is directly modifying definitions by \citet{keya2021equitable}.} Note that group fairness regularization (suffixes ``$_G$(Keya et al.)'' and ``$_G$(R\&P)'') uses the same groups that test set CI and F$_{CG}$ fairness metrics use. As additional baselines, we also use the pseudo value-based approaches proposed by \citet{rahman2022fair}, namely their Fair DeepPseudo and Fair PseudoNAM methods (abbreviated as ``FIDP'' and ``FIPNAM'' respectively; note that these abbreviations are the same as the ones used by \citet{rahman2022fair} and, moreover, following Rahman and Purushotham's paper and implementation, FIDP and FIPNAM specifically use individual fairness regularization).

In terms of hyperparameter tuning, we use the strategy by \citet{keya2021equitable}: the final hyperparameter setting used per dataset and per method is determined based on a preset rule in practice that allows up to a 5\% degradation in the validation set C$^{td}$ from the classical Cox model (for the linear setting) or DeepSurv (for the nonlinear setting) while minimizing the validation set CI fairness metric or F$_{CG}$ fairness metric (for details, see Appendix~\ref{additional_exp_results}).

\paragraph{DeepHit and SODEN}
For DeepHit \citep{tang2022soden}, we denote its heuristic DRO variant as \textsc{dro-deephit} and its sample splitting DRO variant as \textsc{dro-deephit (split)}. For SODEN \citep{tang2022soden}, there is only one DRO variant to consider which we denote as \textsc{dro-soden}.

In terms of baselines, we consider the original DeepHit and SODEN models that do not account for fairness. We further adapt the group-based fairness regularization that accounts for censoring from \citet{rahman2022fair} to each of DeepHit and SODEN separately as additional baselines (\textsc{\tableDeepHitRPGroup} and \tableSODENRPGroup).

The hyperparameter setting used per dataset and per method is also determined based on a preset rule in practice that allows up to a 5\% degradation in the validation set C$^{td}$ from the original model (that does not encourage fairness) while minimizing the validation set CI fairness metric or F$_{CG}$ fairness metric. Hyperparameter grids for all methods are in Appendix~\ref{sec:hyperparameters-compute-env}, where we also provide information on the compute environment that we used.

\subsection{Experimental Results}\label{sec:experimental-results}

\paragraph{Cox models}

\begin{table*}[t!]
\vspace{-1em}
\caption[]{\small Cox model test set accuracy and fairness metrics on the FLC (age) dataset. We report mean and standard deviation (in parentheses) across 10 experimental repeats (each repeat holds out a different 20\% of the training data as a validation set for hyperparameter tuning; the test set is the same across experimental repeats). Higher is better for metrics with ``$\uparrow$'', while lower is better for metrics with ``$\downarrow$''. The best results are shown in bold for linear and, separately, nonlinear models. When one of our methods outperforms all baselines (in linear and, separately, nonlinear models), we highlight the corresponding cell in \mybox[fill=green!30]{green}. Evaluation metrics are reported to 4 decimal places unless the number is exactly equal to~0 (in which case we just state 0 without using a decimal point) or smaller than $10^{-4}$ (in which case we report the number in scientific notation). Note that achieving F$_{CI}$ or F$_{CG}$ scores that are exactly~0 is due to the manner in which these fairness metrics are defined.\footnotemark} \vspace{-1em}
\centering
\setlength\tabcolsep{0.1pt}
\renewcommand{\arraystretch}{0.5}
{\tiny %
\renewcommand{\belowrulesep}{0.1pt}
\renewcommand{\aboverulesep}{0.1pt}
\begin{tabular}{ccccccc|ccccc}
\toprule
\multirow{3}{*}{} & \multirow{3}{*}{Methods} & \multicolumn{5}{c|}{CI-based Tuning}                                                    & \multicolumn{5}{c}{F$_{CG}$-based Tuning}                             \\ \cmidrule{3-12}
&                   & \multicolumn{2}{c|}{Accuracy Metrics}  & \multicolumn{3}{c|}{Fairness Metrics}  & \multicolumn{2}{c|}{Accuracy Metrics}  & \multicolumn{3}{c}{Fairness Metrics} \\ \cmidrule{3-12}
&                   & \multicolumn{1}{c}{C$^{td}$$\uparrow$}  & \multicolumn{1}{c|}{IBS$\downarrow$}  &\multicolumn{1}{c}{CI(\%)$\downarrow$}& \multicolumn{1}{c}{F$_{CI}$$\downarrow$}& F$_{CG}$$\downarrow$ & \multicolumn{1}{c}{C$^{td}$$\uparrow$}  & \multicolumn{1}{c|}{IBS$\downarrow$}  &\multicolumn{1}{c}{CI(\%)$\downarrow$}& \multicolumn{1}{c}{F$_{CI}$$\downarrow$}& F$_{CG}$$\downarrow$                 \\ \midrule

\multirow{8}{*}{\rotatebox{90}{Linear \ \ \ \ \ \ \ \ \ \ \ \ \ \ }} &        \tableCox            & \multicolumn{1}{c}{{\textbf{\makecell{0.8032 \\(0.0002)}}}}  & \multicolumn{1}{c|}{\makecell{0.1739 \\(0.0004)}}  & \multicolumn{1}{c}{\makecell{0.5350 \\(0.0413)}}& \multicolumn{1}{c}{\makecell{0.0249 \\(0.0002)}}& \makecell{0.0044 \\(2.8919e-05)} & \multicolumn{1}{c}{\textbf{{\makecell{0.8032 \\(0.0002)}}}}  & \multicolumn{1}{c|}{\makecell{0.1739 \\(0.0004)}}  & \multicolumn{1}{c}{\makecell{0.5350 \\(0.0413)}}& \multicolumn{1}{c}{\makecell{0.0249 \\(0.0002)}}& \makecell{0.0044 \\(2.8919e-05)}              \\ \cmidrule{3-12}
&          \tableCoxKeyaInd          & \multicolumn{1}{c}{{\makecell{0.7937 \\(0.0068)}}}  & \multicolumn{1}{c|}{\makecell{0.1414 \\(0.0073)}}  & \multicolumn{1}{c}{\makecell{0.5400 \\(0.3270)}}& \multicolumn{1}{c}{\makecell{0.0129 \\(0.0028)}}& \makecell{0.0021 \\(0.0005)} & \multicolumn{1}{c}{{\makecell{0.7923 \\(0.0074)}}}  & \multicolumn{1}{c|}{\makecell{0.1334 \\(0.0034)}}  & \multicolumn{1}{c}{\makecell{0.4010 \\(0.2631)}}& \multicolumn{1}{c}{\makecell{0.0068 \\(0.0006)}}& \makecell{0.0010 \\(0.0001)}                \\ \cmidrule{3-12}
&          \tableCoxRPInd          & \multicolumn{1}{c}{{{\makecell{0.8029 \\(0.0005)}}}}  & \multicolumn{1}{c|}{\makecell{0.1735 \\(0.0023)}}  & \multicolumn{1}{c}{\makecell{0.4660 \\(0.1551)}}& \multicolumn{1}{c}{\makecell{0.0247 \\(0.0009)}}& \makecell{0.0043 \\(0.0002)} & \multicolumn{1}{c}{{\makecell{0.8020 \\(0.0007)}}}  & \multicolumn{1}{c|}{\makecell{0.1700 \\(0.0034)}}  & \multicolumn{1}{c}{\makecell{0.2530 \\(0.2658)}}& \multicolumn{1}{c}{\makecell{0.0233 \\(0.0014)}}& \makecell{0.0040 \\(0.0003)}                \\ \cmidrule{3-12}
&        \tableCoxKeyaGroup            & \multicolumn{1}{c}{{\makecell{0.7974 \\(0.0117)}}}  & \multicolumn{1}{c|}{\makecell{0.1492 \\(0.0077)}}  & \multicolumn{1}{c}{\makecell{0.3410 \\(0.3011)}}& \multicolumn{1}{c}{\makecell{0.0123 \\(0.0043)}}& \makecell{0.0024 \\(0.0007)} & \multicolumn{1}{c}{{\makecell{0.7862 \\(0.0133)}}}  & \multicolumn{1}{c|}{\makecell{0.1413 \\(0.0035)}}  & \multicolumn{1}{c}{\makecell{0.5360 \\(0.3888)}}& \multicolumn{1}{c}{\makecell{0.0079 \\(0.0029)}}& \makecell{0.0016 \\(0.0004)}                \\ \cmidrule{3-12}
&        \tableCoxRPGroup            & \multicolumn{1}{c}{{\makecell{0.8029 \\(0.0005)}}}  & \multicolumn{1}{c|}{\makecell{0.1735 \\(0.0023)}}  & \multicolumn{1}{c}{\makecell{0.4660 \\(0.1551)}}& \multicolumn{1}{c}{\makecell{0.0247 \\(0.0009)}}& \makecell{0.0043 \\(0.0002)} & \multicolumn{1}{c}{{\makecell{0.8015 \\(0.0003)}}}  & \multicolumn{1}{c|}{\makecell{0.1673 \\(0.0004)}}  & \multicolumn{1}{c}{\textbf{\makecell{0.0390 \\(0.0243)}}}& \multicolumn{1}{c}{\makecell{0.0222 \\(0.0002)}}& \makecell{0.0038 \\(3.3934e-05)}                \\ \cmidrule{3-12}
&        \tableCoxKeyaInt            & \multicolumn{1}{c}{{\makecell{0.7870 \\(0.0029)}}}  & \multicolumn{1}{c|}{{\makecell{0.1400 \\(0.0005)}}}  & \multicolumn{1}{c}{\makecell{1.0790 \\(0.1098)}}& \multicolumn{1}{c}{{\makecell{0.0073 \\(0.0002)}}}& \makecell{0.0016 \\(0.0001)} & \multicolumn{1}{c}{{\makecell{0.7875 \\(0.0021)}}}  & \multicolumn{1}{c|}{\makecell{0.1402 \\(0.0004)}}  & \multicolumn{1}{c}{\makecell{1.1190 \\(0.1073)}}& \multicolumn{1}{c}{\makecell{0.0073 \\(0.0002)}}& \makecell{0.0016 \\(0.0001)}                \\ \cmidrule{2-12} 
&       \tableDROCox             & \multicolumn{1}{c}{{\makecell{0.7959 \\(0.0036)}}}  & \multicolumn{1}{c|}{\makecell{0.1408 \\(0.0050)}}  & \multicolumn{1}{c}{\1\textbf{\makecell{0.0510 \\(0.0401)}}}& \multicolumn{1}{c}{\makecell{0.0078 \\(0.0051)}}& \1{\makecell{0.0012 \\(0.0008)} } & \multicolumn{1}{c}{{\makecell{0.7958 \\(0.0049)}}}  & \multicolumn{1}{c|}{\1\textbf{\makecell{0.1330 \\(0.0002)}}}  & \multicolumn{1}{c}{{\makecell{0.1620 \\(0.1132)}}}& \multicolumn{1}{c}{\1\textbf{\makecell{~~~~~0~~~~~\\~~~~~(0)~~~~~}}}& \1\textbf{\makecell{~~~~~0~~~~~\\~~~~~(0)~~~~~}}              \\ \cmidrule{3-12}
&       \tableDROCoxSplit            & \multicolumn{1}{c}{{\makecell{0.7964 \\(0.0045))}}}  & \multicolumn{1}{c|}{\1\textbf{\makecell{0.1389 \\(0.0008)}}}  & \multicolumn{1}{c}{\1\makecell{0.2350 \\(0.1277)}}& \multicolumn{1}{c}{\1\textbf{\makecell{~~~~~0~~~~~\\~~~~~(0)~~~~~}}}& \1\textbf{\makecell{~~~~~0~~~~~\\~~~~~(0)~~~~~}} & \multicolumn{1}{c}{{\makecell{0.7964 \\(0.0045))}}}  & \multicolumn{1}{c|}{\1\textbf{\makecell{0.1389 \\(0.0008)}}}  & \multicolumn{1}{c}{\makecell{0.2350 \\(0.1277)}}& \multicolumn{1}{c}{\1\textbf{\makecell{~~~~~0~~~~~\\~~~~~(0)~~~~~}}}& \1\textbf{\makecell{~~~~~0~~~~~\\~~~~~(0)~~~~~}}               \\ \cmidrule{3-12}
&       \tableExactDROCox            & \multicolumn{1}{c}{{\makecell{0.7821 \\(0.0142)}}}  & \multicolumn{1}{c|}{{\makecell{0.3916 \\(0.0487)}}}  & \multicolumn{1}{c}{\makecell{0.9838 \\(0.4567)}}& \multicolumn{1}{c}{{\makecell{0.0094 \\(0.0016)}}}& {\makecell{0.0019 \\(0.0003)} } & \multicolumn{1}{c}{{\makecell{0.7821 \\(0.0142)}}}  & \multicolumn{1}{c|}{{\makecell{0.3916 \\(0.0487)}}}  & \multicolumn{1}{c}{\makecell{0.9838 \\(0.4567)}}& \multicolumn{1}{c}{{\makecell{0.0094 \\(0.0016)}}}& {\makecell{0.0019 \\(0.0003)} }               \\ \cmidrule{1-12}
\multirow{8}{*}{\rotatebox{90}{Nonlinear \ \ \ \ \ \ \ \ \ \ \ \ \ \ \  }} &          DeepSurv         & \multicolumn{1}{c}{{{\makecell{0.8070 \\(0.0014)}}}}  & \multicolumn{1}{c|}{\makecell{0.1767 \\(0.0018)}} & \multicolumn{1}{c}{\makecell{0.2940 \\(0.2147)}}& \multicolumn{1}{c}{\makecell{0.0259 \\(0.0004)}}& \makecell{0.0050 \\(0.0003)} & \multicolumn{1}{c}{{\makecell{0.8070 \\(0.0014)}}}  & \multicolumn{1}{c|}{\makecell{0.1767 \\(0.0018)}} & \multicolumn{1}{c}{\makecell{0.2940 \\(0.2147)}}& \multicolumn{1}{c}{\makecell{0.0259 \\(0.0004)}}& \makecell{0.0050 \\(0.0003)}                \\  \cmidrule{3-12}
&          \tableDeepSurvKeyaInd          & \multicolumn{1}{c}{{\makecell{0.7884 \\(0.0070)}}} & \multicolumn{1}{c|}{\makecell{0.1441 \\(0.0130)}} & \multicolumn{1}{c}{\makecell{0.3700 \\(0.2523)}}& \multicolumn{1}{c}{\makecell{0.0127 \\(0.0080)}}& \makecell{0.0025 \\(0.0017)} & \multicolumn{1}{c}{{\makecell{0.7994 \\(0.0069)}}}  & \multicolumn{1}{c|}{\makecell{0.1672 \\(0.0051)}}  & \multicolumn{1}{c}{\makecell{0.6310 \\(0.5316)}}& \multicolumn{1}{c}{\makecell{0.0245 \\(0.0014)}}& \makecell{0.0050 \\(0.0005)}               \\  \cmidrule{3-12}
&          \tableDeepSurvRPInd          & \multicolumn{1}{c}{{\makecell{0.8070 \\(0.0033)}}}  & \multicolumn{1}{c|}{\makecell{0.1736 \\(0.0086)}}  & \multicolumn{1}{c}{\makecell{0.2300 \\(0.1471)}}& \multicolumn{1}{c}{{\makecell{0.0246 \\(0.0040)}}}& {\makecell{0.0047 \\(0.0008)}} & \multicolumn{1}{c}{{\makecell{0.8086 \\(0.0015)}}}  & \multicolumn{1}{c|}{\makecell{0.1766 \\(0.0024)}}  & \multicolumn{1}{c}{{\makecell{0.1560 \\(0.0956)}}}& \multicolumn{1}{c}{\makecell{0.0258 \\(0.0011)}}& \makecell{0.0050 \\(0.0002)}               \\  \cmidrule{3-12}
&        \tableDeepSurvKeyaGroup            & \multicolumn{1}{c}{{\makecell{0.7990 \\(0.0120)}}} & \multicolumn{1}{c|}{\makecell{0.4190 \\(0.2487)}} & \multicolumn{1}{c}{\makecell{0.2490 \\(0.1646)}}& \multicolumn{1}{c}{\makecell{0.0071 \\(0.0069)}}& \makecell{0.0015 \\(0.0013)} & \multicolumn{1}{c}{{\makecell{0.8061 \\(0.0020)}}}  & \multicolumn{1}{c|}{\makecell{0.4713 \\(0.2142)}}  & \multicolumn{1}{c}{\makecell{0.2700 \\(0.2260)}}& \multicolumn{1}{c}{{\makecell{0.0070 \\(0.0081)}}}& {\makecell{0.0014 \\(0.0016)}}             \\ \cmidrule{3-12}
&        \tableDeepSurvRPGroup            & \multicolumn{1}{c}{{{\makecell{0.8069 \\(0.0033)}}}} & \multicolumn{1}{c|}{\makecell{0.1735 \\(0.0086)}} & \multicolumn{1}{c}{\makecell{0.2580 \\(0.1661)}}& \multicolumn{1}{c}{\makecell{0.0245 \\(0.0040)}}& \makecell{0.0047 \\(0.0008)} & \multicolumn{1}{c}{\textbf{{\makecell{0.8086 \\(0.0015)}}}}  & \multicolumn{1}{c|}{\makecell{0.1766 \\(0.0024)}}  & \multicolumn{1}{c}{\textbf{\makecell{0.1560 \\(0.0956)}}}& \multicolumn{1}{c}{\makecell{0.0258 \\(0.0011)}}& \makecell{0.0050 \\(0.0002)}              \\ 
\cmidrule{3-12}
&          \tableDeepSurvKeyaInt          & \multicolumn{1}{c}{{\makecell{0.7751 \\(0.0018)}}} & \multicolumn{1}{c|}{{\makecell{0.1357 \\(0.0002)}}} & \multicolumn{1}{c}{\makecell{0.4300 \\(0.1091)}}& \multicolumn{1}{c}{\makecell{0.0037 \\(0.0001)}}& \makecell{0.0008 \\(1.3494e-05)} & \multicolumn{1}{c}{{\makecell{0.7751 \\(0.0018)}}}  & \multicolumn{1}{c|}{\makecell{0.1357 \\(0.0002)}}  & \multicolumn{1}{c}{\makecell{0.4300 \\(0.1091)}}& \multicolumn{1}{c}{\makecell{0.0037 \\(0.0001)}}& \makecell{0.0008 \\(1.3494e-05)}                \\
\cmidrule{3-12}
&        \tableDeepPseudoRPInd            & \multicolumn{1}{c}{{{\textbf{\makecell{0.8077 \\(0.0022)}}}}} & \multicolumn{1}{c|}{\textbf{\makecell{0.1228 \\(0.0019)}}} & \multicolumn{1}{c}{\makecell{0.2530 \\(0.0974)}}& \multicolumn{1}{c}{\makecell{0.0239 \\(0.0018)}}& \makecell{0.0048 \\(0.0004)} & \multicolumn{1}{c}{{{\makecell{0.8077 \\(0.0022)}}}} & \multicolumn{1}{c|}{\textbf{\makecell{0.1228 \\(0.0019)}}} & \multicolumn{1}{c}{\makecell{0.2530 \\(0.0974)}}& \multicolumn{1}{c}{\makecell{0.0239 \\(0.0018)}}& \makecell{0.0048 \\(0.0004)}               \\ \cmidrule{3-12}
&        \tablePseudoNAMRPInd            & \multicolumn{1}{c}{{{\makecell{0.7829 \\(0.0037)}}}} & \multicolumn{1}{c|}{\makecell{0.1810 \\(0.0050)}} & \multicolumn{1}{c}{\makecell{0.3660 \\(0.0508)}}& \multicolumn{1}{c}{\makecell{0.0251 \\(0.0006)}}& \makecell{0.0052 \\(0.0004)} & \multicolumn{1}{c}{{{\makecell{0.7829 \\(0.0037)}}}} & \multicolumn{1}{c|}{\makecell{0.1810 \\(0.0050)}} & \multicolumn{1}{c}{\makecell{0.3660 \\(0.0508)}}& \multicolumn{1}{c}{\makecell{0.0251 \\(0.0006)}}& \makecell{0.0052 \\(0.0004)}              \\
\cmidrule{2-12} 
&          \tableDeepDROCox          & \multicolumn{1}{c}{{\makecell{0.8068 \\(0.0024)}}} & \multicolumn{1}{c|}{\makecell{0.1595 \\(0.0135)}} & \multicolumn{1}{c}{\1\textbf{\makecell{0.0730 \\(0.0822)}}}& \multicolumn{1}{c}{\makecell{0.0189 \\(0.0056)}}& \makecell{0.0036 \\(0.0013)} & \multicolumn{1}{c}{{\makecell{0.7781 \\(0.0091)}}}  & \multicolumn{1}{c|}{{\makecell{0.1331 \\(0.0002)}}}  & \multicolumn{1}{c}{\makecell{2.4300 \\(0.3462)}}& \multicolumn{1}{c}{\1\makecell{0.0001 \\(3.1257e-05)}}& \1\makecell{9.9660e-06 \\(3.5999e-06)}                \\ \cmidrule{3-12}
&       \tableDeepDROCoxSplit            & \multicolumn{1}{c}{{\makecell{0.7784 \\(0.0092)}}} & \multicolumn{1}{c|}{\makecell{0.1647 \\(0.0037)}} & \multicolumn{1}{c}{\makecell{2.3210 \\(0.3590)}}& \multicolumn{1}{c}{\1\textbf{\makecell{~~~~~0~~~~~\\~~~~~(0)~~~~~}}}& \1\textbf{\makecell{~~~~~0~~~~~\\~~~~~(0)~~~~~}} & \multicolumn{1}{c}{{\makecell{0.7784 \\(0.0092)}}}  & \multicolumn{1}{c|}{\makecell{0.1647 \\(0.0037)}}  & \multicolumn{1}{c}{\makecell{2.3210 \\(0.3590)}}& \multicolumn{1}{c}{\1\textbf{\makecell{~~~~~0~~~~~\\~~~~~(0)~~~~~}}}& \1\textbf{\makecell{~~~~~0~~~~~\\~~~~~(0)~~~~~}} \\ \cmidrule{3-12}
&       \tableDeepExactDROCox            & \multicolumn{1}{c}{{\makecell{0.8048 \\(0.0011)}}}  & \multicolumn{1}{c|}{{\makecell{0.1363 \\(0.0016)}}}  & \multicolumn{1}{c}{\makecell{0.5050 \\(0.2489)}}& \multicolumn{1}{c}{{\makecell{0.0197 \\(0.0005)}}}& {\makecell{0.0038 \\(0.0001)} } & \multicolumn{1}{c}{{\makecell{0.8048 \\(0.0011)}}}  & \multicolumn{1}{c|}{{\makecell{0.1363 \\(0.0016)}}}  & \multicolumn{1}{c}{\makecell{0.5050 \\(0.2489)}}& \multicolumn{1}{c}{{\makecell{0.0197 \\(0.0005)}}}& {\makecell{0.0038 \\(0.0001)} }              
\\ \bottomrule
\end{tabular}
}
\label{tab:general_performance_CI}
\end{table*}

\begin{table*}[pth!]
\caption{\small Cox model test set scores on the SUPPORT (gender) dataset, in the same format as Table~\ref{tab:general_performance_CI}.}
\centering
\setlength\tabcolsep{0.5pt}
\renewcommand{\arraystretch}{0.5}
{\tiny %
\renewcommand{\belowrulesep}{0.1pt}
\renewcommand{\aboverulesep}{0.1pt}
\begin{tabular}{ccccccc|ccccc}
\toprule
\multirow{3}{*}{} & \multirow{3}{*}{Methods} & \multicolumn{5}{c|}{CI-based Tuning}                                                    & \multicolumn{5}{c}{F$_{CG}$-based Tuning}                             \\ \cmidrule{3-12}
&                   & \multicolumn{2}{c|}{Accuracy Metrics}  & \multicolumn{3}{c|}{Fairness Metrics}  & \multicolumn{2}{c|}{Accuracy Metrics}  & \multicolumn{3}{c}{Fairness Metrics} \\ \cmidrule{3-12}
&                   & \multicolumn{1}{c}{C$^{td}$$\uparrow$}  & \multicolumn{1}{c|}{IBS$\downarrow$}  &\multicolumn{1}{c}{CI(\%)$\downarrow$}& \multicolumn{1}{c}{F$_{CI}$$\downarrow$}& F$_{CG}$$\downarrow$ & \multicolumn{1}{c}{C$^{td}$$\uparrow$}  & \multicolumn{1}{c|}{IBS$\downarrow$}  &\multicolumn{1}{c}{CI(\%)$\downarrow$}& \multicolumn{1}{c}{F$_{CI}$$\downarrow$}& F$_{CG}$$\downarrow$                 \\ \midrule
\multirow{8}{*}{\rotatebox{90}{Linear \ \ \ \ \ \ \ \ \ \ \ \ \ \ }} &        \tableCox            & \multicolumn{1}{c}{{\makecell{0.6025 \\(0.0005)}}} & \multicolumn{1}{c|}{\makecell{0.2304 \\(0.0015)}} & \multicolumn{1}{c}{\makecell{1.4300 \\(0.0654)}} & \multicolumn{1}{c}{\makecell{0.0054 \\(0.0002)}} & \makecell{0.0028 \\(0.0001)}& \multicolumn{1}{c}{{\textbf{\makecell{0.6025 \\(0.0005)}}}} & \multicolumn{1}{c|}{\makecell{0.2304 \\(0.0015)}} & \multicolumn{1}{c}{\makecell{1.4300 \\(0.0654)}} & \multicolumn{1}{c}{\makecell{0.0054 \\(0.0002)}} & \makecell{0.0028 \\(0.0001)}              \\  \cmidrule{3-12}
&          \tableCoxKeyaInd          & \multicolumn{1}{c}{{\makecell{0.5881 \\(0.0114)}}} & \multicolumn{1}{c|}{\textbf{\makecell{0.2157 \\(0.0060)}}} & \multicolumn{1}{c}{\makecell{0.9650 \\(0.6126)}}& \multicolumn{1}{c}{\makecell{0.0004 \\(0.0004)}} & \makecell{0.0002 \\(0.0002)}& \multicolumn{1}{c}{{\makecell{0.5829 \\(0.0099)}}}  & \multicolumn{1}{c|}{\textbf{\makecell{0.2147 \\(0.0063)}}}  & \multicolumn{1}{c}{\makecell{1.1330 \\(0.6846)}}& \multicolumn{1}{c}{\textbf{\makecell{0\\(0)}}}& \textbf{\makecell{0\\(0)}}                  \\ \cmidrule{3-12}
&          \tableCoxRPInd          & \multicolumn{1}{c}{{\makecell{0.6018 \\(0.0016)}}} & \multicolumn{1}{c|}{\makecell{0.2309 \\(0.0011)}} & \makecell{1.4390 \\(0.1077)}& \multicolumn{1}{c}{\makecell{0.0056 \\(0.0002)}} & \makecell{0.0029 \\(0.0001)}& \multicolumn{1}{c}{{\makecell{0.6022 \\(0.0005)}}}  & \multicolumn{1}{c|}{\makecell{0.2307 \\(0.0012)}}  & \multicolumn{1}{c}{\makecell{1.4060 \\(0.0932)}}& \multicolumn{1}{c}{\makecell{0.0055 \\(0.0002)}}& \makecell{0.0028 \\(0.0001)}                    \\ \cmidrule{3-12}
&        \tableCoxKeyaGroup            & \multicolumn{1}{c}{{\textbf{\makecell{0.6030 \\(0.0007)}}}} & \multicolumn{1}{c|}{\makecell{0.2297 \\(0.0018)}} & \makecell{1.4190 \\(0.0632)}& \multicolumn{1}{c}{\makecell{0.0051 \\(0.0003)}} & \makecell{0.0026 \\(0.0001)}& \multicolumn{1}{c}{{\makecell{0.6024 \\(0.0006)}}}  & \multicolumn{1}{c|}{\makecell{0.2284 \\(0.0009)}}  & \multicolumn{1}{c}{\makecell{1.4360 \\(0.0674)}}& \multicolumn{1}{c}{\makecell{0.0047 \\(0.0001)}}& \makecell{0.0025 \\(4.2335e-05)}                 \\ \cmidrule{3-12}
&        \tableCoxRPGroup            & \multicolumn{1}{c}{{\makecell{0.6018 \\(0.0016)}}} & \multicolumn{1}{c|}{\makecell{0.2309 \\(0.0011)}} & \makecell{1.4390\\ (0.1077)}& \multicolumn{1}{c}{\makecell{0.0056 \\(0.0002)}} & \makecell{0.0029 \\(0.0001)}& \multicolumn{1}{c}{{\makecell{0.6024 \\(0.0007)}}}  & \multicolumn{1}{c|}{\makecell{0.2307 \\(0.0012)}}  & \multicolumn{1}{c}{\makecell{1.4010 \\(0.0931)}}& \multicolumn{1}{c}{\makecell{0.0055 \\(0.0002)}}& \makecell{0.0028 \\(0.0001)}                 \\ \cmidrule{3-12}
&        \tableCoxKeyaInt            & \multicolumn{1}{c}{{\makecell{0.5715 \\(0.0062)}}} & \multicolumn{1}{c|}{\makecell{0.2275 \\(0.0016)}} & \makecell{1.1270 \\(0.2457)}& \multicolumn{1}{c}{\makecell{0.0028 \\(0.0003)}} & \makecell{0.0015 \\(0.0002)}& \multicolumn{1}{c}{{\makecell{0.5631 \\(0.0070)}}}  & \multicolumn{1}{c|}{\makecell{0.2264 \\(0.0017)}}  & \multicolumn{1}{c}{\makecell{0.8650 \\(0.2958)}}& \multicolumn{1}{c}{\makecell{0.0024 \\(0.0003)}}& \makecell{0.0012 \\(0.0002)}               \\ \cmidrule{2-12} 
&       \tableDROCox          & \multicolumn{1}{c}{{\makecell{0.5734 \\(0.0019)}}} & \multicolumn{1}{c|}{\makecell{0.2210 \\(0.0010)}} & \1\makecell{0.4350 \\(0.0674)} & \multicolumn{1}{c}{\1\makecell{0.0002 \\(2.3882e-05)}} & \1\makecell{0.0001 \\(1.3621e-05)}& \multicolumn{1}{c}{{\makecell{0.5641 \\(0.0105)}}}  & \multicolumn{1}{c|}{\makecell{0.2211 \\(0.0010)}}  & \multicolumn{1}{c}{\1\textbf{\makecell{0.3840 \\(0.1830)}}}& \multicolumn{1}{c}{\makecell{0.0001 \\(0.0001)}}& \makecell{0.0001\\(4.8271e-05)}           \\ \cmidrule{3-12}
&       \tableDROCoxSplit          & \multicolumn{1}{c}{{\makecell{0.5701 \\(0.0056)}}} & \multicolumn{1}{c|}{\makecell{0.4569 \\(0.1314)}} & \1\textbf{\makecell{0.3860 \\(0.1163)}} & \multicolumn{1}{c}{\1{\textbf{\makecell{1.1922e-07 \\(2.6445e-07)}}}} & \1{\textbf{\makecell{9.6779e-08 \\(2.1315e-07)}}}& \multicolumn{1}{c}{{\makecell{0.5701 \\(0.0056)}}}  & \multicolumn{1}{c|}{\makecell{0.4570 \\(0.1314)}}  & \multicolumn{1}{c}{\1{\makecell{0.3860 \\(0.1163)}}}& \multicolumn{1}{c}{\makecell{1.1922e-07 \\(2.6445e-07)}}& \makecell{9.6779e-08 \\(2.1315e-07)}                \\ \cmidrule{3-12} 
&       \tableExactDROCox            & \multicolumn{1}{c}{{\makecell{0.5884 \\(0.0063)}}}  & \multicolumn{1}{c|}{{\makecell{0.3122 \\(0.0068)}}}  & \multicolumn{1}{c}{\1\makecell{0.8580 \\(0.2434)}}& \multicolumn{1}{c}{\1{\makecell{8.1822e-06 \\(8.1542e-06)}}}& \1{\makecell{5.2437e-06 \\(5.0535e-06)} } & \multicolumn{1}{c}{{\makecell{0.5884 \\(0.0063)}}}  & \multicolumn{1}{c|}{{\makecell{0.3122 \\(0.0068)}}}  & \multicolumn{1}{c}{\1\makecell{0.8580 \\(0.2434)}}& \multicolumn{1}{c}{{\makecell{8.1822e-06 \\(8.1542e-06)}}}& {\makecell{5.2437e-06 \\(5.0535e-06)} } \\ \cmidrule{1-12}
\multirow{8}{*}{\rotatebox{90}{Nonlinear \ \ \ \ \ \ \ \ \ \ \ \ \ \ \   }} &          DeepSurv         & \multicolumn{1}{c}{{\makecell{0.6108 \\(0.0029)}}} & \multicolumn{1}{c|}{\makecell{0.2417 \\(0.0016)}} & \makecell{1.6220 \\(0.3303)}& \multicolumn{1}{c}{\makecell{0.0090 \\(0.0002)}} & \makecell{0.0046 \\(0.0001)}& \multicolumn{1}{c}{{\makecell{0.6108 \\(0.0029)}}} & \multicolumn{1}{c|}{\makecell{0.2417 \\(0.0016)}} & \makecell{1.6220 \\(0.3303)}& \multicolumn{1}{c}{\makecell{0.0090 \\(0.0002)}} & \makecell{0.0046 \\(0.0001)}             \\  \cmidrule{3-12}
&          \tableDeepSurvKeyaInd          & \multicolumn{1}{c}{{\makecell{0.5984 \\(0.0124)}}} & \multicolumn{1}{c|}{\makecell{0.2376 \\(0.0182)}} & \makecell{1.3280 \\(0.7670)}& \multicolumn{1}{c}{{\makecell{0.0061 \\(0.0036)}}} & {\makecell{0.0031 \\(0.0019)}}& \multicolumn{1}{c}{{\makecell{0.6031 \\(0.0059)}}}  & \multicolumn{1}{c|}{\makecell{0.2459 \\(0.0102)}}  & \multicolumn{1}{c}{\textbf{\makecell{1.1590 \\(0.8626)}}}& \multicolumn{1}{c}{{\makecell{0.0090 \\(0.0007)}}}& {\makecell{0.0046 \\(0.0004)}}                \\  \cmidrule{3-12}
&          \tableDeepSurvRPInd          & \multicolumn{1}{c}{{\makecell{0.6100 \\(0.0070)}}} & \multicolumn{1}{c|}{\makecell{0.2383 \\(0.0075)}} & \makecell{1.6100 \\(0.3374)}& \multicolumn{1}{c}{\makecell{0.0080 \\(0.0023)}} & \makecell{0.0041 \\(0.0012)}& \multicolumn{1}{c}{{\textbf{\makecell{0.6115 \\(0.0051)}}}}  & \multicolumn{1}{c|}{\makecell{0.2444 \\(0.0036)}}  & \multicolumn{1}{c}{\makecell{1.5410 \\(0.4066)}}& \multicolumn{1}{c}{{\makecell{0.0097 \\(0.0009)}}}& {\makecell{0.0050 \\(0.0004)} }              \\  \cmidrule{3-12}
&        \tableDeepSurvKeyaGroup           & \multicolumn{1}{c}{{\makecell{0.5982 \\(0.0109)}}} & \multicolumn{1}{c|}{\makecell{0.2436 \\(0.0121)}} & \makecell{1.6540 \\(0.3892)}& \multicolumn{1}{c}{\makecell{0.0090 \\(0.0036)}} & \makecell{0.0046 \\(0.0019)}& \multicolumn{1}{c}{{\makecell{0.6034 \\(0.0037)}}}  & \multicolumn{1}{c|}{\makecell{0.2499 \\(0.0024)}}  & \multicolumn{1}{c}{\makecell{1.2390 \\(0.4314)}}& \multicolumn{1}{c}{\makecell{0.0111 \\(0.0003)}}& \makecell{0.0057 \\(0.0001)}              \\ \cmidrule{3-12}
&        \tableDeepSurvRPGroup            & \multicolumn{1}{c}{{\textbf{\makecell{0.6105 \\(0.0055)}}}} & \multicolumn{1}{c|}{\makecell{0.2408 \\(0.0067)}} & \makecell{1.5410 \\(0.3661)}& \multicolumn{1}{c}{\makecell{0.0087 \\(0.0019)}} & \makecell{0.0045 \\(0.0010)}& \multicolumn{1}{c}{{\makecell{0.6115 \\(0.0051)}}}  & \multicolumn{1}{c|}{\makecell{0.2444 \\(0.0036)}}  & \multicolumn{1}{c}{\makecell{1.5410 \\(0.4066)}}& \multicolumn{1}{c}{\makecell{0.0097 \\(0.0009)}}& \makecell{0.0050 \\(0.0004)}             \\ 
\cmidrule{3-12}
&          \tableDeepSurvKeyaInt         & \multicolumn{1}{c}{{\makecell{0.6015 \\(0.0069)}}} & \multicolumn{1}{c|}{\makecell{0.2378 \\(0.0053)}} & \makecell{1.4110 \\(0.2129)}& \multicolumn{1}{c}{\makecell{0.0066 \\(0.0017)}} & \makecell{0.0034 \\(0.0009)} & \multicolumn{1}{c}{{\makecell{0.5912 \\(0.0012)}}}  & \multicolumn{1}{c|}{\makecell{0.2309 \\(0.0011)}}  & \multicolumn{1}{c}{\makecell{1.5390 \\(0.1303)}}& \multicolumn{1}{c}{\makecell{0.0043 \\(0.0002)}}& \makecell{0.0023 \\(0.0001)}               \\
\cmidrule{3-12}
&        \tableDeepPseudoRPInd            & \multicolumn{1}{c}{{{\makecell{0.5811 \\(0.0090)}}}} & \multicolumn{1}{c|}{\makecell{0.2356 \\(0.0023)}} & \multicolumn{1}{c}{\makecell{1.2670 \\(0.4179)}}& \multicolumn{1}{c}{\makecell{0.0059 \\(0.0005)}}& \makecell{0.0029 \\(0.0003)} & \multicolumn{1}{c}{{{\makecell{0.5811 \\(0.0090)}}}} & \multicolumn{1}{c|}{\makecell{0.2356 \\(0.0023)}} & \multicolumn{1}{c}{\makecell{1.2670 \\(0.4179)}}& \multicolumn{1}{c}{\makecell{0.0059 \\(0.0005)}}& \makecell{0.0029 \\(0.0003)}              \\ \cmidrule{3-12}
&        \tablePseudoNAMRPInd            & \multicolumn{1}{c}{{{\makecell{0.5760 \\(0.0039)}}}} & \multicolumn{1}{c|}{\makecell{0.2330 \\(0.0005)}} & \multicolumn{1}{c}{\makecell{1.0360 \\(0.0448)}}& \multicolumn{1}{c}{\makecell{0.0021 \\(0.0001)}}& \makecell{0.0009 \\(0.0001)} & \multicolumn{1}{c}{{{\makecell{0.5760 \\(0.0039)}}}} & \multicolumn{1}{c|}{\makecell{0.2330 \\(0.0005)}} & \multicolumn{1}{c}{\makecell{1.0360 \\(0.0448)}}& \multicolumn{1}{c}{\makecell{0.0021 \\(0.0001)}}& \makecell{0.0009 \\(0.0001)}              \\ 
\cmidrule{2-12}  
&          \tableDeepDROCox          & \multicolumn{1}{c}{{\makecell{0.5829 \\(0.0067)}}} & \multicolumn{1}{c|}{\1\textbf{\makecell{0.2240 \\(0.0010)}}} & \1{\textbf{\makecell{1.2600 \\(0.4412)}}}& \multicolumn{1}{c}{\1\makecell{~~0.0019~~~\\~~(0.0006)~~~}} & \makecell{0.0010 \\(0.0003)}& \multicolumn{1}{c}{{\makecell{0.5754 \\(0.0120)}}}  & \multicolumn{1}{c|}{\1\textbf{\makecell{0.2227 \\(0.0011)}}}  & \multicolumn{1}{c}{\makecell{1.5550 \\(0.4622)}}& \multicolumn{1}{c}{\1\makecell{~~0.0010~~~\\~~(0.0005)~~~}}& \1\makecell{~~0.0005~~~\\~~(0.0003)~~~}             \\ \cmidrule{3-12}
&       \tableDeepDROCoxSplit            & \multicolumn{1}{c}{{\makecell{0.5772 \\(0.0093)}}} & \multicolumn{1}{c|}{\makecell{0.6387 \\(0.0007)}} & {\makecell{1.5530 \\(0.4682)}}& \multicolumn{1}{c}{\1\textbf{\makecell{~~~~~0~~~~~\\~~~~~(0)~~~~~}}} & \1\textbf{\makecell{~~~~~0~~~~~\\~~~~~(0)~~~~~}}& \multicolumn{1}{c}{{\makecell{0.5772 \\(0.0093)}}}  & \multicolumn{1}{c|}{\makecell{0.6387 \\(0.0007)}}  & \multicolumn{1}{c}{\makecell{1.5530 \\(0.4682)}}& \multicolumn{1}{c}{\1\textbf{\makecell{~~~~~0~~~~~\\~~~~~(0)~~~~~}}}& \1\textbf{\makecell{~~~~~0~~~~~\\~~~~~(0)~~~~~}}       \\ \cmidrule{3-12}      
&       \tableDeepExactDROCox            & \multicolumn{1}{c}{{\makecell{0.5811 \\(0.0065)}}}  & \multicolumn{1}{c|}{{\makecell{0.2621 \\(0.0098)}}}  & \multicolumn{1}{c}{\makecell{2.0490 \\(0.4989)}}& \multicolumn{1}{c}{{\makecell{0.0062 \\(0.0020)}}}& {\makecell{0.0033 \\(0.0010)} } & \multicolumn{1}{c}{{\makecell{0.5811 \\(0.0065)}}}  & \multicolumn{1}{c|}{{\makecell{0.2621 \\(0.0098)}}}  & \multicolumn{1}{c}{\makecell{2.0490 \\(0.4989)}}& \multicolumn{1}{c}{{\makecell{0.0062 \\(0.0020)}}}& {\makecell{0.0033 \\(0.0010)} }\\ \bottomrule
\end{tabular}
}
\label{tab:general_performance_SUPPORT_gender_CI}
\end{table*}

\begin{table*}[pth!]
\vspace{-1em}
\caption{\small Cox model test set scores on the SEER (race) dataset, in the same format as Table~\ref{tab:general_performance_CI}.} 
\centering
\setlength\tabcolsep{0.1pt}
\renewcommand{\arraystretch}{0.5}
{\tiny %
\renewcommand{\belowrulesep}{0.1pt}
\renewcommand{\aboverulesep}{0.1pt}
\begin{tabular}{ccccccc|ccccc}
\toprule
\multirow{3}{*}{} & \multirow{3}{*}{Methods} & \multicolumn{5}{c|}{CI-based Tuning}                                                    & \multicolumn{5}{c}{F$_{CG}$-based Tuning}                             \\ \cmidrule{3-12}
&                   & \multicolumn{2}{c|}{Accuracy Metrics}  & \multicolumn{3}{c|}{Fairness Metrics}  & \multicolumn{2}{c|}{Accuracy Metrics}  & \multicolumn{3}{c}{Fairness Metrics} \\ \cmidrule{3-12}
&                   & \multicolumn{1}{c}{C$^{td}$$\uparrow$}  & \multicolumn{1}{c|}{IBS$\downarrow$}  &\multicolumn{1}{c}{CI(\%)$\downarrow$}& \multicolumn{1}{c}{F$_{CI}$$\downarrow$}& F$_{CG}$$\downarrow$ & \multicolumn{1}{c}{C$^{td}$$\uparrow$}  & \multicolumn{1}{c|}{IBS$\downarrow$}  &\multicolumn{1}{c}{CI(\%)$\downarrow$}& \multicolumn{1}{c}{F$_{CI}$$\downarrow$}& F$_{CG}$$\downarrow$                 \\ \midrule
\multirow{8}{*}{\rotatebox{90}{Linear \ \ \ \ \ \ \ \ \ \ \ \ \ \ }} &        \tableCox            & \multicolumn{1}{c}{{\makecell{0.7025 \\(0.0003)}}} & \multicolumn{1}{c|}{\makecell{0.2128 \\(0.0009)}} & \makecell{2.5200 \\(0.0431)} & \multicolumn{1}{c}{\makecell{0.0256 \\(0.0006)}} & \makecell{0.0204 \\(0.0005)} & \multicolumn{1}{c}{{\textbf{\makecell{0.7025 \\(0.0003)}}}} & \multicolumn{1}{c|}{\makecell{0.2128 \\(0.0009)}} & \makecell{2.5200 \\(0.0431)} &\multicolumn{1}{c}{\makecell{0.0256 \\(0.0006)}} & \makecell{0.0204 \\(0.0005)}            \\  \cmidrule{3-12}
&          \tableCoxKeyaInd         & \multicolumn{1}{c}{{\makecell{0.6894 \\(0.0046)}}} & \multicolumn{1}{c|}{\textbf{\makecell{0.1837 \\(0.0027)}}} & \makecell{1.9750 \\(0.6480)} & \multicolumn{1}{c}{\makecell{0.0005 \\(0.0001)}} & \makecell{0.0004 \\(0.0001)}& \multicolumn{1}{c}{{\makecell{0.6894 \\(0.0046)}}}  & \multicolumn{1}{c|}{\textbf{\makecell{0.1837 \\(0.0027)}}}  & \multicolumn{1}{c}{{\makecell{1.9750 \\(0.6480)}}}&\multicolumn{1}{c}{\makecell{0.0005 \\(0.0001)}} & \makecell{0.0004 \\(0.0001)}                              \\ \cmidrule{3-12}
&          \tableCoxRPInd         & \multicolumn{1}{c}{{\makecell{0.7032 \\(0.0025)}}} & \multicolumn{1}{c|}{\makecell{0.2103 \\(0.0031)}} & \makecell{2.4590 \\(0.0886)} & \multicolumn{1}{c}{\makecell{0.0235 \\(0.0022)}} & \makecell{0.0186 \\(0.0017)} & \multicolumn{1}{c}{{\makecell{0.7035 \\(0.0025)}}}  & \multicolumn{1}{c|}{\makecell{0.2097 \\(0.0031)}}  & \multicolumn{1}{c}{\makecell{2.4520 \\(0.0862)}}& \multicolumn{1}{c}{\makecell{0.0231 \\(0.0021)}}& \makecell{0.0183 \\(0.0016)}                            \\ \cmidrule{3-12}
&        \tableCoxKeyaGroup            & \multicolumn{1}{c}{{\makecell{0.6952 \\(0.0146)}}} & \multicolumn{1}{c|}{\makecell{0.2073 \\(0.0049)}} & \makecell{2.8690 \\(0.5267)} & \multicolumn{1}{c}{\makecell{0.0216 \\(0.0041)}} & \makecell{0.0175 \\(0.0035)} & \multicolumn{1}{c}{{\makecell{0.6952 \\(0.0146)}}}  & \multicolumn{1}{c|}{\makecell{0.2073 \\(0.0049)}}  & \multicolumn{1}{c}{\makecell{2.8690 \\(0.5267)}}&\multicolumn{1}{c}{\makecell{0.0216 \\(0.0041)}} & \makecell{0.0175 \\(0.0035)}                             \\ \cmidrule{3-12}
&        \tableCoxRPGroup            & \multicolumn{1}{c}{{\textbf{\makecell{0.7037 \\(0.0025)}}}} & \multicolumn{1}{c|}{\makecell{0.2089 \\(0.0020)}} & \makecell{2.4790 \\(0.0611)} & \multicolumn{1}{c}{\makecell{0.0226 \\(0.0017)}} & \makecell{0.0179 \\(0.0014)} & \multicolumn{1}{c}{{\textbf{\makecell{0.7037 \\(0.0025)}}}} & \multicolumn{1}{c|}{\makecell{0.2089 \\(0.0020)}} & \makecell{2.4790 \\(0.0611)} & \multicolumn{1}{c}{\makecell{0.0226 \\(0.0017)}} & \makecell{0.0179 \\(0.0014)}                              \\ \cmidrule{3-12}
&        \tableCoxKeyaInt           & \multicolumn{1}{c}{{\makecell{0.6494 \\(0.0016)}}} & \multicolumn{1}{c|}{\makecell{0.1963 \\(0.0012)}} & \makecell{2.1290 \\(0.2573)} & \multicolumn{1}{c}{\makecell{0.0107 \\(0.0010)}} & \makecell{0.0087 \\(0.0008)}& \multicolumn{1}{c}{{\makecell{0.6494 \\(0.0016)}}}  & \multicolumn{1}{c|}{\makecell{0.1963 \\(0.0012)}}  & \multicolumn{1}{c}{\makecell{2.1290 \\(0.2573)}}&\multicolumn{1}{c}{\makecell{0.0107 \\(0.0010)}} & \makecell{0.0087 \\(0.0008)}                              \\ \cmidrule{2-12} 
&       \tableDROCox          & \multicolumn{1}{c}{{\makecell{0.6927 \\(0.0069)}}} & \multicolumn{1}{c|}{\makecell{0.1868 \\(0.0004)}} & {\makecell{2.3090 \\(0.5215)}} & \multicolumn{1}{c}{\1\textbf{\makecell{~~~~~0~~~~~\\~~~~~(0)~~~~~}}} & \1\textbf{\makecell{~~~~~0~~~~~\\~~~~~(0)~~~~~}} & \multicolumn{1}{c}{{\makecell{0.6927 \\(0.0069)}}}  & \multicolumn{1}{c|}{\makecell{0.1868 \\(0.0004)}}  & \multicolumn{1}{c}{{\makecell{2.3090 \\(0.5215)}}}&\multicolumn{1}{c}{\1\textbf{\makecell{~~~~~0~~~~~\\~~~~~(0)~~~~~}}} & \1\textbf{\makecell{~~~~~0~~~~~\\~~~~~(0)~~~~~}}                          \\ \cmidrule{3-12}
&       \tableDROCoxSplit           & \multicolumn{1}{c}{{\makecell{0.6872 \\(0.0047)}}} & \multicolumn{1}{c|}{\makecell{0.1869 \\(0.0004)}} & \makecell{2.8280 \\(0.7434)} & \multicolumn{1}{c}{\1\textbf{\makecell{~~~~~0~~~~~\\~~~~~(0)~~~~~}}} & \1\textbf{\makecell{~~~~~0~~~~~\\~~~~~(0)~~~~~}} & \multicolumn{1}{c}{{\makecell{0.6872 \\(0.0047)}}}  & \multicolumn{1}{c|}{\makecell{0.1869 \\(0.0004)}}  & \multicolumn{1}{c}{\makecell{2.8280 \\(0.7434)}}& \multicolumn{1}{c}{\1\textbf{\makecell{~~~~~0~~~~~\\~~~~~(0)~~~~~}}} & \1\textbf{\makecell{~~~~~0~~~~~\\~~~~~(0)~~~~~}}                           \\ \cmidrule{3-12} 
&       \tableExactDROCox            & \multicolumn{1}{c}{{\makecell{0.6833 \\(0.0060)}}}  & \multicolumn{1}{c|}{{\makecell{0.2422 \\(0.0044)}}}  & \multicolumn{1}{c}{\1\textbf{\makecell{1.3020 \\(0.3474)}}}& \multicolumn{1}{c}{{\makecell{0.0056 \\(0.0005)}}}& {\makecell{0.0045 \\(0.0004)} } & \multicolumn{1}{c}{{\makecell{0.6833 \\(0.0060)}}}  & \multicolumn{1}{c|}{{\makecell{0.2422 \\(0.0044)}}}  & \multicolumn{1}{c}{\1\textbf{\makecell{1.3020 \\(0.3474)}}}& \multicolumn{1}{c}{{\makecell{0.0056 \\(0.0005)}}}& {\makecell{0.0045 \\(0.0004)} } \\ \cmidrule{1-12}
\multirow{8}{*}{\rotatebox{90}{Nonlinear \ \ \ \ \ \ \ \ \ \ \ \ \ \ \   }} &          DeepSurv         & \multicolumn{1}{c}{{\textbf{\makecell{0.7095 \\(0.0014)}}}} & \multicolumn{1}{c|}{\makecell{0.2200 \\(0.0012)}} & \makecell{2.5990 \\(0.1189)} & \multicolumn{1}{c}{\makecell{0.0309 \\(0.0006)}} & \makecell{0.0249 \\(0.0004)}& \multicolumn{1}{c}{{\textbf{\makecell{0.7095 \\(0.0014)}}}} & \multicolumn{1}{c|}{\makecell{0.2200 \\(0.0012)}} & \makecell{2.5990 \\(0.1189)} & \multicolumn{1}{c}{\makecell{0.0309 \\(0.0006)}} & \makecell{0.0249 \\(0.0004)}                         \\  \cmidrule{3-12}
&          \tableDeepSurvKeyaInd          & \multicolumn{1}{c}{{\makecell{0.6982 \\(0.0045)}}} & \multicolumn{1}{c|}{\makecell{0.2127 \\(0.0032)}} & \makecell{1.5740 \\(0.6970)} & \multicolumn{1}{c}{{\makecell{0.0291 \\(0.0014)}}} & {\makecell{0.0235 \\(0.0012)} } & \multicolumn{1}{c}{{\makecell{0.6982 \\(0.0045)}}}  & \multicolumn{1}{c|}{\makecell{0.2127 \\(0.0032)}}  & \multicolumn{1}{c}{\makecell{1.5740 \\(0.6970)}}& \multicolumn{1}{c}{{\makecell{0.0291 \\(0.0014)}}} & {\makecell{0.0235 \\(0.0012)} }                       \\  \cmidrule{3-12}
&          \tableDeepSurvRPInd         & \multicolumn{1}{c}{{\makecell{0.7064 \\(0.0021)}}} & \multicolumn{1}{c|}{\makecell{0.2168 \\(0.0012)}} & \makecell{2.5120 \\(0.1847)} & \multicolumn{1}{c}{\makecell{0.0288 \\(0.0006)}} & \makecell{0.0233 \\(0.0004)} & \multicolumn{1}{c}{{\makecell{0.7064 \\(0.0021)}}} & \multicolumn{1}{c|}{\makecell{0.2168 \\(0.0012)}} & \makecell{2.5120 \\(0.1847)} & \multicolumn{1}{c}{\makecell{0.0288 \\(0.0006)}} & \makecell{0.0233 \\(0.0004)}                          \\  \cmidrule{3-12}
&        \tableDeepSurvKeyaGroup           & \multicolumn{1}{c}{{\makecell{0.7034 \\(0.0016)}}} & \multicolumn{1}{c|}{\makecell{0.2154 \\(0.0007)}} & \makecell{2.5920 \\(0.1468)} & \multicolumn{1}{c}{\makecell{0.0278 \\(0.0010)}} & \makecell{0.0229 \\(0.0008)} & \multicolumn{1}{c}{{\makecell{0.7034 \\(0.0016)}}}  & \multicolumn{1}{c|}{\makecell{0.2154 \\(0.0007)}}  & \multicolumn{1}{c}{\makecell{2.5920 \\(0.1468)}}&\multicolumn{1}{c}{\makecell{0.0278 \\(0.0010)}} & \makecell{0.0229 \\(0.0008)}                          \\ \cmidrule{3-12}
&        \tableDeepSurvRPGroup            & \multicolumn{1}{c}{{\makecell{0.7062 \\(0.0017)}}} & \multicolumn{1}{c|}{\makecell{0.2169 \\(0.0010)}} & \makecell{2.5010 \\(0.1626)} & \multicolumn{1}{c}{\makecell{0.0289 \\(0.0005)}} & \makecell{0.0234 \\(0.0004)}& \multicolumn{1}{c}{{\makecell{0.7062 \\(0.0017)}}} & \multicolumn{1}{c|}{\makecell{0.2169 \\(0.0010)}} & \makecell{2.5010 \\(0.1626)} & \multicolumn{1}{c}{\makecell{0.0289 \\(0.0005)}} & \makecell{0.0234 \\(0.0004)}                         \\
\cmidrule{3-12}
&          \tableDeepSurvKeyaInt         & \multicolumn{1}{c}{{\makecell{0.6537 \\(0.0054)}}} & \multicolumn{1}{c|}{\makecell{0.1998 \\(0.0008)}} & \textbf{\makecell{1.0480 \\(0.4252)}} & \multicolumn{1}{c}{\makecell{0.0136 \\(0.0012)}} & \makecell{0.0111 \\(0.0010)}& \multicolumn{1}{c}{{\makecell{0.6537 \\(0.0054)}}}  & \multicolumn{1}{c|}{\makecell{0.1998 \\(0.0008)}}  & \multicolumn{1}{c}{{\textbf{\makecell{1.0480 \\(0.4252)}}}}& \multicolumn{1}{c}{\makecell{0.0136 \\(0.0012)}} & \makecell{0.0111 \\(0.0010)}                            \\
\cmidrule{3-12}
&        \tableDeepPseudoRPInd            & \multicolumn{1}{c}{{{\makecell{0.7086 \\(0.0030)}}}} & \multicolumn{1}{c|}{\textbf{\makecell{0.1824 \\(0.0033)}}} & \multicolumn{1}{c}{\makecell{2.3290 \\(0.2906)}}& \multicolumn{1}{c}{\makecell{0.0168 \\(0.0055)}}& \makecell{0.0120 \\(0.0040)} & \multicolumn{1}{c}{{{\makecell{0.7086 \\(0.0030)}}}} & \multicolumn{1}{c|}{\makecell{0.1824 \\(0.0033)}} & \multicolumn{1}{c}{\makecell{2.3290 \\(0.2906)}}& \multicolumn{1}{c}{\makecell{0.0168 \\(0.0055)}}& \makecell{0.0120 \\(0.0040)}              \\ \cmidrule{3-12}
&        \tablePseudoNAMRPInd            & \multicolumn{1}{c}{{{\makecell{0.7022 \\(0.0118)}}}} & \multicolumn{1}{c|}{\makecell{0.2226 \\(0.0019)}} & \multicolumn{1}{c}{\makecell{2.3480 \\(0.2087)}}& \multicolumn{1}{c}{\makecell{0.0181 \\(0.0020)}}& \makecell{0.0129 \\(0.0016)} & \multicolumn{1}{c}{{{\makecell{0.7022 \\(0.0118)}}}} & \multicolumn{1}{c|}{\makecell{0.2226 \\(0.0019)}} & \multicolumn{1}{c}{\makecell{2.3480 \\(0.2087)}}& \multicolumn{1}{c}{\makecell{0.0181 \\(0.0020)}}& \makecell{0.0129 \\(0.0016)}              \\
\cmidrule{2-12}  
&          \tableDeepDROCox          & \multicolumn{1}{c}{{\makecell{0.6830 \\(0.0050)}}} & \multicolumn{1}{c|}{{\makecell{0.1869 \\(0.0004)}}} & \makecell{2.5810 \\(0.5244)} & \multicolumn{1}{c}{\1\textbf{\makecell{5.3651e-06 \\(6.3580e-06)}}} & \1\textbf{\makecell{5.3233e-06 \\(6.2580e-06)}}& \multicolumn{1}{c}{{\makecell{0.6830 \\(0.0050)}}}  & \multicolumn{1}{c|}{{\makecell{0.1869 \\(0.0004)}}}  & \multicolumn{1}{c}{{\makecell{2.5810 \\(0.5244)}}}& \multicolumn{1}{c}{\1\textbf{\makecell{5.3651e-06 \\(6.3580e-06)}}} & \1\textbf{\makecell{5.3233e-06 \\(6.2580e-06)}}                      \\ \cmidrule{3-12}
&       \tableDeepDROCoxSplit             & \multicolumn{1}{c}{{\makecell{0.6829 \\(0.0049)}}} & \multicolumn{1}{c|}{\makecell{0.1881 \\(0.0012)}} & \multicolumn{1}{c}{{\makecell{2.4880 \\(0.5154)}}} & \multicolumn{1}{c}{\1\makecell{6.3123e-06 \\(7.2058e-06)}} & \1\makecell{6.2466e-06 \\(7.0785e-06)}& \multicolumn{1}{c}{{\makecell{0.6829 \\(0.0049)}}}  & \multicolumn{1}{c|}{\makecell{0.1881 \\(0.0012)}}  & \multicolumn{1}{c}{{\makecell{2.4880 \\(0.5154)}}}& \multicolumn{1}{c}{\1\makecell{6.3123e-06 \\(7.2058e-06)}} & \1\makecell{6.2466e-06 \\(7.0785e-06)} \\ \cmidrule{3-12} 
&       \tableDeepExactDROCox            & \multicolumn{1}{c}{{\makecell{0.7057 \\(0.0014)}}}  & \multicolumn{1}{c|}{\1\textbf{\makecell{0.1597 \\(0.0003)}}}  & \multicolumn{1}{c}{\makecell{2.5030 \\(0.2540)}}& \multicolumn{1}{c}{{\makecell{0.0277 \\(0.0004)}}}& {\makecell{0.0225 \\(0.0003)} } & \multicolumn{1}{c}{{\makecell{0.7057 \\(0.0014)}}}  & \multicolumn{1}{c|}{\1\textbf{\makecell{0.1597 \\(0.0003)}}}  & \multicolumn{1}{c}{\makecell{2.5030 \\(0.2540)}}& \multicolumn{1}{c}{{\makecell{0.0277 \\(0.0004)}}}& {\makecell{0.0225 \\(0.0003)} } \\ \bottomrule
\end{tabular}
}
\label{tab:general_performance_SEER2_race_CI}
\end{table*}

We compare \textsc{dro-cox} and \textsc{dro-cox (split)} against various baselines using a similar experimental setup as \citet{keya2021equitable}.
Specifically, we report the test set evaluation metrics for FLC (using age to evaluate CI and F$_{CG}$) in Table~\ref{tab:general_performance_CI}, SUPPORT (gender) in Table~\ref{tab:general_performance_SUPPORT_gender_CI}, and SEER (race) in Table~\ref{tab:general_performance_SEER2_race_CI}. Experimental results using other sensitive attributes for the datasets have similar trends and are in Appendix~\ref{additional_exp_results}. From these tables, we have the following observations:
\begin{itemize}[leftmargin=*,itemsep=0pt,parsep=0pt,topsep=0pt,partopsep=0pt]
    \item Among linear methods, the heuristic \textsc{dro-cox} method consistently outperforms baselines in terms of the CI fairness metric (and often on the other fairness metrics too) while still achieving reasonably high accuracy scores. A similar trend holds among nonlinear methods for the heuristic deep \textsc{dro-cox} variant.

    \item The performance difference (in terms of both accuracy and fairness) between the heuristic \textsc{dro-cox} and sample-splitting-based \textsc{dro-cox (split)} is not clear cut; sometimes one performs better than the other and vice versa. This holds for their linear variants as well as, separately, their nonlinear (deep) variants.
    
    \item As expected, the unregularized Cox and DeepSurv models often have (among) the highest accuracy scores but tend to have poor performance on fairness metrics. %
    
    \item The baselines that are regularized variants of Cox and DeepSurv typically do not simultaneously achieve low scores across all fairness metrics. Even though some of these can work well with some of the metrics by \citet{keya2021equitable}, they clearly do not work as well as our \textsc{dro-cox} variants when it comes to the CI fairness metric that actually accounts for accuracy.
\end{itemize}
\smallskip
\textit{Effect of $\alpha$.}
To show how $\alpha$ trades off between fairness and accuracy, we show results for \mbox{\textsc{dro-cox}} in the linear setting across all datasets (using age for evaluating F$_G$ and CI) in Figure~\ref{fig:sensitive_analysis}, where we use c-index as the accuracy metric. It is clear that accuracy tends to increase when $\alpha$ increases from 0.1 to 0.3 on FLC and SEER, and from 0.3 to 0.5 on SUPPORT. However, the increase in $\alpha$ results in worse scores across fairness metrics.

\begin{figure*}[t!]
\captionsetup[subfigure]{justification=centering}
\centering
\vspace{-.1em}
        \begin{subfigure}[b]{0.33\textwidth}
                \includegraphics[width=\linewidth]{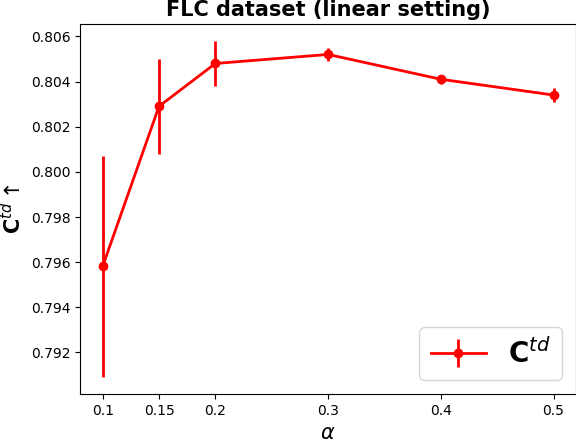}
        \end{subfigure}%
        \begin{subfigure}[b]{0.325\textwidth}
                \includegraphics[width=\linewidth]{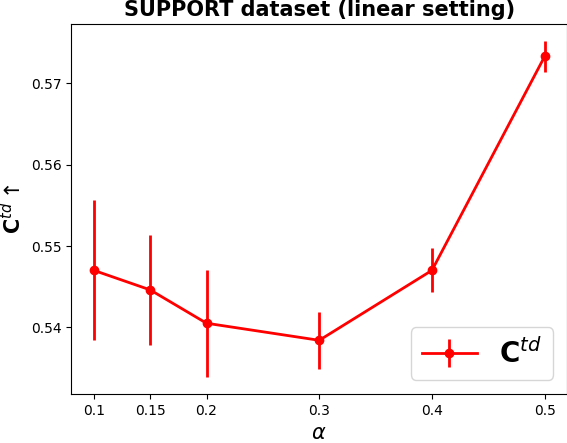}
        \end{subfigure}%
        \begin{subfigure}[b]{0.33\textwidth}
                \includegraphics[width=\linewidth]{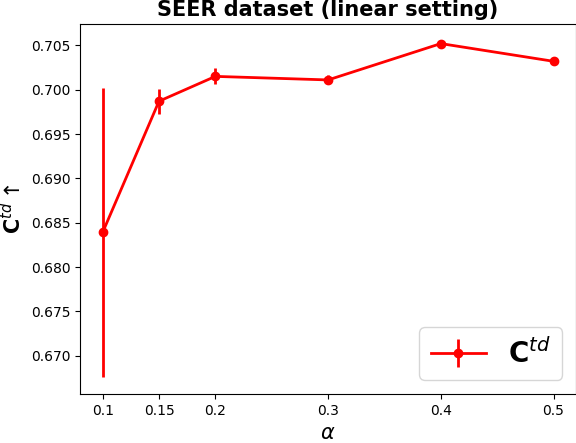}
        \end{subfigure}%

        \begin{subfigure}[b]{0.32\textwidth}
                \includegraphics[width=\linewidth]{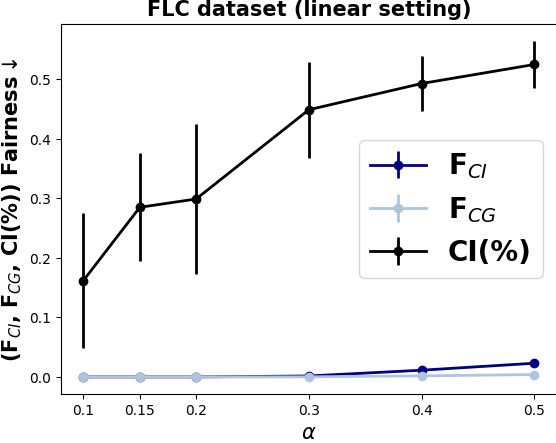}
                \caption{FLC (age)}
        \end{subfigure}%
        \begin{subfigure}[b]{0.32\textwidth}
                \includegraphics[width=\linewidth]{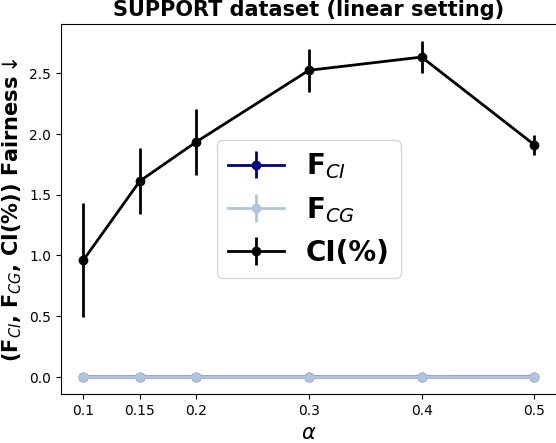}
                \caption{SUPPORT (age)}
        \end{subfigure}%
        \begin{subfigure}[b]{0.32\textwidth}
                \includegraphics[width=\linewidth]{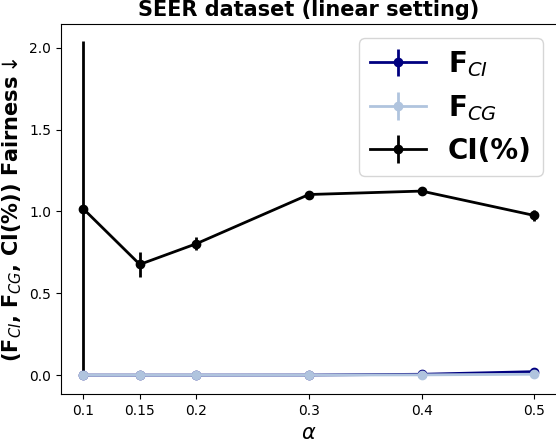}
                \caption{SEER (age)}
        \end{subfigure}%

\caption{\small Effect of $\alpha$ on test set accuracy (c-index; higher is better) and fairness metrics (F$_{CI}$, F$_{CG}$, and CI; lower is better for all fairness metrics) of \textsc{dro-cox} on four datasets.}
\vspace{-2em}
\label{fig:sensitive_analysis}
\end{figure*}

\smallskip
\noindent
\textit{Additional experiments.} Across all methods, instead of minimizing the validation set CI fairness metric during hyperparameter tuning (tolerating a small degradation in the validation set C$^{td}$), we also tried instead minimizing the validation set F$_{CG}$ metric and found similar results (see the rightmost columns under the heading ``F$_{CG}$-based Tuning'' in Tables~\ref{tab:general_performance_CI}, \ref{tab:general_performance_SUPPORT_gender_CI}, and \ref{tab:general_performance_SEER2_race_CI}).\footnotetext{Censoring-based individual and group fairness metrics (F$_{CI}$ and F$_{CG}$ respectively) by \citet{rahman2022fair}---which we formally define in Appendix~\ref{sec:fairness-measures}---depend on a user-specified scale constant $\gamma>0$ that must be specified in advance (of running experiments). A higher value of $\gamma$ makes it easier for a survival model to achieve an F$_{CI}$ or F$_{CG}$ score that is \emph{exactly} and not just approximately equal to~0. We set $\gamma=0.01$ for all datasets and it turns out that in this case, for the Cox model, it is possible for our DRO variants to achieve F$_{CI}$ or F$_{CG}$ scores that are exactly 0. We point out that we have found that if we decrease $\gamma$, then we no longer get exactly 0 for F$_{CI}$ or F$_{CG}$. Ultimately, this issue of F$_{CI}$ or F$_{CG}$ being exactly 0 is due to how they are defined by \citet{rahman2022fair}, and if one did not want these scores to be exactly 0, one would have to tune on $\gamma$ in a manner that could depend on the dataset. For more details, see Appendix~\ref{sec:fairness-measures}.}

We further conduct a number of supplemental experiments that we summarize the findings for now. %
First, we show that our \textsc{dro-cox (split)} procedure is somewhat robust to the choice of $n_1=|\mathcal{D}_1|$ and $n_2=|\mathcal{D}_2|$. Second, we show what happens if $\mathcal{D}_1$ and $\mathcal{D}_2$ have different censoring rates, where the main finding is that \textsc{dro-cox (split)} can still work well when there is a large imbalance in censoring rates between $\mathcal{D}_1$ and $\mathcal{D}_2$. Third, if \textsc{dro-cox (split)} did not use both losses $L_{\text{DRO}}^{\text{split}}(\theta,\eta,\mathcal{D}_{1}\mid\mathcal{D}_{2})$ and $L_{\text{DRO}}^{\text{split}}(\theta,\eta,\mathcal{D}_{2}\mid\mathcal{D}_{1})$ (i.e., if it only used one of these), then it performs worse. %
For details on these additional experiments including a formal definition of a quantity that controls the amount of imbalance in censoring rates between $\mathcal{D}_1$ and $\mathcal{D}_2$, %
see Appendix~\ref{additional_exp_results}. %

\paragraph{DeepHit}

\begin{table*}[t!]
\vspace{-1em}
\caption{\small DeepHit test set scores on the FLC, SUPPORT, SEER datasets when hyperparameter tuning is based on CI and F$_{CG}$.} 
\centering
\setlength\tabcolsep{1.6pt}
\renewcommand{\arraystretch}{0.5}
{\tiny %
\renewcommand{\belowrulesep}{0.1pt}
\renewcommand{\aboverulesep}{0.1pt}
\begin{tabular}{ccccccc|ccccc}
\toprule
\multirow{3}{*}{Datasets} & \multirow{3}{*}{Methods} & \multicolumn{5}{c|}{CI-based Tuning}                                                    & \multicolumn{5}{c}{F$_{CG}$-based Tuning}                             \\ \cmidrule{3-12}
&                   & \multicolumn{2}{c|}{Accuracy Metrics}  & \multicolumn{3}{c|}{Fairness Metrics}  & \multicolumn{2}{c|}{Accuracy Metrics}  & \multicolumn{3}{c}{Fairness Metrics} \\ \cmidrule{3-12}
&                   & \multicolumn{1}{c}{C$^{td}$$\uparrow$}  & \multicolumn{1}{c|}{IBS$\downarrow$}  &\multicolumn{1}{c}{CI(\%)$\downarrow$}& \multicolumn{1}{c}{F$_{CI}$$\downarrow$}& F$_{CG}$$\downarrow$ & \multicolumn{1}{c}{C$^{td}$$\uparrow$}  & \multicolumn{1}{c|}{IBS$\downarrow$}  &\multicolumn{1}{c}{CI(\%)$\downarrow$}& \multicolumn{1}{c}{F$_{CI}$$\downarrow$}& F$_{CG}$$\downarrow$                 \\ \midrule

\multirow{3}{*}{\makecell{FLC \\(age)}} &      DeepHit             & \multicolumn{1}{c}{\makecell{0.7937 \\(0.0080)}} & \multicolumn{1}{c|}{{\makecell{0.1560 \\(0.0204)}}} & \multicolumn{1}{c}{\makecell{1.1950 \\(0.7885)}} & \multicolumn{1}{c}{\makecell{0.0108 \\(0.0012)}} & \makecell{0.0022 \\(0.0003)} & \multicolumn{1}{c}{\textbf{\makecell{0.7937 \\(0.0080)}}} & \multicolumn{1}{c|}{{\makecell{0.1560 \\(0.0204)}}} & \multicolumn{1}{c}{\makecell{1.1950 \\(0.7885)}} & \multicolumn{1}{c}{\makecell{0.0108 \\(0.0012)}} & \makecell{0.0022 \\(0.0003)} \\ \cmidrule{3-12} 
                  &       \textsc{\tableDeepHitRPGroup}            & \multicolumn{1}{c}{{\makecell{0.7825 \\(0.0237)}}} & \multicolumn{1}{c|}{\textbf{\makecell{0.1449 \\(0.0201)}}} & \multicolumn{1}{c}{\makecell{1.2340 \\(0.6046)}} & \multicolumn{1}{c}{\makecell{0.0097 \\(0.0023)}} & \makecell{0.0021 \\(0.0003)} & \multicolumn{1}{c}{\makecell{0.7446 \\(0.0051)}} & \multicolumn{1}{c|}{\textbf{\makecell{0.1326 \\(0.0027)}}} & \multicolumn{1}{c}{\makecell{2.1110 \\(0.4770)}} & \multicolumn{1}{c}{{\makecell{0.0064 \\(0.0001)}}} & {\makecell{0.0018 \\(0.0001)}} \\ \cmidrule{2-12}
                  &       \textsc{dro-deephit}            & \multicolumn{1}{c}{\1\textbf{\makecell{0.7956 \\(0.0051)}}} & \multicolumn{1}{c|}{\makecell{0.1971 \\(0.0543)}} & \multicolumn{1}{c}{\1\makecell{1.0430 \\(0.4835)}} & \multicolumn{1}{c}{\1\makecell{0.0084 \\(0.0028)}} & \1\makecell{0.0017 \\(0.0006)} & \multicolumn{1}{c}{\makecell{0.7821 \\(0.0101)}} & \multicolumn{1}{c|}{\makecell{0.2754 \\(0.0076)}} & \multicolumn{1}{c}{\1\makecell{1.0180 \\(0.5330)}} & \multicolumn{1}{c}{\1\textbf{\makecell{0.0026 \\(0.0005)}}} & \1\textbf{\makecell{0.0006 \\(0.0001)}} \\ \cmidrule{3-12} 
                  &       \textsc{dro-deephit (split)}            & \multicolumn{1}{c}{\makecell{0.7748 \\(0.0189)}} & \multicolumn{1}{c|}{\makecell{0.2264 \\(0.0623)}} & \multicolumn{1}{c}{\1\textbf{\makecell{0.9950 \\(0.4556)}}} & \multicolumn{1}{c}{\1\textbf{\makecell{0.0067 \\(0.0036)}}} & \1\textbf{\makecell{0.0015 \\(0.0007)}} & \multicolumn{1}{c}{\makecell{0.7622 \\(0.0122)}} & \multicolumn{1}{c|}{\makecell{0.2734 \\(0.0092)}} & \multicolumn{1}{c}{\1\textbf{\makecell{0.9270 \\(0.5411)}}} & \multicolumn{1}{c}{\1\makecell{0.0027 \\(0.0009)}} & \1\makecell{0.0007 \\(0.0002)} \\ \midrule
\multirow{3}{*}{\makecell{FLC \\(gender)}} &         DeepHit          & \multicolumn{1}{c}{\makecell{0.7937 \\(0.0080)}} & \multicolumn{1}{c|}{{\makecell{0.1560 \\(0.0204)}}} & \multicolumn{1}{c}{{\makecell{0.4990 \\(0.3792)}}} & \multicolumn{1}{c}{\makecell{0.0108 \\(0.0012)}} & \makecell{0.0055 \\(0.0006)} & \multicolumn{1}{c}{\textbf{\makecell{0.7937 \\(0.0080)}}} & \multicolumn{1}{c|}{\textbf{\makecell{0.1560 \\(0.0204)}}} & \multicolumn{1}{c}{\textbf{\makecell{0.4990 \\(0.3792)}}} & \multicolumn{1}{c}{\makecell{0.0108 \\(0.0012)}} & \makecell{0.0055 \\(0.0006)} \\ \cmidrule{3-12} 
                  &       \textsc{\tableDeepHitRPGroup}            & \multicolumn{1}{c}{{\makecell{0.7840 \\(0.0245)}}} & \multicolumn{1}{c|}{\textbf{\makecell{0.1489 \\(0.0212)}}} & \multicolumn{1}{c}{\makecell{0.5170 \\(0.3982)}} & \multicolumn{1}{c}{\makecell{0.0099 \\(0.0021)}} & \makecell{0.0050 \\(0.0011)} & \multicolumn{1}{c}{\textbf{\makecell{0.7937 \\(0.0080)}}} & \multicolumn{1}{c|}{\textbf{\makecell{0.1560 \\(0.0204)}}} & \multicolumn{1}{c}{\textbf{\makecell{0.4990 \\(0.3792)}}} & \multicolumn{1}{c}{{\makecell{0.0108 \\(0.0012)}}} & {\makecell{0.0055 \\(0.0006)}} \\ \cmidrule{2-12}
                  &         \textsc{dro-deephit}          & \multicolumn{1}{c}{\1\textbf{\makecell{0.7956 \\(0.0051)}}} & \multicolumn{1}{c|}{\makecell{0.1971 \\(0.0543)}} & \multicolumn{1}{c}{\1\textbf{\makecell{0.4320 \\(0.4786)}}} & \multicolumn{1}{c}{\1\makecell{0.0084 \\(0.0028)}} & \1\makecell{0.0043 \\(0.0014)} & \multicolumn{1}{c}{\makecell{0.7821 \\(0.0101)}} & \multicolumn{1}{c|}{\makecell{0.2754 \\(0.0076)}} & \multicolumn{1}{c}{\makecell{1.3700 \\(0.6702)}} & \multicolumn{1}{c}{\1\textbf{\makecell{0.0026 \\(0.0005)}}} & \1\makecell{0.0013 \\(0.0002)}\\ \cmidrule{3-12} 
                  &       \textsc{dro-deephit (split)}            & \multicolumn{1}{c}{\makecell{0.7748 \\(0.0189)}} & \multicolumn{1}{c|}{\makecell{0.2264 \\(0.0623)}} & \multicolumn{1}{c}{\makecell{1.3100 \\(0.9915)}} & \multicolumn{1}{c}{\textbf{\1\makecell{0.0067 \\(0.0036)}}} & \1\textbf{\makecell{0.0034 \\(0.0018)}} & \multicolumn{1}{c}{\makecell{0.7622 \\(0.0122)}} & \multicolumn{1}{c|}{\makecell{0.2734 \\(0.0092)}} & \multicolumn{1}{c}{\makecell{1.9350 \\(0.7234)}} & \multicolumn{1}{c}{\1\makecell{0.0027 \\(0.0009)}} & \1\textbf{\makecell{0.0014 \\(0.0004)}} \\ \midrule
\multirow{3}{*}{\makecell{SUPPORT \\(age)}} &        DeepHit           & \multicolumn{1}{c}{\textbf{\makecell{0.6029 \\(0.0071)}}} & \multicolumn{1}{c|}{{\makecell{0.2151 \\(0.0067)}}} & \multicolumn{1}{c}{\makecell{3.5910 \\(0.3987)}} & \multicolumn{1}{c}{\makecell{0.0055 \\(0.0008)}} & \makecell{0.0026 \\(0.0004)}& \multicolumn{1}{c}{\textbf{\makecell{0.6029 \\(0.0071)}}} & \multicolumn{1}{c|}{{\makecell{0.2151 \\(0.0067)}}} & \multicolumn{1}{c}{\makecell{3.5910 \\(0.3987)}} & \multicolumn{1}{c}{\makecell{0.0055 \\(0.0008)}} & \makecell{0.0026 \\(0.0004)} \\ \cmidrule{3-12} 
                  &       \textsc{\tableDeepHitRPGroup}            & \multicolumn{1}{c}{{\makecell{0.5775 \\(0.0050)}}} & \multicolumn{1}{c|}{\textbf{\makecell{0.2123 \\(0.0009)}}} & \multicolumn{1}{c}{\textbf{\makecell{1.1940 \\(0.8221)}}} & \multicolumn{1}{c}{\makecell{0.0046 \\(0.0006)}} & \makecell{0.0023 \\(0.0003)} & \multicolumn{1}{c}{\makecell{0.5766 \\(0.0033)}} & \multicolumn{1}{c|}{\textbf{\makecell{0.2126 \\(0.0007)}}} & \multicolumn{1}{c}{\textbf{\makecell{1.0230 \\(0.4416)}}} & \multicolumn{1}{c}{{\makecell{0.0044 \\(0.0002)}}} & {\makecell{0.0022 \\(0.0001)}} \\ \cmidrule{2-12}
                  &         \textsc{dro-deephit}          & \multicolumn{1}{c}{\makecell{0.5932 \\(0.0159)}} & \multicolumn{1}{c|}{\makecell{0.2447 \\(0.0147)}} & \multicolumn{1}{c}{\makecell{2.9160 \\(0.8347)}} & \multicolumn{1}{c}{\1\textbf{\makecell{0.0014 \\(0.0009)}}} & \1\textbf{\makecell{0.0007 \\(0.0004)}}& \multicolumn{1}{c}{\makecell{0.5899 \\(0.0154)}} & \multicolumn{1}{c|}{\makecell{0.2493 \\(0.0159)}} & \multicolumn{1}{c}{{\makecell{3.3740 \\(0.6078)}}} & \multicolumn{1}{c}{\1\textbf{\makecell{0.0007 \\(0.0002)}}} & \1\textbf{\makecell{0.0003 \\(0.0001)}} \\ \cmidrule{3-12} 
                  &      \textsc{dro-deephit (split)}             & \multicolumn{1}{c}{\makecell{0.5753 \\(0.0236)}} & \multicolumn{1}{c|}{\makecell{0.2225 \\(0.0112)}} & \multicolumn{1}{c}{{\makecell{2.7280 \\(0.9570)}}} & \multicolumn{1}{c}{\1\makecell{0.0044 \\(0.0013)}} & {\1\makecell{0.0021 \\(0.0006)}} & \multicolumn{1}{c}{\makecell{0.5792 \\(0.0234)}} & \multicolumn{1}{c|}{\makecell{0.2392 \\(0.0268)}} & \multicolumn{1}{c}{\makecell{3.5270 \\(0.7331)}} & \multicolumn{1}{c}{\1\makecell{0.0037 \\(0.0019)}} & \1\makecell{0.0018 \\(0.0009)} \\ \midrule
\multirow{3}{*}{\makecell{SUPPORT \\(gender)}} &        DeepHit           & \multicolumn{1}{c}{\textbf{\makecell{0.6029 \\(0.0071)}}} & \multicolumn{1}{c|}{{\makecell{0.2151 \\(0.0067)}}} & \multicolumn{1}{c}{{\makecell{0.5880 \\(0.2895)}}} & \multicolumn{1}{c}{\makecell{0.0055 \\(0.0008)}} & \makecell{0.0028 \\(0.0004)} & \multicolumn{1}{c}{\textbf{\makecell{0.6029 \\(0.0071)}}} & \multicolumn{1}{c|}{{\makecell{0.2151 \\(0.0067)}}} & \multicolumn{1}{c}{\textbf{\makecell{0.5880 \\(0.2895)}}} & \multicolumn{1}{c}{\makecell{0.0055 \\(0.0008)}} & \makecell{0.0028 \\(0.0004)} \\ \cmidrule{3-12} 
                  &       \textsc{\tableDeepHitRPGroup}            & \multicolumn{1}{c}{{\makecell{0.5767 \\(0.0034)}}} & \multicolumn{1}{c|}{\textbf{\makecell{0.2126 \\(0.0008)}}} & \multicolumn{1}{c}{\makecell{0.6960 \\(0.3183)}} & \multicolumn{1}{c}{\makecell{0.0044 \\(0.0002)}} & \makecell{0.0022 \\(0.0001)} & \multicolumn{1}{c}{\makecell{0.5773 \\(0.0039)}} & \multicolumn{1}{c|}{\textbf{\makecell{0.2125 \\(0.0007)}}} & \multicolumn{1}{c}{\makecell{0.7600 \\(0.2994)}} & \multicolumn{1}{c}{{\makecell{0.0043 \\(0.0002)}}} & {\makecell{0.0022 \\(0.0001)}} \\ \cmidrule{2-12}
                  &         \textsc{dro-deephit}          & \multicolumn{1}{c}{\makecell{0.5932 \\(0.0159)}} & \multicolumn{1}{c|}{\makecell{0.2447 \\(0.0147)}} & \multicolumn{1}{c}{\makecell{1.1980 \\(0.6834)}} & \multicolumn{1}{c}{\1\textbf{\makecell{0.0014 \\(0.0009)}}} & \1\textbf{\makecell{0.0007 \\(0.0005)}} & \multicolumn{1}{c}{\makecell{0.5899 \\(0.0154)}} & \multicolumn{1}{c|}{\makecell{0.2493 \\(0.0159)}} & \multicolumn{1}{c}{\makecell{1.4460 \\(0.4235)}} & \multicolumn{1}{c}{\1\textbf{\makecell{0.0007 \\(0.0002)}}} & \1\textbf{\makecell{0.0004 \\(0.0001)}} \\ \cmidrule{3-12} 
                  &      \textsc{dro-deephit (split)}             & \multicolumn{1}{c}{\makecell{0.5753 \\(0.0236)}} & \multicolumn{1}{c|}{\makecell{0.2225 \\(0.0112)}} & \multicolumn{1}{c}{\1\textbf{\makecell{0.5160 \\(0.3942)}}} & \multicolumn{1}{c}{\makecell{0.0044 \\(0.0013)}} & \makecell{0.0022 \\(0.0006)} & \multicolumn{1}{c}{\makecell{0.5792 \\(0.0234)}} & \multicolumn{1}{c|}{\makecell{0.2392 \\(0.0268)}} & \multicolumn{1}{c}{\makecell{0.7550 \\(0.5022)}} & \multicolumn{1}{c}{\1\makecell{0.0037 \\(0.0019)}} & \1\makecell{0.0019 \\(0.0010)} \\ \midrule
\multirow{3}{*}{\makecell{SUPPORT \\(race)}} &       DeepHit            & \multicolumn{1}{c}{\textbf{\makecell{0.6029 \\(0.0071)}}} & \multicolumn{1}{c|}{{\makecell{0.2151 \\(0.0067)}}} & \multicolumn{1}{c}{\makecell{1.2250 \\(0.4454)}} & \multicolumn{1}{c}{\makecell{0.0055 \\(0.0008)}} & \makecell{0.0033 \\(0.0005)} & \multicolumn{1}{c}{\textbf{\makecell{0.6029 \\(0.0071)}}} & \multicolumn{1}{c|}{{\makecell{0.2151 \\(0.0067)}}} & \multicolumn{1}{c}{{\makecell{1.2250 \\(0.4454)}}} & \multicolumn{1}{c}{\makecell{0.0055 \\(0.0008)}} & \makecell{0.0033 \\(0.0005)} \\ \cmidrule{3-12} 
                  &       \textsc{\tableDeepHitRPGroup}            & \multicolumn{1}{c}{{\makecell{0.5767 \\(0.0031)}}} & \multicolumn{1}{c|}{\textbf{\makecell{0.2126 \\(0.0008)}}} & \multicolumn{1}{c}{\textbf{\makecell{0.7290 \\(0.4122)}}} & \multicolumn{1}{c}{\makecell{0.0044 \\(0.0002)}} & \makecell{0.0026 \\(0.0001)} & \multicolumn{1}{c}{\makecell{0.5813 \\(0.0108)}} & \multicolumn{1}{c|}{\textbf{\makecell{0.2144 \\(0.0041)}}} & \multicolumn{1}{c}{\textbf{\makecell{0.7400 \\(0.4211)}}} & \multicolumn{1}{c}{{\makecell{0.0043 \\(0.0003)}}} & {\makecell{0.0026 \\(0.0002)}} \\ \cmidrule{2-12}
                  &        \textsc{dro-deephit}           & \multicolumn{1}{c}{\makecell{0.5932 \\(0.0159)}} & \multicolumn{1}{c|}{\makecell{0.2447 \\(0.0147)}} & \multicolumn{1}{c}{{\makecell{1.0630 \\(0.5174)}}} & \multicolumn{1}{c}{\1\textbf{\makecell{0.0014 \\(0.0009)}}} & \1\textbf{\makecell{0.0009 \\(0.0005)}} & \multicolumn{1}{c}{\makecell{0.5899 \\(0.0154)}} & \multicolumn{1}{c|}{\makecell{0.2493 \\(0.0159)}} & \multicolumn{1}{c}{\makecell{1.4220 \\(0.4302)}} & \multicolumn{1}{c}{\1\textbf{\makecell{0.0007 \\(0.0002)}}} & \1\textbf{\makecell{0.0004 \\(0.0001)}} \\ \cmidrule{3-12} 
                  &     \textsc{dro-deephit (split)}              & \multicolumn{1}{c}{\makecell{0.5753 \\(0.0236)}} & \multicolumn{1}{c|}{\makecell{0.2225 \\(0.0112)}} & \multicolumn{1}{c}{\makecell{1.1930 \\(0.4449)}} & \multicolumn{1}{c}{\makecell{0.0044 \\(0.0013)}} & \makecell{0.0027 \\(0.0008)}& \multicolumn{1}{c}{\makecell{0.5792 \\(0.0234)}} & \multicolumn{1}{c|}{\makecell{0.2392 \\(0.0268)}} & \multicolumn{1}{c}{\makecell{1.5640 \\(0.6744)}} & \multicolumn{1}{c}{\1\makecell{0.0037 \\(0.0019)}} & \1\makecell{0.0022 \\(0.0012)} \\ \midrule
\multirow{3}{*}{\makecell{SEER \\(age)}} &       DeepHit            & \multicolumn{1}{c}{\textbf{\makecell{0.7156 \\(0.0047)}}} & \multicolumn{1}{c|}{\textbf{\makecell{0.1715 \\(0.0038)}}} & \multicolumn{1}{c}{\makecell{1.4450 \\(0.2901)}} & \multicolumn{1}{c}{\makecell{0.0122 \\(0.0011)}} & \makecell{0.0038 \\(0.0003)} & \multicolumn{1}{c}{\textbf{\makecell{0.7156 \\(0.0047)}}} & \multicolumn{1}{c|}{\textbf{\makecell{0.1715 \\(0.0038)}}} & \multicolumn{1}{c}{\makecell{1.4450 \\(0.2901)}} & \multicolumn{1}{c}{\makecell{0.0122 \\(0.0011)}} & \makecell{0.0038 \\(0.0003)} \\ \cmidrule{3-12} 
                  &       \textsc{\tableDeepHitRPGroup}            & \multicolumn{1}{c}{{\makecell{0.7122 \\(0.0086)}}} & \multicolumn{1}{c|}{\makecell{0.1743 \\(0.0064)}} & \multicolumn{1}{c}{\makecell{1.4160 \\(0.2443)}} & \multicolumn{1}{c}{\makecell{0.0105 \\(0.0029)}} & \makecell{0.0034 \\(0.0007)} & \multicolumn{1}{c}{\makecell{0.6987 \\(0.0025)}} & \multicolumn{1}{c|}{\makecell{0.1801 \\(0.0021)}} & \multicolumn{1}{c}{\makecell{2.0960 \\(0.4633)}} & \multicolumn{1}{c}{{\makecell{0.0046 \\(0.0002)}}} & {\makecell{0.0019 \\(0.0001)}} \\ \cmidrule{2-12}
                  &        \textsc{dro-deephit}           & \multicolumn{1}{c}{\makecell{0.7112 \\(0.0084)}} & \multicolumn{1}{c|}{\makecell{0.2794 \\(0.0871)}} & \multicolumn{1}{c}{\1\textbf{\makecell{0.7990 \\(0.3281)}}} & \multicolumn{1}{c}{\1\textbf{\makecell{0.0061 \\(0.0041)}}} & \1\textbf{\makecell{0.0020 \\(0.0013)}} & \multicolumn{1}{c}{\makecell{0.6951 \\(0.0051)}} & \multicolumn{1}{c|}{\makecell{0.4122 \\(0.0304)}} & \multicolumn{1}{c}{\1\makecell{1.3800 \\(0.5574)}} & \multicolumn{1}{c}{\1\textbf{\makecell{0.0002 \\(0.0002)}}} & \1\textbf{\makecell{0.0001 \\(0.0001)}} \\ \cmidrule{3-12} 
                  &      \textsc{dro-deephit (split)}             & \multicolumn{1}{c}{\makecell{0.6969 \\(0.0211)}} & \multicolumn{1}{c|}{\makecell{0.2073 \\(0.0464)}} & \multicolumn{1}{c}{\1\makecell{1.0240 \\(0.3449)}} & \multicolumn{1}{c}{\makecell{0.0107 \\(0.0016)}} & \makecell{0.0038 \\(0.0004)} & \multicolumn{1}{c}{\makecell{0.6963 \\(0.0224)}} & \multicolumn{1}{c|}{\makecell{0.2063 \\(0.0419)}} & \multicolumn{1}{c}{\1\textbf{\makecell{1.2630 \\(0.7467)}}} & \multicolumn{1}{c}{\makecell{0.0098 \\(0.0025)}} & \makecell{0.0034 \\(0.0005)} \\ \midrule
\multirow{3}{*}{\makecell{SEER \\(race)}} &      DeepHit             & \multicolumn{1}{c}{\textbf{\makecell{0.7156 \\(0.0047)}}} & \multicolumn{1}{c|}{\textbf{\makecell{0.1715 \\(0.0038)}}} & \multicolumn{1}{c}{\makecell{3.2820 \\(0.5958)}} & \multicolumn{1}{c}{\makecell{0.0122 \\(0.0011)}} & \makecell{0.0099 \\(0.0009)} & \multicolumn{1}{c}{\textbf{\makecell{0.7156 \\(0.0047)}}} & \multicolumn{1}{c|}{\textbf{\makecell{0.1715 \\(0.0038)}}} & \multicolumn{1}{c}{\makecell{3.2820 \\(0.5958)}} & \multicolumn{1}{c}{\makecell{0.0122 \\(0.0011)}} & \makecell{0.0099 \\(0.0009)} \\ \cmidrule{3-12} 
                  &       \textsc{\tableDeepHitRPGroup}            & \multicolumn{1}{c}{{\makecell{0.7132 \\(0.0073)}}} & \multicolumn{1}{c|}{\makecell{0.1728 \\(0.0045)}} & \multicolumn{1}{c}{\makecell{3.1330 \\(0.8321)}} & \multicolumn{1}{c}{\makecell{0.0113 \\(0.0028)}} & \makecell{0.0091 \\(0.0023)} & \multicolumn{1}{c}{\makecell{0.6987 \\(0.0048)}} & \multicolumn{1}{c|}{\makecell{0.1806 \\(0.0028)}} & \multicolumn{1}{c}{\textbf{\makecell{1.6760 \\(0.6385)}}} & \multicolumn{1}{c}{{\makecell{0.0045 \\(0.0023)}}} & {\makecell{0.0034 \\(0.0019)}} \\ \cmidrule{2-12}
                  &       \textsc{dro-deephit}            & \multicolumn{1}{c}{\makecell{0.7112 \\(0.0084)}} & \multicolumn{1}{c|}{\makecell{0.2794 \\(0.0871)}} & \multicolumn{1}{c}{\1\makecell{3.0120 \\(0.5652)}} & \multicolumn{1}{c}{\1\textbf{\makecell{0.0061 \\(0.0041)}}} & \1\textbf{\makecell{0.0049 \\(0.0033)}} & \multicolumn{1}{c}{\makecell{0.6951 \\(0.0051)}} & \multicolumn{1}{c|}{\makecell{0.4122 \\(0.0304)}} & \multicolumn{1}{c}{\makecell{3.2520 \\(1.7820)}} & \multicolumn{1}{c}{\1\textbf{\makecell{0.0002 \\(0.0002)}}} & \1\textbf{\makecell{0.0002 \\(0.0002)}} \\ \cmidrule{3-12} 
                  &       \textsc{dro-deephit (split)}            & \multicolumn{1}{c}{\makecell{0.6969 \\(0.0211)}} & \multicolumn{1}{c|}{\makecell{0.2073 \\(0.0464)}} & \multicolumn{1}{c}{\1\textbf{\makecell{2.7700 \\(0.5636)}}} & \multicolumn{1}{c}{\1\makecell{0.0107 \\(0.0016)}} & \1\makecell{0.0085 \\(0.0014)} & \multicolumn{1}{c}{\makecell{0.6963 \\(0.0224)}} & \multicolumn{1}{c|}{\makecell{0.2063 \\(0.0419)}} & \multicolumn{1}{c}{{\makecell{3.0070 \\(0.8355)}}} & \multicolumn{1}{c}{\makecell{0.0098 \\(0.0025)}} & \makecell{0.0078 \\(0.0021)} \\ \bottomrule
\end{tabular}
}
\label{tab:general_performance_deephit_CI}
\end{table*}

We now compare \textsc{dro-deephit} and \textsc{dro-deephit (split)} to the original DeepHit method \citep{lee2018deephit} and the regularized variant \textsc{\tableDeepHitRPGroup}.
We report the test performance on all three datasets in Table \ref{tab:general_performance_deephit_CI}. According to the results in Table \ref{tab:general_performance_deephit_CI}, we have the following observations:
\begin{itemize}[leftmargin=*,itemsep=0pt,parsep=0pt,topsep=0pt,partopsep=0pt]
    \item Our DRO variants can achieve better CI performance than the original DeepHit method on most of datasets with different sensitive attributes when using a CI-based hyperparameter tuning strategy. It is also clear that \textsc{dro-deephit} and \textsc{dro-deephit (split)} can achieve lower values on F$_{CI}$ and F$_{CG}$ on all datasets when using an F$_{CG}$-based hyperparameter tuning strategy. These results indicate that our DRO methods can encourage fairness for DeepHit and can obtain better fairness scores than \textsc{\tableDeepHitRPGroup}.
    \item We find that our DRO variants outperform DeepHit on F$_{CI}$ and F$_{CG}$ metrics when using CI-based hyperparameter tuning. However, we find that our DRO variants cannot always achieve the best scores on the CI fairness metric when using F$_{CG}$-based hyperparameter tuning. We conclude that the CI metric may reflect fairness in the F$_{CG}$ fairness metric but the reverse may not be true.
    \item It is hard to distinguish which method is better between \textsc{dro-deephit} and \textsc{dro-deephit (split)}. For both methods, as expected, they have slightly lower performance than the DeepHit method on accuracy metrics. However, \textsc{dro-deephit} method has the best C$^{td}$ performance on the FLC dataset in Table \ref{tab:general_performance_deephit_CI}.
\end{itemize}

\begin{table*}[t!]
\vspace{-1em}
\caption{\small SODEN test set scores on the FLC, SUPPORT, SEER datasets when hyperparameter tuning is based on CI and F$_{CG}$.} 
\centering
\setlength\tabcolsep{0.1pt}
\renewcommand{\arraystretch}{0.5}
{\tiny %
\renewcommand{\belowrulesep}{0.1pt}
\renewcommand{\aboverulesep}{0.1pt}
\begin{tabular}{ccccccc|ccccc}
\toprule
\multirow{3}{*}{Datasets} & \multirow{3}{*}{Methods} & \multicolumn{5}{c|}{CI-based Tuning}                                                    & \multicolumn{5}{c}{F$_{CG}$-based Tuning}                             \\ \cmidrule{3-12}
&                   & \multicolumn{2}{c|}{Accuracy Metrics}  & \multicolumn{3}{c|}{Fairness Metrics}  & \multicolumn{2}{c|}{Accuracy Metrics}  & \multicolumn{3}{c}{Fairness Metrics} \\ \cmidrule{3-12}
&                   & \multicolumn{1}{c}{C$^{td}$$\uparrow$}  & \multicolumn{1}{c|}{IBS$\downarrow$}  &\multicolumn{1}{c}{CI(\%)$\downarrow$}& \multicolumn{1}{c}{F$_{CI}$$\downarrow$}& F$_{CG}$$\downarrow$ & \multicolumn{1}{c}{C$^{td}$$\uparrow$}  & \multicolumn{1}{c|}{IBS$\downarrow$}  &\multicolumn{1}{c}{CI(\%)$\downarrow$}& \multicolumn{1}{c}{F$_{CI}$$\downarrow$}& F$_{CG}$$\downarrow$                 \\ \midrule
\multirow{2}{*}{\makecell{FLC \\(age)}} &      SODEN             & \multicolumn{1}{c}{\makecell{0.7785 \\(0.0175)}} & \multicolumn{1}{c|}{{\makecell{0.1482 \\(0.0138)}}} & \multicolumn{1}{c}{\makecell{1.1790 \\(0.6098)}} & \multicolumn{1}{c}{\makecell{0.0004 \\(0.0009)}} & \makecell{0.0001 \\(0.0002)} & \multicolumn{1}{c}{\makecell{0.7785 \\(0.0175)}} & \multicolumn{1}{c|}{{\makecell{0.1482 \\(0.0138)}}} & \multicolumn{1}{c}{{\textbf{\makecell{1.1790 \\(0.6098)}}}} & \multicolumn{1}{c}{{\makecell{0.0004 \\(0.0009)}}} & {\makecell{0.0001 \\(0.0002)}} \\ \cmidrule{3-12} 
                  &      \tableSODENRPGroup             & \multicolumn{1}{c}{\makecell{0.7832 \\(0.0138)}} & \multicolumn{1}{c|}{{\makecell{0.1454 \\(0.0134)}}} & \multicolumn{1}{c}{\makecell{0.8248 \\(0.6491)}} & \multicolumn{1}{c}{\makecell{0.0006 \\(0.0007)}} & \makecell{0.0002 \\(0.0002)} & \multicolumn{1}{c}{\textbf{\makecell{0.7807 \\(0.0140)}}} & \multicolumn{1}{c|}{{\textbf{\makecell{0.1454 \\(0.0134)}}}} & \multicolumn{1}{c}{{{\makecell{1.7324 \\(1.2715)}}}} & \multicolumn{1}{c}{{\makecell{0.0001 \\(0.0001)}}} & {{\makecell{1.2814e-05 \\(2.2052e-05)}}} \\ \cmidrule{2-12} 
                  &        \textsc{dro-soden}           & \multicolumn{1}{c}{\1\textbf{\makecell{0.7857 \\(0.0124)}}} & \multicolumn{1}{c|}{\1\textbf{\makecell{0.1434 \\(0.0141)}}} & \multicolumn{1}{c}{\1\textbf{\makecell{1.0401 \\(0.7724)}}} & \multicolumn{1}{c}{\1\textbf{\makecell{1.6903e-05 \\(3.9865e-05)}}} & \1\textbf{\makecell{6.2555e-06 \\(1.4369e-05)}}& \multicolumn{1}{c}{{\makecell{0.7787 \\(0.0134)}}} & \multicolumn{1}{c|}{\makecell{0.1619 \\(0.0247)}} & \multicolumn{1}{c}{{\makecell{1.4140 \\(0.7545)}}} & \multicolumn{1}{c}{\1{\textbf{\makecell{2.4385e-06 \\(5.8775e-06)}}}} & {\1\textbf{\makecell{8.9596e-07 \\(2.1956e-06)}}} \\ \midrule
\multirow{2}{*}{\makecell{FLC \\(gender)}} &        SODEN           & \multicolumn{1}{c}{\makecell{0.7785 \\(0.0175)}} & \multicolumn{1}{c|}{{\makecell{0.1482 \\(0.0138)}}} & \multicolumn{1}{c}{\makecell{1.3822 \\(0.6028)}} & \multicolumn{1}{c}{\textbf{\makecell{0.0004 \\(0.0009)}}} & {\textbf{\makecell{ 0.0002 \\(0.0005)}}}& \multicolumn{1}{c}{\makecell{0.7785 \\(0.0175)}} & \multicolumn{1}{c|}{{\makecell{0.1482 \\(0.0138)}}} & \multicolumn{1}{c}{{{\makecell{1.3822 \\(0.6028)}}}} & \multicolumn{1}{c}{{\makecell{0.0004 \\(0.0009)}}} & {\makecell{0.0002 \\(0.0005)}} \\ \cmidrule{3-12} 
                  &      \tableSODENRPGroup             & \multicolumn{1}{c}{\makecell{0.7824 \\(0.0126)}} & \multicolumn{1}{c|}{{\makecell{0.1496 \\(0.0117)}}} & \multicolumn{1}{c}{\makecell{0.8252 \\(0.3665)}} & \multicolumn{1}{c}{\makecell{0.0005 \\(0.0006)}} & {\makecell{0.0003 \\(0.0003)}} & \multicolumn{1}{c}{\textbf{\makecell{0.7832 \\(0.0126)}}} & \multicolumn{1}{c|}{{\textbf{\makecell{0.1452 \\(0.0131)}}}} & \multicolumn{1}{c}{{{\textbf{\makecell{1.0564 \\(0.4984)}}}}} & \multicolumn{1}{c}{{\makecell{0.0001 \\(0.0001)}}} & {\makecell{2.5966e-05 \\(4.4394e-05)}} \\ \cmidrule{2-12}
                  &         \textsc{dro-soden}          & \multicolumn{1}{c}{\1\textbf{\makecell{0.7857 \\(0.0100)}}} & \multicolumn{1}{c|}{\1\textbf{\makecell{0.1350 \\(0.0069)}}} & \multicolumn{1}{c}{\1\textbf{\makecell{0.7115 \\(0.3545)}}} & \multicolumn{1}{c}{{\makecell{0.0008 \\(0.0023)}}} & {\makecell{0.0004 \\(0.0011)}}& \multicolumn{1}{c}{{\makecell{0.7811 \\(0.0131)}}} & \multicolumn{1}{c|}{\makecell{0.1592 \\(0.0258)}} & \multicolumn{1}{c}{{\makecell{1.5226 \\(0.7068)}}} & \multicolumn{1}{c}{{\1\textbf{\makecell{2.4385e-06 \\(5.8775e-06)}}}} & {\1\textbf{\makecell{1.5323e-06 \\(3.5413e-06)}}} \\ \midrule
\multirow{2}{*}{\makecell{SUPPORT \\(age)}} &         SODEN          & \multicolumn{1}{c}{{\textbf{\makecell{0.6276 \\(0.0101)}}}} & \multicolumn{1}{c|}{\textbf{\makecell{0.1933 \\(0.0013)}}} & \multicolumn{1}{c}{\makecell{2.6275 \\(0.2490)}} & \multicolumn{1}{c}{\makecell{0.0081 \\(0.0007)}} & \makecell{0.0035 \\(0.0003)}& \multicolumn{1}{c}{{\textbf{\makecell{0.6276 \\(0.0101)}}}} & \multicolumn{1}{c|}{\textbf{\makecell{0.1933 \\(0.0013)}}} & \multicolumn{1}{c}{{\makecell{2.6275 \\(0.2490)}}} & \multicolumn{1}{c}{{\makecell{0.0081 \\(0.0007)}}} & {\makecell{0.0035 \\(0.0003)}} \\ \cmidrule{3-12} 
                  &      \tableSODENRPGroup             & \multicolumn{1}{c}{{\makecell{0.6162 \\(0.0118)}}} & \multicolumn{1}{c|}{{\makecell{0.2073 \\(0.0134)}}} & \multicolumn{1}{c}{\makecell{1.9914 \\(0.4342)}} & \multicolumn{1}{c}{\makecell{0.0059 \\(0.0018)}} & \makecell{0.0025 \\(0.0006)} & \multicolumn{1}{c}{{\makecell{0.6070 \\(0.0080)}}} & \multicolumn{1}{c|}{{\makecell{0.2147 \\(0.0099)}}} & \multicolumn{1}{c}{{{\textbf{\makecell{1.6135 \\(0.2891)}}}}} & \multicolumn{1}{c}{{\makecell{0.0043 \\(0.0008)}}} & {\makecell{0.0020 \\(0.0004)}} \\ \cmidrule{2-12}
                  &        \textsc{dro-soden}           & \multicolumn{1}{c}{\makecell{0.6080 \\(0.0161)}} & \multicolumn{1}{c|}{\makecell{0.2002 \\(0.0095)}} & \multicolumn{1}{c}{\1\textbf{\makecell{1.9901 \\(0.4576)}}} & \multicolumn{1}{c}{\1\textbf{\makecell{0.0045 \\(0.0022)}}} & \1\textbf{\makecell{0.0021 \\(0.0012)}}& \multicolumn{1}{c}{\makecell{0.5996 \\(0.0128)}} & \multicolumn{1}{c|}{\makecell{0.2041 \\(0.0082)}} & \multicolumn{1}{c}{{{\makecell{1.9980 \\(0.4681)}}}} & \multicolumn{1}{c}{{\1\textbf{\makecell{0.0031 \\(0.0011)}}}} & {\1\textbf{\makecell{0.0019 \\(0.0011)}}} \\ \midrule
\multirow{2}{*}{\makecell{SUPPORT \\(gender)}} &      SODEN             & \multicolumn{1}{c}{{\textbf{\makecell{0.6276 \\(0.0101)}}}} & \multicolumn{1}{c|}{\textbf{\makecell{0.1933 \\(0.0013)}}} & \multicolumn{1}{c}{\makecell{1.7548 \\(0.1958)}} & \multicolumn{1}{c}{\makecell{0.0081 \\(0.0007)}} & \makecell{0.0041 \\(0.0003)}& \multicolumn{1}{c}{{\textbf{\makecell{0.6276 \\(0.0101)}}}} & \multicolumn{1}{c|}{\textbf{\makecell{0.1933 \\(0.0013)}}} & \multicolumn{1}{c}{{\makecell{1.7548 \\(0.1958)}}} & \multicolumn{1}{c}{{\makecell{0.0081 \\(0.0007)}}} & {\makecell{0.0041 \\(0.0003)}} \\ \cmidrule{3-12} 
                  &      \tableSODENRPGroup             & \multicolumn{1}{c}{{\makecell{0.6263 \\(0.0077)}}} & \multicolumn{1}{c|}{{\makecell{0.1960 \\(0.0057)}}} & \multicolumn{1}{c}{\makecell{1.6308 \\(0.3285)}} & \multicolumn{1}{c}{\makecell{0.0074 \\(0.0012)}} & \makecell{0.0038 \\(0.0006)} & \multicolumn{1}{c}{{\makecell{0.6083 \\(0.0070)}}} & \multicolumn{1}{c|}{{\makecell{0.2147 \\(0.0099)}}} & \multicolumn{1}{c}{{{\textbf{\makecell{1.2239 \\(0.1634)}}}}} & \multicolumn{1}{c}{{\makecell{0.0043 \\(0.0008)}}} & {\makecell{0.0023 \\(0.0004)}} \\ \cmidrule{2-12}
                  &       \textsc{dro-soden}            & \multicolumn{1}{c}{\makecell{0.6177 \\(0.0118)}} & \multicolumn{1}{c|}{\makecell{0.1943 \\(0.0023)}} & \multicolumn{1}{c}{\1\textbf{\makecell{1.6282 \\(0.2094)}}} & \multicolumn{1}{c}{\1\textbf{\makecell{0.0065 \\(0.0016)}}} & \1\textbf{\makecell{0.0033 \\(0.0008)}}& \multicolumn{1}{c}{\makecell{0.5996 \\(0.0128)}} & \multicolumn{1}{c|}{\makecell{0.2041 \\(0.0082)}} & \multicolumn{1}{c}{{{\makecell{1.5995 \\(0.1943)}}}} & \multicolumn{1}{c}{{\1\textbf{\makecell{0.0031 \\(0.0011)}}}} & {\1\textbf{\makecell{0.0016 \\(0.0006)}}} \\ \midrule
\multirow{2}{*}{\makecell{SUPPORT \\(race)}} &       SODEN            & \multicolumn{1}{c}{{\textbf{\makecell{0.6276 \\(0.0101)}}}} & \multicolumn{1}{c|}{\textbf{\makecell{0.1933 \\(0.0013)}}} & \multicolumn{1}{c}{\makecell{1.6910 \\(0.2182)}} & \multicolumn{1}{c}{\makecell{0.0081 \\(0.0007)}} & \makecell{0.0048 \\(0.0004)}& \multicolumn{1}{c}{{\textbf{\makecell{0.6276 \\(0.0101)}}}} & \multicolumn{1}{c|}{\textbf{\makecell{0.1933 \\(0.0013)}}} & \multicolumn{1}{c}{{\textbf{\makecell{1.6910 \\(0.2182)}}}} & \multicolumn{1}{c}{{\makecell{0.0081 \\(0.0007)}}} & {\makecell{0.0048 \\(0.0004)}} \\ \cmidrule{3-12} 
                  &      \tableSODENRPGroup             & \multicolumn{1}{c}{{\makecell{0.6137 \\(0.0085)}}} & \multicolumn{1}{c|}{{\makecell{0.2045 \\(0.0115)}}} & \multicolumn{1}{c}{\makecell{1.6304 \\(0.1344)}} & \multicolumn{1}{c}{\makecell{0.0058 \\(0.0016)}} & \makecell{0.0036 \\(0.0009)} & \multicolumn{1}{c}{{\makecell{0.6089 \\(0.0073)}}} & \multicolumn{1}{c|}{{\makecell{0.2164 \\(0.0101)}}} & \multicolumn{1}{c}{{{\makecell{1.7705 \\(0.4111)}}}} & \multicolumn{1}{c}{{\makecell{0.0043 \\(0.0008)}}} & {\makecell{0.0027 \\(0.0005)}} \\ \cmidrule{2-12}
                  &        \textsc{dro-soden}           & \multicolumn{1}{c}{\makecell{0.6113 \\(0.0143)}} & \multicolumn{1}{c|}{\makecell{0.1993 \\(0.0093)}} & \multicolumn{1}{c}{\1\textbf{\makecell{1.3418 \\(0.4286)}}} & \multicolumn{1}{c}{\1\textbf{\makecell{0.0052 \\(0.0025)}}} & \1\textbf{\makecell{0.0032 \\(0.0015)}}& \multicolumn{1}{c}{\makecell{0.5996 \\(0.0128)}} & \multicolumn{1}{c|}{\makecell{0.2041 \\(0.0082)}} & \multicolumn{1}{c}{{{\makecell{1.1979 \\(0.4273)}}}} & \multicolumn{1}{c}{{\1\textbf{\makecell{0.0031 \\(0.0011)}}}} & {\1\textbf{\makecell{0.0019 \\(0.0006)}}} \\ \midrule
\multirow{2}{*}{\makecell{SEER \\(age)}} &        SODEN           & \multicolumn{1}{c}{\textbf{\makecell{0.7132 \\(0.0017)}}} & \multicolumn{1}{c|}{\textbf{\makecell{0.1550 \\(0.0009)}}} & \multicolumn{1}{c}{\textbf{\makecell{0.8531 \\(0.0940)}}} & \multicolumn{1}{c}{\makecell{0.0280 \\(0.0014)}} & {\makecell{0.0075 \\(0.0006)}}& \multicolumn{1}{c}{\textbf{\makecell{0.7132 \\(0.0017)}}} & \multicolumn{1}{c|}{\textbf{\makecell{0.1550 \\(0.0009)}}} & \multicolumn{1}{c}{{\textbf{\makecell{0.8531 \\(0.0940)}}}} & \multicolumn{1}{c}{{\makecell{0.0280 \\(0.0014)}}} & {\makecell{0.0075 \\(0.0006)}} \\ \cmidrule{3-12} 
                  &      \tableSODENRPGroup             & \multicolumn{1}{c}{\makecell{0.7131 \\(0.0017)}} & \multicolumn{1}{c|}{{\makecell{0.1556 \\(0.0011)}}} & \multicolumn{1}{c}{\makecell{0.8541 \\(0.1562)}} & \multicolumn{1}{c}{\makecell{0.0277 \\(0.0011)}} & {\makecell{0.0075 \\(0.0006)}} & \multicolumn{1}{c}{\makecell{0.7122 \\(0.0009)}} & \multicolumn{1}{c|}{{\makecell{0.1561 \\(0.0012)}}} & \multicolumn{1}{c}{{{\makecell{0.9110 \\(0.0948)}}}} & \multicolumn{1}{c}{{\makecell{0.0276 \\(0.0013)}}} & {\makecell{0.0072 \\(0.0004)}} \\ \cmidrule{2-12}
                  &        \textsc{dro-soden}           & \multicolumn{1}{c}{\makecell{0.7026 \\(0.0116)}} & \multicolumn{1}{c|}{\makecell{0.1757 \\(0.0293)}} & \multicolumn{1}{c}{{\makecell{1.1275 \\(0.3644)}}} & \multicolumn{1}{c}{\1\textbf{\makecell{0.0227 \\(0.0054)}}} & {\textbf{\makecell{0.0071 \\(0.0010)}}}& \multicolumn{1}{c}{\makecell{0.6980 \\(0.0108)}} & \multicolumn{1}{c|}{\makecell{0.2008 \\(0.0367)}} & \multicolumn{1}{c}{{\makecell{1.4280 \\(0.5242)}}} & \multicolumn{1}{c}{{\1\textbf{\makecell{0.0161 \\(0.0095)}}}} & {\1\textbf{\makecell{0.0049 \\(0.0021)}}} \\ \midrule
\multirow{2}{*}{\makecell{SEER \\(race)}} &        SODEN           & \multicolumn{1}{c}{\textbf{\makecell{0.7132 \\(0.0017)}}} & \multicolumn{1}{c|}{\textbf{\makecell{0.1550 \\(0.0009)}}} & \multicolumn{1}{c}{\makecell{2.4948 \\(0.1341)}} & \multicolumn{1}{c}{\makecell{0.0280 \\(0.0014)}} & \makecell{0.0227 \\(0.0011)}& \multicolumn{1}{c}{\textbf{\makecell{0.7132 \\(0.0017)}}} & \multicolumn{1}{c|}{\textbf{\makecell{0.1550 \\(0.0009)}}} & \multicolumn{1}{c}{{\makecell{2.4948 \\(0.1341)}}} & \multicolumn{1}{c}{{\makecell{0.0280 \\(0.0014)}}} & {\makecell{0.0227 \\(0.0011)}} \\ \cmidrule{3-12} 
                  &      \tableSODENRPGroup             & \multicolumn{1}{c}{\makecell{0.7124 \\(0.0016)}} & \multicolumn{1}{c|}{{\makecell{0.1558 \\(0.0011)}}} & \multicolumn{1}{c}{\makecell{2.4390 \\(0.1775)}} & \multicolumn{1}{c}{\makecell{0.0273 \\(0.0013)}} & \makecell{0.0220 \\(0.0011)} & \multicolumn{1}{c}{\makecell{0.7123 \\(0.0016)}} & \multicolumn{1}{c|}{{\makecell{0.1561 \\(0.0013)}}} & \multicolumn{1}{c}{{{\makecell{2.4747 \\(0.1869)}}}} & \multicolumn{1}{c}{{\makecell{0.0271 \\(0.0013)}}} & {\makecell{0.0218 \\(0.0011)}} \\ \cmidrule{2-12}
                  &        \textsc{dro-soden}           & \multicolumn{1}{c}{\makecell{0.6913 \\(0.0109)}} & \multicolumn{1}{c|}{\makecell{0.2079 \\(0.0373)}} & \multicolumn{1}{c}{\1\textbf{\makecell{1.6398 \\(0.4948)}}} & \multicolumn{1}{c}{\1\textbf{\makecell{0.0167 \\(0.0057)}}} & \1\textbf{\makecell{0.0132 \\(0.0047)}}& \multicolumn{1}{c}{\makecell{0.6893 \\(0.0055)}} & \multicolumn{1}{c|}{\makecell{0.2191 \\(0.0269)}} & \multicolumn{1}{c}{{\1\textbf{\makecell{1.7676 \\(0.4458)}}}} & \multicolumn{1}{c}{{\1\textbf{\makecell{0.0126 \\(0.0059)}}}} & {\1\textbf{\makecell{0.0099 \\(0.0047)}}} \\ 
                  \bottomrule
\end{tabular}
}
\label{tab:general_performance_SODEN_CI}
\end{table*}

\paragraph{SODEN}

We conduct experiments to compare the accuracy and fairness of SODEN and 
 \tableSODENRPGroup~to \textsc{dro-soden}.
Our experimental results are reported in Table \ref{tab:general_performance_SODEN_CI} (CI-based hyperparameter tuning and F$_{CG}$-based hyperparameter tuning). From these results, we have the following observations:
\begin{itemize}[leftmargin=*,itemsep=0pt,parsep=0pt,topsep=0pt,partopsep=0pt]
    \item When tuning hyperparameters based on CI, it is clear that \textsc{dro-soden} outperforms the other methods on the CI fairness metric for FLC and SUPPORT datasets. Meanwhile, F$_{CI}$ and F$_{CG}$ are also reduced by using \textsc{dro-soden} while accuracy scores become a little lower than those of SODEN. However, we find \textsc{dro-soden} can achieve a slightly higher C$^{td}$ scores on the FLC dataset.
    
    \item When tuning hyperparameters based on F$_{CG}$, we find \textsc{dro-soden} also can achieve better performance on F$_{CG}$ than the corresponding values that are from the CI-based tuning since we tune hyperparameters based on this metric. In addition, \textsc{dro-soden} can obtain the best F$_{CG}$ and F$_{CI}$ scores compared to the baselines.
\end{itemize}

\subsection{Accuracy Fairness Tradeoff Comparison Across DRO Variants of Different Survival Models}
\label{sec:fairness-accuracy-tradeoff}

\begin{figure*}[t!]
\captionsetup[subfigure]{justification=centering}
\centering
        \begin{subfigure}[b]{0.325\textwidth}
                \includegraphics[width=\linewidth]{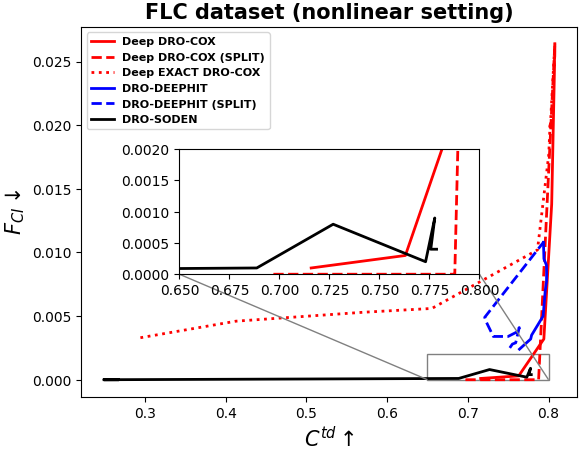}
        \end{subfigure}%
        \begin{subfigure}[b]{0.33\textwidth}
                \includegraphics[width=\linewidth]{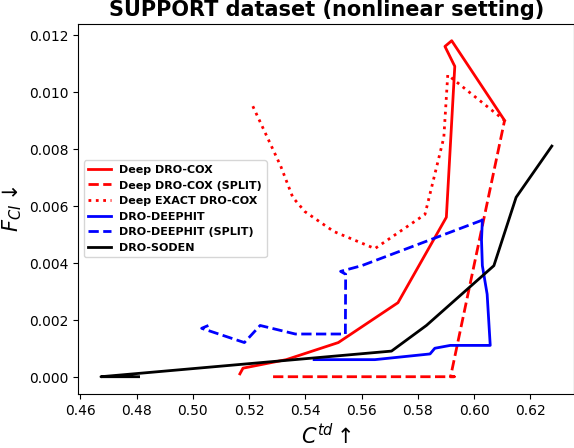}
        \end{subfigure}%
        \begin{subfigure}[b]{0.33\textwidth}
                \includegraphics[width=\linewidth]{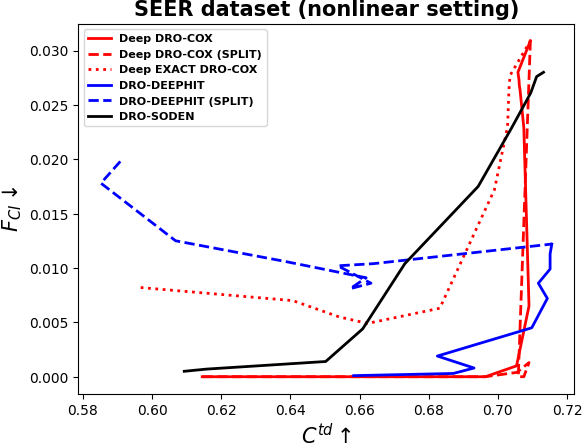}
        \end{subfigure}%
        
        \begin{subfigure}[b]{0.33\textwidth}            \includegraphics[width=\linewidth]{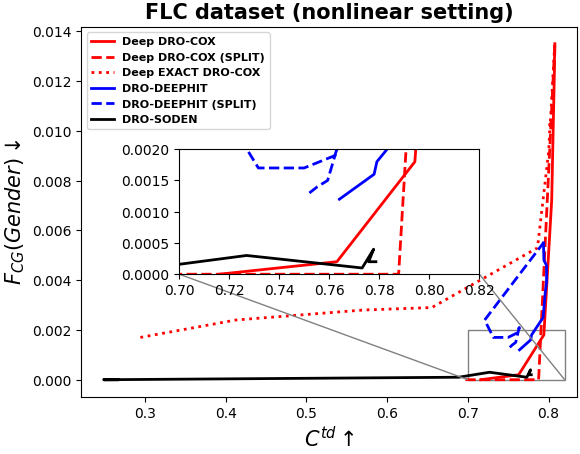}
        \end{subfigure}%
        \begin{subfigure}[b]{0.33\textwidth}
                \includegraphics[width=\linewidth]{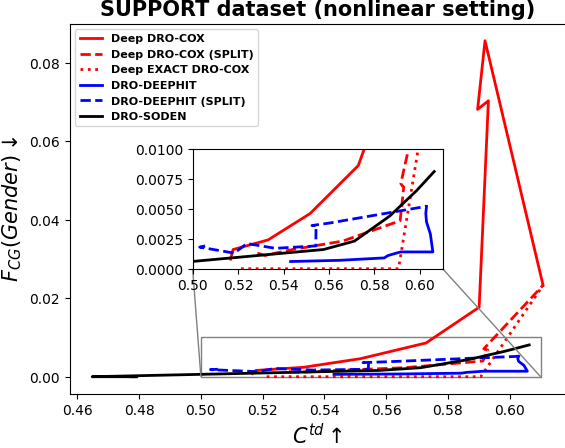}
        \end{subfigure}%
        \begin{subfigure}[b]{0.33\textwidth}
                \includegraphics[width=\linewidth]{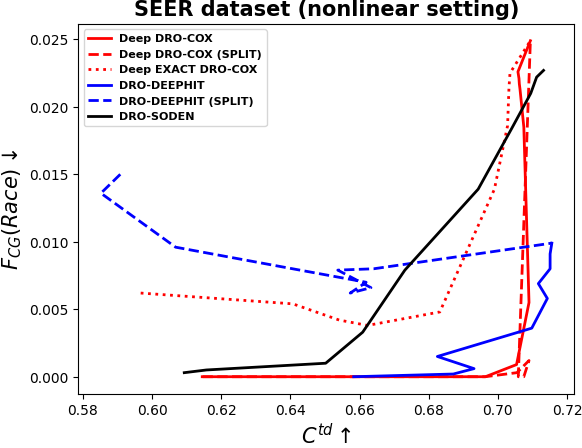}
        \end{subfigure}%
        \vspace{-1em}
        \caption{Comparison of all proposed fairness methods in terms of F$_{CI}$ (first row) and F$_{CG}$ (second row) with C$^{td}$ on FLC, SUPPORT, and SEER datasets. Each line is drawn based on the various values of $\alpha$ (from 0.1 to 1.0). In each subfigure, the closer the curve is to the lower right corner, the better the performance.}\label{fig:all_methods_comparison}
\end{figure*}

We can compare the tradeoff between accuracy and tradeoff across different DRO variants. Specifically, for the \textsc{deep dro-cox}, \textsc{deep dro-cox (split)}, \textsc{dro-deephit}, \textsc{dro-deephit (split)}, and \textsc{dro-soden} models, we test them under the nonlinear setting on FLC, SUPPORT, and SEER datasets. We evaluate the F$_{CI}$, F$_{CG}$, and C$^{td}$ scores of all methods using different values of $\alpha$ from 0.1 to 1.0 and then plot accuracy vs fairness curves for each dataset, as shown in Figure \ref{fig:all_methods_comparison}. Note that in each plot, being closer to the lower right is considered better, corresponding to a model having an $\alpha$ value that achieves as low of a fairness metric score (either F$_{CI}$ or F$_{CG}$) as possible (which is considered more fair) and as high of a C$^{td}$ score as possible. From the figure, we find that the \textsc{dro-deephit} method outperforms the other methods.

\section{Discussion} \label{sec:conclusion}

We have shown a general strategy for converting a wide class of survival models into DRO variants that encourage fairness. The key idea is to write the overall loss in terms of individual losses, which in turn could be used in a DRO framework. When there is coupling so that an individual loss technically is not ``individual'' as it depends on multiple data points, we introduced a sample splitting approach that is compliant with DRO theory. We also showed that the heuristic approach that ignores this coupling problem and naively runs an existing DRO algorithm (that assumes that there is no coupling) works in practice about as well as the sample splitting version. When a survival model used does not have this coupling issue (such as SODEN or the Cox model when using the full Cox loss as in Section~\ref{sec:dro-cox-exact}), then existing DRO machinery directly works; there is no need to use any sample splitting. Specifically for the Cox model, we derive an exact DRO approach that does not use any sample splitting, where the trick is to lift the problem to a higher dimensional (in terms of the number of survival model parameters) space where the coupling issue vanishes.

We now discuss some extensions of our work as well as open questions.

\paragraph{Competing risks}
In various time-to-event prediction problems, we aim to predict the time until the earliest of multiple competing events happen along with which such event happens. For instance, consider hospitalized patients who are in a coma. We may be interested in predicting their time until awakening. However, it could be that they die before awakening. Thus, the two critical events that compete as to which happens first is awakening vs death. Meanwhile, by the time we stop collecting training data, some patients could still be in a coma (so that their outcome is censored). Such a setting is referred to as a \emph{competing risks} problem (see, for instance, Chapter~8 of the textbook by \citet{kalbfleisch2002statistical}).

Our DRO conversion framework can easily accommodate the competing risks setting. To illustrate this, consider the DeepHit model \citep{lee2018deephit} that was originally designed to handle competing risks (and that we actually simplified in our exposition in Section \ref{sec:deephit} and Example~\ref{ex:deephit} to be for the standard survival analysis setting). Suppose that there are $\delta_{\max}\in\mathbb{N}$ competing events, which we simply label as the events $1, 2, \dots, \delta_{\max}$. Now each training point still is represented by the triple $(X_i,Y_i,\Delta_i)$ but $Y_i$ is the time until the earliest critical event happens (or the censoring time if censoring happened prior to any critical event happening), and $\Delta_i\in\{0,1,\dots,\delta_{\max}\}$ indicates which critical event happened first (with the special value of 0 meaning that censoring happened first).

Then in general, DeepHit aims to estimate the so-called \emph{cumulative incidence function} (CIF) \citep{gray1988class,fine1999proportional} that is specific to each event $\delta\in[\delta_{\max}]$:
\[
F_{\delta}(t|x)
\triangleq \mathbb{P}( Y \le t, \Delta = \delta \mid X = x).
\]
To estimate the CIF, DeepHit uses a user-specified discrete time grid $t_1<t_2<\cdots<t_m$. Letting random variable $T$ be the time until the earliest event happens for a data point with feature vector $X$, and letting random variable $\Delta$ indicate which of the critical events is the earliest to happen (also for feature vector $X$), we define
\begin{equation}
\mathbb{P}(T = t_j, \Delta = \delta \mid X = x)
\triangleq f_{\delta, j}(x; \theta)
\quad\text{for }\delta\in[\delta_{\max}]\text{ and }j\in[m],
\label{eq:competing-risk-pmf}
\end{equation}
where neural network
\begin{align*}
f(x;\theta) = \big(
& f_{1,1}(x;\theta), f_{1,2}(x;\theta), \dots, f_{1,m}(x;\theta), \\
& f_{2,1}(x;\theta), f_{2,2}(x;\theta), \dots, f_{2,m}(x;\theta), \\
& \dots, \\
& f_{\delta_{\max},1}(x;\theta), f_{\delta_{\max},2}(x;\theta), \dots, f_{\delta_{\max},m}(x;\theta) \big) \in [0,1]^{\delta_{\max} \cdot m}
\end{align*}
has parameters $\theta$ and maps a raw input $x\in\mathcal{X}$ to a probability distribution over $\delta_{\max} \cdot m$ entries. Note that when there is only one critical event of interest (so that $\delta_{\max}=1$), then equation~\eqref{eq:competing-risk-pmf} reduces to equation~\eqref{eq:deephit-model}).

In particular, DeepHit's estimate of the CIF is given by
\[
\widehat{F}_{\delta}(t_\ell|x)
\triangleq \sum_{j=1}^\ell f_{\delta,j}(x; \theta)
\qquad\text{for }\delta\in[\delta_{\max}]\text{ and }\ell\in[m].
\]
Moreover, the loss function of DeepHit in this case is
\[
L^\text{DeepHit-general}(\theta)
\triangleq \frac{1}{n}\sum_{i=1}^n L_i^{\text{DeepHit-general}}(\theta),
\]
where the $i$-th individual loss is
\begin{align*}
&L^{\text{DeepHit-general}}_i(\theta) \nonumber\\
&=
     \beta\cdot
     \bigg[
       - \mathbf{1}\{\Delta_i \ne 0\}\log(f_{\Delta_i,\kappa(Y_i)}(X_i;\theta))
       - \mathbf{1}\{\Delta_i = 0\} \log\bigg( 1 - \sum_{\delta=1}^{\delta_{\max}}\sum_{\ell=1}^{\kappa(Y_i)}f_{\delta,\ell}(X_i;\theta) \bigg)
     \bigg]
\nonumber \\
&\quad + (1-\beta)\cdot
    \frac{1}{n} \cdot \mathbf{1}\{\Delta_i \ne 0\} \sum_{\substack{j\in[n]\text{ s.t.}\\\kappa(Y_j) > \kappa(Y_i)}}
     \exp\bigg(\frac{\sum_{\ell=1}^{\kappa(Y_i)}[f_{\Delta_i,\ell}(X_j;\theta) - f_{\Delta_i,\ell}(X_i;\theta)]}{\sigma}\bigg).
\end{align*}
Note that the model hyperparameters are the same as what we had presented earlier in Section~\ref{sec:deephit}, with the only minor difference being that in practice, for the ranking loss, it could be helpful to weight the contributions of different competing events differently (i.e., for the $i$-th point to have a ranking loss contribution, it needs to have $\Delta_i\ne0$, in which case we multiply by a scalar weight hyperparameter specific to the event type $\Delta_i\in[\delta_{\max}]$).

In particular, because we can write the loss function as the average of individual loss terms, we can convert this model into a DRO variant using either the heuristic or sample splitting approaches we presented in Section~\ref{sec:dro-split}.

\paragraph{Tuning subpopulation probability threshold $\bm{\alpha}$}
In using DRO, tuning the subpopulation probability threshold $\alpha$ can significantly impact the results. In our experiments, we tuned $\alpha$ using one of two different fairness metrics (CI or F$_{CG}$) on a validation set. While DRO (whether heuristic or using sample splitting) itself does not require the user to specify which features to treat as sensitive in the training loss, we are effectively using some information about which features to treat as sensitive as it shows up in computing the validation set fairness metric. We do this primarily because this is how other researchers have tuned hyperparameters for fair survival models. An open question thus arises of whether we could tune $\alpha$ in some other way in practice that either does not use a validation set or which does not require using a validation set fairness metric that knows which features to treat as sensitive. Fundamentally this is about coming up with other fairness evaluation metrics that can be used for the validation set.

\paragraph{Choosing ``optimal'' splits}
For simplicity, in how we presented our split DRO approach, we used 2-fold cross-fitting, where a key step is randomly splitting the training data into the two sets $\mathcal{D}_1$ and $\mathcal{D}_2$ (as a reminder, we explain how to use more than 2 folds in Appendix~\ref{sec:k-fold-cross-fitting}). Furthermore, our main theoretical guarantee for split DRO (Theorem~\ref{thm:main-result}) assumes that $|\mathcal{D}_1|=|\mathcal{D}_2|$. We defer identifying an ``optimal'' split to future work. Some interesting problems that arise include coming up with some notion of optimality of a split, and figuring out the number of folds that would be optimal. It could even be that instead of randomly splitting the training data, there could be some better non-random optimization-based approach for data splitting.

\paragraph{Impact of DRO on evaluation metrics that are not about fairness}
Lastly, we point out that it would be interesting to empirically study how converting a survival model into its DRO variant impacts other metrics aside from the accuracy or fairness metrics we considered, such as calibration metrics \citep{haider2020effective, goldstein2020x}. Ultimately, we suspect that DRO variants of survival models potentially have interesting properties that make them useful beyond encouraging fairness.

\acks{Shu Hu is supported by the U.S.~National Science Foundation (NSF) CRII award IIS-2434967, the National Artificial Intelligence Research Resource (NAIRR) Pilot, and TACC Lonestar6. George H.~Chen is supported by NSF CAREER award \#2047981. The views, opinions, and/or findings expressed are those of the authors and should not be interpreted as representing the official views or policies of the NSF or the NAIRR Pilot.}

\newpage

\FloatBarrier
\appendix
\counterwithin{figure}{section}
\counterwithin{table}{section}
\counterwithin{equation}{section}

\section{More Details on Cox Models}

\subsection{Estimating the Baseline Hazard and Conditional Survival Function}
\label{sec:breslow}

After learning the log partial hazard function $f(\cdot;\theta)$ (or, equivalently, learning the parameters $\theta$), a standard approach to estimating the baseline hazard function $h_0$ is to use the so-called Breslow method \citep{breslow1972discussion}. In what follows, we use $\widehat{\theta}$ to denote the learned estimate of $\theta$.

The Breslow method estimates a discretized version of $h_0$. Specifically, let $t_1 < t_2 < \cdots < t_m$ denote the unique times when critical event happened in the training data. Let $d_j$ denote the number of critical events that occurred at time $t_j$ for $j\in[m]$. Then we compute the following estimate of $h_0$ at the $j$-th time step:
\[
\widehat{h}_{0,j}
\triangleq \frac{d_j}{\sum_{i\in[n]\text{ s.t.~}Y_i \ge t_j} \exp(f(x_i;\widehat{\theta}))}
\qquad\text{for }j\in[m].
\]
After estimating the baseline hazard function, estimating the survival function is straightforward. Recall that $S(t|x)=\exp\big(-\int_0^t h(u|x)\textrm{d}u\big)$. Then combining this equation with the proportional hazards assumption (i.e., the factorization in equation~\eqref{eq:hazard-factorization}), we get
\begin{equation}
S(t|x)
= \exp\Big(-\int_0^t h_0(u)\exp(f(x;\theta)) \textrm{d}u \Big) 
= \exp\Big(\Big[-\!\!\!\!\!\!\underbrace{\int_0^t h_0(u)\textrm{d}u}_{\text{abbreviate as }H_0(t)}\!\!\!\!\!\!\Big] \exp(f(x;\theta)) \Big).
\label{eq:how-to-estimate-S-helper}
\end{equation}
We can estimate $H_0(t)$ via a summation in place of an integration:
\[
\widehat{H}_0(t) \triangleq \sum_{j\in[m]\text{ s.t.~}t_j \le t} \widehat{h}_{0,j}
\qquad\text{for }t\ge0.
\]
Thus, by plugging in $\widehat{H}_0$ in place of $H_0$ and $\widehat{\theta}$ in place of $\theta$ in equation~\eqref{eq:how-to-estimate-S-helper}, we obtain the conditional survival function estimate $\widehat{S}(t|x) \triangleq \exp(-\widehat{H}_0(t)\exp(f(x;\widehat{\theta})))$.

\subsection{The Proportional Hazards Assumption and the Shape of the Conditional Survival Function}
\label{sec:proportional-hazards-survival-curve-shape}

The proportional hazards assumption constrains the shape of the conditional survival function. Recall that for any two real numbers $a,b\in\mathbb{R}$, we have $\exp(a\cdot b) = (\exp(a))^b$. Then equation~\eqref{eq:how-to-estimate-S-helper} (which was derived using the proportional hazard assumption) is equal to
\[
S(t|x)
= \exp\big( H_0(t) \exp(f(x;\theta)) \big)
= [\underbrace{\exp(H_0(t))}_{\triangleq S_0(t)}]^{\exp(f(x;\theta))}.
\]
In other words, under the proportional hazards assumption, the conditional survival function $S(\cdot|x)$ must necessarily be a power of the so-called baseline survival function $S_0(\cdot)$.

\subsection{Details on the Full Cox Loss}
\label{sec:cox-full-details}

Throughout this section, we assume that we have done the preprocessing stated for Proposition~\ref{prop:cox-full-vs-partial}, namely that for $i\in[n]$ such that $\Delta_i=0$, we have set $Y_i = t_{\kappa(Y_i,0)}$.

\paragraph{Deriving the full Cox loss function}
For a survival model specified in terms of a hazard function, the full likelihood (see, e.g., Section 3.2 of \citet{kalbfleisch2002statistical}) is
\begin{equation}
\text{likelihood}
\triangleq
  \prod_{i=1}^n
    \bigg\{
      [h(Y_{i}|X_{i})]^{\Delta_i}
      \exp\bigg(
        -\int_0^{Y_i} h(u|X_i) du
      \bigg)
    \bigg\}.
\label{eq:full-likelihood}
\end{equation}
For the Cox model, recall that the hazard function is given by
\[
h(t|x)=h_0(t) e^{f(x;\theta)}.
\]
Under the assumption that $h_0$ is piecewise constant, as given in equation~\eqref{eq:hazard-piecewise-constant}, we have
\[
h(t|x)
= \begin{cases}
  e^{\psi_\ell + f(x;\theta)}
  & \text{if }t\in(t_{\ell-1},t_\ell]\text{ for }\ell\in[m], \\
  0 & \text{otherwise}.
  \end{cases}
\]
Plugging this expression for $h(t|x)$ into the full likelihood (equation~\eqref{eq:full-likelihood}), we get
\begin{equation}
\text{likelihood}
=\prod_{i=1}^n
   \bigg\{
     [e^{\psi_\ell + f(X_i;\theta)}]^{\Delta_{i}}\exp\bigg(-e^{f(X_i;\theta)}\sum_{\ell=1}^{\kappa(Y_{i},\Delta_i)}(t_{\ell}-t_{\ell-1})e^{\psi_{\ell}}\bigg)
   \bigg\},
\label{eq:cox-full-likelihood-helper}
\end{equation}
where we have crucially used the preprocessing of the censoring data's observed times in evaluating the integral. (If we did not do the preprocessing, we could still come up with an expression for the integral but the math gets messy.)

Taking the negative log of both sides of equation~\eqref{eq:cox-full-likelihood-helper}, we get:
\[
-\text{log likelihood}
=-\sum_{i=1}^n
  \bigg\{
    \Delta_i[\psi_{\kappa(Y_i,\Delta_i)} + f(X_i;\theta)]
    -
    e^{f(X_i;\theta)}
    \sum_{\ell=1}^{\kappa(Y_i,\Delta_i)}
      (t_{\ell}-t_{\ell-1})e^{\psi_{\ell}}
  \bigg\}
\]
Multiplying both sides by $\frac{1}{n}$, we get the full loss (i.e., equation~\eqref{eq:cox-full}), which we reproduce here for convenience:
\[
L^{\text{Cox-full}}(\theta,\psi)
=\frac{1}{n}
 \sum_{i=1}^n
   \bigg\{
     -
     \Delta_i[\psi_{\kappa(Y_{i},\Delta_i)} + f(X_i;\theta)]
     +
     e^{f(X_i;\theta)}
     \sum_{\ell=1}^{\kappa(Y_i,\Delta_i)}
       (t_{\ell}-t_{\ell-1})e^{\psi_{\ell}}
   \bigg\}.
\]
\paragraph{Proof of Proposition~\ref{prop:cox-full-vs-partial}}
First, we do some re-indexing (to introduce summation over the unique times in which critical events happen):
\begin{align}
&L^{\text{Cox-full}}(\theta,\psi) \nonumber\\
&\!=
  \frac{1}{n}
  \sum_{i=1}^n
    \bigg\{
      -\Delta_i \psi_{\kappa(Y_{i},\Delta_i)}
      -
      \Delta_i
      f(X_i;\theta)
      +
      e^{f(X_i;\theta)}
      \sum_{\ell=1}^{\kappa(Y_{i},\Delta_i)}
        (t_{\ell}-t_{\ell-1})e^{\psi_\ell}
    \bigg\} \nonumber\\
&\!=\!
  -\frac{1}{n}\!\hspace{1pt}
   \sum_{i=1}^n
     \Delta_i \psi_{\kappa(Y_{i},\Delta_i)}
  -
  \frac{1}{n}\!\hspace{1pt}
  \sum_{i=1}^n
    \Delta_i f(X_i;\theta)
  +
  \frac{1}{n}\!\hspace{1pt}
  \sum_{i=1}^n
    e^{f(X_i;\theta)}
    \sum_{\ell=1}^{\kappa(Y_i,\Delta_i)}
      (t_{\ell}-t_{\ell-1})e^{\psi_\ell}
  \nonumber\\
&\!=\!
  -\frac{1}{n}\!\hspace{1pt}
   \sum_{\ell=1}^m
   \underbrace{
     \sum_{j=1}^n
       \Delta_j
       \mathbf{1}\{Y_j=t_\ell\}
   }_{\triangleq d_\ell} \psi_\ell
  -
  \frac{1}{n}\!\hspace{1pt}
  \sum_{i=1}^n
    \Delta_i f(X_i;\theta)
  +
  \frac{1}{n}\!\hspace{1pt}
  \sum_{\ell=1}^m
    (t_{\ell}-t_{\ell-1})
    e^{\psi_\ell}
    \sum_{j=1}^n
      \mathbf{1}\{Y_j\ge t_{\ell}\}
      e^{f(X_j;\theta)}.
\label{eq:cox-full-reindexed}
\end{align}
Taking the derivative of $L^{\text{Cox-full}}(\theta,\psi)$ with respect to $\psi_\ell$ for $\ell\in[m]$, we get
\[
\frac{\partial L^{\text{Cox-full}}(\theta,\psi)}{\partial\psi_\ell}
=-\frac{d_\ell}{n}
 +
 \bigg[
   \frac{1}{n}
   (t_{\ell}-t_{\ell-1})
   \sum_{j=1}^n
     \mathbf{1}\{Y_j\ge\ell\}e^{f(X_j;\theta)}
 \bigg]e^{\psi_\ell}.
\]
By setting this derivative to 0, we get that the optimal value of $\psi_\ell$ is
\[
\widehat{\psi}_\ell
=\log
   \frac{d_\ell}
        {(t_{\ell}-t_{\ell-1})
         \sum_{j=1}^n
           \mathbf{1}\{Y_j\ge t_{\ell}\}
           e^{f(X_j;\theta)}}.
\]
One can verify that indeed $[\frac{\partial^2 L^{\text{Cox-full}}(\theta,\psi)}{\partial\psi_\ell^2}]_{\psi_{\ell}=\widehat{\psi}_{\ell}}>0$ so that this optimal value corresponds to a minimum.

Finally, we plug in $\widehat{\psi}\triangleq(\widehat{\psi}_{1},\dots,\widehat{\psi}_{m})$ in place of $\psi=(\psi_{1},\dots,\psi_{m})$ in $L^{\text{Cox-full}}(\theta,\psi)$ (using equation~\eqref{eq:cox-full-reindexed}):
\begin{align*}
 & L^{\text{Cox-full}}(\theta,\widehat{\psi}) \\
 & =-\frac{1}{n}\sum_{i=1}^{n}\Delta_{i}f(X_i;\theta)-\frac{1}{n}\sum_{\ell=1}^{m}d_{\ell}\widehat{\psi}_{\ell}+\frac{1}{n}\sum_{\ell=1}^{m}(t_{\ell}-t_{\ell-1})e^{\widehat{\psi}_{\ell}}\sum_{j=1}^{n}\mathbf{1}\{Y_{j}\ge t_{\ell}\}e^{f(X_j;\theta)}\\
 & =-\frac{1}{n}\sum_{i=1}^{n}\Delta_{i}f(X_i;\theta)-\frac{1}{n}\sum_{\ell=1}^{m}d_{\ell}\log\frac{d_{\ell}}{(t_{\ell}-t_{\ell-1})\sum_{j=1}^{n}\mathbf{1}\{Y_{j}\ge t_{\ell}\}e^{f(X_j;\theta)}}+\frac{1}{n}\sum_{\ell=1}^{m}d_{\ell}\\
 & =-\frac{1}{n}\sum_{i=1}^{n}\Delta_{i}f(X_i;\theta)-\frac{1}{n}\sum_{\ell=1}^{m}d_{\ell}\bigg[\log\frac{d_{\ell}}{t_{\ell}-t_{\ell-1}}-\log\sum_{j=1}^{n}\mathbf{1}\{Y_{j}\ge t_{\ell}\}e^{f(X_j;\theta)}\bigg]+\frac{1}{n}\sum_{\ell=1}^{m}d_{\ell}\\
 & =-\frac{1}{n}\sum_{i=1}^{n}\Delta_{i}f(X_i;\theta)+\frac{1}{n}\sum_{\ell=1}^{m}d_{\ell}\log\sum_{j=1}^{n}\mathbf{1}\{Y_{j}\ge t_{\ell}\}e^{f(X_j;\theta)}+\underbrace{\text{constant}}_{\text{w.r.t.~}\theta}\\
 & =-\frac{1}{n}\sum_{i=1}^{n}\Delta_{i}f(X_i;\theta)+\frac{1}{n}\sum_{\ell=1}^{m}\bigg[\sum_{i=1}^{n}\Delta_{i}\mathbf{1}\{Y_{i}=t_{\ell}\}\bigg]\log\sum_{j=1}^{n}\mathbf{1}\{Y_{j}\ge t_{\ell}\}e^{f(X_j;\theta)}+\text{constant}\\
 & =-\frac{1}{n}\sum_{i=1}^{n}\Delta_{i}f(X_i;\theta)+\frac{1}{n}\sum_{i=1}^{n}\Delta_{i}\sum_{\ell=1}^{m}\mathbf{1}\{Y_{i}=t_{\ell}\}\log\sum_{j=1}^{n}\mathbf{1}\{Y_{j}\ge t_{\ell}\}e^{f(X_j;\theta)}+\text{constant}\\
 & =-\frac{1}{n}\sum_{i=1}^{n}\Delta_{i}f(X_i;\theta)+\frac{1}{n}\sum_{i=1}^{n}\Delta_{i}\log\sum_{j=1}^{n}\mathbf{1}\{Y_{j}\ge Y_{i}\}e^{f(X_j;\theta)}+\text{constant}\\
 & =\frac{1}{n}\sum_{i=1}^{n}-\Delta_{i}\bigg[f(X_i;\theta)-\log\sum_{j=1}^{n}\mathbf{1}\{Y_{j}\ge Y_{i}\}e^{f(X_j;\theta)}\bigg]+\text{constant}\\
 & =L^{\text{Cox}}(\theta)+\text{constant}. \tag*{$\blacksquare$}
\end{align*}

\section{Cross-Fitting With More Than Two Folds}
\label{sec:k-fold-cross-fitting}

For example, for some pre-specified number of folds $K_{\text{folds}}$, we could randomly partition the training data into $K_{\text{folds}}$ roughly equal-size sets $\mathcal{D}_1,\dots,\mathcal{D}_{K_{\text{folds}}}$. Then for each $k\in[K_{\text{folds}}]$, we could either set
\[
L^{\text{split}}_{\text{DRO}}(\theta, \eta^{(1)}, \dots, \eta^{(K_{\text{folds}})})
\triangleq
  \sum_{k=1}^{K_{\text{folds}}}
    L_{\text{DRO}}^{\text{split}}\big(\theta, \eta^{(k)}, \mathcal{D}_k \mid ([n] \setminus \mathcal{D}_k)\big),
\]
or
\[
L^{\text{split}}_{\text{DRO}}(\theta, \eta^{(1)}, \dots, \eta^{(K_{\text{folds}})})
\triangleq
  \sum_{k=1}^{K_{\text{folds}}}
    L_{\text{DRO}}^{\text{split}}\big(\theta, \eta^{(k)}, ([n] \setminus \mathcal{D}_k) \mid \mathcal{D}_k\big).
\]

\section{Proof of Theorem \ref{thm:main-result}\label{sec:pf-main-result}}

We prove the following.
\begin{proposition}
\label{prop:main-result-more-general}(Slightly more general version of Theorem \ref{thm:main-result}) Let $n\ge2$ and randomly split the training data into $\mathcal{D}_{1}$ and $\mathcal{D}_{2}$ of sizes $n_{1}\ge1$ and $n_{2}=n-n_{1}$. Let $\omega>0$. Suppose that Assumptions A1--A6 hold. If $\phi_{\text{transform}}(s)=s$, then define
\begin{align*}
M & \triangleq M_{\text{indiv}}+n_{2}M_{\text{couple-max}},\\
M' & \triangleq(M_{\text{couple-max}}-M_{\text{couple-min}})\sqrt{\frac{\omega n_{2}}{\zeta}}.
\end{align*}
If instead $\phi_{\text{transform}}(s)=\log(1+s)$, then define
\begin{align*}
M & \triangleq M_{\text{indiv}}+\log(1+n_{2}M_{\text{couple-max}}),\\
M' & \triangleq\frac{(M_{\text{couple-max}}-M_{\text{couple-min}})}{\zeta M_{\text{couple-min}}}\sqrt{\frac{8\omega}{n_{2}}}.
\end{align*}
Then with probability at least
\[
1-2\bigg[\frac{M}{(C_{\alpha}-1)\big[\sqrt{\frac{2\omega}{n_{1}}}\max\{2,\frac{C_{\alpha}}{C_{\alpha}-1}\}M+(2\mathcal{L}+1)M'\big]}+\mathbb{N}(M',\mathcal{X})\bigg]e^{-\omega}-me^{-\frac{n_{2}\zeta}{8}}
\]
over randomness in the training data, we have
\begin{align*}
 & \bigg|\inf_{\eta\in\mathbb{R}}L_{\text{DRO}}^{\text{split}}(\theta,\eta,\mathcal{D}_{1}\mid\mathcal{D}_{2})-\inf_{\eta\in\mathbb{R}}R_{\text{DRO}}^{\text{split}}(\theta,\eta)\bigg|\\
 & \quad\le10C_{\alpha}^{2}\Bigg[\frac{1}{\sqrt{n_{1}}}\max\Big\{2,\frac{C_{\alpha}}{C_{\alpha}-1}\Big\}\big(\sqrt{2\omega}+1\big)M+(2\mathcal{L}+1)M'\Bigg].
\end{align*}
\end{proposition}

Theorem \ref{thm:main-result} corresponds to Proposition \ref{prop:main-result-more-general}, where we assume $n$ to be even and we set $n_{1}=n_{2}=n/2$.

We define
\[
L_{\text{DRO}}^{\text{split},*}(\theta,\eta)\triangleq C_{\alpha}\sqrt{\mathbb{E}_{(X,Y,\Delta)\sim\mathbb{P}}\big[[L_{\text{indiv}}(\theta;X,Y,\Delta)-\eta]_{+}^{2}\big]}+\eta.
\]
The main goal in the proof is to bound
\begin{align*}
 & |L_{\text{DRO}}^{\text{split}}(\theta,\eta,\mathcal{D}_{1}\mid\mathcal{D}_{2})-R_{\text{DRO}}^{\text{split}}(\theta,\eta)|\\
 & \quad=|L_{\text{DRO}}^{\text{split}}(\theta,\eta,\mathcal{D}_{1}\mid\mathcal{D}_{2})-\mathbb{E}[L_{\text{DRO}}^{\text{split}}(\theta,\eta,\mathcal{D}_{1}\mid\mathcal{D}_{2})]\\
 & \quad\quad\;+\mathbb{E}[L_{\text{DRO}}^{\text{split}}(\theta,\eta,\mathcal{D}_{1}\mid\mathcal{D}_{2})]-L_{\text{DRO}}^{\text{split},*}(\theta,\eta,\mathcal{D}_{1}\mid\mathcal{D}_{2})\\
 & \quad\quad\;+L_{\text{DRO}}^{\text{split},*}(\theta,\eta,\mathcal{D}_{1}\mid\mathcal{D}_{2})-R_{\text{DRO}}^{\text{split}}(\theta,\eta)\big|\\
 & \quad\le\underbrace{|L_{\text{DRO}}^{\text{split}}(\theta,\eta,\mathcal{D}_{1}\mid\mathcal{D}_{2})-\mathbb{E}[L_{\text{DRO}}^{\text{split}}(\theta,\eta,\mathcal{D}_{1}\mid\mathcal{D}_{2})]\big|}_{\triangleq\spadesuit}\\
 & \quad\quad\;+\underbrace{\big|\mathbb{E}[L_{\text{DRO}}^{\text{split}}(\theta,\eta,\mathcal{D}_{1}\mid\mathcal{D}_{2})]-L_{\text{DRO}}^{\text{split},*}(\theta,\eta,\mathcal{D}_{1}\mid\mathcal{D}_{2})\big|}_{\triangleq\heartsuit}\\
 & \quad\quad\;+\underbrace{\big|L_{\text{DRO}}^{\text{split},*}(\theta,\eta,\mathcal{D}_{1}\mid\mathcal{D}_{2})-R_{\text{DRO}}^{\text{split}}(\theta,\eta)\big|}_{\triangleq\clubsuit},
\end{align*}
where we have used the triangle inequality. The bulk of the proof is in upper-bounding $\spadesuit$, $\heartsuit$, and $\clubsuit$. Prior to bounding these, we collect two important lemmas. The first establishes that $L_{\text{indiv}}(\theta;x,y,\delta)$ and $R_{\text{indiv}}(\theta;x,y,\Delta)$ are bounded. Note that we defer all proofs of lemmas to subsections immediately following this main proof outline.
\begin{lemma}
\label{lem:L-indiv-R-indiv-bounded}Under Assumption A4, for all $(x,y,\delta)\in\mathcal{Z}$, if $\phi_{\text{transform}}$ is the identity function, then
\[
L_{\text{indiv}}(\theta;x,y,\delta),R_{\text{indiv}}(\theta;x,y,\delta)\in[0,M_{\text{indiv}}+n_{2}M_{\text{couple-max}}].
\]
Otherwise if $\phi_{\text{transform}}(s)=\log(1+s)$, then
\[
L_{\text{indiv}}(\theta;x,y,\delta),R_{\text{indiv}}(\theta;x,y,\delta)\in[0,M_{\text{indiv}}+\log(1+n_{2}M_{\text{couple-max}})].
\]
\end{lemma}

In fact, $M$ is defined in Proposition \ref{prop:main-result-more-general} precisely based on the upper bounds in Lemma \ref{lem:L-indiv-R-indiv-bounded}.

The next lemma says that even though we are optimizing over $\eta\in\mathbb{R}$, we actually only need to consider $\eta$ within a closed interval that depends on $M$.
\begin{lemma}
\label{lem:eta-limited}(Slight variant of Lemma 9 of \citet{duchi2021learning}) We have
\[
\inf_{\eta\in\mathbb{R}}L_{\text{DRO}}^{\text{split}}(\theta,\eta,\mathcal{D}_{1}\mid\mathcal{D}_{2})=\inf_{\eta\in[-\frac{1}{C_{\alpha}-1}M,M]}L_{\text{DRO}}^{\text{split}}(\theta,\eta,\mathcal{D}_{1}\mid\mathcal{D}_{2}),
\]
and similarly
\[
\inf_{\eta\in\mathbb{R}}R_{\text{DRO}}^{\text{split}}(\theta,\eta)=\inf_{\eta\in[-\frac{1}{C_{\alpha}-1}M,M]}R_{\text{DRO}}^{\text{split}}(\theta,\eta).
\]
\end{lemma}

Now we present the bounds on $\spadesuit$, $\heartsuit$, and $\clubsuit$ in three successive lemmas.
\begin{lemma}
\label{lem:spadesuit}(Bound on $\spadesuit$; appears in the proof of Theorem 2 of \citet{duchi2021learning}) Let $\omega>0$. We have
\[
\mathbb{P}\Bigg(\underbrace{\spadesuit\ge C_{\alpha}\max\Big\{2,\frac{C_{\alpha}}{C_{\alpha}-1}\Big\} M\sqrt{\frac{2\omega}{n_{1}}}}_{\mathcal{E}_{\text{bad spade}}(\eta)}\Bigg)\le2e^{-\omega}.
\]
Note that for this probabilistic bound, the event $\mathcal{E}_{\text{bad spade}}(\eta)$, as the notation suggests, depends on $\eta$ as $\spadesuit$ depends on $\eta$ (later we union bound $\mathcal{E}_{\text{bad spade}}(\eta)$ over a finite choice of options of $\eta$).
\end{lemma}

The statement of this lemma depends on Lemma \ref{lem:L-indiv-R-indiv-bounded} (since the constant $M$ shows up), and the proof crucially uses the fact that from Lemma \ref{lem:eta-limited}, we know that it suffices to only consider $\eta\in[-\frac{1}{C_{\alpha}-1}M,M]$.
\begin{lemma}
\label{lem:heartsuit}(Bound on $\heartsuit$; appears in the proof of Theorem 2 of \citet{duchi2021learning}) We have
\[
\heartsuit\le C_{\alpha}\sqrt{\max\Big\{2,\frac{C_{\alpha}}{C_{\alpha}-1}\Big\} M}\cdot\frac{1}{\sqrt{n_{1}}}.
\]
\end{lemma}

Once again, the statement of this lemma depends on Lemma \ref{lem:L-indiv-R-indiv-bounded} (due to the constant $M$ showing up), and the proof again uses the fact that $\eta\in[-\frac{1}{C_{\alpha}-1}M,M]$.
\begin{lemma}
\label{lem:clubsuit}(Bound on $\clubsuit$) Under Assumptions A1, A2, and A4, we have
\begin{align*}
\clubsuit & \le C_{\alpha}\sup_{(x,y,\delta)\in\mathcal{Z}}|L_{\text{indiv}}(\theta;x,y,\delta)-R_{\text{indiv}}(\theta;x,y,\delta)|\\
 & =C_{\alpha}\underbrace{\max_{(x,y,\delta)\in\mathcal{Z}}|L_{\text{indiv}}(\theta;x,y,\delta)-R_{\text{indiv}}(\theta;x,y,\delta)|}_{\triangleq\diamondsuit}.
\end{align*}
Importantly, the supremum is attained by a specific point in the set $\mathcal{Z}\triangleq\mathcal{X}\times\{t_{1},t_{2},\dots,t_{m}\}\times\{0,1\}$.
\end{lemma}

To help upper-bound $\diamondsuit$, we make use of the following lemma.
\begin{lemma}
\label{lem:enough-points-for-every-time}(Enough data points in $\mathcal{D}_{2}$ for every time index) Define the bad event
\begin{align*}
&\mathcal{E}_{\text{bad time}} \\
& \triangleq\bigcup_{\ell=1}^{m}\Big\{\text{the number of points in }\{(X_{i},Y_{i},\Delta_{i})\}_{i\in\mathcal{D}_{2}}\text{ with observed time equal to }t_{\ell}\text{ is}\le\frac{n_{2}\zeta}{2}\Big\}\\
 & =\bigcup_{\ell=1}^{m}\Big\{\sum_{i\in\mathcal{D}_{2}}\mathbf{1}\{Y_{i}=t_{\ell}\}\le\frac{n_{2}\zeta}{2}\Big\}.
\end{align*}
Under Assumption A2,
\[
\mathbb{P}(\mathcal{E}_{\text{bad time}})\le me^{-\frac{n_{2}\zeta}{8}}.
\]
 
\end{lemma}

The reason that Lemma \ref{lem:enough-points-for-every-time} is helpful is that it ensures that the adjacency sets we get are large enough. Specifically, note that by Assumption A3, we use
\[
\mathcal{A}^{*}((x,y,\delta),\mathcal{C})=\begin{cases}
\emptyset & \text{if }\delta=0,\\
\{(x',y',\delta')\in\mathcal{C}:\kappa(y')\ge\kappa(y) & \text{otherwise}.
\end{cases}
\]
In particular, for any time index $t_{\ell}$ for $\ell\in[m]$, the set
\begin{align*}
\mathcal{A}^{*}((x,t_{\ell},1),\{(X_{i},Y_{i},\Delta_{i})\}_{i\in\mathcal{D}_{2}}) & =\{(X_{i},Y_{i},\Delta_{i}):i\in\mathcal{D}_{2}\text{ and }\kappa(Y_{i})\ge\ell\}\\
 & =\{(X_{i},Y_{i},\Delta_{i}):i\in\mathcal{D}_{2}\text{ and }Y_{i}\ge t_{\ell}\},
\end{align*}
has cardinality
\[
\sum_{\widetilde{\ell}=\ell}^{m}\sum_{i\in\mathcal{D}_{2}}\mathbf{1}\{Y_{i}=t_{\widetilde{\ell}}\}\ge\sum_{i\in\mathcal{D}_{2}}\mathbf{1}\{Y_{i}=t_{\ell}\}>\frac{n_{2}\zeta}{2},
\]
where the strict inequality occurs when $\mathcal{E}_{\text{bad time}}$does not happen.

Before we bound $\diamondsuit$ from Lemma \ref{lem:clubsuit}, we collect one more lemma.
\begin{lemma}
\label{lem:L-indiv-vs-R-indiv}Let $\omega>0$. Let $(x,y,\delta)\in\mathcal{Z}$. Under Assumptions A1---A4, when $\delta=0$, trivially
\[
|L_{\text{indiv}}(\theta;x,y,\delta)-R_{\text{indiv}}(\theta;x,y,\delta)|=0.
\]
Otherwise, suppose that the bad event $\mathcal{E}_{\text{bad time}}$ in Lemma \ref{lem:enough-points-for-every-time} does not happen:
\begin{itemize}
\item If $\phi_{\text{transform}}(s)=s$, then
\[
\mathbb{P}\bigg(|L_{\text{indiv}}(\theta;x,y,\delta)-R_{\text{indiv}}(\theta;x,y,\delta)|\ge(M_{\text{couple-max}}-M_{\text{couple-min}})\sqrt{\frac{\omega n_{2}}{\zeta}}\bigg)\le2e^{-\omega}.
\]
\item If $\phi_{\text{transform}}(s)=\log(1+s)$, then
\[
\mathbb{P}\bigg(|L_{\text{indiv}}(\theta;x,y,\delta)-R_{\text{indiv}}(\theta;x,y,\delta)|\ge\frac{(M_{\text{couple-max}}-M_{\text{couple-min}})}{\zeta M_{\text{couple-min}}}\sqrt{\frac{8\omega}{n_{2}}}\bigg)\le2e^{-\omega}.
\]
\end{itemize}
\end{lemma}

Note that $M'$ from Proposition \ref{prop:main-result-more-general} is precisely defined based on the bounds in Lemma \ref{lem:L-indiv-vs-R-indiv}. Moreover, we now define the bad event based on Lemma \ref{lem:L-indiv-vs-R-indiv}:
\[
\mathcal{E}_{\text{bad couples}}(x,y,\delta)\triangleq\big\{|L_{\text{indiv}}(\theta;x,y,\delta)-R_{\text{indiv}}(\theta;x,y,\delta)|\ge M'\big\}.
\]
Note that this bad event depends on $(x,y,\delta)$. By compactness of $\mathcal{X}$ (Assumption A1) and the time grid being finite (Assumption A2), let $\mathcal{Q}$ be an $M'$-cover of minimal size for $\mathcal{X}$ in Euclidean norm (so that $|\mathcal{Q}|=\mathbb{N}(M',\mathcal{X})$); denote the elements of $\mathcal{Q}$ by $q_{1},q_{2},\dots,q_{\mathbb{N}(M',\mathcal{X})}$, and for $x\in\mathcal{X}$, let $j(x)\in[\mathbb{N}(\varepsilon,\mathcal{X})]$ be the index of the closest point (in Euclidean distance) from $\mathcal{Q}$ to $x$ (breaking ties arbitrarily). Then we shall union bound over $\mathcal{E}_{\text{bad couples}}(x,y,\delta)$ for all $x\in\mathcal{Q}$, $y\in\{t_{1},\dots,t_{m}\}$, and $\delta\in\{0,1\}$. Importantly, by ensuring that $\mathcal{E}_{\text{bad couples}}(x,y,\delta)$ holds at all these coordinates means that
\begin{align*}
\diamondsuit & =\max_{(x,y,\delta)}|L_{\text{indiv}}(\theta;x,y,\delta)-R_{\text{indiv}}(\theta;x,y,\delta)|\\
 & =\max_{(x,y,\delta)}\big\{|L_{\text{indiv}}(\theta;x,y,\delta)-L_{\text{indiv}}(\theta;q_{j(x)},y,\delta)|\\
 & \quad\quad\quad\;+|L_{\text{indiv}}(\theta;q_{j(x)},y,\delta)-R_{\text{indiv}}(\theta;q_{j(x)},y,\delta)|\\
 & \quad\quad\quad\;+|R_{\text{indiv}}(\theta;q_{j(x)},y,\delta)-R_{\text{indiv}}(\theta;x,y,\delta)|\big\}\\
 & \le\max_{(x,y,\delta)}\{\mathcal{L}M'+M'+\mathcal{L}M'\}\\
 & =(2\mathcal{L}+1)M',
\end{align*}
where the inequality uses the coordinate-based Lipschitz continuity of $L^{*}$ (and thus also $L_{\text{indiv}}$ and $R_{\text{indiv}}$) to obtain the two different $\mathcal{L}M'$ terms, whereas the $M'$ term comes from the bound from Lemma \ref{lem:L-indiv-vs-R-indiv}.

At this point, when the bad events of Lemma \ref{lem:spadesuit} (this bad event depends on $\eta$), Lemma \ref{lem:enough-points-for-every-time}, and Lemma \ref{lem:L-indiv-vs-R-indiv} (this bad event depends on $(x,y,\delta)$) do not happen,
\begin{align*}
 & |L_{\text{DRO}}^{\text{split}}(\theta,\eta,\mathcal{D}_{1}\mid\mathcal{D}_{2})-R_{\text{DRO}}^{\text{split}}(\theta,\eta)|\\
 & \quad\le\spadesuit+\heartsuit+\clubsuit\\
 & \quad\le C_{\alpha}\max\Big\{2,\frac{C_{\alpha}}{C_{\alpha}-1}\Big\} M\sqrt{\frac{2\omega}{n_{1}}}+C_{\alpha}\sqrt{\max\Big\{2,\frac{C_{\alpha}}{C_{\alpha}-1}\Big\} M}\cdot\frac{1}{\sqrt{n_{1}}}+C_{\alpha}(2\mathcal{L}+1)M'\\
 & \quad\le\underbrace{C_{\alpha}\Bigg[\frac{1}{\sqrt{n_{1}}}\max\Big\{2,\frac{C_{\alpha}}{C_{\alpha}-1}\Big\}\big(\sqrt{2\omega}+1\big)M+(2\mathcal{L}+1)M'\Bigg]}_{\triangleq\varepsilon_{\omega,n_{1},n_{2}}}.
\end{align*}
We now apply yet another covering argument. For positive integers $i\le\frac{C_{\alpha}}{C_{\alpha}-1}\frac{M}{\varepsilon_{\omega,n_{1},n_{2}}}$, we define
\[
\eta_{i}\triangleq-\frac{1}{C_{\alpha}-1}M+i\varepsilon_{\omega,n_{1},n_{2}}.
\]
By construction, $\{\eta_{1},\eta_{2}\dots\}$ forms an $\varepsilon_{\omega,n_{1},n_{2}}$-cover of $[-\frac{1}{C_{\alpha}-1}M,M]$, meaning that for any $\eta\in[-\frac{1}{C_{\alpha}-1}M,M]$, there exists an integer $i(\eta)\in[1,\frac{C_{\alpha}}{C_{\alpha}-1}\frac{M}{\varepsilon_{\omega,n_{1},n_{2}}}]$ such that $|\eta-\eta_{i(\eta)}|\le\varepsilon_{\omega,n_{1},n_{2}}$. The size of this $\varepsilon_{\omega,n_{1},n_{2}}$-cover is bounded above by $\frac{C_{\alpha}}{C_{\alpha}-1}\frac{M}{\varepsilon_{\omega,n_{1},n_{2}}}$. We have
\begin{align*}
 & \sup_{\eta\in[-\frac{1}{C_{\alpha}-1}M,M]}|L_{\text{DRO}}^{\text{split}}(\theta,\eta,\mathcal{D}_{1}\mid\mathcal{D}_{2})-R_{\text{DRO}}^{\text{split}}(\theta,\eta)|\\
 & \le\sup_{\eta\in[-\frac{1}{C_{\alpha}-1}M,M]}\big|L_{\text{DRO}}^{\text{split}}(\theta,\eta,\mathcal{D}_{1}\mid\mathcal{D}_{2})-L_{\text{DRO}}^{\text{split}}(\theta,\eta_{i(\eta)},\mathcal{D}_{1}\mid\mathcal{D}_{2})\\
 & \quad\phantom{\sup_{\eta\in[-\frac{1}{C_{\alpha}-1}M,M]}|}~+L_{\text{DRO}}^{\text{split}}(\theta,\eta_{i(\eta)},\mathcal{D}_{1}\mid\mathcal{D}_{2})-R_{\text{DRO}}^{\text{split}}(\theta,\eta_{i(\eta)})+R_{\text{DRO}}^{\text{split}}(\theta,\eta_{i(\eta)})-R_{\text{DRO}}^{\text{split}}(\theta,\eta)\big|\\
 & \le\sup_{\eta\in[-\frac{1}{C_{\alpha}-1}M,M]}\Big\{\big|L_{\text{DRO}}^{\text{split}}(\theta,\eta,\mathcal{D}_{1}\mid\mathcal{D}_{2})-L_{\text{DRO}}^{\text{split}}(\theta,\eta_{i(\eta)},\mathcal{D}_{1}\mid\mathcal{D}_{2})\big|\\
 & \quad\phantom{\sup_{\eta\in[-\frac{1}{C_{\alpha}-1}M,M]}\Big\{|}~+\big|L_{\text{DRO}}^{\text{split}}(\theta,\eta_{i(\eta)},\mathcal{D}_{1}\mid\mathcal{D}_{2})-R_{\text{DRO}}^{\text{split}}(\theta,\eta_{i(\eta)})\big|+\big|R_{\text{DRO}}^{\text{split}}(\theta,\eta_{i(\eta)})-R_{\text{DRO}}^{\text{split}}(\theta,\eta)\big|\Big\}\\
 & \le\sup_{i\in[1,\frac{C_{\alpha}}{C_{\alpha}-1}\frac{M}{\varepsilon_{\omega,n_{1},n_{2}}}]}\big\{(1+C_{\alpha})\varepsilon_{\omega,n_{1},n_{2}}+\big|L_{\text{DRO}}^{\text{split}}(\theta,\eta_{i},\mathcal{D}_{1}\mid\mathcal{D}_{2})-R_{\text{DRO}}^{\text{split}}(\theta,\eta_{i(\eta)})\big|+(1+C_{\alpha})\varepsilon_{\omega,n_{1},n_{2}}\big\},
\end{align*}
where the first and third terms inside the supremum objective have been bounded in the last step using the fact that $\eta\mapsto L_{\text{DRO}}^{\text{split}}(\theta,\eta,\mathcal{D}_{1}\mid\mathcal{D}_{2})$ and $\eta\mapsto R_{\text{DRO}}^{\text{split}}(\theta,\eta)$ are each $(1+C_{\alpha})$-Lipschitz. Meanwhile, the second term in the objective is upper-bounded by $\varepsilon_{\omega,n_{1},n_{2}}$ when none of the bad events happen, where for the bad event of Lemma \ref{lem:spadesuit} (the probabilistic bound for $\spadesuit$) we now have to union bound over it not happening across all the points $\eta_{1},\eta_{2},\dots$ in the $\varepsilon_{\omega,n_{1},n_{2}}$-cover. We thus conclude that with probability at least
\begin{align*}
 & 1-\underbrace{\frac{C_{\alpha}}{C_{\alpha}-1}\frac{M}{\varepsilon_{\omega,n_{1},n_{2}}}}_{\substack{\text{upper bound on size of}\\
\varepsilon_{\omega,n_{1},n_{2}}\text{-cover}
}
}\cdot\underbrace{2e^{-\omega}}_{\text{from Lemma \ref{lem:spadesuit}}}-\underbrace{me^{-\frac{n_{2}\zeta}{8}}}_{\text{from Lemma \ref{lem:enough-points-for-every-time}}}-\mathbb{N}(M',\mathcal{X})\underbrace{2e^{-\omega}}_{\text{from Lemma \ref{lem:L-indiv-vs-R-indiv}}}\\
 & \ge1-2\Big[\frac{C_{\alpha}}{C_{\alpha}-1}\frac{M}{\varepsilon_{\omega,n_{1},n_{2}}}+\mathbb{N}(M',\mathcal{X})\Big]e^{-\omega}-me^{-\frac{n_{2}\zeta}{8}}\\
 & =1-2\bigg[\frac{C_{\alpha}}{C_{\alpha}-1}\cdot\frac{M}{C_{\alpha}\big[\frac{1}{\sqrt{n_{1}}}\max\{2,\frac{C_{\alpha}}{C_{\alpha}-1}\}\big(\sqrt{2\omega}+1\big)M+(2\mathcal{L}+1)M'\big]}+\mathbb{N}(M',\mathcal{X})\bigg]e^{-\omega}-me^{-\frac{n_{2}\zeta}{8}}\\
 & =1-2\bigg[\frac{M}{(C_{\alpha}-1)\big[\frac{1}{\sqrt{n_{1}}}\max\{2,\frac{C_{\alpha}}{C_{\alpha}-1}\}\big(\sqrt{2\omega}+1\big)M+(2\mathcal{L}+1)M'\big]}+\mathbb{N}(M',\mathcal{X})\bigg]e^{-\omega}-me^{-\frac{n_{2}\zeta}{8}}\\
 & \ge1-2\bigg[\frac{M}{(C_{\alpha}-1)\big[\sqrt{\frac{2\omega}{n_{1}}}\max\{2,\frac{C_{\alpha}}{C_{\alpha}-1}\}M+(2\mathcal{L}+1)M'\big]}+\mathbb{N}(M',\mathcal{X})\bigg]e^{-\omega}-me^{-\frac{n_{2}\zeta}{8}},
\end{align*}
we have
\begin{align*}
 & \sup_{\eta\in[-\frac{1}{C_{\alpha}-1}M,M]}|L_{\text{DRO}}^{\text{split}}(\theta,\eta,\mathcal{D}_{1}\mid\mathcal{D}_{2})-R_{\text{DRO}}^{\text{split}}(\theta,\eta)|\\
 & \quad\le2(1+C_{\alpha})\varepsilon_{\omega,n_{1},n_{2}}+\varepsilon_{\omega,n_{1},n_{2}}\\
 & \quad=(3+2C_{\alpha})\varepsilon_{\omega,n_{1},n_{2}}\\
 & \quad=(3+2C_{\alpha})C_{\alpha}\Bigg[\frac{1}{\sqrt{n_{1}}}\max\Big\{2,\frac{C_{\alpha}}{C_{\alpha}-1}\Big\}\big(\sqrt{2\omega}+1\big)M+(2\mathcal{L}+1)M'\Bigg]\\
 & \quad\le10C_{\alpha}^{2}\Bigg[\frac{1}{\sqrt{n_{1}}}\max\Big\{2,\frac{C_{\alpha}}{C_{\alpha}-1}\Big\}\big(\sqrt{2\omega}+1\big)M+(2\mathcal{L}+1)M'\Bigg].
\end{align*}
At this point, using Lemma \ref{lem:eta-limited}, we have
\begin{align*}
 & \bigg|\inf_{\eta\in\mathbb{R}}L_{\text{DRO}}^{\text{split}}(\theta,\eta,\mathcal{D}_{1}\mid\mathcal{D}_{2})-\inf_{\eta\in\mathbb{R}}R_{\text{DRO}}^{\text{split}}(\theta,\eta)\bigg|\\
 & \quad=\bigg|\inf_{\eta\in[-\frac{1}{C_{\alpha}-1}M,M]}L_{\text{DRO}}^{\text{split}}(\theta,\eta,\mathcal{D}_{1}\mid\mathcal{D}_{2})-\inf_{\eta\in[-\frac{1}{C_{\alpha}-1}M,M]}R_{\text{DRO}}^{\text{split}}(\theta,\eta)\bigg|\\
 & \quad\le\sup_{\eta\in[-\frac{1}{C_{\alpha}-1}M,M]}|L_{\text{DRO}}^{\text{split}}(\theta,\eta,\mathcal{D}_{1}\mid\mathcal{D}_{2})-R_{\text{DRO}}^{\text{split}}(\theta,\eta)|\\
 & \quad\le10C_{\alpha}^{2}\Bigg[\frac{1}{\sqrt{n_{1}}}\max\Big\{2,\frac{C_{\alpha}}{C_{\alpha}-1}\Big\}\big(\sqrt{2\omega}+1\big)M+(2\mathcal{L}+1)M'\Bigg],
\end{align*}
where the first inequality is a standard result that holds when optimizing over bounded functions. This finishes the proof of Proposition \ref{prop:main-result-more-general}.$\hfill\blacksquare$%

\subsection{Proof of Lemma \ref{lem:L-indiv-R-indiv-bounded}}

Under Assumption A4, when $\phi_{\text{transform}}$ is the identity function,
\begin{align*}
L_{\text{indiv}}(\theta;x,y,\delta) & =L^{*}\big((x,y,\delta),\mathcal{A}^{*}\big((x,y,\delta),\{(X_{j},Y_{j},\Delta_{j}):j\in\mathcal{D}_{2}\}\big);\theta\big)\\
 & \le M_{\text{indiv}}+|\mathcal{D}_{2}|M_{\text{couple-max}}\\
 & =M_{\text{indiv}}+n_{2}M_{\text{couple-max}}.
\end{align*}
Assumption A4 trivially also implies that $L_{\text{indiv}}(\theta;x,y,\delta)\ge0$. Hence, $L_{\text{indiv}}(\theta;x,y,\delta)\in[0,M_{\text{indiv}}+n_{2}M_{\text{couple-max}}]$. By a similar argument, $R_{\text{indiv}}(\theta;x,y,\delta)\in[0,M_{\text{indiv}}+n_{2}M_{\text{couple-max}}]$.

If instead $\phi_{\text{transform}}(s)=\log(1+s)$, then
\begin{align*}
L_{\text{indiv}}(\theta;x,y,\delta) & =L^{*}\big((x,y,\delta),\mathcal{A}^{*}\big((x,y,\delta),\{(X_{j},Y_{j},\Delta_{j}):j\in\mathcal{D}_{2}\}\big);\theta\big)\\
 & \le M_{\text{indiv}}+\log(1+|\mathcal{D}_{2}|M_{\text{couple-max}})\\
 & =M_{\text{indiv}}+\log(1+n_{2}M_{\text{couple-max}}).
\end{align*}
Again, we have $L_{\text{indiv}}(\theta;x,y,\delta)\ge0$, so $L_{\text{indiv}}(\theta;x,y,\delta)\in[0,M_{\text{indiv}}+\log(1+n_{2}M_{\text{couple-max}})]$. Similarly, $R_{\text{indiv}}(\theta;x,y,\delta)\in[0,M_{\text{indiv}}+\log(1+n_{2}M_{\text{couple-max}})]$.$\hfill\blacksquare$

\subsection{Proof of Lemma \ref{lem:eta-limited}}

Using Proposition~\ref{prop:sample-split}, we have
\[
L_{\text{DRO}}^{\text{split}}(\theta,\eta,\mathcal{D}_{1}\mid\mathcal{D}_{2})=C_{\alpha}\sqrt{\frac{1}{n_{1}}\sum_{i\in\mathcal{D}_{1}}[L_{\text{indiv}}(\theta;X_{i},Y_{i},\Delta_{i})-\eta]_{+}^{2}}+\eta.
\]
Since $L_{\text{indiv}}(\theta;X_{i},Y_{i},\Delta_{i})\in[0,M]$ (from Lemma \ref{lem:L-indiv-R-indiv-bounded}), this means that when $\eta\ge M$, we have $[L_{\text{indiv}}(\theta;X_{i},Y_{i},\Delta_{i})-\eta]_{+}=0$ for all $i\in\mathcal{D}_{1}$ in which case $L_{\text{DRO}}^{\text{split}}(\theta,\eta,\mathcal{D}_{1}\mid\mathcal{D}_{2})=\eta$.

Meanwhile,
\begin{align*}
L_{\text{DRO}}^{\text{split}}\Big(\theta,-\frac{1}{C_{\alpha}-1}M,\mathcal{D}_{1}\mid\mathcal{D}_{2}\Big) & =C_{\alpha}\sqrt{\frac{1}{n_{1}}\sum_{i\in\mathcal{D}_{1}}[L_{\text{indiv}}(\theta;X_{i},Y_{i},\Delta_{i})+\frac{1}{C_{\alpha}-1}M]_{+}^{2}}-\frac{1}{C_{\alpha}-1}M\\
 & \ge C_{\alpha}\sqrt{\frac{1}{n_{1}}\sum_{i\in\mathcal{D}_{1}}[0+\frac{1}{C_{\alpha}-1}M]_{+}^{2}}-\frac{1}{C_{\alpha}-1}M\\
 & =\frac{C_{\alpha}}{C_{\alpha}-1}M-\frac{1}{C_{\alpha}-1}M\\
 & =M.
\end{align*}
Since $\eta\mapsto L_{\text{DRO}}^{\text{split}}(\theta,\eta,\mathcal{D}_{1}\mid\mathcal{D}_{2})$ is convex, and $L_{\text{DRO}}^{\text{split}}(\theta,-\frac{1}{C_{\alpha}-1}M,\mathcal{D}_{1}\mid\mathcal{D}_{2})=M$ and $L_{\text{DRO}}^{\text{split}}(\theta,\eta,\mathcal{D}_{1}\mid\mathcal{D}_{2})=\eta$ for all $\eta\ge M$, then it must be that
\[
\inf_{\eta\in\mathbb{R}}L_{\text{DRO}}^{\text{split}}(\theta,\eta,\mathcal{D}_{1}\mid\mathcal{D}_{2})=\inf_{\eta\in[-\frac{1}{C_{\alpha}-1}M,M]}L_{\text{DRO}}^{\text{split}}(\theta,\eta,\mathcal{D}_{1}\mid\mathcal{D}_{2}).
\]
Using the same reasoning,
\[
\inf_{\eta\in\mathbb{R}}R_{\text{DRO}}^{\text{split}}(\theta,\eta)=\inf_{\eta\in[-\frac{1}{C_{\alpha}-1}M,M]}R_{\text{DRO}}^{\text{split}}(\theta,\eta).\hfill\tag*{\ensuremath{\blacksquare}}
\]

\subsection{Proof of Lemma \ref{lem:spadesuit}}

We define
\[
\Xi_{i}\triangleq[L_{\text{indiv}}(\theta;X_{i},Y_{i},\Delta_{i})-\eta]_{+}\quad\text{for }i\in\mathcal{D}_{1}.
\]
As a consequence of Lemma \ref{lem:eta-limited}, it suffices to only consider $\eta\in[-\frac{1}{C_{\alpha}-1}M,M]$. Hence,
\begin{align*}
\Xi_{i} & =[L_{\text{indiv}}(\theta;X_{i},Y_{i},\Delta_{i})-\eta]_{+}\\
 & \le|L_{\text{indiv}}(\theta;X_{i},Y_{i},\Delta_{i})-\eta|\\
 & \le|L_{\text{indiv}}(\theta;X_{i},Y_{i},\Delta_{i})|+|\eta|\\
 & \le M+|\eta|\\
 & \le M+\max\Big\{\frac{1}{C_{\alpha}-1}M,M\Big\}\\
 & =\max\Big\{\frac{1}{C_{\alpha}-1}M+M,2M\Big\}\\
 & =\max\Big\{\frac{C_{\alpha}}{C_{\alpha}-1}M,2M\Big\}.
\end{align*}
Meanwhile, trivially $\Xi_{i}\ge0$, so
\begin{equation}
\Xi_{i}\in\Big[0,\max\Big\{2,\frac{C_{\alpha}}{C_{\alpha}-1}\Big\} M\Big].\label{eq:Xi-bound}
\end{equation}
Next, by Lemma 7 of \citet{duchi2021learning}, $L_{\text{DRO}}^{\text{split}}(\theta,\eta,\mathcal{D}_{1}\mid\mathcal{D}_{2})$ is $\frac{C_{\alpha}}{\sqrt{n_{1}}}$ Lipschitz with respect to the vector $(\Xi_{i})_{i\in\mathcal{D}_{1}}$ in Euclidean norm. Then by Lemma 6 of \citet{duchi2021learning}, for any $\widetilde{\omega}>0$,
\begin{align*}
 & \mathbb{P}\big(|L_{\text{DRO}}^{\text{split}}(\theta,\eta,\mathcal{D}_{1}\mid\mathcal{D}_{2})-\mathbb{E}[L_{\text{DRO}}^{\text{split}}(\theta,\eta,\mathcal{D}_{1}\mid\mathcal{D}_{2})]|\ge\widetilde{\omega}\big)\\
 & \quad\le2\exp\bigg(-\frac{\widetilde{\omega}{}^{2}n_{1}}{2C_{\alpha}^{2}\big(\max\{2,\frac{C_{\alpha}}{C_{\alpha}-1}\}M\big)^{2}}\bigg).
\end{align*}
We do a change of variables. Let $\omega>0$. Plugging in
\[
\widetilde{\omega}=C_{\alpha}\max\Big\{2,\frac{C_{\alpha}}{C_{\alpha}-1}\Big\} M\sqrt{\frac{2\omega}{n_{1}}},
\]
we get that
\begin{align*}
 & \mathbb{P}\Bigg(\overbrace{|L_{\text{DRO}}^{\text{split}}(\theta,\eta,\mathcal{D}_{1}\mid\mathcal{D}_{2})-\mathbb{E}[L_{\text{DRO}}^{\text{split}}(\theta,\eta,\mathcal{D}_{1}\mid\mathcal{D}_{2})]|}^{\spadesuit}\ge C_{\alpha}\max\Big\{2,\frac{C_{\alpha}}{C_{\alpha}-1}\Big\} M\sqrt{\frac{2\omega}{n_{1}}}\Bigg)\\
 & \quad\le2e^{-\omega}.\hfill\tag*{\ensuremath{\blacksquare}}
\end{align*}

\subsection{Proof of Lemma \ref{lem:heartsuit}}

Recall from bound (\ref{eq:Xi-bound}) that for $i\in\mathcal{D}_{1}$, the variable $\Xi_{i}=[L_{\text{indiv}}(\theta;X_{i},Y_{i},\Delta_{i})-\eta]_{+}$ satisfies
\[
\Xi_{i}\in\Big[0,\max\Big\{2,\frac{C_{\alpha}}{C_{\alpha}-1}\Big\} M\Big].
\]
Thus, we trivially have 
\[
\mathbb{E}[|\Xi_{i}|^{4}]\le\bigg(\max\Big\{2,\frac{C_{\alpha}}{C_{\alpha}-1}\Big\} M\bigg)^{2}\mathbb{E}[|\Xi_{i}|^{2}],
\]
which means that applying Lemma 8 of \citet{duchi2021learning}, we get
\[
\mathbb{E}\Bigg[\sqrt{\frac{1}{n_{1}}\sum_{i\in\mathcal{D}_{1}}|\Xi_{i}|^{2}}\Bigg]\ge\sqrt{\mathbb{E}[|\Xi_{i}|^{2}]}-\sqrt{\max\Big\{2,\frac{C_{\alpha}}{C_{\alpha}-1}\Big\} M}\cdot\frac{1}{\sqrt{n_{1}}}.
\]
This means that
\begin{align*}
 & L_{\text{DRO}}^{\text{split},*}(\theta,\eta,\mathcal{D}_{1}\mid\mathcal{D}_{2})-\mathbb{E}[L_{\text{DRO}}^{\text{split}}(\theta,\eta,\mathcal{D}_{1}\mid\mathcal{D}_{2})]\\
 & \quad=C_{\alpha}\Bigg(\sqrt{\mathbb{E}[|\Xi_{i}|^{2}]}-\mathbb{E}\Bigg[\sqrt{\frac{1}{n_{1}}\sum_{i\in\mathcal{D}_{1}}|\Xi_{i}|^{2}}\Bigg]\Bigg)\\
 & \quad\le C_{\alpha}\sqrt{\max\Big\{2,\frac{C_{\alpha}}{C_{\alpha}-1}\Big\} M}\cdot\frac{1}{\sqrt{n_{1}}}.
\end{align*}
Separately, by Jensen's inequality,
\begin{align*}
 & \mathbb{E}[L_{\text{DRO}}^{\text{split}}(\theta,\eta,\mathcal{D}_{1}\mid\mathcal{D}_{2})]-L_{\text{DRO}}^{\text{split},*}(\theta,\eta,\mathcal{D}_{1}\mid\mathcal{D}_{2})\\
 & \quad=C_{\alpha}\Bigg(\mathbb{E}\Bigg[\sqrt{\frac{1}{n_{1}}\sum_{i\in\mathcal{D}_{1}}|\Xi_{i}|^{2}}\Bigg]-\sqrt{\mathbb{E}[|\Xi_{i}|^{2}]}\Bigg)\\
 & \quad\le C_{\alpha}\Bigg(\sqrt{\mathbb{E}\bigg[\frac{1}{n_{1}}\sum_{i\in\mathcal{D}_{1}}|\Xi_{i}|^{2}\bigg]}-\sqrt{\mathbb{E}[|\Xi_{i}|^{2}]}\Bigg)\\
 & \quad=C_{\alpha}\Big(\sqrt{\mathbb{E}[|\Xi_{i}|^{2}]}-\sqrt{\mathbb{E}[|\Xi_{i}|^{2}]}\Big)\\
 & \quad=0.
\end{align*}
Hence,
\begin{align*}
\heartsuit & =\big|\mathbb{E}[L_{\text{DRO}}^{\text{split}}(\theta,\eta,\mathcal{D}_{1}\mid\mathcal{D}_{2})]-L_{\text{DRO}}^{\text{split},*}(\theta,\eta,\mathcal{D}_{1}\mid\mathcal{D}_{2})\big|\\
 & \le C_{\alpha}\sqrt{\max\Big\{2,\frac{C_{\alpha}}{C_{\alpha}-1}\Big\} M}\cdot\frac{1}{\sqrt{n_{1}}}.\hfill\tag*{\ensuremath{\blacksquare}}
\end{align*}

\subsection{Proof of Lemma \ref{lem:clubsuit}\label{subsec:lem-clubsuit-pf}}

First off, under Assumptions A1, A2, and A4, we have
\begin{align*}
&\sup_{(x,y,\delta)\in\mathcal{Z}}|L_{\text{indiv}}(\theta;x,y,\delta)-R_{\text{indiv}}(\theta;x,y,\delta)|\\
 & \quad=\max_{(x,y,\delta)\in\mathcal{Z}}|L_{\text{indiv}}(\theta;x,y,\delta)-R_{\text{indiv}}(\theta;x,y,\delta)|\\
 & \quad=\max_{y\in\{t_{1},\dots,t_{m}\}}\max_{\delta\in\{0,1\}}\max_{x\in\mathcal{X}}|L_{\text{indiv}}(\theta;x,y,\delta)-R_{\text{indiv}}(\theta;x,y,\delta)|.
\end{align*}
The reason the supremum is attained (so that it is equal to the max) is because not only do we discretize time to a finite grid (Assumption A2) so that maximizing over the $m$ values of $y$ and the 2 values of $\delta$ does not present any issues in the supremum being attained, we further assume that $\mathcal{X}$ is compact (Assumption A1) and $L^{*}((x,y,\delta),\mathcal{C};\theta)$ is continuous in the coordinate $x$ with respect to Euclidean norm (Assumption A4(c)), which ensures that the supremum over $x\in\mathcal{X}$ is equal to the max over $x\in\mathcal{X}$.

Next, using the fact that for any $a,b,c,d\in\mathbb{R}$, we have $\max\{a+b,c+d\}\le\max\{a,c\}+\max\{b,d\}$ (which implies that $\max\{a+b,0\}\le\max\{a,0\}+\max\{b,0\}=[a]_{+}+[b]_{+}$),
\begin{align*}
 & \mathbb{E}_{(X,Y,\Delta)\sim\mathbb{P}}\big[[L_{\text{indiv}}(\theta;X,Y,\Delta)-\eta]_{+}^{2}\big]\\
 & \quad=\mathbb{E}_{(X,Y,\Delta)\sim\mathbb{P}}\big[[L_{\text{indiv}}(\theta;X,Y,\Delta)-R_{\text{indiv}}(\theta;X,Y,\Delta)+R_{\text{indiv}}(\theta;X,Y,\Delta)-\eta]_{+}^{2}\big]\\
 & \quad\le\mathbb{E}_{(X,Y,\Delta)\sim\mathbb{P}}\big[\big([L_{\text{indiv}}(\theta;X,Y,\Delta)-R_{\text{indiv}}(\theta;X,Y,\Delta)]_{+}+[R_{\text{indiv}}(\theta;X,Y,\Delta)-\eta]_{+}\big)^{2}\big].
\end{align*}
Taking the square root of both sides, we get
\begin{align}
 & \sqrt{\mathbb{E}_{(X,Y,\Delta)\sim\mathbb{P}}\big[[L_{\text{indiv}}(\theta;X,Y,\Delta)-\eta]_{+}^{2}\big]}\nonumber \\
 & \quad\le\sqrt{\mathbb{E}_{(X,Y,\Delta)\sim\mathbb{P}}\big[\big([L_{\text{indiv}}(\theta;X,Y,\Delta)-R_{\text{indiv}}(\theta;X,Y,\Delta)]_{+}+[R_{\text{indiv}}(\theta;X,Y,\Delta)-\eta]_{+}\big)^{2}\big]}.\label{eq:clubsuit-helper1}
\end{align}
Applying Minkowski's inequality,
\begin{align}
 & \sqrt{\mathbb{E}_{(X,Y,\Delta)\sim\mathbb{P}}\big[\big([L_{\text{indiv}}(\theta;X,Y,\Delta)-R_{\text{indiv}}(\theta;X,Y,\Delta)]_{+}+[R_{\text{indiv}}(\theta;X,Y,\Delta)-\eta]_{+}\big)^{2}\big]}\nonumber \\
 & \quad\le\sqrt{\mathbb{E}_{(X,Y,\Delta)\sim\mathbb{P}}\big[[L_{\text{indiv}}(\theta;X,Y,\Delta)-R_{\text{indiv}}(\theta;X,Y,\Delta)]_{+}^{2}\big]}\nonumber \\
 &\quad\quad+\sqrt{\mathbb{E}_{(X,Y,\Delta)\sim\mathbb{P}}\big[[R_{\text{indiv}}(\theta;X,Y,\Delta)-\eta]_{+}^{2}\big]}.\label{eq:clubsuit-helper2}
\end{align}
Next, we have
\begin{align}
 & \sqrt{\mathbb{E}_{(X,Y,\Delta)\sim\mathbb{P}}\big[[L_{\text{indiv}}(\theta;X,Y,\Delta)-R_{\text{indiv}}(\theta;X,Y,\Delta)]_{+}^{2}\big]}\nonumber \\
 & \quad\le\sqrt{\mathbb{E}_{(X,Y,\Delta)\sim\mathbb{P}}\big[(L_{\text{indiv}}(\theta;X,Y,\Delta)-R_{\text{indiv}}(\theta;X,Y,\Delta))^{2}\big]}\nonumber \\
 & \quad\le\sqrt{\mathbb{E}_{(X,Y,\Delta)\sim\mathbb{P}}\big[\max_{(x,y,\delta)\in\mathcal{Z}}(L_{\text{indiv}}(\theta;x,y,\delta)-R_{\text{indiv}}(\theta;x,y,\delta))^{2}\big]}\nonumber \\
 & \quad=\sqrt{\max_{(x,y,\delta)\in\mathcal{Z}}(L_{\text{indiv}}(\theta;x,y,\delta)-R_{\text{indiv}}(\theta;x,y,\delta))^{2}}\nonumber \\
 & \quad=\max_{(x,y,\delta)\in\mathcal{Z}}|L_{\text{indiv}}(\theta;x,y,\delta)-R_{\text{indiv}}(\theta;x,y,\delta)|.\label{eq:clubsuit-helper3}
\end{align}
Combining inequalities (\ref{eq:clubsuit-helper1}), (\ref{eq:clubsuit-helper2}), and (\ref{eq:clubsuit-helper3}), we obtain
\begin{align*}
 & \sqrt{\mathbb{E}_{(X,Y,\Delta)\sim\mathbb{P}}\big[[L_{\text{indiv}}(\theta;X,Y,\Delta)-\eta]_{+}^{2}\big]}\\
 & \quad\le\sqrt{\mathbb{E}_{(X,Y,\Delta)\sim\mathbb{P}}\big[[L_{\text{indiv}}(\theta;X,Y,\Delta)-R_{\text{indiv}}(\theta;X,Y,\Delta)]_{+}^{2}\big]}\\
 &\quad\quad+\sqrt{\mathbb{E}_{(X,Y,\Delta)\sim\mathbb{P}}\big[[R_{\text{indiv}}(\theta;X,Y,\Delta)-\eta]_{+}^{2}\big]}\\
 & \quad\le\max_{(x,y,\delta)\in\mathcal{Z}}|L_{\text{indiv}}(\theta;x,y,\delta)-R_{\text{indiv}}(\theta;x,y,\delta)|+\sqrt{\mathbb{E}_{(X,Y,\Delta)\sim\mathbb{P}}\big[[R_{\text{indiv}}(\theta;X,Y,\Delta)-\eta]_{+}^{2}\big]},
\end{align*}
i.e.,
\begin{align}
 & \sqrt{\mathbb{E}_{(X,Y,\Delta)\sim\mathbb{P}}\big[[L_{\text{indiv}}(\theta;X,Y,\Delta)-\eta]_{+}^{2}\big]}-\sqrt{\mathbb{E}_{(X,Y,\Delta)\sim\mathbb{P}}\big[[R_{\text{indiv}}(\theta;X,Y,\Delta)-\eta]_{+}^{2}\big]}\nonumber \\
 & \quad\le\max_{(x,y,\delta)\in\mathcal{Z}}|L_{\text{indiv}}(\theta;x,y,\delta)-R_{\text{indiv}}(\theta;x,y,\delta)|.\label{eq:clubsuit-helper4}
\end{align}
Repeating the same proof ideas but with the roles of $L_{\text{indiv}}$ and $R_{\text{indiv}}$ swapped, we would instead obtain the bound
\begin{align}
 & \sqrt{\mathbb{E}_{(X,Y,\Delta)\sim\mathbb{P}}\big[[R_{\text{indiv}}(\theta;X,Y,\Delta)-\eta]_{+}^{2}\big]}-\sqrt{\mathbb{E}_{(X,Y,\Delta)\sim\mathbb{P}}\big[[L_{\text{indiv}}(\theta;X,Y,\Delta)-\eta]_{+}^{2}\big]}\nonumber \\
 & \quad\le\max_{(x,y,\delta)\in\mathcal{Z}}|L_{\text{indiv}}(\theta;x,y,\delta)-R_{\text{indiv}}(\theta;x,y,\delta)|.\label{eq:clubsuit-helper5}
\end{align}
Thus, inequalities (\ref{eq:clubsuit-helper4}) and (\ref{eq:clubsuit-helper5}) together imply that
\begin{align*}
 & \bigg|\sqrt{\mathbb{E}_{(X,Y,\Delta)\sim\mathbb{P}}\big[[L_{\text{indiv}}(\theta;X,Y,\Delta)-\eta]_{+}^{2}\big]}-\sqrt{\mathbb{E}_{(X,Y,\Delta)\sim\mathbb{P}}\big[[R_{\text{indiv}}(\theta;X,Y,\Delta)-\eta]_{+}^{2}\big]}\bigg|\\
 & \quad\le\max_{(x,y,\delta)\in\mathcal{Z}}|L_{\text{indiv}}(\theta;x,y,\delta)-R_{\text{indiv}}(\theta;x,y,\delta)|.
\end{align*}
Therefore,
\begin{align*}
 & |L_{\text{DRO}}^{\text{split},*}(\theta,\eta)-R_{\text{DRO}}^{\text{split}}(\theta,\eta)|\\
 & \quad=C_{\alpha}\bigg|\sqrt{\mathbb{E}_{(X,Y,\Delta)\sim\mathbb{P}}\big[[L_{\text{indiv}}(\theta;X,Y,\Delta)-\eta]_{+}^{2}\big]}-\sqrt{\mathbb{E}_{(X,Y,\Delta)\sim\mathbb{P}}\big[[R_{\text{indiv}}(\theta;X,Y,\Delta)-\eta]_{+}^{2}\big]}\bigg|\\
 & \quad\le C_{\alpha}\max_{(x,y,\delta)\in\mathcal{Z}}|L_{\text{indiv}}(\theta;x,y,\delta)-R_{\text{indiv}}(\theta;x,y,\delta)|.\hfill\tag*{\ensuremath{\blacksquare}}
\end{align*}

\subsection{Proof of Lemma \ref{lem:enough-points-for-every-time}}

For each time index $\ell\in[m]$, by the multiplicative Chernoff bound and Assumption A2,
\[
\mathbb{P}\bigg(\sum_{i\in\mathcal{D}_{2}}\mathbf{1}\{Y_{i}=t_{\ell}\}\le\frac{1}{2}n_{2}\mathbb{P}(Y=t_{\ell})\bigg)\le e^{-\frac{\mathbb{E}[\sum_{i\in\mathcal{D}_{2}}\mathbf{1}\{Y_{i}=t_{\ell}\}]}{8}}=e^{-\frac{n_{2}\mathbb{P}(Y=t_{\ell})}{8}}\le e^{-\frac{n_{2}\zeta}{8}}.
\]
Note that $\sum_{i\in\mathcal{D}_{2}}\mathbf{1}\{Y_{i}=t_{\ell}\}\le\frac{1}{2}n_{2}\zeta$ implies that $\sum_{i\in\mathcal{D}_{2}}\mathbf{1}\{Y_{i}=t_{\ell}\}\le\frac{1}{2}n_{2}\mathbb{P}(Y=t_{\ell})$ since $\mathbb{P}(Y=t_{\ell})\ge\zeta$ by Assumption A2. This means that
\[
\mathbb{P}\bigg(\sum_{i\in\mathcal{D}_{2}}\mathbf{1}\{Y_{i}=t_{\ell}\}\le\frac{1}{2}n_{2}\zeta\bigg)\le\mathbb{P}\bigg(\sum_{i\in\mathcal{D}_{2}}\mathbf{1}\{Y_{i}=t_{\ell}\}\le\frac{1}{2}n_{2}\mathbb{P}(Y=t_{\ell})\bigg)\le e^{-\frac{n_{2}\zeta}{8}}.
\]
Union-bounding over all $m$ time indices yields the claim.$\hfill\blacksquare$

\subsection{Proof of Lemma \ref{lem:L-indiv-vs-R-indiv}}

Let $(x,y,\delta)\in\mathcal{Z}$. If $\delta=0$, then we obtain the trivial equality
\[
|L_{\text{indiv}}(\theta;x,y,\delta)-R_{\text{indiv}}(\theta;x,y,\delta)|=0,
\]
since by Assumption A4, $L_{\text{indiv}}(\theta;x,y,0)$ has no coupling terms, in which case it is exactly equal to $R_{\text{indiv}}(\theta;x,y,0)$.

For the remainder of this lemma's proof, we assume that $\delta\ne0$. In this case, the adjacency set of $(x,y,\delta)$ could be nonempty. First, we introduce the shorthand notation where $\mathcal{N}_{\mathcal{D}_{2}}\subseteq[n]$ denotes the indices of training data in $\mathcal{D}_{2}$ that are considered adjacent to data point $(x,y,\delta)$, and $\mathcal{N}_{\text{fresh}}\in[n_{2}]$ is analogously defined but for the fresh sample of $n_{2}$ data points (used in the definition of $R_{\text{indiv}}$). Formally,
\begin{align*}
\mathcal{N}_{\mathcal{D}_{2}} & \triangleq\Big\{ i\in\mathcal{D}_{2}:(X_{i},Y_{i},\Delta_{i})\in\mathcal{A}^{*}\big((x,y,\delta),\{(X_{j},Y_{j},\Delta_{j}):j\in\mathcal{D}_{2}\}\big)\Big\},\\
\mathcal{N}_{\text{fresh}} & \triangleq\Big\{ i\in[n_{2}]:(X_{i}',Y_{i}',\Delta_{i}')\in\mathcal{A}^{*}\big((x,y,\delta),\{(X_{j}',Y_{j}',\Delta_{j}'):j\in[n_{2}]\}\big)\Big\}.
\end{align*}
When the event $\mathcal{E}_{\text{bad time}}$ in Lemma \ref{lem:enough-points-for-every-time} does not happen, we are guaranteed that $|\mathcal{N}_{\mathcal{D}_{2}}|>\frac{n_{2}\zeta}{2}>0$ (by the definition of DeepHit's adjacency function, when $\delta=1$, $\mathcal{N}_{\mathcal{D}_{2}}$ would at least contain all points in $\mathcal{D}_{2}$ with the same time index as $y$, for which there are more than $\frac{n_{2}\zeta}{2}$ such data points).

Then when $\phi_{\text{transform}}$ is the identity function,
\begin{align*}
 & |L_{\text{indiv}}(\theta;x,y,\delta)-R_{\text{indiv}}(\theta;x,y,\delta)|\\
 & \quad=\bigg|\phi_{\text{indiv}}((x,y,\delta);\theta)+\sum_{j\in\mathcal{N}_{\mathcal{D}_{2}}^{*}}\phi_{\text{couple}}((x,y,\delta),(X_{j},Y_{j},\Delta_{j});\theta)\\
 & \quad\quad\;-\mathbb{E}_{\{(X_{i}',Y_{i}',\Delta_{i}')\}_{i=1}^{n_{2}}}\bigg[\phi_{\text{indiv}}((x,y,\delta);\theta)+\sum_{j\in\mathcal{N}_{\text{fresh}}^{*}}\phi_{\text{couple}}((x,y,\delta),(X_{j}',Y_{j}',\Delta_{j}');\theta)\bigg]\bigg|\\
 & \quad=\bigg|\sum_{j\in\mathcal{N}_{\mathcal{D}_{2}}^{*}}\phi_{\text{couple}}((x,y,\delta),(X_{j},Y_{j},\Delta_{j});\theta)\\
 & \quad\quad\;-\mathbb{E}_{\{(X_{i}',Y_{i}',\Delta_{i}')\}_{i=1}^{n_{2}}}\bigg[\sum_{j\in\mathcal{N}_{\text{fresh}}^{*}}\phi_{\text{couple}}((x,y,\delta),(X_{j}',Y_{j}',\Delta_{j}');\theta)\bigg]\bigg|.
\end{align*}
The key observation is that by construction, $\sum_{j\in\mathcal{N}_{\mathcal{D}_{2}}^{*}}\phi_{\text{couple}}((x,y,\delta),(X_{j},Y_{j},\Delta_{j});\theta)$ has the same distribution as $\sum_{j\in\mathcal{N}_{\text{fresh}}^{*}}\phi_{\text{couple}}((x,y,\delta),(X_{j}',Y_{j}',\Delta_{j}');\theta)$, so
\begin{align*}
 & \mathbb{E}_{\{(X_{i},Y_{i},\Delta_{i})\}_{i\in\mathcal{D}_{2}}}\bigg[\sum_{j\in\mathcal{N}_{\mathcal{D}_{2}}^{*}}\phi_{\text{couple}}((x,y,\delta),(X_{j},Y_{j},\Delta_{j});\theta)\bigg]\\
 & \quad=\mathbb{E}_{\{(X_{i}',Y_{i}',\Delta_{i}')\}_{i=1}^{n_{2}}}\bigg[\sum_{j\in\mathcal{N}_{\text{fresh}}^{*}}\phi_{\text{couple}}((x,y,\delta),(X_{j}',Y_{j}',\Delta_{j}');\theta)\bigg].
\end{align*}
Hence, denoting
\begin{equation}
\Phi\triangleq\sum_{j\in\mathcal{N}_{\mathcal{D}_{2}}^{*}}\phi_{\text{couple}}((x,y,\delta),(X_{j},Y_{j},\Delta_{j});\theta),\label{eq:Phi}
\end{equation}
we have
\[
|L_{\text{indiv}}(\theta;x,y,\delta)-R_{\text{indiv}}(\theta;x,y,\delta)|=|\Phi-\mathbb{E}[\Phi]|.
\]

When event $\mathcal{E}_{\text{bad time}}$ does not happen, we know that $|\mathcal{N}_{\mathcal{D}_{2}}|\ge\lceil\frac{n_{2}\zeta}{2}\rceil$. This means that $\Phi$ is a nonempty sum of i.i.d.~nonnegative random variables each bounded within $[M_{\text{couple-min}},M_{\text{couple-max}}]$. Then by Hoeffding's inequality,%

{} for any $\widetilde{\omega}>0$,
\begin{align*}
 & \mathbb{P}\Big(|\Phi-\mathbb{E}[\Phi]|\ge\widetilde{\omega}|\mathcal{N}_{\mathcal{D}_{2}}|~\bigg|~|\mathcal{N}_{\mathcal{D}_{2}}|\ge\Big\lceil\frac{n_{2}\zeta}{2}\Big\rceil\Big)\\
 & \quad=\frac{\sum_{\ell=\lceil\frac{n_{2}\zeta}{2}\rceil}^{n_{2}}\mathbb{P}\big(|\Phi-\mathbb{E}[\Phi]|\le\widetilde{\omega}|\mathcal{N}_{\mathcal{D}_{2}}|~\big|~|\mathcal{N}_{\mathcal{D}_{2}}|=\ell\big)\mathbb{P}(|\mathcal{N}_{\mathcal{D}_{2}}|=\ell)}{\mathbb{P}(|\mathcal{N}_{\mathcal{D}_{2}}|\ge\lceil\frac{n_{2}\zeta}{2}\rceil)}\\
 & \quad\le\frac{\sum_{\ell=\lceil\frac{n_{2}\zeta}{2}\rceil}^{n_{2}}2\exp\bigg(-\frac{2(\widetilde{\omega}\ell)^{2}}{\ell(M_{\text{couple-max}}-M_{\text{couple-min}})^{2}}\bigg)\mathbb{P}(|\mathcal{N}_{\mathcal{D}_{2}}|=\ell)}{\mathbb{P}(|\mathcal{N}_{\mathcal{D}_{2}}|\ge\lceil\frac{n_{2}\zeta}{2}\rceil)}\\
 & \quad=\frac{\sum_{\ell=\lceil\frac{n_{2}\zeta}{2}\rceil}^{n_{2}}2\exp\bigg(-\frac{2\widetilde{\omega}^{2}\ell}{(M_{\text{couple-max}}-M_{\text{couple-min}})^{2}}\bigg)\mathbb{P}(|\mathcal{N}_{\mathcal{D}_{2}}|=\ell)}{\mathbb{P}(|\mathcal{N}_{\mathcal{D}_{2}}|\ge\lceil\frac{n_{2}\zeta}{2}\rceil)}\\
 & \quad\le\frac{\sum_{\ell=\lceil\frac{n_{2}\zeta}{2}\rceil}^{n_{2}}2\exp\bigg(-\frac{2\widetilde{\omega}^{2}(\frac{n_{2}\zeta}{2})}{(M_{\text{couple-max}}-M_{\text{couple-min}})^{2}}\bigg)\mathbb{P}(|\mathcal{N}_{\mathcal{D}_{2}}|=\ell)}{\mathbb{P}(|\mathcal{N}_{\mathcal{D}_{2}}|\ge\lceil\frac{n_{2}\zeta}{2}\rceil)}\\
 & \quad=2\exp\bigg(-\frac{\widetilde{\omega}^{2}n_{2}\zeta}{(M_{\text{couple-max}}-M_{\text{couple-min}})^{2}}\bigg).
\end{align*}
Now we do a change of variables. Let $\omega>0$, and set
\[
\widetilde{\omega}=(M_{\text{couple-max}}-M_{\text{couple-min}})\sqrt{\frac{\omega}{\zeta n_{2}}}.
\]
Then we have
\begin{align*}
 & \mathbb{P}\Big(|\Phi-\mathbb{E}[\Phi]|\ge(M_{\text{couple-max}}-M_{\text{couple-min}})\sqrt{\frac{\omega}{\zeta n_{2}}}|\mathcal{N}_{\mathcal{D}_{2}}|~\bigg|~|\mathcal{N}_{\mathcal{D}_{2}}|\ge\Big\lceil\frac{n_{2}\zeta}{2}\Big\rceil\Big)\\
 & \quad\le2e^{-\omega}.
\end{align*}
In summary, when $\mathcal{E}_{\text{bad time}}$ does not happen, with probability at least $1-2e^{-\omega}$, we have
\begin{align*}
|L_{\text{indiv}}(\theta;x,y,\delta)-R_{\text{indiv}}(\theta;x,y,\delta)| & =|\Phi-\mathbb{E}[\Phi]|\\
 & \le(M_{\text{couple-max}}-M_{\text{couple-min}})\sqrt{\frac{\omega}{\zeta n_{2}}}|\mathcal{N}_{\mathcal{D}_{2}}|\\
 & \le(M_{\text{couple-max}}-M_{\text{couple-min}})\sqrt{\frac{\omega}{\zeta n_{2}}}\cdot n_{2}\\
 & =(M_{\text{couple-max}}-M_{\text{couple-min}})\sqrt{\frac{\omega n_{2}}{\zeta}}.
\end{align*}
Now let's consider when instead $\phi_{\text{transform}}(s)=\log(1+s)$, and as a reminder we assume $\delta\ne0$. Then
\begin{align*}
 & |L_{\text{indiv}}(\theta;x,y,\delta)-R_{\text{indiv}}(\theta;x,y,\delta)|\\
 & =\bigg|\phi_{\text{indiv}}((x,y,\delta);\theta)+\log\bigg(1+\sum_{j\in\mathcal{N}_{\mathcal{D}_{2}}}\phi_{\text{couple}}((x,y,\delta),(X_{j},Y_{j},\Delta_{j});\theta)\bigg)\\
 & \quad\;-\mathbb{E}_{\{(X_{i}',Y_{i}',\Delta_{i}')\}_{i=1}^{n_{2}}}\bigg[\phi_{\text{indiv}}((x,y,\delta);\theta)+\delta\log\bigg(1+\sum_{j\in\mathcal{N}_{\text{fresh}}}\phi_{\text{couple}}((x,y,\delta),(X_{j}',Y_{j}',\Delta_{j}');\theta)\bigg)\bigg]\bigg|\\
 & =\bigg|\log\bigg(1+\sum_{j\in\mathcal{N}_{\mathcal{D}_{2}}}\phi_{\text{couple}}((x,y,\delta),(X_{j},Y_{j},\Delta_{j});\theta)\bigg)\\
 & \quad\;-\mathbb{E}_{\{(X_{i}',Y_{i}',\Delta_{i}')\}_{i=1}^{n_{2}}}\bigg[\log\bigg(1+\sum_{j\in\mathcal{N}_{\text{fresh}}}\phi_{\text{couple}}((x,y,\delta),(X_{j}',Y_{j}',\Delta_{j}');\theta)\bigg)\bigg]\bigg|.
\end{align*}
By a similar argument as we used for proving the case where $\phi_{\text{transform}}$ is the identity function, the key observation is that
\begin{align*}
 & \mathbb{E}_{\{(X_{i},Y_{i},\Delta_{i})\}_{i\in\mathcal{D}_{2}}}\bigg[\log\bigg(1+\sum_{j\in\mathcal{N}_{\mathcal{D}_{2}}}\phi_{\text{couple}}((x,y,\delta),(X_{j},Y_{j},\Delta_{j});\theta)\bigg)\bigg]\\
 & \quad=\mathbb{E}_{\{(X_{i}',Y_{i}',\Delta_{i}')\}_{i=1}^{n_{2}}}\bigg[\log\bigg(1+\sum_{j\in\mathcal{N}_{\text{fresh}}}\phi_{\text{couple}}((x,y,\delta),(X_{j}',Y_{j}',\Delta_{j}');\theta)\bigg)\bigg].
\end{align*}
We now define $\Gamma_{j}\triangleq\phi_{\text{couple}}((x,y,\delta),(X_{j},Y_{j},\Delta_{j});\theta)$ for each $j\in\mathcal{N}_{\mathcal{D}_{2}}$. Again, when event $\mathcal{E}_{\text{bad time}}$ does not happen, we are guaranteed that $|\mathcal{N}_{\mathcal{D}_{2}}|\ge\frac{n_{2}\zeta}{2}$, i.e., $\mathcal{N}_{\mathcal{D}_{2}}$ is nonempty. Note that the map $(\Gamma_{j})_{j\in\mathcal{N}_{\mathcal{D}_{2}}}\mapsto\log(1+\sum_{j\in\mathcal{N}_{\mathcal{D}_{2}}}\Gamma_{j})$ is concave. We now show that this map is Lipschitz continuous with respect to the Euclidean norm by showing what a valid Lipschitz constant is for the map. Note that for $i\in\mathcal{N}_{\mathcal{D}_{2}}$,
\[
\frac{\partial\log(1+\sum_{j\in\mathcal{N}_{\mathcal{D}_{2}}}\Gamma_{j})}{\partial\Gamma_{i}}=\frac{1}{1+\sum_{j\in\mathcal{N}_{\mathcal{D}_{2}}}\Gamma_{j}}.
\]
Then
\begin{align*}
\Big\|\nabla\log\Big(1+\sum_{j\in\mathcal{N}_{\mathcal{D}_{2}}}\Gamma_{j}\Big)\Big\|_{2} & =\sqrt{\sum_{i\in\mathcal{N}_{\mathcal{D}_{2}}}\Big(\frac{\partial\log(1+\sum_{j\in\mathcal{N}_{\mathcal{D}_{2}}}\Gamma_{j})}{\partial\Gamma_{i}}\Big)^{2}}\\
 & =\sqrt{\frac{|\mathcal{N}_{\mathcal{D}_{2}}|}{(1+\sum_{j\in\mathcal{N}_{\mathcal{D}_{2}}}\Gamma_{j})^{2}}}\\
 & \le\sqrt{\frac{|\mathcal{N}_{\mathcal{D}_{2}}|}{(\sum_{j\in\mathcal{N}_{\mathcal{D}_{2}}}\Gamma_{j})^{2}}}\\
 & \le\sqrt{\frac{|\mathcal{N}_{\mathcal{D}_{2}}|}{(\frac{n_{2}\zeta}{2}M_{\text{couple-min}})^{2}}}\\
 & \le\sqrt{\frac{n_{2}}{(\frac{n_{2}\zeta}{2}M_{\text{couple-min}})^{2}}}\\
 & =\frac{2}{\zeta M_{\text{couple-min}}}\cdot\frac{1}{\sqrt{n_{2}}}.
\end{align*}
In other words, the map $(\Gamma_{j})_{j\in\mathcal{N}_{\mathcal{D}_{2}}}\mapsto\log(1+\sum_{j\in\mathcal{N}_{\mathcal{D}_{2}}}\Gamma_{j})$ has Lipschitz constant $\frac{2}{\zeta M_{\text{couple-min}}}\cdot\frac{1}{\sqrt{n_{2}}}$ when event $\mathcal{E}_{\text{bad time}}$ does not happen. Then applying Lemma 6 of \citet{duchi2021learning}, for any $\widetilde{\omega}>0$,
\begin{align*}
 & \mathbb{P}\bigg(\bigg|\log\Big(1+\sum_{j\in\mathcal{N}_{\mathcal{D}_{2}}}\Gamma_{j}\Big)-\mathbb{E}\bigg[\log\Big(1+\sum_{j\in\mathcal{N}_{\mathcal{D}_{2}}}\Gamma_{j}\Big)\bigg]\bigg|\ge\widetilde{\omega}\,\bigg|\,|\mathcal{N}_{\mathcal{D}_{2}}|\ge\frac{n_{2}\zeta}{2}\bigg)\\
 & \quad\le2\exp\Big(-\frac{\widetilde{\omega}^{2}}{2(\frac{2}{\zeta M_{\text{couple-min}}}\cdot\frac{1}{\sqrt{n_{2}}})^{2}(M_{\text{couple-max}}-M_{\text{couple-min}})^{2}}\Big)\\
 & \quad=2\exp\Big(-\frac{\widetilde{\omega}^{2}\zeta^{2}M_{\text{couple-min}}^{2}n_{2}}{8(M_{\text{couple-max}}-M_{\text{couple-min}})^{2}}\Big).
\end{align*}
Let $\omega>0$ and set $\widetilde{\omega}=\frac{(M_{\text{couple-max}}-M_{\text{couple-min}})}{\zeta M_{\text{couple-min}}}\sqrt{\frac{8\omega}{n_{2}}}$. Then
\begin{align*}
 & \mathbb{P}\bigg(\overbrace{\bigg|\log\Big(1+\sum_{j\in\mathcal{N}_{\mathcal{D}_{2}}^{*}}\Gamma_{j}\Big)-\mathbb{E}\bigg[\log\Big(1+\sum_{j\in\mathcal{N}_{\mathcal{D}_{2}}^{*}}\Gamma_{j}\Big)\bigg]\bigg|}^{=|L_{\text{indiv}}(\theta;x,y,\delta)-R_{\text{indiv}}(\theta;x,y,\delta)|}\ge\frac{(M_{\text{couple-max}}-M_{\text{couple-min}})}{\zeta M_{\text{couple-min}}}\sqrt{\frac{8\omega}{n_{2}}}\bigg)\\
 & \quad\le2e^{-\omega}.\hfill\tag*{\ensuremath{\blacksquare}}
\end{align*}

\section{Proof of Corollary \ref{cor:linear-cox-split-dro-guarantee}\label{sec:pf-linear-cox-split-dro-guarantee}}

The proof of this corollary consists of two main parts. First, we check that the Assumptions A1--A4 needed by Theorem \ref{thm:main-result} hold. Then we apply Theorem \ref{thm:main-result}, where we impose constraints on $n$ and $d$ so that we can simplify the probability bound in equation (\ref{eq:main-result-prob}).

\paragraph*{Verifying Assumptions A1--A4}

Assumption A1 clearly holds since $\mathcal{X}$ is the unit ball in $\mathbb{R}^{d}$, which is compact. There is no need to check Assumption A2 in that we are directly assuming it. Similarly, Assumption A3 also trivially holds (for discrete time, the adjacency function for the Cox model is the same as for DeepHit).

We proceed to verify Assumption A4. For the Cox model where $f(x;\theta)=\theta^{\top}x$, we have
\begin{align*}
L^{*}((x,y,\delta),\mathcal{C};\theta) & =-\delta\Bigg[\theta^{\top}x-\log\bigg(\exp(\theta^{\top}x)+\sum_{(x',y',\delta')\in\mathcal{C}}\exp(\theta^{\top}x')\bigg)\bigg]\\
 & =-\delta\Bigg[\log\exp(\theta^{\top}x)-\log\bigg(\exp(\theta^{\top}x)+\sum_{(x',y',\delta')\in\mathcal{C}}\exp(\theta^{\top}x')\bigg)\bigg]\\
 & =-\delta\log\bigg(\frac{\exp(\theta^{\top}x)}{\exp(\theta^{\top}x)+\sum_{(x',y',\delta')\in\mathcal{C}}\exp(\theta^{\top}x')}\bigg)\\
 & =\delta\log\bigg(\frac{\exp(\theta^{\top}x)+\sum_{(x',y',\delta')\in\mathcal{C}}\exp(\theta^{\top}x')}{\exp(\theta^{\top}x)}\bigg)\\
 & =\delta\log\Big(1+\sum_{(x',y',\delta')\in\mathcal{C}}\exp\big(\theta^{\top}(x'-x)\big)\Big),
\end{align*}
which corresponds to Assumption A4 where $\phi_{\text{transform}}(s)=\log(1+s)$, $\phi_{\text{indiv}}$ always outputs 0 (so $M_{\text{indiv}}=0)$, and
\[
\phi_{\text{couple}}((x,y,\delta),(x',y',\delta');\theta)=\exp\big(\theta^{\top}(x'-x)\big).
\]
In this case, since $\mathcal{X}$ and $\Theta$ are constrained to be within the unit ball, by the Cauchy-Schwarz inequality,
\[
|\theta^{\top}(x'-x)|\le\underbrace{\|\theta\|_{2}}_{\le1}\underbrace{\|x'-x\|_{2}}_{\substack{\le2\text{ since a unit ball}\\
\text{has diameter 2}
}
}\le2.
\]
In particular,
\[
\theta^{\top}(x'-x)\in[-2,2],
\]
so the largest $\phi_{\text{couple}}((x,y,\delta),(x',y',\delta');\theta)$ can be is
\[
\exp\big(\theta^{\top}(x'-x)\big)\le\exp(2)\triangleq M_{\text{couple-max}},
\]
whereas the smallest is
\[
\exp\big(\theta^{\top}(x'-x)\big)\ge\exp(-2)\triangleq M_{\text{couple-min}}.
\]
Meanwhile, to check that $L^{*}((x,y,\delta),\mathcal{C};\theta)$ satisfies Lipschitz continuity, first note that when $\delta=0$, the $L^{*}((x,y,\delta),\mathcal{C};\theta)=0$, so there is nothing to show. When $\delta\ne0$, we have
\[
L^{*}((x,y,\delta),\mathcal{C};\theta)=\log\Big(1+\sum_{(x',y',\delta')\in\mathcal{C}}\exp\big(\theta^{\top}(x'-x)\big)\Big).
\]
Taking the gradient with respect to $x$, we get
\[
\frac{\partial L^{*}((x,y,\delta),\mathcal{C};\theta)}{\partial x}=-\frac{\sum_{(x',y',\delta')\in\mathcal{C}}\exp\big(\theta^{\top}(x'-x)\big)}{1+\sum_{(x',y',\delta')\in\mathcal{C}}\exp\big(\theta^{\top}(x'-x)\big)}\theta.
\]
Then
\begin{align*}
\Big\|\frac{\partial L^{*}((x,y,\delta),\mathcal{C};\theta)}{\partial x}\Big\|_{2} & =\bigg\|-\frac{\sum_{(x',y',\delta')\in\mathcal{C}}\exp\big(\theta^{\top}(x'-x)\big)}{1+\sum_{(x',y',\delta')\in\mathcal{C}}\exp\big(\theta^{\top}(x'-x)\big)}\theta\bigg\|_{2}\\
 & =\underbrace{\frac{\sum_{(x',y',\delta')\in\mathcal{C}}\exp\big(\theta^{\top}(x'-x)\big)}{1+\sum_{(x',y',\delta')\in\mathcal{C}}\exp\big(\theta^{\top}(x'-x)\big)}}_{\le1\text{ (this is a probability from a softmax calculation)}}\underbrace{\|\theta\|_{2}}_{\le1}\\
 & \le1.
\end{align*}
Thus, $L^{*}((x,y,\delta),\mathcal{C};\theta)$ is 1-Lipschitz (i.e., the constant $\mathcal{L}$ in Assumption A4(c) is 1). At this point we have verified that Assumption A4 holds.

\paragraph*{Applying Theorem \ref{thm:main-result}}

We now apply Theorem \ref{thm:main-result}. In this case, we have
\begin{equation}
M\triangleq\log\Big(1+\frac{e^{2}}{2}n\Big)\qquad\text{and}\qquad M'\triangleq\frac{4(e^{2}-e^{-2})}{\zeta e^{-2}}\sqrt{\frac{\omega}{n}}.\label{eq:linear-cox-dro-split-M-M'}
\end{equation}
By a standard result (see, for instance, Corollary 4.2.13 of \citet{vershynin2018high}), for all $\varepsilon\in(0,1]$,
\begin{equation}
\mathbb{N}(\varepsilon,\mathcal{X})\le\Big(\frac{3}{\varepsilon}\Big)^{d}.\label{eq:standard-unit-ball-covering-bound}
\end{equation}
Since the Theorem \ref{thm:main-result}'s probability bound (\ref{eq:main-result-prob}) depends on $\mathbb{N}(M',\mathcal{X})$, we first verify that $M'\le1$ so that inequality (\ref{eq:standard-unit-ball-covering-bound}) holds. To do this, we make use of the lower branch $W_{-1}$ of the Lambert W function and the standard result that
\begin{equation}
-1-\sqrt{2s}-s<W_{-1}(-e^{-s-1})\quad\text{for}\quad s>0.\label{eq:lambert-w-bound}
\end{equation}
By assumption,
\begin{align*}
n & \ge\Big(\frac{4(e^{2}-e^{-2})}{\zeta e^{-2}}\Big)^{2}\Big(\frac{d+1}{2}\Big)e^{\sqrt{2\log\Big(\big(\frac{4(e^{2}-e^{-2})}{\zeta e^{-2}}\big)^{2}\big(\frac{d+1}{2}\big)\Big)-1}}\\
 & =e^{\log\Big(\big(\frac{4(e^{2}-e^{-2})}{\zeta e^{-2}}\big)^{2}\big(\frac{d+1}{2}\big)\Big)+\sqrt{2\log\Big(\big(\frac{4(e^{2}-e^{-2})}{\zeta e^{-2}}\big)^{2}\big(\frac{d+1}{2}\big)\Big)-1}}.
\end{align*}
Inequality (\ref{eq:lambert-w-bound}) (with $s=\log\big((\frac{4(e^{2}-e^{-2})}{\zeta e^{-2}})^{2}(\frac{d+1}{2})\big)-1$) implies that
\[
e^{\log\Big(\big(\frac{4(e^{2}-e^{-2})}{\zeta e^{-2}}\big)^{2}\big(\frac{d+1}{2}\big)\Big)+\sqrt{2\log\Big(\big(\frac{4(e^{2}-e^{-2})}{\zeta e^{-2}}\big)^{2}\big(\frac{d+1}{2}\big)\Big)-1}}>e^{-W_{-1}\Big(-\frac{1}{\Big(\frac{4(e^{2}-e^{-2})}{\zeta e^{-2}}\Big)^{2}\Big(\frac{d+1}{2}\Big)}\Big)},
\]
so that
\[
n\ge e^{-W_{-1}\Big(-\frac{1}{\Big(\frac{4(e^{2}-e^{-2})}{\zeta e^{-2}}\Big)^{2}\Big(\frac{d+1}{2}\Big)}\Big)}.
\]
This implies that
\[
M'=\frac{4(e^{2}-e^{-2})}{\zeta e^{-2}}\sqrt{\frac{\omega}{n}}=\frac{4(e^{2}-e^{-2})}{\zeta e^{-2}}\sqrt{\frac{\frac{d+1}{2}\log n}{n}}\le1
\]
as desired. Then using inequality (\ref{eq:standard-unit-ball-covering-bound}), the probability bound in equation (\ref{eq:main-result-prob}) satisfies the bound
\begin{align*}
 & 1-2\bigg[\frac{M}{(C_{\alpha}-1)\big[2\sqrt{\frac{\omega}{n}}\max\{2,\frac{C_{\alpha}}{C_{\alpha}-1}\}M+(2\mathcal{L}+1)M'\big]}+\mathbb{N}(M',\mathcal{X})\bigg]e^{-\omega}-me^{-\frac{n\zeta}{16}}\\
 & =1-2\bigg[\frac{\log\big(1+\frac{e^{2}}{2}n\big)}{(C_{\alpha}-1)\big[2\sqrt{\frac{\omega}{n}}\max\{2,\frac{C_{\alpha}}{C_{\alpha}-1}\}\log\big(1+\frac{e^{2}}{2}n\big)+\frac{12(e^{2}-e^{-2})}{\zeta e^{-2}}\sqrt{\frac{\omega}{n}}\big]}+\mathbb{N}(M',\mathcal{X})\bigg]e^{-\omega}-me^{-\frac{n\zeta}{16}}\\
 & \ge1-2\Bigg[\frac{\log\big(1+\frac{e^{2}}{2}n\big)}{(C_{\alpha}-1)\big[2\sqrt{\frac{\omega}{n}}\max\{2,\frac{C_{\alpha}}{C_{\alpha}-1}\}\log\big(1+\frac{e^{2}}{2}n\big)+\frac{12(e^{2}-e^{-2})}{\zeta e^{-2}}\sqrt{\frac{\omega}{n}}\big]}+\bigg(\frac{3}{\frac{4(e^{2}-e^{-2})}{\zeta e^{-2}}\sqrt{\frac{\omega}{n}}}\bigg)^{d}\Bigg]e^{-\omega}\\
 &\phantom{=}-me^{-\frac{n\zeta}{16}}\\
 & =1-2\Bigg[\frac{\log\big(1+\frac{e^{2}}{2}n\big)}{2(C_{\alpha}-1)\max\{2,\frac{C_{\alpha}}{C_{\alpha}-1}\}\log\big(1+\frac{e^{2}}{2}n\big)+\frac{12(C_{\alpha}-1)(e^{2}-e^{-2})}{\zeta e^{-2}}}\sqrt{\frac{n}{\omega}}+\bigg(\frac{3\zeta e^{-2}}{4(e^{2}-e^{-2})}\bigg)^{d}\Big(\frac{n}{\omega}\Big)^{d/2}\Bigg]e^{-\omega}\\
 &\phantom{=}~-me^{-\frac{n\zeta}{16}}\\
 & =1-2\Bigg[\frac{\log\big(1+\frac{e^{2}}{2}n\big)}{\Upsilon_{1}\log\big(1+\frac{e^{2}}{2}n\big)+\Upsilon_{2}}\Big(\frac{n}{\omega}\Big)^{1/2}+\Upsilon_{3}\Big(\frac{n}{\omega}\Big)^{d/2}\Bigg]e^{-\omega}-me^{-\frac{n\zeta}{16}},
\end{align*}
where
\begin{align*}
\Upsilon_{1} & \triangleq2(C_{\alpha}-1)\max\Big\{2,\frac{C_{\alpha}}{C_{\alpha}-1}\Big\},\\
\Upsilon_{2} & \triangleq\frac{12(C_{\alpha}-1)(e^{2}-e^{-2})}{\zeta e^{-2}},\\
\Upsilon_{3} & \triangleq\bigg(\frac{3\zeta e^{-2}}{4(e^{2}-e^{-2})}\bigg)^{d}.
\end{align*}
Recall that we have the assumption
\[
n\ge2e^{-\frac{6(e^{4}-1)}{\zeta\max\{2,\frac{C_{\alpha}}{C_{\alpha}-1}\}}-2}=\frac{2e^{-\Upsilon_{2}/\Upsilon_{1}}}{e^{2}}.
\]
This implies that
\[
n>\frac{2(e^{-\Upsilon_{2}/\Upsilon_{1}}-1)}{e^{2}},
\]
which further implies that
\[
\frac{\log\big(1+\frac{e^{2}}{2}n\big)}{\Upsilon_{1}\log\big(1+\frac{e^{2}}{2}n\big)+\Upsilon_{2}}\le\frac{1}{\Upsilon_{1}}.
\]
Consequently,
\begin{align*}
 & 1-2\Bigg[\frac{\log\big(1+\frac{e^{2}}{2}n\big)}{\Upsilon_{1}\log\big(1+\frac{e^{2}}{2}n\big)+\Upsilon_{2}}\Big(\frac{n}{\omega}\Big)^{1/2}+\Upsilon_{3}\Big(\frac{n}{\omega}\Big)^{d/2}\Bigg]e^{-\omega}-me^{-\frac{n\zeta}{16}}\\
 & \quad\ge1-2\Bigg[\frac{1}{\Upsilon_{1}}\Big(\frac{n}{\omega}\Big)^{1/2}+\Upsilon_{3}\Big(\frac{n}{\omega}\Big)^{d/2}\Bigg]e^{-\omega}-me^{-\frac{n\zeta}{16}}.
\end{align*}
Next, we use the fact that we set $\omega=\frac{d+1}{2}\log n$. Let $W_{-1}$ to be the lower branch of the Lambert W function. The statement of the corollary also assumes that
\[
n\ge e^{\sqrt{2(\log\frac{d+1}{2}-1)}+\log\frac{d+1}{2}}.
\]
A standard bound on $W_{-1}$ is that for any $s>0$, we have $-1-\sqrt{2s}-s<W_{-1}(-e^{-s-1})$. Plugging in $s=\log\frac{d+1}{2}-1$ (which is guaranteed to be positive since we assume that $d\ge5>2e-1$), we get that
\[
n\ge e^{\sqrt{2(\log\frac{d+1}{2}-1)}+\log\frac{d+1}{2}}\ge e^{-W_{-1}(-\frac{2}{d+1})}.
\]
This in turn implies that $n\ge\omega=\frac{d+1}{2}\log n$. Then since $n\ge\omega$, we have
\begin{align*}
 & 1-2\Bigg[\frac{1}{\Upsilon_{1}}\Big(\frac{n}{\omega}\Big)^{1/2}+\Upsilon_{3}\Big(\frac{n}{\omega}\Big)^{d/2}\Bigg]e^{-\omega}-me^{-\frac{n\zeta}{16}}\\
 & \quad\ge1-2\Big(\frac{1}{\Upsilon_{1}}+\Upsilon_{3}\Big)\Big(\frac{n}{\omega}\Big)^{d/2}e^{-\omega}-me^{-\frac{n\zeta}{16}}.
\end{align*}
Lastly, because we assume that $n\ge e^{\frac{2}{d+1}}$ and $d\ge5>0$, then these two conditions imply that
\[
\frac{d}{2}\log n-\frac{d}{2}\log\omega-\omega\le-\frac{1}{2}\log n,
\]
which means that
\[
\underbrace{\Big(\frac{n}{\omega}\Big)^{d/2}e^{-\omega}}_{=n^{d/2}\omega^{-d/2}e^{-\omega}}\le n^{-1/2}.
\]
Thus,
\begin{align*}
 & 1-2\Big(\frac{1}{\Upsilon_{1}}+\Upsilon_{3}\Big)\Big(\frac{n}{\omega}\Big)^{d/2}e^{-\omega}-me^{-\frac{n\zeta}{16}}\\
 & \quad\ge1-2\Big(\frac{1}{\Upsilon_{1}}+\Upsilon_{3}\Big)\frac{1}{\sqrt{n}}-me^{-\frac{n\zeta}{16}}.
\end{align*}
This results in the simplified probability bound in equation (\ref{eq:linear-cox-split-dro-prob}).

Finally, we plug in $M$ and $M'$ from equation (\ref{eq:linear-cox-dro-split-M-M'}) as well as $\omega=\frac{d+1}{2}\log n$ into bound (\ref{eq:main-result-loss-bound}) to arrive at (\ref{eq:linear-cox-split-dro-loss-bound}), which completes the proof of the corollary.$\hfill\blacksquare$ 

\section{Proof of Corollary \ref{cor:deephit-split-dro-guarantee}\label{sec:pf-deephit-split-dro-guarantee}}

\paragraph*{Verifying Assumptions A1--A4}

Since $\mathcal{X}$ is the unit ball in $\mathbb{R}^{2}$, Assumption A1 is satisfied. There is no need to check Assumptions A2 or A3 (we directly assume A2, and A3 says that we are using the DeepHit adjacency function, which is the case since we are analyzing DeepHit). As for Assumption A4, we now describe how the bounds on $M_{\text{indiv}}$, $M_{\text{couple-min}}$, and $M_{\text{couple-max}}$ are obtained.

First, let's look at
\[
\phi_{\text{indiv}}((x,y,\delta);\theta)=\beta\cdot\big[-\delta\log(f_{\kappa(y)}(x;\theta))-(1-\delta)\log(S_{\kappa(y)}(x;\theta))\big].
\]
Note that this function is nonnegative since log probabilities are negative and are at most 0, i.e., $\log(f_{\kappa(y)}(x;\theta))\le0$ and $\log(S_{\kappa(y)}(x;\theta))\le0$. Since $f_{\kappa(y)}(x;\theta)\ge\varrho$, this means that
\[
-\beta\cdot\delta\cdot\log(f_{\kappa(y)}(x;\theta))\le-\beta\log(f_{\kappa(y)}(x;\theta))\le-\beta\log\varrho=\beta\log\frac{1}{\varrho}.
\]
Meanwhile,
\[
-\beta(1-\delta)\log(S_{\kappa(y)}(x;\theta))\le-\beta\log(S_{\kappa(y)}(x;\theta))\le-\beta\log\varrho=\beta\log\frac{1}{\varrho},
\]
where the last inequality holds because $S_{j}(x;\theta)$ monotonically decreases as we go to later time indices; the smallest it gets is $S_{m-1}(x;\theta)=f_{m}(x;\theta)\ge\varrho$ (where we have used the assumption that within the training data, no observed time corresponds to index $m$). Thus, we can take $M_{\text{indiv}}=\beta\log\frac{1}{\varrho}$.

Next, we look at
\[
\phi_{\text{couple}}((x,y,\delta),(x',y',\delta'),\mathcal{C};\theta)=(1-\beta)\cdot\frac{1}{n}\cdot\exp\Big(\frac{S_{\kappa(y)}(x;\theta)-S_{\kappa(y)}(x';\theta)}{\sigma}\Big).
\]
Here, the main observation is that $S_{j}(x;\theta)\in[\varrho,1]$ for $j\in[m-1]$. Hence, $S_{\kappa(y)}(x;\theta)-S_{\kappa(y)}(x';\theta)\in[\varrho-1,1-\varrho]$, from which we conclude that
\[
\phi_{\text{couple}}((x,y,\delta),(x',y',\delta'),\mathcal{C};\theta)\in[\underbrace{(1-\beta)\cdot\frac{1}{n}\cdot e^{(\varrho-1)/\sigma}}_{M_{\text{couple-min}}},\;\underbrace{(1-\beta)\cdot\frac{1}{n}\cdot e^{(1-\varrho)/\sigma}}_{M_{\text{couple-max}}}].
\]
Now we check the Lipschitz constant. When $\delta=0$, then
\begin{align*}
L^{*}((x,y,\delta),\mathcal{C};\theta) & =\phi_{\text{indiv}}((x,y,\delta);\theta)+\delta\phi_{\text{transform}}\Big(\sum_{(x',y',\delta')\in\mathcal{C}}\phi_{\text{couple}}((x,y,\delta),(x',y',\delta');\theta)\Big)\\
 & =\phi_{\text{indiv}}((x,y,\delta);\theta)\\
 & =-\beta\log(f_{\kappa(y)}(x;\theta)).
\end{align*}
Note that $s\mapsto\log s$ defined on the interval $[\varrho,\infty)$ has Lipschitz constant $\frac{1}{\varrho}$. We are composing $s\mapsto\log s$ with $f_{\kappa(y)}(x;\theta)$, which we assumed is 1-Lipschitz, so $\log(f_{\kappa(y)}(x;\theta))$ is $\frac{1}{\varrho}$-Lipschitz. Finally by multiplying by $-\beta$, we have that $L^{*}((x,y,\delta),\mathcal{C};\theta)$ is $\frac{\beta}{\varrho}$-Lipschitz.

Next, we consider when $\delta=1$. In this case,
\begin{align*}
L^{*}((x,y,\delta),\mathcal{C};\theta) & =\phi_{\text{indiv}}((x,y,\delta);\theta)+\delta\phi_{\text{transform}}\Big(\sum_{(x',y',\delta')\in\mathcal{C}}\phi_{\text{couple}}((x,y,\delta),(x',y',\delta');\theta)\Big)\\
 & =-\beta\log(S_{\kappa(y)}(x;\theta))+(1-\beta)\cdot\frac{1}{n}\cdot\exp\Big(\frac{S_{\kappa(y)}(x;\theta)-S_{\kappa(y)}(x';\theta)}{\sigma}\Big).
\end{align*}
First off, we show that $-\beta\log(S_{\kappa(y)}(x;\theta))$ is $\frac{\beta(m-1)}{\varrho}$-Lipschitz. Note that $S_{j}(x;\theta)=\sum_{\ell=j+1}^{m}f_{\ell}(x;\theta)$ is the sum of at most $(m-1$) functions that are each 1-Lipschitz, so it is $(m-1)$-Lipschitz. As stated earlier, $S_{\kappa(y)}(x;\theta)\ge\varrho$, and $s\mapsto\log(s)$ defined over $[\varrho,\infty)$ is $\frac{1}{\varrho}$-Lipschitz. Thus, $x\mapsto\log(S_{\kappa(y)}(x;\theta))$ is $\frac{m-1}{\varrho}$-Lipschitz. Finally, $x\mapsto-\beta\log(S_{\kappa(y)}(x;\theta))$ is $\frac{\beta(m-1)}{\varrho}$-Lipschitz.

Now we consider the term $(1-\beta)\cdot\frac{1}{n}\cdot\exp\Big(\frac{S_{\kappa(y)}(x;\theta)-S_{\kappa(y)}(x';\theta)}{\sigma}\Big)$. Note that $S_{\kappa(y)}(x;\theta)-S_{\kappa(y)}(x';\theta)$ is the difference of two $(m-1)$-Lipschitz functions, so it is $2(m-1)$-Lipschitz. Next, $\frac{S_{\kappa(y)}(x;\theta)-S_{\kappa(y)}(x';\theta)}{\sigma}$ is $\frac{2(m-1)}{\sigma}$-Lipschitz. Note that
\[
\frac{S_{\kappa(y)}(x;\theta)-S_{\kappa(y)}(x';\theta)}{\sigma}\in\Big[\frac{\varrho-1}{\sigma},\frac{1-\varrho}{\sigma}\Big].
\]
Observe that the map $s\mapsto\exp(s)$ defined over the interval $[\frac{\varrho-1}{\sigma},\frac{1-\varrho}{\sigma}]$ is Lipschitz with Lipschitz constant 
\[
\frac{e^{\frac{1-\varrho}{\sigma}}-e^{\frac{\varrho-1}{\sigma}}}{\frac{1-\varrho}{\sigma}-\frac{\varrho-1}{\sigma}}=\frac{\sigma(e^{\frac{1-\varrho}{\sigma}}-e^{\frac{\varrho-1}{\sigma}})}{2(1-\varrho)}.
\]
Then $x\mapsto\exp\big(\frac{S_{\kappa(y)}(x;\theta)-S_{\kappa(y)}(x';\theta)}{\sigma}\big)$ has Lipschitz constant
\[
\frac{\sigma(e^{\frac{1-\varrho}{\sigma}}-e^{\frac{\varrho-1}{\sigma}})}{2(1-\varrho)}\cdot\frac{2(m-1)}{\sigma}=\frac{(e^{\frac{1-\varrho}{\sigma}}-e^{\frac{\varrho-1}{\sigma}})(m-1)}{(1-\varrho)}.
\]
Finally, $x\mapsto(1-\beta)\cdot\frac{1}{n}\cdot\exp\big(\frac{S_{\kappa(y)}(x;\theta)-S_{\kappa(y)}(x';\theta)}{\sigma}\big)$ has Lipschitz constant
\[
\frac{(1-\beta)(e^{\frac{1-\varrho}{\sigma}}-e^{\frac{\varrho-1}{\sigma}})(m-1)}{n(1-\varrho)}.
\]
We conclude that $x\mapsto L^{*}((x,y,\delta),\mathcal{C};\theta)$ has Lipschitz constant
\[
\frac{\beta(m-1)}{\varrho}+\frac{(1-\beta)(e^{\frac{1-\varrho}{\sigma}}-e^{\frac{\varrho-1}{\sigma}})(m-1)}{n(1-\varrho)}.
\]
Under the assumption that
\[
n\ge\frac{\varrho(1-\beta)(e^{\frac{1-\varrho}{\sigma}}-e^{\frac{\varrho-1}{\sigma}})}{(1-\varrho)\beta},
\]
we have
\begin{align*}
 & \frac{\beta(m-1)}{\varrho}+\frac{(1-\beta)(e^{\frac{1-\varrho}{\sigma}}-e^{\frac{\varrho-1}{\sigma}})(m-1)}{n(1-\varrho)}\\
 & \quad\le\frac{\beta(m-1)}{\varrho}+\frac{(1-\beta)(e^{\frac{1-\varrho}{\sigma}}-e^{\frac{\varrho-1}{\sigma}})(m-1)}{\frac{\varrho(1-\beta)(e^{\frac{1-\varrho}{\sigma}}-e^{\frac{\varrho-1}{\sigma}})}{(1-\varrho)\beta}(1-\varrho)}\\
 & \quad=\frac{2\beta(m-1)}{\varrho}\\
 & \quad\triangleq\mathcal{L}.
\end{align*}
Note that for simplicity, we have used a somewhat loose bound on the Lipschitz constant. This finishes our verification of Assumptions A1--A4. 

\paragraph*{Applying Theorem \ref{thm:main-result}}

We begin by noting that in this setup,
\begin{align*}
M & =M_{\text{indiv}}+\frac{M_{\text{couple-max}}}{2}n\\
 & =\beta\log\frac{1}{\varrho}+\frac{(1-\beta)\cdot\frac{1}{n}\cdot e^{(1-\varrho)/\sigma}}{2}n\\
 & =\beta\log\frac{1}{\varrho}+\frac{(1-\beta)e^{(1-\varrho)/\sigma}}{2},
\end{align*}
and
\begin{align*}
M' & =(M_{\text{couple-max}}-M_{\text{couple-min}})\sqrt{\frac{\omega n}{2\zeta}}\\
 & =\Big((1-\beta)\cdot\frac{1}{n}\cdot e^{(1-\varrho)/\sigma}-(1-\beta)\cdot\frac{1}{n}\cdot e^{(\varrho-1)/\sigma}\Big)\sqrt{\frac{\omega n}{2\zeta}}\\
 & =\sqrt{\frac{\omega}{n}}\Big(\frac{(1-\beta)e^{(1-\varrho)/\sigma}-(1-\beta)e^{(\varrho-1)/\sigma}}{\sqrt{2\zeta}}\Big).
\end{align*}
Just as in the proof of Corollary \ref{cor:linear-cox-split-dro-guarantee}, we begin by showing that $M'\le1$, which ensures that
\[
\mathbb{N}(M',\mathcal{X})\le\Big(\frac{3}{M'}\Big)^{d}.
\]
We have
\begin{align*}
M' & =\sqrt{\frac{\omega}{n}}\Big(\frac{(1-\beta)e^{(1-\varrho)/\sigma}-(1-\beta)e^{(\varrho-1)/\sigma}}{\sqrt{2\zeta}}\Big)\\
 & =\sqrt{\frac{\frac{d+1}{2}\log n}{n}}\Big(\frac{(1-\beta)e^{(1-\varrho)/\sigma}-(1-\beta)e^{(\varrho-1)/\sigma}}{\sqrt{2\zeta}}\Big)\\
 & =\sqrt{\frac{\log n}{n}}\Bigg(\frac{(1-\beta)e^{(1-\varrho)/\sigma}-(1-\beta)e^{(\varrho-1)/\sigma}}{2}\sqrt{\frac{d+1}{\zeta}}\Bigg).
\end{align*}
Under the assumptions that
\begin{align*}
n & \ge[(1-\beta)e^{(1-\varrho)/\sigma}-(1-\beta)e^{(\varrho-1)/\sigma}]^{2}\big(\frac{d+1}{4\zeta}\big)e^{\sqrt{2\log\big([(1-\beta)e^{(1-\varrho)/\sigma}-(1-\beta)e^{(\varrho-1)/\sigma}]^{2}\big(\frac{d+1}{4\zeta}\big)\big)}}\\
 & =e^{\log\Big([(1-\beta)e^{(1-\varrho)/\sigma}-(1-\beta)e^{(\varrho-1)/\sigma}]^{2}\big(\frac{d+1}{4\zeta}\big)\Big)+\sqrt{2\log\Big([(1-\beta)e^{(1-\varrho)/\sigma}-(1-\beta)e^{(\varrho-1)/\sigma}]^{2}\big(\frac{d+1}{4\zeta}\big)\Big)}},
\end{align*}
and that
\[
\log[(1-\beta)e^{(1-\varrho)/\sigma}-(1-\beta)e^{(\varrho-1)/\sigma}]^{2}\Big(\frac{d+1}{4\zeta}\Big)>1,
\]
then by inequality (\ref{eq:lambert-w-bound}) (with $s=\log[(1-\beta)e^{(1-\varrho)/\sigma}-(1-\beta)e^{(\varrho-1)/\sigma}]^{2}\Big(\frac{d+1}{4\zeta}\Big)-1$), we have
\[
n\ge e^{-W_{-1}\Big(-\frac{1}{[(1-\beta)e^{(1-\varrho)/\sigma}-(1-\beta)e^{(\varrho-1)/\sigma}]^{2}\big(\frac{d+1}{4\zeta}\big)}\Big)},
\]
which implies that $M'\le1$.

Now that we have shown that $M'\le1$ so that $\mathbb{N}(M',\mathcal{X})\le(\frac{3}{M'})^{d}$, the probability in equation (\ref{eq:main-result-prob}) satisfies the bound
\begin{align*}
 & 1-2\bigg[\frac{M}{(C_{\alpha}-1)\big[2\sqrt{\frac{\omega}{n}}\max\{2,\frac{C_{\alpha}}{C_{\alpha}-1}\}M+(2\mathcal{L}+1)M'\big]}+\mathbb{N}(M',\mathcal{X})\bigg]e^{-\omega}-me^{-\frac{n\zeta}{16}}\\
 & =1-2\bigg[\frac{\beta\log\frac{1}{\varrho}+\frac{(1-\beta)e^{(1-\varrho)/\sigma}}{2}}{(C_{\alpha}-1)\big[2\sqrt{\frac{\omega}{n}}\max\{2,\frac{C_{\alpha}}{C_{\alpha}-1}\}(\beta\log\frac{1}{\varrho}+\frac{(1-\beta)e^{(1-\varrho)/\sigma}}{2})+(2\mathcal{L}+1)\sqrt{\frac{\omega}{n}}\Big(\frac{(1-\beta)e^{(1-\varrho)/\sigma}-(1-\beta)e^{(\varrho-1)/\sigma}}{\sqrt{2\zeta}}\Big)\big]}\\
 & \phantom{=1-2\bigg[}~+\mathbb{N}(M',\mathcal{X})\bigg]e^{-\omega}-me^{-\frac{n\zeta}{16}}\\
 & =1-2\bigg[\frac{1}{(C_{\alpha}-1)\Big[2\max\{2,\frac{C_{\alpha}}{C_{\alpha}-1}\}+(2\mathcal{L}+1)\Big(\frac{\big((1-\beta)e^{(1-\varrho)/\sigma}-(1-\beta)e^{(\varrho-1)/\sigma}\big)}{\big(2\beta\log\frac{1}{\varrho}+(1-\beta)e^{(1-\varrho)/\sigma}\big)}\sqrt{\frac{2}{\zeta}}\Big)\Big]}\sqrt{\frac{n}{\omega}}+\mathbb{N}(M',\mathcal{X})\bigg]e^{-\omega}\\
 &\phantom{=}~ -me^{-\frac{n\zeta}{16}}\\
 & \le1-2\bigg[\frac{1}{(C_{\alpha}-1)\Big[2\max\{2,\frac{C_{\alpha}}{C_{\alpha}-1}\}+(2\mathcal{L}+1)\Big(\frac{\big((1-\beta)e^{(1-\varrho)/\sigma}-(1-\beta)e^{(\varrho-1)/\sigma}\big)}{\big(2\beta\log\frac{1}{\varrho}+(1-\beta)e^{(1-\varrho)/\sigma}\big)}\sqrt{\frac{2}{\zeta}}\Big)\Big]}\sqrt{\frac{n}{\omega}}\\
 & \phantom{=1-2\bigg[}~+\bigg(\frac{3}{\sqrt{\frac{\omega}{n}}\Big(\frac{(1-\beta)e^{(1-\varrho)/\sigma}-(1-\beta)e^{(\varrho-1)/\sigma}}{\sqrt{2\zeta}}\Big)}\bigg)^{d}\bigg]e^{-\omega}-me^{-\frac{n\zeta}{16}}\\
 & =1-2\bigg[\Psi_{1}\sqrt{\frac{n}{\omega}}+\Psi_{2}\Big(\frac{n}{\omega}\Big)^{d/2}\bigg]e^{-\omega}-me^{-\frac{n\zeta}{16}},
\end{align*}
where
\begin{align*}
\Psi_{1} & \triangleq\frac{1}{(C_{\alpha}-1)\Big[2\max\{2,\frac{C_{\alpha}}{C_{\alpha}-1}\}+(2\mathcal{L}+1)\Big(\frac{\big((1-\beta)e^{(1-\varrho)/\sigma}-(1-\beta)e^{(\varrho-1)/\sigma}\big)}{\big(2\beta\log\frac{1}{\varrho}+(1-\beta)e^{(1-\varrho)/\sigma}\big)}\sqrt{\frac{2}{\zeta}}\Big)\Big]},\\
\Psi_{2} & \triangleq\bigg(\frac{3\sqrt{2\zeta}}{(1-\beta)e^{(1-\varrho)/\sigma}-(1-\beta)e^{(\varrho-1)/\sigma}}\bigg)^{d}.
\end{align*}
Using the same reasoning as in the proof of Corollary \ref{cor:linear-cox-split-dro-guarantee}, when $n\ge e^{\sqrt{2(\log\frac{d+1}{2}-1)}+\log\frac{d+1}{2}}$ (which we assume in the corollary statement), we are guaranteed that $n\ge\omega$. Hence, we have
\begin{align*}
 & 1-2\bigg[\Psi_{1}\sqrt{\frac{n}{\omega}}+\Psi_{2}\Big(\frac{n}{\omega}\Big)^{d/2}\bigg]e^{-\omega}-me^{-\frac{n\zeta}{16}}\\
 & \quad\ge1-2(\Psi_{1}+\Psi_{2})\Big(\frac{n}{\omega}\Big)^{d/2}e^{-\omega}-me^{-\frac{n\zeta}{16}}.
\end{align*}
Moreover, when $n\ge e^{\frac{2}{d+1}}$ and $d\ge5>0$ (we assume both of these), using the same reasoning as the proof of Corollary \ref{cor:linear-cox-split-dro-guarantee},
\[
\Big(\frac{n}{\omega}\Big)^{d/2}e^{-\omega}\le\frac{1}{\sqrt{n}}.
\]
Hence,
\begin{align*}
 & 1-2(\Psi_{1}+\Psi_{2})\Big(\frac{n}{\omega}\Big)^{d/2}e^{-\omega}-me^{-\frac{n\zeta}{16}}\\
 & \quad\ge1-2(\Psi_{1}+\Psi_{2})\frac{1}{\sqrt{n}}-me^{-\frac{n\zeta}{16}}.
\end{align*}
In the statement of the corollary, $\Psi\triangleq2(\Psi_{1}+\Psi_{2})$. As for the loss bound (\ref{eq:main-result-loss-bound}), we simply plug in the values of $M$ and $M'$ specific to the DeepHit setup here, and we also plug in $\omega=\frac{d+1}{2}\log n$. This finishes the proof.$\hfill\blacksquare$

\section{Fairness Metrics}
\label{sec:fairness-measures}

In this paper, we use the individual, group, and intersectional fairness metrics defined by \citet{keya2021equitable}, the concordance imparity (CI) metric by \citet{zhang2022longitudinal}, and also censoring-based individual and censoring-based group fairness metrics by \citet{rahman2022fair}. For all of these fairness metrics, lower is considered better, where the minimum possible value is~0. We point out that the fairness metrics by \citet{keya2021equitable}~and \citet{rahman2022fair} can readily be treated as regularizers (i.e., they could be included as additional loss terms during model training). Moreover, the individual fairness metric by \citet{keya2021equitable} and the censoring-based individual and censoring-based group fairness metrics by \citet{rahman2022fair} crucially depend on a scaling constant $\gamma>0$ that must be set by the user in advance: if $\gamma$ is set to be higher, then it becomes easier for a survival model to achieve a score of exactly (and not just approximately) 0 for these particular fairness metrics.

Note that in Section~\ref{sec:Experiments} of the main paper, we use the fairness metrics by Keya et al.~as regularizers in baseline methods and not as evaluation metrics. However, we include additional experimental results that use the individual and group fairness metrics by Keya et al.~as evaluation metrics in Appendix~\ref{additional_exp_results} (specifically, see Tables~\ref{tab:Flc_age_more_fair_performance}--\ref{tab:more_fair_performance_SODEN}).

We begin by explaining the fairness metrics proposed by \citet{keya2021equitable} as these were the earliest fairness metrics we are aware of that were developed for survival analysis. Note that Keya et al.~focused on Cox proportional hazards models. For such models, we can take the predicted outcome for a feature vector $x$ to be the so-called \emph{partial hazard} $\widetilde{h}(x) \triangleq \exp( f(x;\theta) )$; this is the same as the hazard function given in equation \eqref{eq:hazard-factorization} except where we exclude the baseline hazard factor $h_0(t)$. Note that once we exclude $h_0(t)$, then $\widetilde{h}$ no longer depends on time~$t$. We state the fairness metrics in terms of a collection of $N_{\text{test}}$ test patients with data $(X_1^{\text{test}},Y_1^{\text{test}},\Delta_1^{\text{test}}),\dots,(X_{N_{\text{test}}}^{\text{test}},Y_{N_{\text{test}}}^{\text{test}},\Delta_{N_{\text{test}}}^{\text{test}})$. Note that the fairness metrics by \citet{keya2021equitable} only use the test feature vectors $X_1^{\text{test}},\dots,X_{N_{\text{test}}}^{\text{test}}$ and ignores the test patients' observed times and event~indicators. Also, at the end of this section, we point out that the individual and group fairness metrics by \citet{keya2021equitable} are sensitive to the scale of the log partial hazard~$f(\cdot;\theta)$.

\paragraph{Individual fairness}
Roughly, \citet{keya2021equitable} consider a model to be fair across individuals (patients) if similar individuals have similar predicted outcomes. To operationalize this notion of fairness in the context of Cox models, Keya et al.~define the individual fairness metric
\begin{equation*}
\begin{aligned}
\textrm{F}_I \triangleq \sum_{i=1}^{N_{\text{test}}}\sum_{j=i+1}^{N_{\text{test}}}\big[&|\widetilde{h}(X_i^{\text{test}})-\widetilde{h}(X_j^{\text{test}})|-\gamma \|X_i^{\text{test}} - X_j^{\text{test}}\|\big]_+,
\end{aligned}
\end{equation*}
where $\gamma$ is a predefined scale factor (0.01 in our experiments). As a reminder, $[\hspace{2pt}\cdot\hspace{2pt}]_+$ is the ReLU function (so that $[a]_+=\max\{0,a\}$ for any $a\in\mathbb{R}$). Importantly, we point out that by setting $\gamma$ to be larger, then more terms in the summation become~0 (since within the ReLU expression, we are subtracting a larger quantity, making it more likely that after applying ReLU, we get~0). If $\gamma$ is set to be too large, then it is possible that all terms in the summation become~0 (i.e., the fairness metric becomes exactly and not just approximately equal to~0).

Note that this individual fairness metric is actually just penalizing $\widetilde{h}$ for not being Lipschitz continuous (as empirically evaluated over the test data). Specifically, $\widetilde{h}$ is defined to be $\gamma$-Lipschitz continuous if
\[
|\widetilde{h}(x)-\widetilde{h}(x')|
\le \gamma \|x - x'\|
\quad\text{for all }x,x'\in\mathcal{X}.
\]
Meanwhile, when F$_I$ is equal to 0, then it means that
\begin{align*}
&|\widetilde{h}(X_i^{\text{test}})-\widetilde{h}(X_j^{\text{test}})|
\le \gamma \|X_i^{\text{test}} - X_j^{\text{test}}\| \quad\text{for all }i,j\in\{1,\dots,N_{\text{test}}\}.
\end{align*}
As a technical remark, in the definition of F$_I$ and also $\gamma$-Lipschitz continuity, the metric used to measure distances between feature vectors does not have to be Euclidean. For example, we can replace $\|X_i^{\text{test}} - X_j^{\text{test}}\|$ with $\rho(X_i^{\text{test}}, X_j^{\text{test}})$, where $\rho:\mathcal{X}\times\mathcal{X}\rightarrow[0,\infty)$ is a user-specified metric.

The individual fairness metric by \citet{keya2021equitable} can be modified to support survival models that do not assume the proportional hazards assumption (such as DeepHit and SODEN) in a straightforward manner: we simply replace the hazard function $\widetilde{h}(x)$ by the estimated survival function $\widehat{S}(t|x)$ to obtain the following time-dependent fairness metric:
\begin{equation*}
\begin{aligned}
\textrm{F}_I(t) \triangleq \sum_{i=1}^{N_{\text{test}}}\sum_{j=i+1}^{N_{\text{test}}}\big[&|\widehat{S}(t|X_i^{\text{test}})-\widehat{S}(t|X_j^{\text{test}})|-\gamma \|X_i^{\text{test}} - X_j^{\text{test}}\|\big]_+.
\end{aligned}
\end{equation*}

\paragraph{Group fairness}
\citet{keya2021equitable} consider a model is fair across a user-specified set of groups if these different groups have similar predicted outcomes. Keya et al.~define the group fairness metric F$_G$ to look at the maximum deviation of a group's average predicted outcome to the overall population's average predicted outcome. Specifically, let $\mathcal{G}$ be the user-specified set of groups to consider (for example, there could be two groups: everyone with age at most 65 years, and everyone older than 65 years), where each group $g\in\mathcal{G}$ is a subset of the test set indices $\{1,\dots,N_{\text{test}}\}$ (so that using this notation, group $g$ has size $|g|$); the different groups should form a partition of the test set (so that the groups are disjoint and their union is the entire test set). Then
\begin{equation*}
\begin{aligned}
\textrm{F}_G
\triangleq \max_{g\in \mathcal{G}}
    \bigg|
      \underbrace{\frac{1}{|g|}\sum_{i\in g}\widetilde{h}(X_i^{\text{test}})}_{\substack{\text{average predicted}\\\text{outcome of group }g}}
      -
      \underbrace{\frac{1}{N_{\text{test}}}\sum_{i=1}^{N_{\text{test}}}\widetilde{h}( X_i^{\text{test}})}_{\substack{\text{average predicted}\\\text{outcome of population}}}
    \bigg|.
\end{aligned}
\end{equation*}
Once again, for survival models that do not assume a proportional hazards assumption (such as DeepHit and SODEN), we can instead replace $\widetilde{h}(x)$ with $\widetilde{S}(t|x)$ to obtain the following time-dependent group fairness metric:
\begin{equation*}
\begin{aligned}
\textrm{F}_G(t)
\triangleq \max_{g\in \mathcal{G}}
    \bigg|
      \underbrace{\frac{1}{|g|}\sum_{i\in g}\widehat{S}(t|X_i^{\text{test}})}_{\substack{\text{average predicted}\\\text{outcome of group }g}}
      -
      \underbrace{\frac{1}{N_{\text{test}}}\sum_{i=1}^{N_{\text{test}}}\widehat{S}(t|X_i^{\text{test}})}_{\substack{\text{average predicted}\\\text{outcome of population}}}
    \bigg|.
\end{aligned}
\end{equation*}

\paragraph{Intersectional fairness} 
\citet{keya2021equitable} consider a notion of intersectional fairness that accounts for multiple sensitive attributes. For example, in the FLC dataset, we have 2 different sensitive attributes, age and gender. For each of these sensitive attributes, we can partition the test set into groups. Specifically, let $\mathcal{G}_1$ be a partition of the test set into different age groups (for example, two groups: at most 65 years old and over 65 years old), and let $\mathcal{G}_2$ be a partition of the test set into different gender groups (for example, two groups: female and male). Then intersectional fairness looks at every intersection of age/gender groups (continuing from the previous examples, we would have four intersectional subgroups: at most 65 years old and female, at most 65 years and male, over 65 years old and female, over 65 years old and male).

The notation here is a bit more involved. The set of all intersectional subgroups of $\mathcal{G}_1$ and $\mathcal{G}_2$ is given by the Cartesian product $\mathcal{G}_1\times\mathcal{G}_2$. Note that $s\in\mathcal{G}_1\times\mathcal{G}_2$ means that $s=(s_1,s_2)$, where $s_1\in\mathcal{G}_1$ and $s_2\in\mathcal{G}_2$. More generally, if there are $J$ sensitive attributes, corresponding to groupings $\mathcal{G}_1,\mathcal{G}_2,\dots,\mathcal{G}_J$, then the set of all intersectional subgroups would be $\mathcal{S} \triangleq \mathcal{G}_1\times\mathcal{G}_2\times\cdots\mathcal{G}_J$. Now $s\in\mathcal{S}$ is a list consisting of $J$ different subsets of test patients (i.e., $s=(s_1,s_2,\dots,s_J)$, where $s_1\in\mathcal{G}_1$, $\dots$, $s_J\in\mathcal{G}_J$). The intersection of these $J$ subsets (i.e., $\cap_{j=1}^J s_j \subset \{1,\dots,N_{\text{test}}\}$) is precisely the set of test patients that intersectional subgroup $s$ corresponds to. Then the average predicted outcome for intersectional subgroup $s$ is
\[
\widetilde{\mathbf{h}}(s) \triangleq
\frac{1}{|\cap_{j=1}^J s_j|}~ \sum_{i\in\cap_{j=1}^J s_j} \widetilde{h}(X_i^{\text{test}}).
\]
Then the intersection fairness metric F$_{\cap}$ by
\citet{keya2021equitable} is the worst-case log ratio of expected predicted outcomes between two intersectional subgroups:
\begin{equation*} %
\textrm{F}_{\cap} \triangleq \max_{s, s' \in \mathcal{S}}
\Big|
\log \frac{ \widetilde{\mathbf{h}}(s) }
          { \widetilde{\mathbf{h}}(s') }
\Big|.
\end{equation*}

\paragraph{Concordance imparity} We now describe an alternative metric for group fairness called concordance imparity (CI) that asks that a survival analysis model achieves similar prediction accuracy for different groups. For ease of exposition, we only state the CI metric by \cite{zhang2022longitudinal} in terms of a single sensitive attribute that has already been discretized (e.g., the attribute is already discrete or we have a pre-specified discretization rule); this special case is sufficient for our experiments. We denote the set of possible discretized values of this sensitive attribute as $\mathcal{A}$.
For example, $\mathcal{A}$ could correspond to age and we could have
$\mathcal{A}=\{\text{``age}\le65\text{''}, \text{``age}>65\text{''}\}$, 
i.e., $\mathcal{A}$ consists of the different groups to consider. We refer the reader to the Zhang and Weiss's original paper for their more general definition of CI that can handle a continuous sensitive attribute via an automatic discretization strategy that they propose.

Assuming that the sensitive attribute has already been discretized into the set $\mathcal{A}$, the CI metric looks at a variant of the standard survival analysis accuracy metric of concordance index \citep{harrell1982evaluating} that Zhang and Weiss call the \emph{concordance fraction} (CF), which is specific to each sensitive attribute value $a\in\mathcal{A}$. The CI metric is then defined to be the worst-case difference between the CF scores of any two $a,a'\in\mathcal{A}$ where $a\ne a'$. The pseudocode can be found in Algorithm \ref{alg:CI-discrete}; note that to keep the notation from getting clunky, we drop the superscript ``test'' from the test feature vectors, observed times, and event indicators in the pseudocode but we still use $N_{\text{test}}$ to denote the number of test patients. Also, in the pseudocode, we let $A_i\in\mathcal{A}$ denote the sensitive attribute value for the $i$-th test patient, where we assume that $A_i$ can directly be computed based on the $i$-th test patient's feature vector. For example, when age (which is not discretized) is one of the features and $\mathcal{A}$ consists of the two age groups previously stated ($\le65$ or $>65$), then since we know the discretization rule used, we can readily determine which age group in $\mathcal{A}$ that any test patient is in.

Importantly, to calculate the CI metric, a way to calculate a risk score is required to compute the CF scores. How to define a risk score differs across models. For Cox models, we can take the risk score to be the log partial hazard function $f(\cdot;\theta)$. For DeepHit and SODEN models, we take the risk score to be the estimated survival probability $\widehat{S}(t|x)$ and therefore we need to replace $f(\cdot;\theta)$ with $\widehat{S}(t|x)$ before using Algorithm \ref{alg:CI-discrete}. Since different values of time $t$ can have different estimated $\widehat{S}(t|x)$ values, we would obtain different value of the CI fairness metric for different $t$. We test three different values of $t$ (the 25$^{th}$, 50$^{th}$, and 75$^{th}$ percentile of the observed times in the test data) and use the average value for the final CI score.

\paragraph{Censoring-based individual fairness}
Individual fairness F$_{I}$ does not consider censoring information that is available. \citet{rahman2022fair} defined a censoring-based individual fairness metric as follows:
\begin{equation*}
\begin{aligned}
\textrm{F}_{CI}(t) \triangleq \frac{1}{|N_c|\times |N_{uc}|}\sum_{\substack{i\in N_c, j\in N_{uc}\\\text{s.t.~}Y_j\geq Y_i}}\big[&|\widehat{S}(t|X_i^{\text{test}})-\widehat{S}(t|X_j^{\text{test}})|-\gamma \|X_i^{\text{test}} - X_j^{\text{test}}\|\big]_+,
\end{aligned}
\end{equation*}
where $N_c$ and $N_{uc}$ are the index sets of censored and uncensored data, and $\widehat{S}(t|X)$ is the estimated survival probability at time $t$ for patient $X$. Similar to in F$_{I}$, the scale factor $\gamma$ is a predefined (0.01 in our experiments). This fairness metric ensures that a censored patient and an uncensored patient who have similar features should also have similar predictions whenever the observed time from the uncensored patient is larger than that of the censored patient. Similar to the CI fairness metric and following the experimental settings of \cite{rahman2022fair}, we test three different $t$ values (25$^{th}$, 50$^{th}$, 75$^{th}$ percentile of the observed times in the test data) and use their average value to calculate the final F$_{CI}$ score.

As a warning, just as with the individual fairness metric F$_I$ by \citet{keya2021equitable} that we described earlier, if $\gamma$ is set higher, then it becomes easier for the F$_{CI}(t)$ metric to become exactly and not just approximately equal to~0.

\paragraph{Censoring-based group fairness}
\citet{rahman2022fair} also modified the F$_G$ metric by \citet{keya2021equitable} to account for censoring information. Reusing notation from the definitions of F$_G$ and F$_{CI}(t)$, we now define the censoring-based group fairness metric
\begin{equation*}
\begin{aligned}
\textrm{F}_{CG}(t) \triangleq \frac{1}{|N_c|\times |N_{uc}|}\sum_{g\in\mathcal{G}}\sum_{\substack{i\in N_{c,g}, j\in N_{uc,g}\\\text{s.t.~}Y_j\geq Y_i}}\big[&|\widehat{S}(t|X_i^{\text{test}})-\widehat{S}(t|X_j^{\text{test}})|-\gamma \|X_i^{\text{test}} - X_j^{\text{test}}\|\big]_+,
\end{aligned}
\end{equation*}
where $N_{c,g}$ and $N_{uc,g}$ are the index sets of censored and uncensored in group $g$, and $\widehat{S}(t|X)$ is the estimated survival probability at time $t$ for patient $X$.
Similar to the setting in censoring-based individual fairness, we use three different $t$ to test the value of F$_{CG}(t)$ and use their average for the final reported F$_{CG}$ score. Once again, if $\gamma$ is set too large, then it becomes easier for F$_{CG}(t)$ to be exactly~0.

\begin{algorithm}[p!]\footnotesize
    \caption{Concordance Imparity (CI) with a discrete sensitive attribute\vspace{1pt}}\label{alg:CI-discrete}
    \SetAlgoLined
    \KwIn{Test dataset $\{(X_i,Y_i,\Delta_i)\}_{i=1}^{N_{\text{test}}}$, risk score $f(\cdot;\theta)$ (from an already trained model),
    set of sensitive attribute values $\mathcal{A}$ (so that each $a\in\mathcal{A}$ corresponds to a different group), $A_1,\dots,A_{N_{\text{test}}}\in\mathcal{A}$ says which sensitive attribute value each test patient has}
    \KwOut{CI score} 

    \For{$a\in\mathcal{A}$}{
        Initialize the numerator count $\mathbf{N}(a) \leftarrow 0$ and denominator count $\mathbf{D}(a) \leftarrow 0$.
    }
    \For{$i=1,\dots,N_{\text{test}}$}{
    \For{$j=1,\dots,N_{\text{test}}$ s.t.~$j\ne i$}{
        \eIf{($Y_i<Y_j$ and $\Delta_i==0$) or ($Y_j<Y_i$ and $\Delta_j==0$) or ($Y_i==Y_j$ and $\Delta_i==0$ and $\Delta_j==0$)}
        {
        \textbf{continue}}{
        Set $\mathbf{D}(A_i) \leftarrow \mathbf{D}(A_i) + 1$.
        }
        \uIf{$Y_i<Y_j$}{
            \uIf{$f(X_i;\theta)>f(X_j;\theta)$}{
            Set $\mathbf{N}(A_i) \leftarrow \mathbf{N}(A_i) + 1$.
            }
            \ElseIf{$f(X_i;\theta)==f(X_j;\theta)$}{
            Set $\mathbf{N}(A_i) \leftarrow \mathbf{N}(A_i) + 0.5$.
            }{}
        }
        \uElseIf{$Y_i>Y_j$}{
            \uIf{$f(X_i;\theta)<f(X_j;\theta)$}{
            Set $\mathbf{N}(A_i) \leftarrow \mathbf{N}(A_i) + 1$.
            }
            \ElseIf{$f(X_i;\theta)==f(X_j;\theta)$}{
            Set $\mathbf{N}(A_i) \leftarrow \mathbf{N}(A_i) + 0.5$.
            }{}
        }
        \ElseIf{$Y_i==Y_j$}{
            \uIf{$\Delta_i==1$ and $\Delta_j==1$}{\eIf{$f(X_i;\theta)\!\!==\!\!f(X_j;\theta)$}{%
            Set $\mathbf{N}(A_i) \leftarrow \mathbf{N}(A_i) + 1$.
            }{%
            Set $\mathbf{N}(A_i) \leftarrow \mathbf{N}(A_i) + 0.5$.
            }}
            \uElseIf{$\Delta_i\!\!==\!\!0$ and $\Delta_j\!\!==\!\!1$ and $f(X_i;\theta)\!<\!f(X_j;\theta)$}{%
            Set $\mathbf{N}(A_i) \leftarrow \mathbf{N}(A_i) + 1$.
            }
            \uElseIf{$\Delta_i\!\!==\!\!1$ and $\Delta_j\!\!==\!\!0$ and $f(X_i;\theta)\!>\!f(X_j;\theta)$}{%
            Set $\mathbf{N}(A_i) \leftarrow \mathbf{N}(A_i) + 1$.
            }
            \Else{%
            Set $\mathbf{N}(A_i) \leftarrow \mathbf{N}(A_i) + 0.5$.
            }
        }{}
    }

    }
    \For{$a\in\mathcal{A}$}{
        Set the concordance fraction of $a$: $\mathbf{CF}(a) \leftarrow \frac{\mathbf{N}(a)}{\mathbf{D}(a)}$.
    }
    
    \Return{$\text{\emph{CI}}\leftarrow\max_{a,a'\in\mathcal{A}\text{ s.t.~}a\ne a'}|\mathbf{CF}(a)-\mathbf{CF}(a')|$}
\end{algorithm}

\vspace{-.75em}
\noindent
\subsection*{Scale Issues with F$_I$ and F$_G$}
\vspace{-.25em}

We point out that the F$_I$ and F$_G$ fairness metrics by \citet{keya2021equitable} are sensitive to the scale of the log partial hazard function $f(\cdot;\theta)$, and thus also the scale of the partial hazard $\widetilde{h}(x)=\exp( f(x;\theta) )$ if they are calculated by using $\widetilde{h}(x)$. For instance, consider a standard linear Cox model with $f(x;\theta)=\theta^T x$, where the parameters $\theta$ have already been learned. Then one way to make the model appear fairer according to the F$_I$ and F$_G$ metrics is to just scale all values in $\theta$ by any positive constant smaller than 1; doing so, the standard accuracy metric of concordance index \citep{harrell1982evaluating} would actually remain unchanged for the model as it only depends on the ranking of the different individuals' (log) partial hazard values. However, an accuracy score that considers each individual's survival function estimate (e.g., integrated IPCW Brier Score \citep{graf1999assessment}) would be affected.

\section{Hyperparameter Tuning and Compute Environment Details}
\label{sec:hyperparameters-compute-env}

\paragraph{Hyperparameters}
\textit{Cox models}:
for nonlinear Cox models, we always use a two-layer MLP with ReLU as the activation function and 24 as the number of hidden units. All models (linear and nonlinear) are trained using Adam \citep{kingma2014adam} in PyTorch 1.7.1 in a batch setting for 500 iterations (except in the case of the exact DRO Cox model on the FLC dataset, where we use 5000 iterations as it took more iterations for the model to converge), only using a CPU and no GPU.

\textit{DeepHit models}: we use three-layer MLP with ReLU activation, batch normalization, and dropout (in 0.1). The number of hidden units is 32. The original DeepHit and \textsc{dro-deephit} models are trained using Adam in PyTorch 1.7.1 in a mini-batch 256 setting for 500 epochs. However, the \textsc{dro-deephit (split)} model is trained using a batch setting for 500 iterations.

\textit{SODEN models}: for the FLC dataset, we use an MLP with 4 layers and 16 hidden units. For SUPPORT and SEER datasets, we use an MLP with 2 layers and 26 hidden units. In addition, RMSprop \citep{tieleman2012lecture} in 128 batch size with a maximum 100 epochs is used to train all models.

We tune on the following hyperparameter grid:
\begin{itemize}[leftmargin=*,itemsep=0pt,parsep=0pt,topsep=0pt,partopsep=0pt]
\item To find the optimal learning rate for each Cox model, we conducted a sweep over values of 0.01, 0.001, and 0.0001. Specifically for the exact DRO Cox model, we used a fixed learning rate of 0.1. For the FIDP, FIPNAM, and DeepHit models, we used a fixed learning rate of 0.01. In the case of SODEN models, the learning rates applied were 0.01, 0.002, and 0.002 for the FLC, SUPPORT, and SEER datasets, respectively.
\item $\lambda$ (only used for baselines; a hyperparameter that controls the tradeoff between the original Cox loss and fairness regularization term): 1, 0.7, 0.4. We set $\lambda=0.1$ for FIDP and FIPNAM.
\item $\alpha$: 0.1, 0.15, 0.2, 0.3, 0.4, 0.5 for \textsc{dro-cox}/\textsc{dro-cox (split)}/\textsc{exact  dro-cox} variants; 0.1, 0.2, 0.3, 0.4, 0.5, 0.6, 0.7, 0.8, 0.9, 1.0 for \textsc{dro-deephit}/\textsc{dro-deephit (split)} variants and \textsc{dro-soden}. 
\end{itemize}
In addition, for \textsc{dro-cox (split)} and \textsc{dro-deephit (split)}, we choose $n_1=n_2=n/2$ (rounding as needed when $n$ is odd, so that $n_1$ might not equal $n_2$).

\paragraph{Compute environment}
All models are implemented with Python 3.8.3, and they are trained and tested on identical compute instances, each with an Intel Core i9-10900K CPU (3.70GHz with 64 GB RAM) and a Quadro RTX 4000 GPU.

\section{Additional Experiments}\label{additional_exp_results}

\paragraph{Using other sensitive attributes in evaluating CI and F$_{CG}$ in Cox models} In Section~\ref{sec:Experiments} of the main paper, for the Cox model, we only showed test set performance metrics ($C^{td}$ and IBS accuracy metrics, and CI, F$_{CI}$, and F$_{CG}$ fairness metrics) for FLC, SUPPORT, and SEER using age, gender, race, and race respectively (in Tables~\ref{tab:general_performance_CI}, \ref{tab:general_performance_SUPPORT_gender_CI}, and~\ref{tab:general_performance_SEER2_race_CI}). We now provide results using gender for FLC (Table~\ref{tab:general_performance_FLC_gender_CI}), age and separately race for SUPPORT (Tables~\ref{tab:general_performance_SUPPORT_age_CI} and~\ref{tab:general_performance_SUPPORT_race_CI}), and age for SEER (Table~\ref{tab:general_performance_SEER2_age_CI}). Our main findings still hold for these additional results. 

We point out that for DeepHit and SODEN models, in Section~\ref{sec:Experiments}, we had already shown results for FLC, SUPPORT, and SEER where per dataset, we consider different sensitive attributes (see Tables~\ref{tab:general_performance_deephit_CI} and~\ref{tab:general_performance_SODEN_CI}).

\paragraph{Using individual and group fairness evaluation metrics by \citet{keya2021equitable}}
Whereas in the main paper, we focused on evaluating test data using CI, F$_{CI}$, and F$_{CG}$ fairness metrics, we now also show results where we use the F$_I$ and F$_G$ fairness metrics by \citet{keya2021equitable} instead. See Tables~\ref{tab:Flc_age_more_fair_performance}--\ref{tab:more_fair_performance_SODEN}. Our main findings still hold using these fairness metrics by Keya et al.

\paragraph{Effect of changing $n_1$ (or $n_2$) for \textsc{dro-cox (split)}} In the above experiments, we set $n_1=n_2=n/2$ (rounding as needed). To evaluate the sensitivity of this setting, we test the model performance using \textsc{dro-cox (split)} under the linear and nonlinear settings, where we set $n_2=0.1n,0.2n,0.3n,0.4n,0.5n$ (corresponding to $n_1=0.9n,0.8n,0.7n,0.6n,0.5n$). We report the test set performance metrics for the FLC dataset (using age for evaluation) in Table \ref{tab:performance_of_diff_split}. From the table, we find that per metric, different settings for $n_1$ and $n_2$ lead to results that, while slightly different, are not dramatically different, i.e., the performance of \textsc{dro-cox (split)} does not appear very sensitive w.r.t.~the choice of $n_1$ and $n_2$.

\paragraph{Effect of changing imbalance in censoring rates across training data splits for \textsc{dro-cox (split)}}
For our split DRO strategy, to see what happens when the two subsets of the training data $\mathcal{D}_1$ and $\mathcal{D}_2$ have different censoring rates, we conduct the following experiment. We first randomly divide the training data into 50/50 pieces $\mathcal{D}_1$/$\mathcal{D}_2$ where we stratify on the censoring rate so that $\mathcal{D}_1$ and $\mathcal{D}_2$ have the same censoring rate. Then, we introduce an censoring rate imbalance by trading, for instance, a randomly chosen censored point from $\mathcal{D}_2$ with a randomly chosen uncensored point from $\mathcal{D}_1$. We could of course trade multiple points.

To formalize a notion of how much imbalance we are introducing, we define a censoring rate \emph{imbalance ratio} as follows. First, note that using the above strategy of trading points between $\mathcal{D}_1$ and $\mathcal{D}_2$ that we stated, the maximum number of points we could possibly trade is given by the \emph{minimum} of the number of uncensored points in $\mathcal{D}_1$ and the number of censored points in $\mathcal{D}_2$. Let's call this maximum number of points we could trade as $n_{\text{max trade}}$. Then we define the imbalance ratio to be a percentage of $n_{\text{max trade}}$ points that we trade. Thus, an imbalance ratio of 80\% means that we trade $0.8 n_{\text{max trade}}$ randomly chosen censored points from $\mathcal{D}_2$ with $0.8 n_{\text{max trade}}$ randomly chosen uncensored points from $\mathcal{D}_1$.

We repeat the same experiment that resulted in Table~\ref{tab:general_performance_CI} specifically for \textsc{dro-cox (split)} (i.e., for simplicity, we only consider the FLC dataset treating age as sensitive), where the only difference now is that we re-train \textsc{dro-cox (split)} using imbalance ratios of 0\%, 20\%, 40\%, 60\%, 80\%, and 100\% (per imbalance ratio, we re-run experiments 10 times). The resulting test set accuracy and fairness metrics are reported in Table~\ref{tab:performance_of_diff_imb_ratio}.

The most important takeaway from Table~\ref{tab:performance_of_diff_imb_ratio} is that our split DRO approach still can work well even with high censoring rate imbalance ratios. For instance, in the linear setting, accounting for the standard deviations that have been reported in the table, at an imbalance ratio of 100\%, the resulting accuracy and fairness metrics are actually within noise of using an imbalance ratio of 0\%. In the nonlinear setting, at an imbalance ratio of 80\%, the model achieves a better mean CI fairness score compared to an imbalance ratio of 0\% while achieving the highest mean $C^{td}$ score (although the mean IBS score increases). Meanwhile, still in the nonlinear setting, at an imbalance of 100\%, the model achieves the lowest IBS and CI fairness scores (within the nonlinear setting). To recap, these findings suggest that our split DRO approach can still work well even at high censoring rate imbalance ratios.

As for whether we should favor low or high censoring rate imbalance ratios, Table~\ref{tab:performance_of_diff_imb_ratio} suggests that in practice, we should just tune on this imbalance ratio since an intermediate imbalance ratio could achieve the best tradeoff of accuracy and fairness scores. For simplicity though, in the main paper, we do not tune on the censoring rate imbalance ratio and stick to just using an imbalance ratio of 0\%. We defer a more thorough investigation of the impact of the imbalance ratio to future work.

\paragraph{The effect of using two losses for \textsc{dro-cox (split)} rather than only one} Recall that \textsc{dro-cox (split)} minimizes the sum of two losses $L_{\text{DRO}}^{\text{split}}(\theta,\eta,\mathcal{D}_{1}\mid\mathcal{D}_{2})$ and $L_{\text{DRO}}^{\text{split}}(\theta,\eta,\mathcal{D}_{2}\mid\mathcal{D}_{1})$. Towards the end of Section~\ref{sec:dro-split}, we said that an approach that only minimizes one of these losses would not use the data as effectively compared to minimizing the sum of these losses. We conducted an experiment to verify this claim, where we refer to the version of \textsc{dro-cox (split)} that only minimizes $L_{\text{DRO}}^{\text{split}}(\theta,\eta,\mathcal{D}_{1}\mid\mathcal{D}_{2})$ as \mbox{\textsc{dro-cox (split, one side)}}. Specifically, we compare \mbox{\textsc{dro-cox (split, one side)}} and \textsc{dro-cox (split)} under linear and nonlinear settings on the FLC dataset using age for evaluation. We report the resulting test set performance metrics in Table~\ref{tab:performance_of_one_side_split}. From the table, we find that \textsc{dro-cox (split)} outperforms \textsc{dro-cox (split, one side)} on most metrics. This experimental finding supports our hypothesis that \textsc{dro-cox (split, one side)} uses data less effectively.

\begin{table} %
\caption{\small Cox model test set scores on the FLC (gender) dataset, in the same format as Table~\ref{tab:general_performance_CI}.} 
\centering
\setlength\tabcolsep{0.1pt}
\renewcommand{\arraystretch}{0.5}
{\tiny %
\renewcommand{\belowrulesep}{0.1pt}
\renewcommand{\aboverulesep}{0.1pt}
\begin{tabular}{ccccccc|ccccc}
\toprule
\multirow{3}{*}{} & \multirow{3}{*}{Methods} & \multicolumn{5}{c|}{CI-based Tuning}                                                    & \multicolumn{5}{c}{F$_{CG}$-based Tuning}                             \\ \cmidrule{3-12}
&                   & \multicolumn{2}{c|}{Accuracy Metrics}  & \multicolumn{3}{c|}{Fairness Metrics}  & \multicolumn{2}{c|}{Accuracy Metrics}  & \multicolumn{3}{c}{Fairness Metrics} \\ \cmidrule{3-12}
&                   & \multicolumn{1}{c}{C$^{td}$$\uparrow$}  & \multicolumn{1}{c|}{IBS$\downarrow$}  &\multicolumn{1}{c}{CI(\%)$\downarrow$}& \multicolumn{1}{c}{F$_{CI}$$\downarrow$}& F$_{CG}$$\downarrow$ & \multicolumn{1}{c}{C$^{td}$$\uparrow$}  & \multicolumn{1}{c|}{IBS$\downarrow$}  &\multicolumn{1}{c}{CI(\%)$\downarrow$}& \multicolumn{1}{c}{F$_{CI}$$\downarrow$}& F$_{CG}$$\downarrow$                 \\ \midrule
\multirow{8}{*}{\rotatebox{90}{Linear \ \ \ \ \ \ \ \ \ \ \ \ \ \ }} &        \tableCox            & \multicolumn{1}{c}{{\textbf{\makecell{0.8032 \\(0.0002)}}}} & \multicolumn{1}{c|}{\makecell{0.1739 \\(0.0004)}} & \makecell{0.8610 \\(0.0197)} & \makecell{0.0249 \\(0.0002)} & \makecell{0.0128 \\(0.0001)} & \multicolumn{1}{c}{{\textbf{\makecell{0.8032 \\(0.0002)}}}} & \multicolumn{1}{c|}{\makecell{0.1739 \\(0.0004)}} & \makecell{0.8610 \\(0.0197)} & \makecell{0.0249 \\(0.0002)} & \makecell{0.0128 \\(0.0001)}             \\  \cmidrule{3-12}
&          \tableCoxKeyaInd         & \multicolumn{1}{c}{{\makecell{0.7932 \\(0.0083)}}} & \multicolumn{1}{c|}{\makecell{0.1368 \\(0.0052)}} & \makecell{1.6750 \\(0.7969)}& \makecell{0.0096 \\(0.0031)} & \makecell{0.0051 \\(0.0016)} & \multicolumn{1}{c}{{\makecell{0.7859 \\(0.0220)}}}  & \multicolumn{1}{c|}{\makecell{0.1334 \\(0.0034)}}  & \multicolumn{1}{c}{\makecell{1.9650 \\(1.3510)}}& \multicolumn{1}{c}{\makecell{0.0067 \\(0.0006)}}& \makecell{0.0035 \\(0.0004)}               \\
\bottomrule
\end{tabular}
}
\label{tab:general_performance_FLC_gender_CI}
\end{table}

\begin{table} %
\caption{\small Cox model test set scores on the SUPPORT (age) dataset, in the same format as Table~\ref{tab:general_performance_CI}.}
\centering
\setlength\tabcolsep{1.6pt}
\renewcommand{\arraystretch}{0.5}
{\tiny %
\renewcommand{\belowrulesep}{0.1pt}
\renewcommand{\aboverulesep}{0.1pt}
\begin{tabular}{ccccccc|ccccc}
\toprule
\multirow{3}{*}{} & \multirow{3}{*}{Methods} & \multicolumn{5}{c|}{CI-based Tuning}                                                    & \multicolumn{5}{c}{F$_{CG}$-based Tuning}                             \\ \cmidrule{3-12}
&                   & \multicolumn{2}{c|}{Accuracy Metrics}  & \multicolumn{3}{c|}{Fairness Metrics}  & \multicolumn{2}{c|}{Accuracy Metrics}  & \multicolumn{3}{c}{Fairness Metrics} \\ \cmidrule{3-12}
&                   & \multicolumn{1}{c}{C$^{td}$$\uparrow$}  & \multicolumn{1}{c|}{IBS$\downarrow$}  &\multicolumn{1}{c}{CI(\%)$\downarrow$}& \multicolumn{1}{c}{F$_{CI}$$\downarrow$}& F$_{CG}$$\downarrow$ & \multicolumn{1}{c}{C$^{td}$$\uparrow$}  & \multicolumn{1}{c|}{IBS$\downarrow$}  &\multicolumn{1}{c}{CI(\%)$\downarrow$}& \multicolumn{1}{c}{F$_{CI}$$\downarrow$}& F$_{CG}$$\downarrow$                 \\ \midrule
\multirow{8}{*}{\rotatebox{90}{Linear \ \ \ \ \ \ \ \ \ \ \ \ \ \ }} &        \tableCox            & \multicolumn{1}{c}{{\textbf{\makecell{0.6025 \\(0.0005)}}}}  & \multicolumn{1}{c|}{\makecell{0.2304 \\(0.0015)}} & \makecell{2.2240 \\(0.1078)} & \makecell{0.0054 \\(0.0002)} & \makecell{0.0023 \\(0.0001)} & \multicolumn{1}{c}{{\textbf{\makecell{0.6025 \\(0.0005)}}}}  & \multicolumn{1}{c|}{\makecell{0.2304 \\(0.0015)}} & \makecell{2.2240 \\(0.1078)} & \makecell{0.0054 \\(0.0002)} & \makecell{0.0023 \\(0.0001)}           \\  \cmidrule{3-12}
&          \tableCoxKeyaInd          & \multicolumn{1}{c}{{\makecell{0.5820 \\(0.0116)}}} & \multicolumn{1}{c|}{\textbf{\makecell{0.2153 \\(0.0076)}}} & \makecell{1.3120 \\(0.7623)}& {\makecell{0.0001 \\(0.0003)}} & {\makecell{3.6626e-05 \\(0.0001)} } & \multicolumn{1}{c}{{\makecell{0.5829 \\(0.0099)}}}  & \multicolumn{1}{c|}{\textbf{\makecell{0.2147 \\(0.0063)}}}  & \multicolumn{1}{c}{{\makecell{1.3600 \\(0.8532)}}}& \multicolumn{1}{c}{\textbf{\makecell{0\\(0)}}}& \textbf{\makecell{0\\(0)}}               \\
\bottomrule
\end{tabular}
}
\label{tab:general_performance_SUPPORT_age_CI}
\end{table}

\begin{table} %
\caption{\small Cox model test set scores on the SUPPORT (race) dataset, in the same format as Table~\ref{tab:general_performance_CI}.} 
\centering
\setlength\tabcolsep{1.6pt}
\renewcommand{\arraystretch}{0.5}
{\tiny %
\renewcommand{\belowrulesep}{0.1pt}
\renewcommand{\aboverulesep}{0.1pt}
\begin{tabular}{ccccccc|ccccc}
\toprule
\multirow{3}{*}{} & \multirow{3}{*}{Methods} & \multicolumn{5}{c|}{CI-based Tuning}                                                    & \multicolumn{5}{c}{F$_{CG}$-based Tuning}                             \\ \cmidrule{3-12}
&                   & \multicolumn{2}{c|}{Accuracy Metrics}  & \multicolumn{3}{c|}{Fairness Metrics}  & \multicolumn{2}{c|}{Accuracy Metrics}  & \multicolumn{3}{c}{Fairness Metrics} \\ \cmidrule{3-12}
&                   & \multicolumn{1}{c}{C$^{td}$$\uparrow$}  & \multicolumn{1}{c|}{IBS$\downarrow$}  &\multicolumn{1}{c}{CI(\%)$\downarrow$}& \multicolumn{1}{c}{F$_{CI}$$\downarrow$}& F$_{CG}$$\downarrow$ & \multicolumn{1}{c}{C$^{td}$$\uparrow$}  & \multicolumn{1}{c|}{IBS$\downarrow$}  &\multicolumn{1}{c}{CI(\%)$\downarrow$}& \multicolumn{1}{c}{F$_{CI}$$\downarrow$}& F$_{CG}$$\downarrow$                 \\ \midrule
\multirow{8}{*}{\rotatebox{90}{Linear \ \ \ \ \ \ \ \ \ \ \ \ \ \ }} &        \tableCox            & \multicolumn{1}{c}{{\textbf{\makecell{0.6025 \\(0.0005)}}}} & \multicolumn{1}{c|}{\makecell{0.2304 \\(0.0015)}} & \makecell{1.4160 \\(0.0696)} & \makecell{0.0054 \\(0.0002)} & \makecell{0.0034 \\(0.0001)} & \multicolumn{1}{c}{{{\textbf{\makecell{0.6025 \\(0.0005)}}}}} & \multicolumn{1}{c|}{\makecell{0.2304 \\(0.0015)}} & \makecell{1.4160 \\(0.0696)} & \makecell{0.0054 \\(0.0002)} & \makecell{0.0034 \\(0.0001)}          \\  \cmidrule{3-12}
&          \tableCoxKeyaInd          & \multicolumn{1}{c}{{\makecell{0.5905 \\(0.0086)}}} & \multicolumn{1}{c|}{\textbf{\makecell{0.2161 \\(0.0054)}}} & \makecell{1.1230\\ (0.6621)} & \makecell{0.0005 \\(0.0005)} & \makecell{0.0003 \\(0.0004)}  & \multicolumn{1}{c}{{\makecell{0.5829 \\(0.0099)}}}  & \multicolumn{1}{c|}{\textbf{\makecell{0.2147 \\(0.0063)}}}  & \multicolumn{1}{c}{\makecell{1.1820 \\(0.5238)}}& \multicolumn{1}{c}{\textbf{\makecell{0\\(0)}}}& \textbf{\makecell{0\\(0)}}             \\
\bottomrule
\end{tabular}
}
\label{tab:general_performance_SUPPORT_race_CI}
\end{table}

\begin{table*}[pth!]
\vspace{-1em}
\caption{\small Cox model test set scores on the SEER (age) dataset, in the same format as Table~\ref{tab:general_performance_CI}.} 
\centering
\setlength\tabcolsep{0.1pt}
\renewcommand{\arraystretch}{0.5}
{\tiny %
\renewcommand{\belowrulesep}{0.1pt}
\renewcommand{\aboverulesep}{0.1pt}
\begin{tabular}{ccccccc|ccccc}
\toprule
\multirow{3}{*}{} & \multirow{3}{*}{Methods} & \multicolumn{5}{c|}{CI-based Tuning}                                                    & \multicolumn{5}{c}{F$_{CG}$-based Tuning}                             \\ \cmidrule{3-12}
&                   & \multicolumn{2}{c|}{Accuracy Metrics}  & \multicolumn{3}{c|}{Fairness Metrics}  & \multicolumn{2}{c|}{Accuracy Metrics}  & \multicolumn{3}{c}{Fairness Metrics} \\ \cmidrule{3-12}
&                   & \multicolumn{1}{c}{C$^{td}$$\uparrow$}  & \multicolumn{1}{c|}{IBS$\downarrow$}  &\multicolumn{1}{c}{CI(\%)$\downarrow$}& \multicolumn{1}{c}{F$_{CI}$$\downarrow$}& F$_{CG}$$\downarrow$ & \multicolumn{1}{c}{C$^{td}$$\uparrow$}  & \multicolumn{1}{c|}{IBS$\downarrow$}  &\multicolumn{1}{c}{CI(\%)$\downarrow$}& \multicolumn{1}{c}{F$_{CI}$$\downarrow$}& F$_{CG}$$\downarrow$                 \\ \midrule
\multirow{8}{*}{\rotatebox{90}{Linear \ \ \ \ \ \ \ \ \ \ \ \ \ \ }} &        \tableCox            & \multicolumn{1}{c}{{{\makecell{0.7025 \\(0.0003)}}}} & \multicolumn{1}{c|}{\makecell{0.2128 \\(0.0009)}} & \makecell{0.9420 \\(0.0271)} & \multicolumn{1}{c}{\makecell{0.0256 \\(0.0006)}} & \makecell{0.0063 \\(0.0001)} & \multicolumn{1}{c}{{\makecell{0.7025 \\(0.0003)}}} & \multicolumn{1}{c|}{\makecell{0.2128 \\(0.0009)}} & \makecell{0.9420 \\(0.0271)} & \multicolumn{1}{c}{\makecell{0.0256 \\(0.0006)}} & \makecell{0.0063 \\(0.0001)}            \\  \cmidrule{3-12}
&          \tableCoxKeyaInd         & \multicolumn{1}{c}{{\makecell{0.6911 \\(0.0049)}}} & \multicolumn{1}{c|}{\makecell{0.1910 \\(0.0041)}} & \makecell{0.6840 \\(0.4589)} & \multicolumn{1}{c}{\makecell{0.0114 \\(0.0044)}} & \makecell{0.0026 \\(0.0011)}  & \multicolumn{1}{c}{{\makecell{0.6877 \\(0.0065)}}}  & \multicolumn{1}{c|}{\textbf{\makecell{0.1838 \\(0.0027)}}}  & \multicolumn{1}{c}{\makecell{0.7570 \\(0.4637)}}& \multicolumn{1}{c}{\makecell{0.0005 \\(0.0002)}}& \makecell{3.2960e-06 \\(4.6576e-06)}                            \\
\bottomrule
\end{tabular}
}
\label{tab:general_performance_SEER2_age_CI}
\end{table*}

\begin{table*}[t!]
\vspace{-1em}
\caption{\small Cox model test set individual and group fairness metrics on the FLC (age) dataset, in the same format as Table~\ref{tab:general_performance_CI}. } 
\centering
\renewcommand{\arraystretch}{0.5}
{\tiny %
\renewcommand{\belowrulesep}{0.1pt}
\renewcommand{\aboverulesep}{0.1pt}
\begin{tabular}{cccc|cc}
\toprule
\multirow{2}{*}{} & \multirow{2}{*}{Methods} & \multicolumn{2}{c|}{CI-based Tuning}                                                    & \multicolumn{2}{c}{F$_{CG}$-based Tuning}                             \\ \cmidrule{3-6}
&                   & \multicolumn{1}{c}{F$_{I}$$\downarrow$}& F$_{G}$$\downarrow$ & \multicolumn{1}{c}{F$_{I}$$\downarrow$}& F$_{G}$$\downarrow$                 \\ \midrule

\multirow{8}{*}{\rotatebox{90}{Linear \ \ \ \ \ \ \ \ \  }} &        \tableCox            & \multicolumn{1}{c}{\makecell{0.0964 (0.0006)}}& \makecell{0.1912 (0.0012)} & \multicolumn{1}{c}{\makecell{0.0964 (0.0006)}}& \makecell{0.1912 (0.0012)}             \\ \cmidrule{3-6}
&          \tableCoxKeyaInd          & \multicolumn{1}{c}{\makecell{0.0496 (0.0121)}}& \makecell{0.1079 (0.0213)} & \multicolumn{1}{c}{\makecell{0.0256 (0.0018)}}& \makecell{0.0658 (0.0030)}                \\ \cmidrule{3-6}
&          \tableCoxRPInd          & \multicolumn{1}{c}{\makecell{0.0958 (0.0030)}}& \makecell{0.1899 (0.0064)} & \multicolumn{1}{c}{\makecell{0.0912 (0.0043)}}& \makecell{0.1803 (0.0092)}                \\ \cmidrule{3-6}
&        \tableCoxKeyaGroup            & \multicolumn{1}{c}{\makecell{0.0526 (0.0167)}}& \makecell{0.1033 (0.0321)} & \multicolumn{1}{c}{\makecell{0.0353 (0.0109)}}& \makecell{0.0716 (0.0236)}                \\ \cmidrule{3-6}
&        \tableCoxRPGroup            & \multicolumn{1}{c}{\makecell{0.1033 (0.0030)}}& \makecell{0.1899 (0.0064)} & \multicolumn{1}{c}{\makecell{0.0951 (0.0006)}}& \makecell{0.1726 (0.0010)}              \\ \cmidrule{3-6}
&        \tableCoxKeyaInt            & \multicolumn{1}{c}{\makecell{0.0325 (0.0006)}}& \makecell{0.0652 (0.0025)} & \multicolumn{1}{c}{\makecell{0.0326 (0.0006)}}& \makecell{0.0661 (0.0023)}               \\
\bottomrule
\end{tabular}
}
\label{tab:Flc_age_more_fair_performance}
\end{table*}

\begin{table} %
\caption{\small Cox model test set individual and group fairness metrics on the FLC (gender) dataset, in the same format as Table~\ref{tab:general_performance_CI}.} 
\centering
\renewcommand{\arraystretch}{0.5}
{\tiny %
\renewcommand{\belowrulesep}{0.1pt}
\renewcommand{\aboverulesep}{0.1pt}
\begin{tabular}{cccc|cc}
\toprule
\multirow{2}{*}{} & \multirow{2}{*}{Methods} & \multicolumn{2}{c|}{CI-based Tuning}                                                    & \multicolumn{2}{c}{F$_{CG}$-based Tuning}                             \\ \cmidrule{3-6}
&                   & \multicolumn{1}{c}{F$_{I}$$\downarrow$}& F$_{G}$$\downarrow$ & \multicolumn{1}{c}{F$_{I}$$\downarrow$}& F$_{G}$$\downarrow$                 \\ \midrule

\multirow{8}{*}{\rotatebox{90}{Linear \ \ \ \ \ \ \ \ \ }} &        \tableCox            & \multicolumn{1}{c}{\makecell{0.0964 (0.0006)}}& \makecell{0.0167 (0.0015)} & \multicolumn{1}{c}{\makecell{0.0964 (0.0006)}}& \makecell{0.0167 (0.0015)}             \\ \cmidrule{3-6}
&          \tableCoxKeyaInd          & \multicolumn{1}{c}{\makecell{0.0367 (0.0118)}}& \makecell{0.0048 (0.0031)} & \multicolumn{1}{c}{\makecell{0.0253 (0.0019)}}& \makecell{0.0065 (0.0058)}                \\ \cmidrule{3-6}
&          \tableCoxRPInd          & \multicolumn{1}{c}{\makecell{0.0931 (0.0045)}}& \makecell{0.0150 (0.0041)} & \multicolumn{1}{c}{\makecell{0.0885 (0.0024)}}& \makecell{0.0122 (0.0026)}                \\ \cmidrule{3-6}
&        \tableCoxKeyaGroup            & \multicolumn{1}{c}{\makecell{0.0784 (0.0119)}}& \makecell{0.0127 (0.0047)} & \multicolumn{1}{c}{\makecell{0.0708 (0.0005)}}& \makecell{0.0097 (0.0011)}                \\ \cmidrule{3-6}
&        \tableCoxRPGroup            & \multicolumn{1}{c}{\makecell{0.0997 (0.0045)}}& \makecell{0.0141 (0.0034)} & \multicolumn{1}{c}{\makecell{0.0951 (0.0006)}}& \makecell{0.0117 (0.0018)}              \\ \cmidrule{3-6}
&        \tableCoxKeyaInt            & \multicolumn{1}{c}{\makecell{0.0325 (0.0006)}}& \makecell{0.0022 (0.0012)} & \multicolumn{1}{c}{\makecell{0.0325 (0.0006)}}& \makecell{0.0022 (0.0012)}               \\
\bottomrule
\end{tabular}
}
\label{tab:Flc_gender_more_fair_performance}
\end{table}

\begin{table*}[t!]
\vspace{-1em}
\caption{\small Cox model test set individual and group fairness metrics on the SUPPORT (age) dataset, in the same format as Table~\ref{tab:general_performance_CI}. } 
\centering
\setlength\tabcolsep{1pt}
\renewcommand{\arraystretch}{0.5}
{\tiny %
\renewcommand{\belowrulesep}{0.1pt}
\renewcommand{\aboverulesep}{0.1pt}
\begin{tabular}{cccc|cc}
\toprule
\multirow{2}{*}{} & \multirow{2}{*}{Methods} & \multicolumn{2}{c|}{CI-based Tuning}                                                    & \multicolumn{2}{c}{F$_{CG}$-based Tuning}                             \\ \cmidrule{3-6}
&                   & \multicolumn{1}{c}{F$_{I}$$\downarrow$}& F$_{G}$$\downarrow$ & \multicolumn{1}{c}{F$_{I}$$\downarrow$}& F$_{G}$$\downarrow$                 \\ \midrule

\multirow{8}{*}{\rotatebox{90}{Linear \ \ \ \ \ \ \ \ \ }} &        \tableCox            & \multicolumn{1}{c}{\makecell{0.0461 (0.0020)}}& \makecell{0.0498 (0.0021)} & \multicolumn{1}{c}{\makecell{0.0461 (0.0020)}}& \makecell{0.0498 (0.0021)}             \\ \cmidrule{3-6}
&          \tableCoxKeyaInd          & \multicolumn{1}{c}{\makecell{0.0010 (0.0031)}}& \makecell{0.0041 (0.0067)} & \multicolumn{1}{c}{\textbf{\makecell{0 (0)}}}& \makecell{0.0006 (0.0004)}                \\ \cmidrule{3-6}
&          \tableCoxRPInd          & \multicolumn{1}{c}{\makecell{0.0479 (0.0020)}}& \makecell{0.0513 (0.0042)} & \multicolumn{1}{c}{\makecell{0.0473 (0.0016)}}& \makecell{0.0517 (0.0039)}                \\ \cmidrule{3-6}
&        \tableCoxKeyaGroup            & \multicolumn{1}{c}{\makecell{0.0409 (0.0019)}}& \makecell{0.0031 (0.0013)} & \multicolumn{1}{c}{\makecell{0.0350 (0.0009)}}& \makecell{0.0031 (0.0008)}                \\ \cmidrule{3-6}
&        \tableCoxRPGroup            & \multicolumn{1}{c}{\makecell{0.0639 (0.0021)}}& \makecell{0.0513 (0.0042)} & \multicolumn{1}{c}{\makecell{0.0632 (0.0016)}}& \makecell{0.0517 (0.0039)}              \\ \cmidrule{3-6}
&        \tableCoxKeyaInt            & \multicolumn{1}{c}{\makecell{0.0242 (0.0031)}}& \makecell{0.0035 (0.0015)} & \multicolumn{1}{c}{\makecell{0.0214 (0.0034)}}& \makecell{0.0034 (0.0014)}               \\
\bottomrule
\end{tabular}
}
\label{tab:SUPPORT_age_more_fair_performance}
\end{table*}

\begin{table} %
\caption{\small Cox model test set individual and group fairness metrics on the SUPPORT (gender) dataset, in the same format as Table~\ref{tab:general_performance_CI}.} 
\centering
\setlength\tabcolsep{1pt}
\renewcommand{\arraystretch}{0.5}
{\tiny %
\renewcommand{\belowrulesep}{0.1pt}
\renewcommand{\aboverulesep}{0.1pt}
\begin{tabular}{cccc|cc}
\toprule
\multirow{2}{*}{} & \multirow{2}{*}{Methods} & \multicolumn{2}{c|}{CI-based Tuning}                                                    & \multicolumn{2}{c}{F$_{CG}$-based Tuning}                             \\ \cmidrule{3-6}
&                   & \multicolumn{1}{c}{F$_{I}$$\downarrow$}& F$_{G}$$\downarrow$ & \multicolumn{1}{c}{F$_{I}$$\downarrow$}& F$_{G}$$\downarrow$                 \\ \midrule

\multirow{8}{*}{\rotatebox{90}{Linear \ \ \ \ \ \ \ \ \  }} &        \tableCox            & \multicolumn{1}{c}{\makecell{0.0461 (0.0020)}}& \makecell{0.0142 (0.0017)} &\multicolumn{1}{c}{\makecell{0.0461 (0.0020)}}& \makecell{0.0142 (0.0017)}             \\ \cmidrule{3-6}
&          \tableCoxKeyaInd          & \multicolumn{1}{c}{\makecell{0.0042 (0.0042)}}& \makecell{0.0026 (0.0020)} & \multicolumn{1}{c}{\textbf{\makecell{0 (0)}}}& \makecell{0.0003 (0.0001)}                \\ \cmidrule{3-6}
&          \tableCoxRPInd          & \multicolumn{1}{c}{\makecell{0.0475 (0.0010)}}& \makecell{0.0145 (0.0055)} & \multicolumn{1}{c}{\makecell{0.0470 (0.0014)}}& \makecell{0.0137 (0.0034)}                \\ \cmidrule{3-6}
&        \tableCoxKeyaGroup            & \multicolumn{1}{c}{\makecell{0.0437 (0.0027)}}& \makecell{0.0010 (0.0006)} & \multicolumn{1}{c}{\makecell{0.0407 (0.0005)}}& \makecell{0.0009 (0.0006)}                \\ \cmidrule{3-6}
&        \tableCoxRPGroup            & \multicolumn{1}{c}{\makecell{0.0635 (0.0011)}}& \makecell{0.0145 (0.0055)} & \multicolumn{1}{c}{\makecell{0.0629 (0.0014)}}& \makecell{0.0133 (0.0031)}              \\ \cmidrule{3-6}
&        \tableCoxKeyaInt            & \multicolumn{1}{c}{\makecell{0.0256 (0.0034)}}& \makecell{0.0038 (0.0012)} & \multicolumn{1}{c}{\makecell{0.0214 (0.0034)}}& \makecell{0.0031 (0.0008)}               \\
\bottomrule
\end{tabular}
}
\label{tab:SUPPORT_gender_more_fair_performance}
\end{table}

\begin{table} %
\caption{\small Cox model test set individual and group fairness metrics on the SUPPORT (race) dataset, in the same format as Table~\ref{tab:general_performance_CI}.} 
\centering
\setlength\tabcolsep{1pt}
\renewcommand{\arraystretch}{0.5}
{\tiny %
\renewcommand{\belowrulesep}{0.1pt}
\renewcommand{\aboverulesep}{0.1pt}
\begin{tabular}{cccc|cc}
\toprule
\multirow{2}{*}{} & \multirow{2}{*}{Methods} & \multicolumn{2}{c|}{CI-based Tuning}                                                    & \multicolumn{2}{c}{F$_{CG}$-based Tuning}                             \\ \cmidrule{3-6}
&                   & \multicolumn{1}{c}{F$_{I}$$\downarrow$}& F$_{G}$$\downarrow$ & \multicolumn{1}{c}{F$_{I}$$\downarrow$}& F$_{G}$$\downarrow$                 \\ \midrule

\multirow{8}{*}{\rotatebox{90}{Linear \ \ \ \ \ \ \ \ \ \ \ \ \ \ }} &        \tableCox            & \multicolumn{1}{c}{\makecell{0.0461 (0.0020)}}& \makecell{0.0115 (0.0042)} & \multicolumn{1}{c}{\makecell{0.0461 (0.0020)}}& \makecell{0.0115 (0.0042)}             \\ \cmidrule{3-6}
&          \tableCoxKeyaInd          & \multicolumn{1}{c}{\makecell{0.0050 (0.0052)}}& \makecell{0.0053 (0.0072)} & \multicolumn{1}{c}{\textbf{\makecell{0 (0)}}}& \makecell{0.0008 (0.0005)}                \\ \cmidrule{3-6}
&          \tableCoxRPInd          & \multicolumn{1}{c}{\makecell{0.0475 (0.0016)}}& \makecell{0.0118 (0.0061)} & \multicolumn{1}{c}{\makecell{0.0470 (0.0014)}}& \makecell{0.0116 (0.0047)}                \\ \cmidrule{3-6}
&        \tableCoxKeyaGroup            & \multicolumn{1}{c}{\makecell{0.0404 (0.0024)}}& \makecell{0.0011 (0.0006)} & \multicolumn{1}{c}{\makecell{0.0396 (0.0006)}}& \makecell{0.0010 (0.0006)}                \\ \cmidrule{3-6}
&        \tableCoxRPGroup            & \multicolumn{1}{c}{\makecell{0.0634 (0.0017)}}& \makecell{0.0118 (0.0061)} & \multicolumn{1}{c}{\makecell{0.0629 (0.0014)}}& \makecell{0.0116 (0.0047)}              \\ \cmidrule{3-6}
&        \tableCoxKeyaInt            & \multicolumn{1}{c}{\makecell{0.0242 (0.0038)}}& \makecell{0.0045 (0.0016)} & \multicolumn{1}{c}{\makecell{0.0214 (0.0034)}}& \makecell{0.0042 (0.0015)}               \\
\bottomrule
\end{tabular}
}
\label{tab:SUPPORT_race_more_fair_performance}
\end{table}

\begin{table*}[t!]
\vspace{-1em}
\caption{\small Cox model test set individual and group fairness metrics on the SEER (age) dataset, in the same format as Table~\ref{tab:general_performance_CI}. } 
\centering
\setlength\tabcolsep{1.5pt}
\renewcommand{\arraystretch}{0.5}
{\tiny %
\renewcommand{\belowrulesep}{0.1pt}
\renewcommand{\aboverulesep}{0.1pt}
\begin{tabular}{cccc|cc}
\toprule
\multirow{2}{*}{} & \multirow{2}{*}{Methods} & \multicolumn{2}{c|}{CI-based Tuning}                                                    & \multicolumn{2}{c}{F$_{CG}$-based Tuning}                             \\ \cmidrule{3-6}
&                   & \multicolumn{1}{c}{F$_{I}$$\downarrow$}& F$_{G}$$\downarrow$ & \multicolumn{1}{c}{F$_{I}$$\downarrow$}& F$_{G}$$\downarrow$                 \\ \midrule

\multirow{8}{*}{\rotatebox{90}{Linear \ \ \ \ \ \ \ \ \  }} &        \tableCox            & \multicolumn{1}{c}{\makecell{0.0642 (0.0014)}}& \makecell{0.1676 (0.0048)} & \multicolumn{1}{c}{\makecell{0.0642 (0.0014)}}& \makecell{0.1676 (0.0048)}            \\ \cmidrule{3-6}
&          \tableCoxKeyaInd          & \multicolumn{1}{c}{\makecell{0.0289 (0.0110)}}& \makecell{0.0923 (0.0268)} & \multicolumn{1}{c}{\makecell{0.0012 (0.0004)}}& \makecell{0.0229 (0.0032)}                \\ \cmidrule{3-6}
&          \tableCoxRPInd          & \multicolumn{1}{c}{\makecell{0.0640 (0.0055)}}& \makecell{0.1670 (0.0156)} & \multicolumn{1}{c}{\makecell{0.0574 (0.0037)}}& \makecell{0.1495 (0.0139)}                \\ \cmidrule{3-6}
&        \tableCoxKeyaGroup            & \multicolumn{1}{c}{\makecell{0.0367 (0.0007)}}& \makecell{0.0026 (0.0012)} & \multicolumn{1}{c}{\makecell{0.0367 (0.0007)}}& \makecell{0.0026 (0.0012)}                \\ \cmidrule{3-6}
&        \tableCoxRPGroup            & \multicolumn{1}{c}{\makecell{0.0727 (0.0062)}}& \makecell{0.1557 (0.0164)} & \multicolumn{1}{c}{\makecell{0.0686 (0.0023)}}& \makecell{0.1462 (0.0102)}              \\ \cmidrule{3-6}
&        \tableCoxKeyaInt            & \multicolumn{1}{c}{\makecell{0.0298 (0.0029)}}& \makecell{0.0046 (0.0012)} & \multicolumn{1}{c}{\makecell{0.0298 (0.0029)}}& \makecell{0.0046 (0.0012)}               \\
\bottomrule
\end{tabular}
}
\label{tab:SEER_age_more_fair_performance}
\end{table*}

\begin{table*}[t!]
\vspace{-1em}
\caption{\small Cox model test set individual and group fairness metrics on the SEER (race) dataset, in the same format as Table~\ref{tab:general_performance_CI}. } 
\centering
\setlength\tabcolsep{1pt}
\renewcommand{\arraystretch}{0.5}
{\tiny %
\renewcommand{\belowrulesep}{0.1pt}
\renewcommand{\aboverulesep}{0.1pt}
\begin{tabular}{cccc|cc}
\toprule
\multirow{2}{*}{} & \multirow{2}{*}{Methods} & \multicolumn{2}{c|}{CI-based Tuning}                                                    & \multicolumn{2}{c}{F$_{CG}$-based Tuning}                             \\ \cmidrule{3-6}
&                   & \multicolumn{1}{c}{F$_{I}$$\downarrow$}& F$_{G}$$\downarrow$ & \multicolumn{1}{c}{F$_{I}$$\downarrow$}& F$_{G}$$\downarrow$                 \\ \midrule

\multirow{8}{*}{\rotatebox{90}{Linear \ \ \ \ \ \ \ \ \ \ \ \ \ \ }} &        \tableCox            & \multicolumn{1}{c}{\makecell{0.0642 (0.0014)}}& \makecell{0.0597 (0.0027)} & \multicolumn{1}{c}{\makecell{0.0642 (0.0014)}}& \makecell{0.0597 (0.0027)}             \\ \cmidrule{3-6}
&          \tableCoxKeyaInd          & \multicolumn{1}{c}{\makecell{0.0011 (0.0002)}}& \makecell{0.0072 (0.0015)} & \multicolumn{1}{c}{\makecell{0.0011 (0.0002)}}& \makecell{0.0072 (0.0015)}                \\ \cmidrule{3-6}
&          \tableCoxRPInd          & \multicolumn{1}{c}{\makecell{0.0595 (0.0050)}}& \makecell{0.0612 (0.0135)} & \multicolumn{1}{c}{\makecell{0.0586 (0.0048)}}& \makecell{0.0634 (0.0110)}                \\ \cmidrule{3-6}
&        \tableCoxKeyaGroup            & \multicolumn{1}{c}{\makecell{0.0541 (0.0098)}}& \makecell{0.0096 (0.0168)} & \multicolumn{1}{c}{\makecell{0.0541 (0.0098)}}& \makecell{0.0096 (0.0168)}                \\ \cmidrule{3-6}
&        \tableCoxRPGroup            & \multicolumn{1}{c}{\makecell{0.0695 (0.0035)}}& \makecell{0.0571 (0.0087)} & \multicolumn{1}{c}{\makecell{0.0696 (0.0038)}}& \makecell{0.0610 (0.0086)}              \\ \cmidrule{3-6}
&        \tableCoxKeyaInt            & \multicolumn{1}{c}{\makecell{0.0298 (0.0029)}}& \makecell{0.0024 (0.0010)} & \multicolumn{1}{c}{\makecell{0.0298 (0.0029)}}& \makecell{0.0024 (0.0010)}               \\
\bottomrule
\end{tabular}
}
\label{tab:SEER_race_more_fair_performance}
\end{table*}

\begin{table*}[t!]
\vspace{-1em}
\caption{\small DeepHit test set individual and group fairness on the FLC, SUPPORT, SEER datasets when hyperparameter tuning is based on CI and F$_{CG}$.} 
\centering
\renewcommand{\arraystretch}{0.5}
{\tiny %
\renewcommand{\belowrulesep}{0.1pt}
\renewcommand{\aboverulesep}{0.1pt}
\begin{tabular}{cccc|cc}
\toprule
\multirow{3}{*}{Datasets} & \multirow{3}{*}{Methods} & \multicolumn{2}{c|}{CI-based Tuning}                                                    & \multicolumn{2}{c}{F$_{CG}$-based Tuning}                             \\ \cmidrule{3-6}
&                   & \multicolumn{1}{c}{F$_{I}$$\downarrow$}& F$_{G}$$\downarrow$ & \multicolumn{1}{c}{F$_{I}$$\downarrow$}& F$_{G}$$\downarrow$                 \\ \midrule

\multirow{3}{*}{\makecell{FLC \\(age)}} &      DeepHit             & \multicolumn{1}{c}{\makecell{0.0330 (0.0049)}} & \makecell{0.0721 (0.0078)} & \multicolumn{1}{c}{\makecell{0.0330 (0.0049)}} & \makecell{0.0721 (0.0078)} \\ \cmidrule{3-6} 
                  &       \textsc{\tableDeepHitRPGroup}            & \multicolumn{1}{c}{\makecell{0.0289 (0.0089)}} & \makecell{0.0615 (0.0193)} & \multicolumn{1}{c}{\makecell{0.0173 (0.0006)}} & \makecell{0.0331 (0.0012)} \\ \cmidrule{2-6}
                  &       \textsc{dro-deephit}            & \multicolumn{1}{c}{\1\makecell{0.0233 (0.0111)}} & \1\makecell{0.0561 (0.0183)} & \multicolumn{1}{c}{\1\textbf{\makecell{0.0022 (0.0012)}}} & \1\makecell{0.0168 (0.0035)} \\ \cmidrule{3-6} 
                  &       \textsc{dro-deephit (split)}            & \multicolumn{1}{c}{\1\textbf{\makecell{0.0177 (0.0140)}}} & \1\textbf{\makecell{0.0426 (0.0252)}} & \multicolumn{1}{c}{\1\makecell{0.0030 (0.0027)}} & \1\textbf{\makecell{0.0160 (0.0051)}} \\ \midrule
\multirow{3}{*}{\makecell{FLC \\(gender)}} &         DeepHit          & \multicolumn{1}{c}{\makecell{0.0330 (0.0049)}} & \makecell{0.0152 (0.0103)} & \multicolumn{1}{c}{\makecell{0.0330 (0.0049)}} & \makecell{0.0152 (0.0103)} \\ \cmidrule{3-6} 
                  &       \textsc{\tableDeepHitRPGroup}            & \multicolumn{1}{c}{\makecell{0.0297 (0.0075)}} & \makecell{0.0132 (0.0110)} & \multicolumn{1}{c}{\makecell{0.0330 (0.0049)}} & \makecell{0.0152 (0.0103)} \\
\bottomrule
\end{tabular}
}
\label{tab:more_fair_performance_deephit}
\end{table*}

\begin{table*}[t!]
\vspace{-1em}
\caption{\small SODEN test set individual and group fairness on the FLC, SUPPORT, SEER datasets when hyperparameter tuning is based on CI and F$_{CG}$.} 
\centering
\renewcommand{\arraystretch}{0.5}
{\tiny %
\renewcommand{\belowrulesep}{0.1pt}
\renewcommand{\aboverulesep}{0.1pt}
\begin{tabular}{cccc|cc}
\toprule
\multirow{3}{*}{Datasets} & \multirow{3}{*}{Methods} & \multicolumn{2}{c|}{CI-based Tuning}                                                    & \multicolumn{2}{c}{F$_{CG}$-based Tuning}                             \\ \cmidrule{3-6}
&                   & \multicolumn{1}{c}{F$_{I}$$\downarrow$}& F$_{G}$$\downarrow$ & \multicolumn{1}{c}{F$_{I}$$\downarrow$}& F$_{G}$$\downarrow$                 \\ \midrule
\multirow{2}{*}{\makecell{FLC \\(age)}} &      SODEN             & \multicolumn{1}{c}{\makecell{0.0019 (0.0037)}} & \makecell{0.0101 (0.0103)} & \multicolumn{1}{c}{{\makecell{0.0019 (0.0037)}}} & {\makecell{0.0101 (0.0103)}} \\ \cmidrule{3-6} 
                  &      \tableSODENRPGroup             & \multicolumn{1}{c}{\makecell{0.0046 (0.0037)}} & \makecell{0.0153 (0.0066)} & \multicolumn{1}{c}{{\makecell{0.0004 (0.0007)}}} & {\makecell{0.0051 (0.0035)}} \\ \cmidrule{2-6} 
                  &        \textsc{dro-soden}           & \multicolumn{1}{c}{\1\textbf{\makecell{0.0001 (0.0003)}}} & \1\textbf{\makecell{0.0046 (0.0024)}} & \multicolumn{1}{c}{\1{\textbf{\makecell{2.2814e-05 (5.7901e-05)}}}} & \1{\textbf{\makecell{0.0034 (0.0017)}}} \\ \midrule
\multirow{2}{*}{\makecell{FLC \\(gender)}} &        SODEN           & \multicolumn{1}{c}{\textbf{\makecell{0.0019 (0.0037)}}} & \textbf{\makecell{ 0.0017 (0.0008)}} & \multicolumn{1}{c}{{\makecell{0.0019 (0.0037)}}} & {\makecell{0.0017 (0.0008)}} \\ \cmidrule{3-6} 
                  &      \tableSODENRPGroup             & \multicolumn{1}{c}{\makecell{0.0031 (0.0036)}} & \makecell{0.0027 (0.0017)} & \multicolumn{1}{c}{{\makecell{0.0004 (0.0007)}}} & {\makecell{0.0013 (0.0007)}} \\ \cmidrule{2-6}
                  &         \textsc{dro-soden}          & \multicolumn{1}{c}{\makecell{0.0041 (0.0112)}} & \makecell{0.0056 (0.0115)} & \multicolumn{1}{c}{{\1\textbf{\makecell{2.2814e-05 (5.7901e-05)}}}} & {\1\textbf{\makecell{0.0012 (0.0006)}}} \\
\bottomrule
\end{tabular}
}
\label{tab:more_fair_performance_SODEN}
\end{table*}

\begin{table} %
\caption{\small Test set scores for \textsc{dro-cox (split)} on the FLC (age) dataset using $n_2=0.1n, 0.2n, 0.3n, 0.4n, 0.5n$ (corresponding to $n_1=0.9n,0.8n,0.7n,0.6n,0.5n$). The format of this table is similar to that of Table~\ref{tab:general_performance_CI} although here we do not bold or highlight any cells, as our main finding here is that the scores are not dramatically different for the different choices for $n_1$ or $n_2$.}
\vspace{-1em}
\centering
{\scriptsize %
\renewcommand{\belowrulesep}{0.1pt}
\renewcommand{\aboverulesep}{0.1pt}
\begin{tabular}{cccc|ccc}
\toprule
\multirow{2}{*}{} & \multirow{2}{*}{$n_2$} & \multicolumn{2}{c|}{Accuracy Metrics}                                                    & \multicolumn{3}{c}{Fairness Metrics}                             \\ \cmidrule{3-7}
&                   & \multicolumn{1}{c}{C$^{td}\uparrow$} & IBS$\downarrow$  &CI(\%)$\downarrow$& \multicolumn{1}{c}{F$_{CI}$$\downarrow$} & \multicolumn{1}{c}{F$_{CG}$$\downarrow$}                  \\ \midrule
\multirow{5}{*}{\rotatebox{90}{Linear }} &        $0.1n$           & \multicolumn{1}{c}{{\makecell{0.7822 (0.0183)}}} & {\makecell{0.1410 (0.0056)}} & \makecell{0.4670 (0.3846)}& \makecell{0 (0)} & \makecell{0 (0)}             \\ \cmidrule{3-7}
&          $0.2n$         & \multicolumn{1}{c}{{\makecell{0.7945 (0.0069)}}} & {\makecell{0.1402 (0.0029)}} & \makecell{0.3610 (0.2667)} & \makecell{0 (0)} & \makecell{0 (0)}                \\ \cmidrule{3-7}
&          $0.3n$         & \multicolumn{1}{c}{{\makecell{0.7970 (0.0037)}}} & {\makecell{0.1397 (0.0025)}} & \makecell{0.2560 (0.1559)}& \makecell{0 (0)} & \makecell{0 (0)}                 \\ \cmidrule{3-7}
&       $0.4n$           & \multicolumn{1}{c}{{\makecell{0.7970 (0.0043)}}} & {\makecell{0.1392 (0.0015)}} & \makecell{0.2940 (0.1387)}& \makecell{0 (0)} & \makecell{0 (0)}                \\ \cmidrule{3-7}
&       $0.5n$           & \multicolumn{1}{c}{{\makecell{0.7964 (0.0045)}}} & {\makecell{0.1389 (0.0008)}} & \makecell{0.2350 (0.1277)}& \makecell{0 (0)} & \makecell{0 (0)}                \\ \cmidrule{1-7}
\multirow{5}{*}{\rotatebox{90}{Nonlinear  }} &        $0.1n$           & \multicolumn{1}{c}{{\makecell{0.7583 (0.0109)}}} & {\makecell{0.1907 (0.0764)}} & \makecell{2.1490 (1.0704)}& \makecell{1.8664e-04 (5.5992e-04)} & \makecell{3.8323e-05 (1.1497e-04)}                \\  \cmidrule{3-7}
&          $0.2n$         & \multicolumn{1}{c}{{\makecell{0.7712 (0.0107)}}} & {\makecell{0.1622 (0.0095)}} & \makecell{2.2640 (0.7685)} & \makecell{2.4905e-05 (7.4715e-05)} & \makecell{5.2623e-06 (1.5787e-05)}                \\ \cmidrule{3-7}
&          $0.3n$         & \multicolumn{1}{c}{{\makecell{0.7709 (0.0205)}}} & {\makecell{0.1650 (0.0025)}} & \makecell{2.3830 (0.4080)} & \makecell{0 (0)} & \makecell{0 (0)}                \\ \cmidrule{3-7}
&       $0.4n$           & \multicolumn{1}{c}{{\makecell{0.7731 (0.0178)}}} & {\makecell{0.1633 (0.0057)}} & \makecell{2.3860 (0.2411)} & \makecell{1.0570e-07 (3.1711e-07)} & \makecell{6.5156e-08 (1.9547e-07)}               \\ \cmidrule{3-7}
&       $0.5n$         & \multicolumn{1}{c}{{\makecell{0.7784 (0.0092)}}}& {\makecell{0.1647 (0.0037)}} & \makecell{2.3210 (0.3590)}& \makecell{0 (0)} & \makecell{0 (0)}                 \\ \bottomrule
\end{tabular}
}
\label{tab:performance_of_diff_split}
\vspace{-1em}
\end{table}

\begin{table} %
\caption{\small Test set scores for \textsc{dro-cox (split)} on the FLC (age) dataset using censoring rate imbalance ratios (abbreviated below as just ``Ratio'') of $0\%$, $20\%$, $40\%$, $60\%$, $80\%$, and $100\%$. The format of this table is similar to that of Table~\ref{tab:general_performance_CI}.}
\vspace{-1em}
\centering
\setlength\tabcolsep{2pt}
{\scriptsize %
\renewcommand{\belowrulesep}{0.1pt}
\renewcommand{\aboverulesep}{0.1pt}
\begin{tabular}{cccc|ccc}
\toprule
\multirow{2}{*}{} & \multirow{2}{*}{Ratio (\%)} & \multicolumn{2}{c|}{Accuracy Metrics}                                                    & \multicolumn{3}{c}{Fairness Metrics}                             \\ \cmidrule{3-7}
&                   & \multicolumn{1}{c}{C$^{td}\uparrow$} & IBS$\downarrow$  &CI(\%)$\downarrow$& \multicolumn{1}{c}{F$_{CI}$$\downarrow$} & \multicolumn{1}{c}{F$_{CG}$$\downarrow$}                  \\ \midrule
\multirow{6}{*}{\rotatebox{90}{Linear  }} &        $0$           & \multicolumn{1}{c}{{\makecell{0.7964 (0.0045)}}} & {\makecell{0.1389 (0.0008)}} & \makecell{0.2350 (0.1277)}& \makecell{0 (0)} & \makecell{0 (0)}              \\ \cmidrule{3-7}
&          $20$         & \multicolumn{1}{c}{{\makecell{0.7977 (0.0053)}}} & {\makecell{0.1393 (0.0009)}} & \makecell{0.1520 (0.0846)} & \makecell{0 (0)} & \makecell{0 (0)}                \\ \cmidrule{3-7}
&          $40$         & \multicolumn{1}{c}{{\makecell{0.7990 (0.0066)}}} & {\makecell{0.1402 (0.0017)}} & \makecell{0.2200 (0.1239)}& \makecell{0 (0)} & \makecell{0 (0)}                 \\ \cmidrule{3-7}
&       $60$           & \multicolumn{1}{c}{{\makecell{0.7965 (0.0070)}}} & {\makecell{0.1410 (0.0022)}} & \makecell{0.2810 (0.1840)}& \makecell{0 (0)} & \makecell{0 (0)}                \\ \cmidrule{3-7}
&       $80$           & \multicolumn{1}{c}{{\makecell{0.7935 (0.0074)}}} & {\makecell{0.1454 (0.0075)}} & \makecell{0.5660 (0.3148)}& \makecell{0 (0)} & \makecell{0 (0)}                \\ \cmidrule{3-7}
&       $100$           & \multicolumn{1}{c}{{\makecell{0.7929 (0.0102)}}} & {\makecell{0.1341 (0.0006)}} & \makecell{0.3970 (0.2665)}& \makecell{0 (0)} & \makecell{0 (0)}                \\ \cmidrule{1-7}
\multirow{5}{*}{\rotatebox{90}{Nonlinear   }} &        $0$           & \multicolumn{1}{c}{{\makecell{0.7784 (0.0092)}}}& {\makecell{0.1647 (0.0037)}} & \makecell{2.3210 (0.3590)}& \makecell{0 (0)} & \makecell{0 (0)}                 \\  \cmidrule{3-7}
&          $20$         & \multicolumn{1}{c}{{\makecell{0.7734 (0.0188)}}} & {\makecell{0.1647 (0.0048)}} & \makecell{2.4380 (0.3981)} & \makecell{5.7227e-08 (1.7168e-07)} & \makecell{3.9546e-08 (1.1864e-07)}                \\ \cmidrule{3-7}
&          $40$         & \multicolumn{1}{c}{{\makecell{0.7753 (0.0187)}}} & {\makecell{0.1677 (0.0028)}} & \makecell{2.1080 (0.5425)} & \makecell{0.0001 (0.0002)} & \makecell{1.3683e-05 (4.1049e-05)}                \\ \cmidrule{3-7}
&       $60$           & \multicolumn{1}{c}{{\makecell{0.7853 (0.0120)}}} & {\makecell{0.1700 (0.0002)}} & \makecell{1.1180 (0.6116)} & \makecell{0 (0)} & \makecell{0 (0)}               \\ \cmidrule{3-7}
&       $80$           & \multicolumn{1}{c}{{\makecell{0.7900 (0.0122)}}} & {\makecell{0.1703 (0.0002)}} & \makecell{0.6790 (0.5659)} & \makecell{0 (0)} & \makecell{0 (0)}               \\ \cmidrule{3-7}
&       $100$         & \multicolumn{1}{c}{{\makecell{0.7577 (0.0059)}}}& {\makecell{0.1646 (0.0021)}} & \makecell{0.6390 (0.4295)}& \makecell{0 (0)} & \makecell{0 (0)}                 \\ \bottomrule
\end{tabular}
}
\label{tab:performance_of_diff_imb_ratio}
\vspace{-1em}
\end{table}

\begin{table} %
\caption{\small Test set scores of \textsc{dro-cox (split, one side)} vs \textsc{dro-cox (split)} on the FLC (age) dataset. The format of this table is the same that of Table~\ref{tab:general_performance_CI} except without any cells highlighted in green as we are not comparing against baselines by previous authors.}\vspace{-1em}
\centering
{\scriptsize %
\renewcommand{\belowrulesep}{0.1pt}
\renewcommand{\aboverulesep}{0.1pt}
\begin{tabular}{cccc|ccc}
\toprule
\multirow{2}{*}{} & \multirow{2}{*}{Methods} & \multicolumn{2}{c|}{Accuracy Metrics}                                                    & \multicolumn{3}{c}{Fairness Metrics}                             \\ \cmidrule{3-7}
&                   & \multicolumn{1}{c}{C$^{td}\uparrow$} & IBS$\downarrow$  &CI(\%)$\downarrow$ & \multicolumn{1}{c}{F$_{CI}$$\downarrow$} & \multicolumn{1}{c}{F$_{CG}$$\downarrow$}                 \\ \midrule
\multirow{2}{*}{\rotatebox{90}{Linear }} &        \textsc{\makecell{DRO-COX \\(SPLIT, ONE SIDE)}}           & \multicolumn{1}{c}{{\makecell{0.7810 \\(0.0109)}}} & {\textbf{\makecell{0.1330 \\(0.0002)}}} & \makecell{0.4060 \\(0.2847)}& \textbf{\makecell{0 \\ (0)}}& \textbf{\makecell{0 \\ (0)}}              \\ \cmidrule{3-7}
&       \tableDROCoxSplit            & \multicolumn{1}{c}{{\textbf{\makecell{0.7964 \\(0.0045)}}}} & {\makecell{0.1389 \\(0.0008)}} & \textbf{\makecell{0.2350
\\(0.1277)}} & \textbf{\makecell{0 \\(0)}}& \textbf{\makecell{0 \\(0)} }               \\ \cmidrule{1-7}
\multirow{2}{*}{\rotatebox{90}{\makecell{Non-\\linear}}} &         \textsc{\makecell{Deep DRO-COX \\(SPLIT, ONE SIDE)}}           & \multicolumn{1}{c}{{\makecell{0.7554 \\(0.0231)}}} & {\textbf{\makecell{0.1332 \\(0.0002)}}} & \textbf{\makecell{1.9000 \\(0.6850)}} & \makecell{1.6544e-04 \\(1.2172e-04)}& \makecell{4.0388e-05 \\(3.1972e-05)}               \\  \cmidrule{3-7}
&       \tableDeepDROCoxSplit          & \multicolumn{1}{c}{{\textbf{\makecell{0.7784 \\(0.0092)}}}} & {\makecell{0.1647 \\(0.0037)}} & \makecell{2.3210
\\(0.3590)}& \textbf{\makecell{0 \\(0)}}& \textbf{\makecell{0 \\(0)} }                \\ \bottomrule
\end{tabular}
}
\label{tab:performance_of_one_side_split}
\end{table}

\vskip 0.2in
\bibliography{main}

\end{document}